\documentclass{iitthesis-au}

\usepackage{xcolor} 
\definecolor{COLORa}{HTML}{131FAA}
\definecolor{COLORb}{HTML}{6633C6}
\definecolor{COLORc}{HTML}{5AA136}
\usepackage{url}
\usepackage[hypertexnames=false]{hyperref} 
\hypersetup{
    colorlinks=true,
    linkcolor=COLORa,
    filecolor=COLORa,     
    urlcolor=COLORa,
    citecolor=COLORa} 
\usepackage{float}
\usepackage{amsmath}
\usepackage{amssymb}

\usepackage{algorithm}
\usepackage[commentColor=black]{algpseudocodex}

\usepackage{graphicx}
\usepackage{rotating}
\usepackage{epsfig}
\usepackage{booktabs}
\usepackage{amsthm}
\usepackage{tikz}
\usepackage{cleveref}
\usepackage{subcaption}
\usepackage{multirow}
\usepackage{enumitem}
\usepackage{caption} 
\newtheorem{theorem}{Theorem}[section]
\newtheorem{corollary}[theorem]{Corollary}
\newtheorem{lemma}[theorem]{Lemma}
\newtheorem{definition}[theorem]{Definition}
\newtheorem{proposition}[theorem]{Proposition}

\newtheorem{condition}[theorem]{Condition}
\newtheorem{remark}[theorem]{Remark}

\usepackage{accents} 

\newcommand{\D}{\mathrm{d}}

\newcommand{\tB}{{\widetilde{B}}}
\newcommand{\tc}{{\widetilde{c}}}

\newcommand{\tD}{{\widetilde{D}}}

\newcommand{\tf}{{\widetilde{f}}}

\newcommand{\tG}{{\widetilde{G}}}
\newcommand{\tH}{{\widetilde{H}}}

\newcommand{\tK}{{\widetilde{K}}}

\newcommand{\tL}{{\widetilde{L}}}

\newcommand{\tr}{{\widetilde{r}}}
\newcommand{\tR}{{\widetilde{R}}}
\newcommand{\ts}{{\widetilde{s}}}
\newcommand{\tS}{{\widetilde{S}}}
\renewcommand{\tt}{{\widetilde{t}}}

\newcommand{\tV}{{\widetilde{V}}}

\newcommand{\ty}{{\widetilde{y}}}

\newcommand{\talpha}{{\widetilde{\alpha}}}

\newcommand{\tEta}{{\widetilde{\eta}}}

\newcommand{\tkappa}{{\widetilde{\kappa}}}
\newcommand{\tlambda}{{\widetilde{\lambda}}}

\newcommand{\tmu}{{\widetilde{\mu}}}

\newcommand{\tsigma}{{\widetilde{\sigma}}}

\newcommand{\ba}{{\boldsymbol{a}}}

\newcommand{\bc}{{\boldsymbol{c}}}

\newcommand{\bD}{{\boldsymbol{D}}}
\newcommand{\be}{{\boldsymbol{e}}}

\newcommand{\bg}{{\boldsymbol{g}}}

\newcommand{\bJ}{{\boldsymbol{J}}}

\newcommand{\bK}{{\boldsymbol{K}}}
\newcommand{\bl}{{\boldsymbol{l}}}

\newcommand{\bM}{{\boldsymbol{M}}}
\newcommand{\bn}{{\boldsymbol{n}}}

\newcommand{\bR}{{\boldsymbol{R}}}
\newcommand{\bs}{{\boldsymbol{s}}}
\newcommand{\bS}{{\boldsymbol{S}}}
\newcommand{\bt}{{\boldsymbol{t}}}
\newcommand{\bT}{{\boldsymbol{T}}}
\newcommand{\bu}{{\boldsymbol{u}}}
\newcommand{\bU}{{\boldsymbol{U}}}
\newcommand{\bv}{{\boldsymbol{v}}}
\newcommand{\bV}{{\boldsymbol{V}}}
\newcommand{\bw}{{\boldsymbol{w}}}

\newcommand{\bx}{{\boldsymbol{x}}}
\newcommand{\bX}{{\boldsymbol{X}}}
\newcommand{\by}{{\boldsymbol{y}}}
\newcommand{\bY}{{\boldsymbol{Y}}}
\newcommand{\bz}{{\boldsymbol{z}}}

\newcommand{\balpha}{{\boldsymbol{\alpha}}}
\newcommand{\bbeta}{{\boldsymbol{\beta}}}

\newcommand{\bDelta}{{\boldsymbol{\Delta}}}

\newcommand{\bvarepsilon}{{\boldsymbol{\varepsilon}}}

\newcommand{\bEta}{{\boldsymbol{\eta}}}
\newcommand{\btheta}{{\boldsymbol{\theta}}}

\newcommand{\blambda}{{\boldsymbol{\lambda}}}

\newcommand{\bmu}{{\boldsymbol{\mu}}}

\newcommand{\brho}{{\boldsymbol{\rho}}}

\newcommand{\btau}{{\boldsymbol{\tau}}}

\newcommand{\bchi}{{\boldsymbol{\chi}}}

\newcommand{\bPsi}{{\boldsymbol{\Psi}}}
\newcommand{\bomega}{{\boldsymbol{\omega}}}

\newcommand{\hbDelta}{{\widehat{\boldsymbol{\Delta}}}}

\newcommand{\tbf}{{\widetilde{\boldsymbol{f}}}}

\newcommand{\tbK}{{\widetilde{\boldsymbol{K}}}}

\newcommand{\tby}{{\widetilde{\boldsymbol{y}}}}

\newcommand{\tbDelta}{{\widetilde{\boldsymbol{\Delta}}}}

\newcommand{\tblambda}{{\widetilde{\boldsymbol{\lambda}}}}

\newcommand{\tbomega}{{\widetilde{\boldsymbol{\omega}}}}

\newcommand{\sbtau}{{\starhat{\boldsymbol{\tau}}}}

\newcommand{\sbomega}{{\starhat{\boldsymbol{\omega}}}}

\newcommand{\hf}{{\widehat{f}}}

\newcommand{\hg}{{\widehat{g}}}

\newcommand{\hn}{{\widehat{n}}}

\newcommand{\hV}{{\widehat{V}}}

\newcommand{\hmu}{{\widehat{\mu}}}
\newcommand{\hnu}{{\widehat{\nu}}}

\newcommand{\hsigma}{{\widehat{\sigma}}}

\newcommand{\cK}{{\widecheck{K}}}

\newcommand{\cP}{{\widecheck{P}}}

\newcommand{\fu}{{\mathfrak{u}}}

\newcommand{\fv}{{\mathfrak{v}}}

\renewcommand{\sf}{{\starhat{f}}}

\newcommand{\stau}{{\starhat{\tau}}}

\newcommand{\mA}{{\mathsf{A}}}

\newcommand{\mB}{{\mathsf{B}}}

\newcommand{\mC}{{\mathsf{C}}}

\newcommand{\mD}{{\mathsf{D}}}

\newcommand{\mE}{{\mathsf{E}}}

\newcommand{\mF}{{\mathsf{F}}}

\newcommand{\mG}{{\mathsf{G}}}
\newcommand{\mh}{{\mathsf{h}}}
\newcommand{\mH}{{\mathsf{H}}}
\newcommand{\mi}{{\mathsf{i}}}
\newcommand{\mI}{{\mathsf{I}}}
\newcommand{\mj}{{\mathsf{j}}}
\newcommand{\mJ}{{\mathsf{J}}}
\newcommand{\mk}{{\mathsf{k}}}
\newcommand{\mK}{{\mathsf{K}}}

\newcommand{\mL}{{\mathsf{L}}}

\newcommand{\mM}{{\mathsf{M}}}

\newcommand{\mP}{{\mathsf{P}}}

\newcommand{\mR}{{\mathsf{R}}}

\newcommand{\mS}{{\mathsf{S}}}
\newcommand{\mT}{{\mathsf{T}}}

\newcommand{\mU}{{\mathsf{U}}}
\newcommand{\mv}{{\mathsf{v}}}
\newcommand{\mV}{{\mathsf{V}}}

\newcommand{\mx}{{\mathsf{x}}}
\newcommand{\mX}{{\mathsf{X}}}
\newcommand{\my}{{\mathsf{y}}}

\newcommand{\mz}{{\mathsf{z}}}

\newcommand{\mGamma}{{\mathsf{\Gamma}}}

\newcommand{\mDelta}{{\mathsf{\Delta}}}

\newcommand{\mLambda}{{\mathsf{\Lambda}}}

\newcommand{\mXi}{{\mathsf{\Xi}}}

\newcommand{\mPi}{{\mathsf{\Pi}}}

\newcommand{\mSigma}{{\mathsf{\Sigma}}}

\newcommand{\mUpsilon}{{\mathsf{\Upsilon}}}

\newcommand{\bmi}{{\boldsymbol{\mathsf{i}}}}

\newcommand{\bmz}{{\boldsymbol{\mathsf{z}}}}

\newcommand{\tmA}{{\widetilde{\mathsf{A}}}}

\newcommand{\tmC}{{\widetilde{\mathsf{C}}}}

\newcommand{\tmE}{{\widetilde{\mathsf{E}}}}

\newcommand{\tmK}{{\widetilde{\mathsf{K}}}}

\newcommand{\tmLambda}{{\widetilde{\mathsf{\Lambda}}}}

\newcommand{\calB}{{\mathcal{B}}}

\newcommand{\calD}{{\mathcal{D}}}

\newcommand{\calE}{{\mathcal{E}}}

\newcommand{\calF}{{\mathcal{F}}}

\newcommand{\calG}{{\mathcal{G}}}

\newcommand{\calK}{{\mathcal{K}}}

\newcommand{\calL}{{\mathcal{L}}}

\newcommand{\calM}{{\mathcal{M}}}

\newcommand{\calN}{{\mathcal{N}}}

\newcommand{\calO}{{\mathcal{O}}}

\newcommand{\calP}{{\mathcal{P}}}

\newcommand{\calR}{{\mathcal{R}}}

\newcommand{\calT}{{\mathcal{T}}}

\newcommand{\calU}{{\mathcal{U}}}

\newcommand{\calX}{{\mathcal{X}}}

\newcommand{\tcalK}{{\widetilde{\mathcal{K}}}}

\newcommand{\bbB}{{\mathbb{B}}}

\newcommand{\bbC}{{\mathbb{C}}}

\newcommand{\bbE}{{\mathbb{E}}}

\newcommand{\bbG}{{\mathbb{G}}}

\newcommand{\bbN}{{\mathbb{N}}}

\newcommand{\bbR}{{\mathbb{R}}}

\newcommand{\bbU}{{\mathbb{U}}}

\newcommand{\bbV}{{\mathbb{V}}}

\newcommand{\bbZ}{{\mathbb{Z}}}
\newcommand{\rC}{{\mathring{C}}}

\newcommand{\rK}{{\mathring{K}}}

\newcommand{\rr}{{\mathring{r}}}

\newcommand{\rmu}{{\mathring{\mu}}}

\newcommand{\rrC}{{\ringring{C}}}

\newcommand{\rrK}{{\ringring{K}}}

\newcommand{\rrr}{{\ringring{r}}}

\newcommand{\dC}{{\dot{C}}}

\newcommand{\dK}{{\dot{K}}}

\newcommand{\dr}{{\dot{r}}}

\newcommand{\ddC}{{\ddot{C}}}

\newcommand{\ddK}{{\ddot{K}}}

\newcommand{\ddr}{{\ddot{r}}}

\newcommand{\dddK}{{\dddot{K}}}

\newcommand{\dddr}{{\dddot{r}}}

\newcommand{\llvert}{\left\lvert}
\newcommand{\rrvert}{\right\rvert}
\newcommand{\llVert}{\left\lVert}
\newcommand{\rrVert}{\right\rVert}
\newcommand{\llangle}{\left\langle}
\newcommand{\rrangle}{\right\rangle}

\newcommand{\bzero}{{\boldsymbol{0}}}
\newcommand{\bone}{{\boldsymbol{1}}}

\newcommand{\simiid}{\overset{\mathrm{IID}}{\sim}}

\newcommand{\Var}{\mathrm{Var}}
\newcommand{\Cov}{\mathrm{Cov}}

\newcommand{\qtq}[1]{\quad\mathrm{#1}\quad}
\newcommand{\qqtqq}[1]{\qquad\mathrm{#1}\qquad}

\newcommand{\trace}{\mathrm{trace}}

\DeclareMathOperator*{\argmin}{argmin}
\DeclareMathOperator*{\argmax}{argmax}

\newcommand{\diag}{\mathrm{diag}}
\newcommand{\wal}{\mathrm{wal}}
\newcommand{\starhat}[1]{\accentset{\star}{#1}}

\newcommand{\approxhat}[1]{\accentset{\approx}{#1}}

\newcommand\ringring[1]{%
  {% make an Ord atom
   \mathop{\kern0pt #1}\limits^{% set a box over the variable
     \vbox to-1.85ex{
       \kern-2ex % lower the ring accents
       \hbox to 0pt{\hss\normalfont\kern.1em \r{}\kern-.45em \r{}\hss}%
       \vss % fill
     }% end of \vbox
   }% end of the superscript
  }% end of \mathop
}

\usepackage{listings}
\definecolor{darkgreen}{rgb}{0,0.6,0}
\lstdefinestyle{Python}{
    showstringspaces=false,
    language        = Python,
    basicstyle      = \small\ttfamily,
    morekeywords = {as},
    keywordstyle    = \color{blue},
    stringstyle     = \color{purple},
    commentstyle    = \color{darkgreen}\ttfamily,
    breaklines = true,
	postbreak=\text{$\hookrightarrow$\space},
	alsoletter = {>,.} ,
    morekeywords = [2]{>>>,...},
    keywordstyle = [2]\color{cyan}\bfseries}

\newcommand{\patchoverfull}{\newline}

\DeclareFontFamily{U}{mathx}{}
\DeclareFontShape{U}{mathx}{m}{n}{<-> mathx10}{}
\DeclareSymbolFont{mathx}{U}{mathx}{m}{n}
\DeclareMathAccent{\widehat}{0}{mathx}{"70}
\DeclareMathAccent{\widecheck}{0}{mathx}{"71}

\begin{document}

\title{Algorithms and Scientific Software for Quasi-Monte Carlo, \\ Fast Gaussian Process Regression, and Scientific Machine Learning}
\author{Aleksei Gregory Sorokin}
\degree{Doctor of Philosophy}
\dept{Applied Mathematics}
\date{December 2025}
% \copyrightnoticetrue
\maketitle

\prelimpages

\begin{acknowledgment} 
\indent Thank you to Mariss, my mother, my father, my brother, my grandmother, and my aunt for your unwavering kindness, compassion, and guidance. Thank you to Professor Fred Hickernell for your mentorship, support, and guidance throughout my undergraduate and graduate studies. Thank you to Sou-Cheng Choi, Mike McCourt, Pieterjan Robbe, Jagadeeswaran Rathinavel, Aadit Jain, and the rest of the \texttt{QMCPy} team for your support and collaboration. Thank you to my thesis committee for your time and feedback. 

This material is based upon work supported by the U.S. Department of Energy, Office of Science, Office of Workforce Development for Teachers and Scientists, Office of Science Graduate Student Research (SCGSR) program. The SCGSR program is administered by the Oak Ridge Institute for Science and Education for the DOE under contract number DE-SC0014664.

This work is also supported by the National Science Foundation DMS Grant No. 2316011. 
\end{acknowledgment}

\begin{authorship} \label{sec:authorship}
    I, Aleksei Gregory Sorokin, attest that the work in this thesis is substantially my own.

	In accordance with the disciplinary norm of authorship in applied mathematics, the following collaborations occurred in the thesis. (Please consult Appendix S of the IIT Faculty Handbook for further information.) 

	My advisor, Prof. Fred J. Hickernell, contributed to the conceptualization and progress of the research project, as is the norm for an applied mathematics supervisor.  

    I am grateful to have had a number of opportunities to work with excellent collaborators across academia, national labs, and industry. Below I list some of these projects and cite our resulting papers. 

    \begin{itemize}
        \item \cite{sorokin.fastgps_probnum25}: \textbf{Fast Gaussian Process Regression for High-Dimensional Functions with Derivative Information.} This project was developed during my 2025 DOE SCGSR (US Department of Energy Office of Science Graduate Student Research) Program fellowship appointment at Sandia National Laboratories in  collaboration with Pieterjan Robbe and Fred J. Hickernell. This paper was published in the \emph{Proceedings of Machine Learning Research (PMLR) in the First International Conference on Probabilistic Numerics 2025} with a Creative Commons Attribution 4.0 International License \url{https://creativecommons.org/licenses/by/4.0}. See \Cref{sec:fmtgps} for details of this project. 
        \item \cite{hickernell.qmc_what_why_how}: \textbf{Quasi-Monte Carlo Methods: What, Why, and How?} This was a collaboration with Fred J. Hickernell and Nathan Kirk which was published in the \emph{Monte Carlo and Quasi-Monte Carlo Methods 2024 Proceedings}. See \Cref{sec:qmc} for details of this project, reproduced with permission from Springer Nature.  
        \item \cite{sorokin.MC_vector_functions_integrals}: \textbf{On Bounding and Approximating Functions of Multiple Expectations Using Quasi-Monte Carlo.} This was a collaboration with Jagadeeswaran Rathinavel which was published in the \emph{Monte Carlo and Quasi-Monte Carlo Methods 2022 Proceedings}. See  \Cref{sec:stop_crit_vectorziation} for details of this project, reproduced with permission from Springer Nature. 
        \item \cite{choi.challenges_great_qmc_software}: \textbf{Challenges in Developing Great Quasi-Monte Carlo Software.} This was a collaboration with Sou-Cheng T. Choi, Yuhan Ding, Fred J. Hickernell, and Jagadeeswaran Rathinavel which was published in the \emph{Monte Carlo and Quasi-Monte Carlo Methods 2022 Proceedings}. See \Cref{sec:qmc} for details of this project, reproduced with permission from Springer Nature. 
        \item \cite{choi.QMC_software}: \textbf{Quasi-Monte Carlo Software.} This was a  collaboration with Sou-Cheng T. Choi, Fred J. Hickernell, Jagadeeswaran Rathinavel, and Michael J. McCourt which was published in the \emph{Monte Carlo and Quasi-Monte Carlo Methods 2020 Proceedings}. See \Cref{sec:qmc} for details of this project, reproduced with permission from Springer Nature.
        \item \cite{sorokin.QMC_IS_QMCPy}: \textbf{(Quasi-)Monte Carlo Importance Sampling with QMCPy.} This was a collaboration with Fred J. Hickernell, Sou-Cheng T. Choi, Michael J. McCourt, and Jagadeeswaran Rathinavel which was published in the \emph{2021 IIT Undergraduate Research Journal}. See \Cref{sec:variable_transforms} for details of this project. 
        \item \cite{sorokin.RTE_DeepONet}: \textbf{A Neural Surrogate Solver for Radiation Transfer.} This project was developed during my 2024 Scientific Machine Learning Researcher appointment at FM (Factory Mutual Insurance Company) in collaboration with Xiaoyi Lu and Yi Wang. The paper was published through the \emph{NeurIPS 2024 Workshop on Data-Driven and Differentiable Simulations, Surrogates, and Solvers} with a Creative Commons Attribution-ShareAlike 4.0 International License \url{https://creativecommons.org/licenses/by-sa/4.0/}. See \Cref{sec:RTE_DeepONet} for details of this project. 
        \item \cite{sorokin.sigopt_mulch}: \textbf{SigOpt Mulch: An Intelligent System for AutoML of Gradient Boosted Trees.} This project was developed during my 2021 Machine Learning Engineer Intern appointment at SigOpt (now a part of Intel) in collaboration with Michael J. McCourt, Xinran Zhu, Eric Hans Lee, and Bolong Cheng. The paper was published in a \emph{2023 Knowledge-Based Systems Journal}. This project is not discussed in this thesis. 
        \item \cite{sorokin.gp4darcy}: \textbf{Computationally Efficient and Error Aware Surrogate Construction for Numerical Solutions of Subsurface Flow Through Porous Media.} This project was developed during my 2023 Graduate Internship appointment at Los Alamos National Laboratory in collaboration with Aleksandra Pachalieva, Daniel O'Malley, James M. Hyman, Fred J. Hickernell, and Nicolas W. Hengartner. The paper was published in a \emph{2024 Advances in Water Resources Journal}. See \Cref{sec:gp4darcy} for details of this project.
        \item \cite{gjergo.GalCEM1}: \textbf{GalCEM. I. An Open-Source Detailed Isotopic Chemical Evolution Code} and \cite{GalCEM.software}: \textbf{GalCEM: GALactic Chemical Evolution Model.} These papers were in collaboration with Eda Gjergo, Anthony Ruth, Emanuele Spitoni, Francesca Matteucci, Xilong Fan, Jinning Liang, Marco Limongi, Yuta Yamazaki, Motohiko Kusakabe, and Toshitaka Kajino. The two papers were published in the \emph{2023 Astrophysical Journal Supplement Series} and \emph{Astrophysics Source Code Library} respectively. We have also recently posted a preprint \cite{gjergo2025massive} on \textbf{Massive Star Formation at Supersolar Metallicities: Constraints on the Initial Mass Function} in collaboration with Eda Gjergo, Zhiyu Zhang, Pavel Kroupa, Zhiqiang Yan, Ziyi Guo, Tereza Jerabkova, Akram Hasani Zoonozi, and Hosein Haghi. This project is not discussed in this thesis. 
        \item \cite{sorokin.2025.ld_randomizations_ho_nets_fast_kernel_mats}: \textbf{QMCPy: A Python Software for Randomized Low-Discrepancy Sequences, Quasi-Monte Carlo, and Fast Kernel Methods.}  This is currently an unpublished preprint. We review this material in \Cref{sec:rldseqs_acm_toms}.
        \item \cite{sorokin.adaptive_prob_failure_GP}: \textbf{Credible Intervals for Probability of Failure with Gaussian Processes.} This was developed during my 2022 Givens Associate Internship appointment at Argonne National Laboratories in collaboration with Vishwas Rao. The paper is currently an unpublished preprint. See \Cref{sec:pfgpci} for details of this project.
        \item \cite{jain.bernstein_betting_confidence_intervals}: \textbf{Empirical Bernstein and Betting Confidence Intervals for Randomized Quasi-Monte Carlo.} This project was a collaboration with Art B. Owen, Aadit Jain, and Fred J. Hickernell. The paper is currently an unpublished preprint. This project is not discussed in this thesis.
        \item \cite{bacho.CHONKNORIS}, \textbf{Operator Learning at Machine Precision.} This project was developed during my 2025 DOE SCGSR fellowship at Sandia National Laboratories under the guidance of Fred J. Hickernell and Pieterjan Robbe along with collaborators from the California Institute of Technology, including Houman Owhadi, Aras Bacho, Xianjin Yang, Th\'eo Bourdais, Edoardo Calvello, and Matthieu Darcy, as well as collaborators from the University of Washington, including Alex Hsu and Bamdad Hosseini. The paper is currently an unpublished preprint. See \Cref{sec:CHONKNORIS} for details of this project. 
        \item \cite{sorokin.FastBayesianMLQMC}, \textbf{Fast Bayesian Multilevel Quasi-Monte Carlo.} This project was developed during my 2025 DOE SCGSR fellowship at Sandia National Laboratories in collaboration with Pieterjan Robbe, Fred J. Hickernell, Gianluca Geraci, and Michael S. Eldred. The paper is currently an unpublished preprint. See \Cref{sec:FastBayesianMLQMC} for details of this project.
    \end{itemize}
\end{authorship}

\tableofcontents
\clearpage

\listoftables
\clearpage

\listoffigures
\clearpage

\clearpage

\begin{abstract}
    \indent Most scientific domains elicit the development of efficient algorithms and accessible scientific software. This thesis unifies our developments in three broad domains: Quasi-Monte Carlo (QMC) methods for efficient high-dimensional integration,  Gaussian process (GP) regression for high-dimensional interpolation with built-in uncertainty quantification, and scientific machine learning (sciML) for modeling partial differential equations (PDEs) with mesh-free solvers. For QMC, we built new algorithms for vectorized error estimation and developed \texttt{QMCPy} (\url{https://qmcsoftware.github.io/QMCSoftware/}): an open-source Python interface to randomized low-discrepancy sequence generators, automatic variable transforms, adaptive error estimation procedures, and diverse use cases. For GPs, we derived new digitally-shift-invariant kernels of higher-order smoothness, developed novel fast multitask GP algorithms, and produced the scalable Python software \texttt{FastGPs} (\url{https://alegresor.github.io/fastgps/}). For sciML, we developed a new algorithm capable of machine precision recovery of PDEs with random coefficients. We have also studied a number of applications including GPs for probability of failure estimation, multilevel GPs for the Darcy flow equation, neural surrogates for modeling radiative transfer, and fast GPs for Bayesian multilevel QMC.
\end{abstract}

\textpages

\Chapter{Introduction} 

My PhD research has spanned a number of different scientific domains and ranged from theoretical developments to algorithmic enhancements to accessible implementations. This chapter discusses motivations and applications for a number of different domains upon which my research has touched. The advantages and limitations of each method will also be highlighted.

\Section{Quasi-Monte Carlo for High-Dimensional Numerical Integration}

\Subsection{Motivation and Applications} 

Low-discrepancy (LD) sequences, also called Quasi-random sequences, judiciously explore the unit cube in high dimensions. They were first developed to replace the independent points used in Monte Carlo simulation for high-dimensional numerical integration. The resulting Quasi-Monte Carlo (QMC) methods have proven theoretically superior to independent Monte Carlo methods for a large class of nicely behaved functions \cite{niederreiter.qmc_book,dick.digital_nets_sequences_book,kroese.handbook_mc_methods,dick2022lattice,lemieux2009monte,sloan1994lattice,dick.high_dim_integration_qmc_way}. \Cref{fig:points} shows the more uniform coverage of LD points compared to independent ones, and \Cref{fig:qmc_convergence} shows the resulting QMC algorithms attain superior convergence rates compared to independent Monte Carlo and grid-based cubature schemes for a six-dimensional Keister function \cite{keister.multidim_quadrature_algorithms_keister_fun}. The QMC advantage has also been shown across a variety of scientific disciplines including:
\begin{enumerate}
    \item \textbf{Financial modeling,} especially for option pricing \cite{joy1996quasi,lai1998applications,l2004quasi,l2009quasi,xu2015high,giles.mlqmc_path_simulation}. 
    \item \textbf{Solving PDEs (partial differential equations),} often with random coefficients which are sampled at LD locations \cite{graham2011quasi,kuo2012quasi,kuo2015multi,graham2015quasi,kuo.application_qmc_elliptic_pde,robbe.multi_index_qmc}.
    \item \textbf{Ray tracing,} for  graphics rendering and simulating light transport \cite{jensen2003monte,raab2006unbiased,waechter2011quasi}.
\end{enumerate}

\begin{figure}[!ht]
    \centering
    \includegraphics[width=1\linewidth]{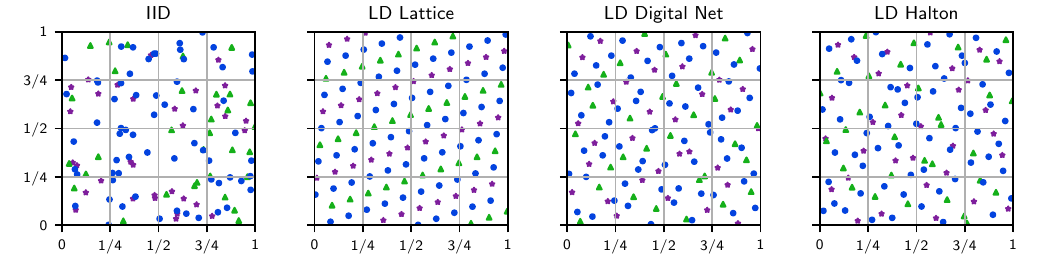}
    \caption{IID (independent identically distributed) and LD (low-discrepancy) point sets. IID points have gaps and clusters while the dependent LD points fill space more evenly. Quasi-Monte Carlo methods use LD points for high-dimensional integral approximation and typically converge faster than IID Monte Carlo methods. The points are colored as follows: purple stars for the first 32 points, green triangles for the next 32, and blue circles for the subsequent 64. The lattice has been randomly shifted, the digital net has been randomized with a nested uniform scramble, and the Halton points have been randomized with a linear matrix scramble and permutation scramble.}
    \label{fig:points}
\end{figure}

\begin{figure}[!ht]
    \centering
    \includegraphics[width=1\linewidth]{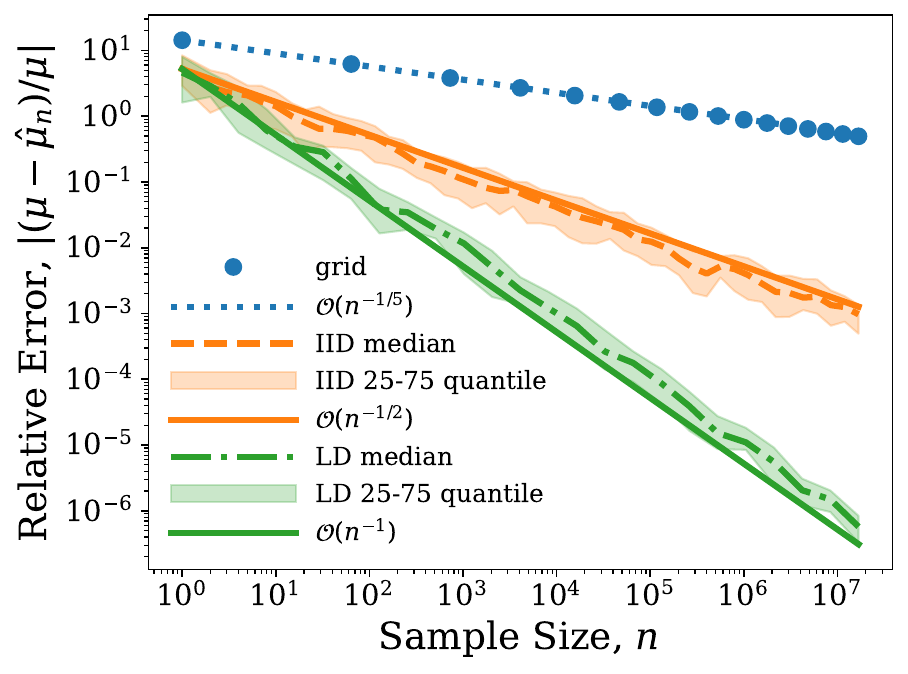}
    \caption{The relative error of approximating the mean of a $d=6$-dimensional Keister function \cite{keister.multidim_quadrature_algorithms_keister_fun} for various choices of nodes. Low-discrepancy (LD) nodes for Quasi-Monte Carlo methods achieve superior convergence rates compared to IID Monte Carlo methods and grid-based cubature schemes.}
    \label{fig:qmc_convergence}
\end{figure}

\Subsection{Randomized Low-Discrepancy Sequences and Randomized QMC}

Randomized LD sequences \cite{owen1995randomly,l2016randomized,owen.variance_alternative_scrambles_digital_net,MATOUSEK1998527,tezuka2002randomization}, such as those shown in \Cref{fig:points}, provide additional advantages:
\begin{enumerate}
    \item \textbf{Enabling QMC error estimation.} One may use independent randomizations of an LD sequence to give independent estimates to the sample mean. These independent estimates may be used to compute standard confidence intervals such as Student's-$t$ confidence intervals, see \Cref{sec:stop_crit_qmc_rep_student_t} or \cite{lecuyer.RQMC_CLT_bootstrap_comparison} or \cite[Chapter 17]{owen.mc_book_practical}.
    \item \textbf{Avoiding boundary observations.} The first point of an unrandomized LD sequence is often the origin, which is undesirable in many applications. For example, transformations such as the inverse CDF of a normal will map the origin to infinite values which could break simulations. Randomized LD sequences almost surely avoid sampling the boundary of the unit cube.
    \item \textbf{Avoid fooling functions.} While mainly a theoretical advantage, randomizing LD sequences negates the probability of pairing to adverse functions. 
\end{enumerate}
For these reasons, randomized LD sequences are typically used in practical applications, and randomization is the default in the vast majority of software libraries for which they are implemented.

While QMC is theoretically superior to independent Monte Carlo for sufficiently regular functions, simply replacing independent points with LD sequences often provides an advantage for functions which do not necessarily satisfy such regularity conditions. For example, later we will show that QMC provides superior error estimates compared to independent Monte Carlo for various option pricing examples. Such examples often contain kinks which destroy the regularity required for QMC methods to provide a theoretically guaranteed advantage. Even so, QMC is shown to provide improved convergence rates compared to IID Monte Carlo for these integrands without sufficient regularity.

Moreover, it is often more difficult to determine if a function satisfies QMC-regularity conditions that it would be to determine an expression for the mean which QMC tries to approximate. In our experience, for numerical integration, replacing IID points with LD points has never worsened performance and often enables superior rates of convergence, even for functions which are not sufficiently regular or for which verifying regularity conditions is infeasible.

\Section{Gaussian Processes for High-Dimensional Interpolation with Uncertainty Quantification}

\Subsection{Motivation and Applications} 

Gaussian process regression models, which we will simply refer to as GPs, are valuable tools for statistical machine learning that come with built-in uncertainty quantification (UQ) \cite{rasmussen.gp4ml,fasshauer.meshfree_approx_methods_matlab}. Specifically, one assumes the unknown function is a draw from a GP so that the posterior distribution conditioned on past function observations is also a GP. The posterior mean is often taken as the regression prediction, and UQ is typically quantified through some function of the posterior covariance. \Cref{fig:noisy_lattice_qgp.4} gives a simple sketch of a GP model. 

\begin{figure}[ht!]
    \centering
    \includegraphics[width=.8\textwidth]{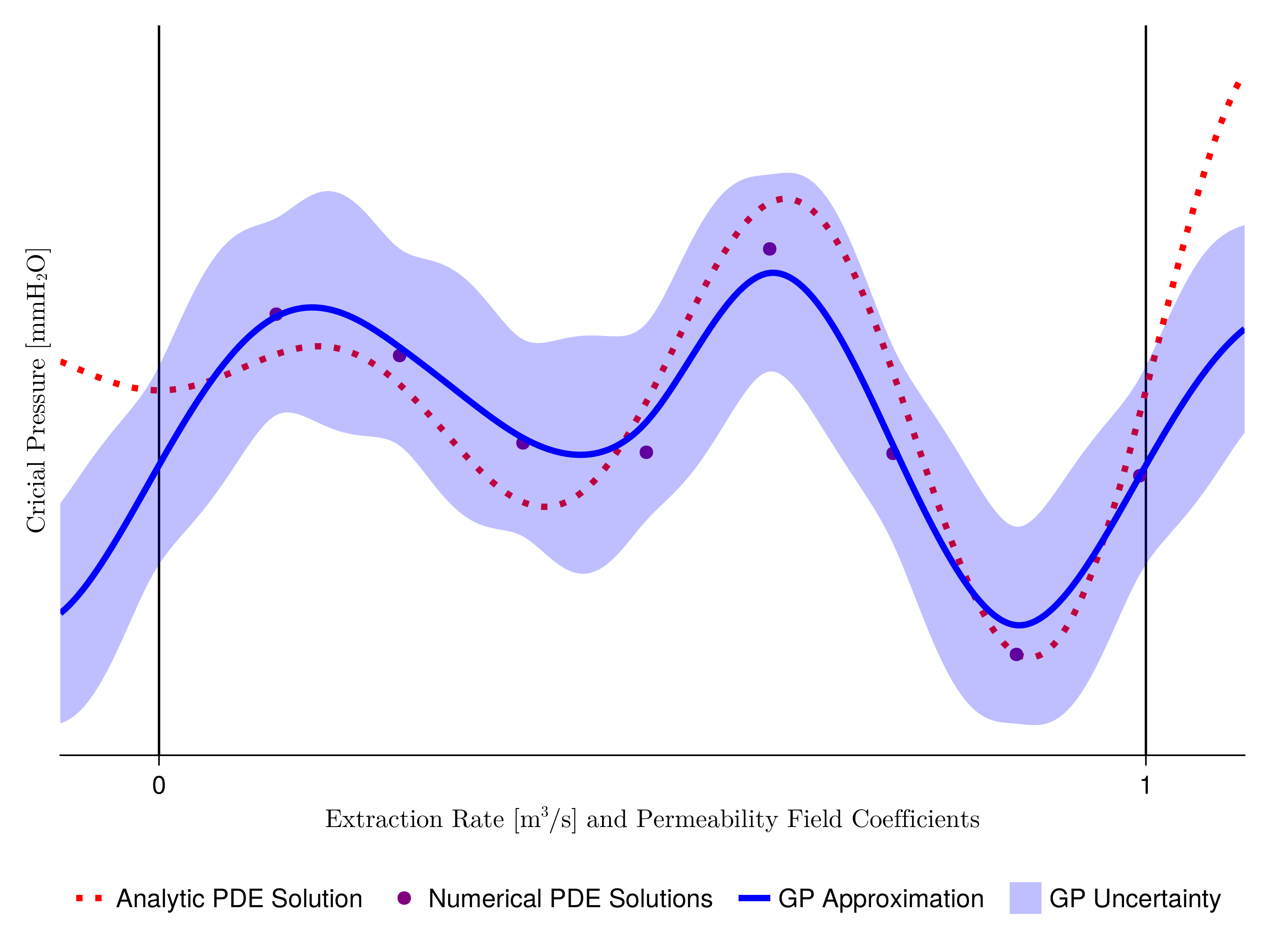}
    \caption{Sketch of a Gaussian process (GP) model which predicts critical pressure in the subsurface from a controlled extraction rate. The true analytic PDE solution is unknown, and we only have access to noisy observations in the form of numerical PDE solutions. We fit a GP in blue where the posterior mean is the GP approximation and the GP uncertainty is quantified through point-wise confidence intervals derived from the posterior variance.}
    \label{fig:noisy_lattice_qgp.4}
\end{figure}

The kernels used for GPs correspond to certain RKHSs (reproducing kernel Hilbert spaces) \cite{aronszajn.theory_of_reproducing_kernels}. In fact, GP regression and RKHS kernel interpolation (or more generally RKHS kernel ridge regression) share a number of similarities. We will give a few connections in \Cref{sec:gp_rkhs_connections}, see \cite{owhadi.book_operator_adapted_wavelets,kanagawa.gp_rkhs_connections} or \cite[Chapter 6]{rasmussen.gp4ml} for more details. The most well know similarities are that the GP posterior mean is equivalent to the optimal RKHS kernel interpolant, and noisy observations for GPs correspond to the optimal RKHS kernel ridge regression prediction which penalizes the RKHS norm of the prediction model. 

GPs, and their equivalents, have found applications across a broad range of scientific domains. These include:
\begin{enumerate}
    \item \textbf{Bayesian optimization,} to find global minima of a function which is expensive to evaluate \cite{bo_tutorial_frazier,snoek2012practical,wu2020practical}. Here the GP uncertainty guides sequential sampling to trade off exploration in regions of high uncertainty and exploitation of regions where the GP prediction is small.  
    \item \textbf{Solving PDEs.} For PDEs with deterministic coefficients, often the solution values and derivatives at collocation points are jointly optimized to minimize the norm of the kernel interpolant in the corresponding RKHS \cite{cockayne2017probabilistic,chen.learning_nonlinear_PDE_GP,chen.sparse_cholesky_PDE_kernel_methods,batlle2025error,long2024kernel}. For PDEs with random coefficients, the GP model often maps the coefficients in a Karhunen--Lo\`eve type expansion to the PDE solution from an existing solver \cite{kaarnioja.kernel_interpolants_lattice_rkhs,kaarnioja.kernel_interpolants_lattice_rkhs_serendipitous,batlle.operator_learning_kernel_methods,sorokin.gp4darcy}.
    \item \textbf{Bayesian cubature,} where the integrand is assumed to be a GP, so the posterior distribution of the integral has a closed form Gaussian distribution \cite{briol2019probabilistic,o1991bayes,rasmussen2003bayesian,rathinavel.bayesian_QMC_lattice,rathinavel.bayesian_QMC_sobol,rathinavel.bayesian_QMC_thesis}. 
    \item \textbf{Reliability analysis,} which is closely related to rare event simulation and level set estimation, tries to quantify the probability that a simulation subject to random conditions will fail \cite{rackwitz2001reliability,renganathan.PF_CAMERA_multifidelity,dubourg.metamodel_IS_reliability_analysis,bae.pf_gp_uncertainty_reduction,zanette.super_LSE_GP,sorokin.adaptive_prob_failure_GP}. Oftentimes, as in Bayesian optimization, uncertainty in the GP model guides sequential sampling to select informant settings at which to simulate.   
\end{enumerate} 

\Subsection{Multitask GPs} \label{sec:background_mtgps}

There has also been a recent push to model correlated simulations using multitask GPs which capture both inter-task and intra-task covariances \cite{bonilla2007multi,chai2010multi}. Specifically, multitask GP kernels model covariances between both the spatial inputs and the correlated simulation tasks. Often this kernel will be written as a product between a kernel over the spatial inputs and a freely-parameterized kernel over the task indices. Multitask GP models have proven valuable in:
\begin{enumerate}
    \item \textbf{Multitask Bayesian optimization,} for simultaneously minimizing multiple objectives or for meta-learning across correlated problems \cite{swersky2013multi,khatamsaz2023multi,huang2021bayesian,dai2020multi}. 
    \item \textbf{Medical treatment analysis,} for modeling treatment effects or correlated patient time series data \cite{alaa2017bayesian,durichen2014multi,futoma2017learning,chen2023multi}.
    \item \textbf{Engineering applications,} where designs effects may be simulated under a variety of controls \cite{williams2008multi,liu2023learning}. 
\end{enumerate}

\Subsection{GPs with Derivatives} \label{sec:background_gps_derivatives} 

Finally, there has been a growing interest in GP models which incorporate derivative observations \cite[Chapter 9.4]{rasmussen.gp4ml}. Such observations have been shown to improve performance across a number of applications:
\begin{enumerate}
    \item \textbf{Bayesian optimization,} for rough surface terrain reconstruction \cite{eriksson2018scaling} and training neural networks with auto-differentiation \cite{padidar2021scaling}.
    \item \textbf{Bayesian cubature,} where the gradient, and possibly Hessian, have been shown to enhance GP predictions \cite{wu2017exploiting}.
    \item \textbf{Modeling dynamical systems,} where one not only observes a function but also the derivatives \cite{solak.gp_derivatives}. 
    \item \textbf{Solving PDEs,} where the PDE solution is modeled as a GP and the function values, including derivative observations, are optimized to minimize the solution's RKHS norm under the PDE constrains at each collocation point \cite{chen.learning_nonlinear_PDE_GP,chen.sparse_cholesky_PDE_kernel_methods}.
\end{enumerate}

\Section{Accelerated GPs for Controlled Experiments} \label{sec:background_fgps}

The main drawback of standard Gaussian process regression methods is they typically require $\calO(n^2d+n^3)$ computations and $\calO(n^2)$ storage in the number of collocation points $n$ and dimension $d$. These costs quickly become prohibitive for many problems, even with access to HPC (high performance computing) systems. The crux of the problem is that one must solve a linear system in the SPD (symmetric positive definite) Gram matrix of pairwise kernel evaluations at collocation points. For this, a Cholesky factorization and back-substitution solves are standard. Often, one will also need to compute the determinant of the Gram matrix in order to optimize kernel hyperparameters. 

A large body of research has gone into accelerating standard Gaussian process regression techniques. In the following subsections we will briefly review popular methods to speed up GP inference and motivate our focus on fast GPs pairing low-discrepancy sequences and (digitally)-shift-invariant kernels. 

\Subsection{Preconditioned Conjugate Gradient Descent with Black Box Matrix Multiplication} \label{sec:pcg_bbmm}

Perhaps the most generally applicable and most widely used GP acceleration method is to solve the Gram matrix linear system using a preconditioned conjugate gradient (PCG) method and, if necessary, use an iterative approximation to the determinant \cite{gardner.gpytorch_GPU_conjugate_gradient}. This enables HPC systems to utilize  efficient matrix multiplication routines for which hardware like GPUs can provide massive benefits. This method, which backends popular software like \texttt{GPyTorch} \cite{gardner.gpytorch_GPU_conjugate_gradient}, reduces the complexity of standard Gaussian process regression to $\mathcal{O}(n^2d)$ when the Gram matrix is well conditioned. However, PCG may fail to provide significant savings when the Gram matrix is ill-conditioned, as often occurs when including derivative-observations or other highly-correlated observations. Moreover, these iterative techniques, and many other accelerated GP techniques, can only provide approximate solutions to the GP fitting problem. 

\Subsection{Forcing Gram Matrix Structure} \label{sec:forcing_gram_matrix_structure}

Exact GPs may be made more affordable by forcing structure into the Gram matrix. Nicely structured Gram matrices may permit reduced storage requirements and/or reduced costs for computing the product, inverse, or determinant. Exact, structured GPs typically have two requirements: 
\begin{enumerate}
    \item One must control the design of experiments in order to use special collocation points. 
    \item One must choose from a restricted class of kernels whose structures will pair nicely with the collocation points.
\end{enumerate}
The first requirement is prohibitive for many GP regression problems where the dataset is given to the modeler. However, many scientific computing applications allow for controlled experiments, for example when building surrogate models to computer experiments or sampling random coefficient inputs to a simulation. The second requirement may be more or less restrictive depending on what information is known about the simulation and which kernels are supported by the GP acceleration technique. 

\Subsection{Structured Kernel Interpolation with Grids and Product Kernels}

One popular structure-inducing method is to pair product kernels with Cartesian grid sampling locations to produce Kronecker product Gram matrices. For $n$ points in $d$ dimensions, one samples $\calO(n^{1/d})$ grid points in each dimension and the resulting Kronecker product Gram matrices require only $\calO(n^{2/d}d)$ storage (compared to the standard $\calO(n^2)$ requirement) and permit GP fitting at only $\calO(n^{2/d}d^2+n^{3/d}d)$ cost (compared to the standard $\calO(n^2d+n^3)$ cost) \cite{saatcci2012scalable,gardner2018product}. Additionally, if the covariance kernel is stationary then each matrix in the Kronecker product becomes Toeplitz and GP fitting costs can be further reduced to $\calO(d n^{1+1/d})$. 

These Kronecker structures can also be exploited for unstructured data by using a technique known as structured kernel interpolation (SKI) \cite{wilson2015kernel}. In essence, SKI uses pseudo-data that lies on a Cartesian grid and then interpolates the kernel between grid points in a sparse manner. This structure is typically used to further accelerate the black-box matrix multiplication in the PCG scheme of \cite{gardner.gpytorch_GPU_conjugate_gradient}.

Using grid points in higher dimensions can be intractable and is often undesirable. For example, in dimension $d=47$ using just $2$ grid points in each dimension would require $n=2^{47} \approx 1.4 \times 10^{14}$ function evaluations whose values would take up over a petabyte of storage in $64$ bit precision. It was proposed in \cite{wilson2014fast} to use partial grid structure and virtual observations. This may yield inexact GP fits as one must approximately solve the Gram matrix system and compute its determinant, often using PCG and block box matrix multiplication as proposed in \cite{gardner.gpytorch_GPU_conjugate_gradient} and discussed in \Cref{sec:pcg_bbmm}. Another drawback of using (partial) grids is their one-dimensional projections only have $n_j$ unique points, which is often quite small for large $d$. \Cref{fig:points_probnum25} shows $64$ grid points in $d=3$ dimensions with each $1$-dimensional projection having only $4$ unique values.

\begin{figure}[!ht]
    \centering
    \includegraphics[width=1\linewidth]{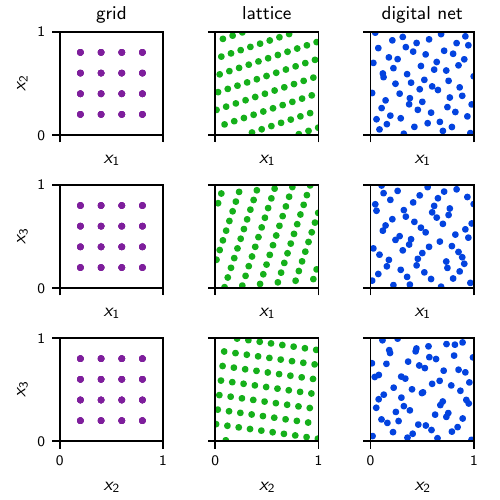}
    \caption{Projections of $n=64$ points in $d=3$ dimensions from a regular grid, shifted lattice point set, and digitally-shifted digital net. Notice the lattice and digital net low-discrepancy point sets more evenly fill the unit cube and have more diverse projections.}
    \label{fig:points_probnum25}
\end{figure}

\Subsection{Fast GPs with Low-Discrepancy Sequences and \\ (Digitally)-Shift-Invariant Kernels} \label{sec:intro_fastgps} 

Low-discrepancy (LD) sequences, especially randomized versions, have also been extensively used as efficient experimental designs. Two simple benefits over grid points are:
\begin{enumerate}
    \item LD sequences have unique points when projected onto subsets of dimensions. This is illustrated in \Cref{fig:points_probnum25} where all $64$ LD points are visible in any marginal projection, while for $64$ grid points only $16$ unique values are visible in any two-dimensional projection and only $4$ unique values are visible in any one-dimensional projection.  
    \item LD sequences attain nice uniformity properties in high dimensions for samples sizes $n=2^p$ for most $p \geq 0$. On the other hand, the minimum number of grid points grows exponentially with the dimension as illustrated in the earlier example in $d=47$ dimensions which would require at least $n=2^{47}$ function evaluations for practical application. 
\end{enumerate}

LD sequences have an additional benefit for accelerating GP regression: Pairing certain LD points with certain kernels yields structured Gram matrices which are diagonalizable by fast transforms \cite{zeng.spline_lattice_digital_net,zeng.spline_lattice_error_analysis}. These fast GP methods  require only $\calO(n)$ storage and $\calO(n \log n + nd)$ computations including kernel hyperparameter optimization. We will focus on two such pairings which were initially proposed in \cite{zeng.spline_lattice_error_analysis,zeng.spline_lattice_digital_net}: 
\begin{enumerate}
    \item Pairing rank-1 lattices with shift-invariant kernels produces circulant Gram matrices diagonalizable by fast Fourier transforms (FFTs) \cite{cooley1965algorithm} which cost only $\calO(n \log n)$. 
    \item Pairing digital nets in base $2$ with digitally-shift-invariant kernels produces Gram matrices diagonalizable by the fast Walsh--Hadamard transform (FWHT) \cite{fino.fwht} which also costs only $\calO(n \log n)$.  
\end{enumerate}

These methods face the same two limitations enumerated in \Cref{sec:forcing_gram_matrix_structure}. The first limitation is that one must control the design of experiments as was the case for exact GPs with grid points and product kernels. The SKI methods \cite{wilson2015kernel}, which use grids as inducing points to accommodate unstructured data, are not immediately applicable as sparse interpolation is not straightforward for LD points. It has recently been shown that certain LD point sets are also quasi-uniform in terms of their covering radius and separation radius \cite{dick2025quasi_lattices,dick2025quasi_digital_nets}. These special LD sequences will lead to tighter error bounds and improved convergence rates when coupled with the space filling terms in Sobolev inequalities which are commonly used in kernel method error analysis. However, other point sets, such as regular grids, can also obtain optimal quasi-uniformity rates of convergence, albeit at samples sizes which occur much less frequently than for LD point sets.

The second and primary limitation of fast GPs is the requirement to use shift-invariant (SI) or digitally-shift-invariant (DSI) kernels supported on the unit cube domain. These structures preclude most popular kernels such as the squared exponential, rational quadratic, and  Mat\'ern family. SI kernels are periodic and kernels of higher-order smoothness are known \cite{kaarnioja.kernel_interpolants_lattice_rkhs,kaarnioja.kernel_interpolants_lattice_rkhs_serendipitous,cools2021fast,cools2020lattice,sloan2001tractability,kuo2004lattice}. However, most functions which are encountered in practice are not periodic, and it is even less common to deal with functions which are both smoothly periodic and have derivatives which are also periodic.  While smoothness-preserving periodizing transforms exist, they often create large bumps in the function which increase variation and therefore degrade SI kernel based estimates. In our experiments, we found that using SI kernels on common benchmark problems often limited performance due to the unreasonable periodicity assumption, even after using periodizing transforms.

While the currently available DSI kernels are not periodic, they are discontinuous. Interestingly, in \Cref{sec:dsi_kernels_smooth_functions} will derive new discontinuous DSI kernels for which the corresponding RKHSs contain smooth functions. Unlike SI kernels, fast GPs with DSI kernels usually achieved competitive error estimates compared to more popular kernels while also enabling accelerated modeling and reduced storage requirements. However, practitioners may be troubled by additional concerns regarding using DSI kernel GPs. The first is that, despite their ability to model smooth functions, the discontinuous nature of DSI kernels means that the corresponding DSI-GPs have discontinuous prior draws, posterior means, and posterior variance estimates. The second troubling aspect is that the RKHSs corresponding to our new DSI kernels contain not only smooth functions but also functions which are not desirably smooth. This often leads to fast DSI-GPs with conservative error estimates as they must also consider functions outside the Sobolev classes of smooth functions which are often studied through the use of Mat\'ern kernels.

\Cref{fig:2d_gp_error} compares fast GP methods with LD points and (digitally)-shift-invariant kernels against a standard GP regression method for the highly oscillatory Ackley function in two dimensions \cite{ackley2012connectionist}. With $4096$ observations, the fast GP methods achieve slightly better accuracy and are over a thousand times faster to fit compared to the standard GP. 

\begin{figure*}[!ht]
    \centering
    \includegraphics[width=1\linewidth]{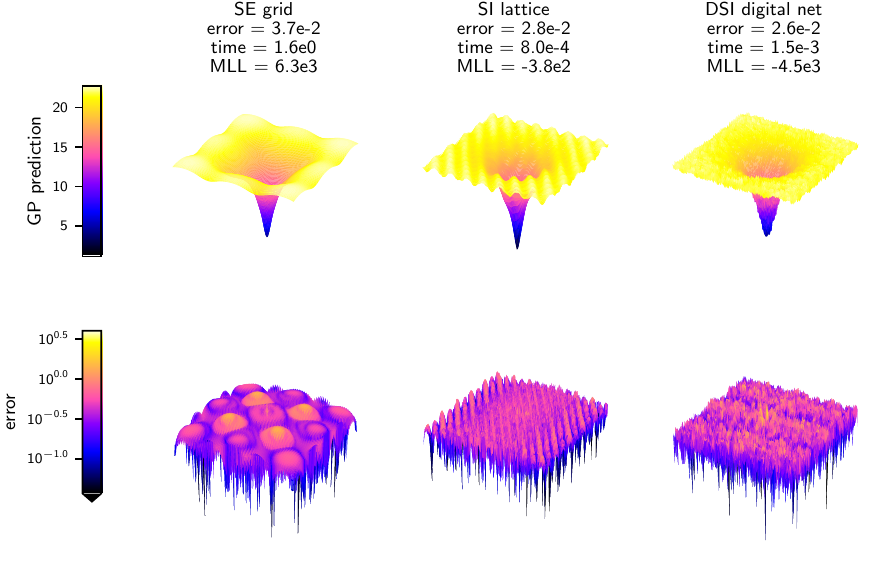}
    \caption{Comparison between standard GP regression using grid points with a squared exponential (SE) kernel and fast GP methods pairing either lattices with a shift-invariant (SI) kernel or digital nets with a DSI kernel. We model the Ackley function in $d=2$ dimensions sampled with $n=4096$ observations. Metrics for $L_2$ relative error, time per optimization step, and marginal log-likelihood (MLL) are also given.}
    \label{fig:2d_gp_error}
\end{figure*}

\Subsection{Applications of Fast GPs} \label{sec:apps_fast_gps}

The structures underlying fast GP methods have also been applied to accelerate kernel interpolation methods for:
\begin{enumerate}
    \item \textbf{Solving PDEs,} often with random coefficients which are sampled at low-discrepancy locations \cite{kaarnioja.kernel_interpolants_lattice_rkhs,kaarnioja.kernel_interpolants_lattice_rkhs_serendipitous,sorokin.gp4darcy}.
    \item \textbf{Fast Bayesian cubature,} which enables uncertainty quantification for Quasi-Monte Carlo with a single randomized low-discrepancy sequence \cite{rathinavel.bayesian_QMC_lattice,rathinavel.bayesian_QMC_sobol,rathinavel.bayesian_QMC_thesis}. This improves upon popular Quasi-Monte Carlo methods which typically require multiple randomized low-discrepancy sequences to enable error estimation. 
    \item \textbf{Fast discrepancy computations,} which may exploit the Gram matrix structures for kernel discrepancy computations in RKHSs (reproducing kernel Hilbert spaces) \cite{hickernell.generalized_discrepancy_quadrature_error_bound,hickernell1998lattice}. 
\end{enumerate}

\Section{Scientific Machine Learning for Modeling PDEs \\ with Mesh-Free Solvers} \label{sec:background_sciml}

Scientific machine learning (sciML) has garnered significant interest in recent years thanks to its ability to automatically solve PDEs. SciML carries a number of advantages over traditional PDE solvers 
\begin{enumerate}
    \item \textbf{Easy setup of solvers,} with just a few lines of code specifying the PDE constraints in a compatible format. 
    \item \textbf{Flexible parameterizations,} through a variety of neural network architectures or GP kernel specifications.
    \item \textbf{Scalable training,} enabled by extensive work on GPU-acceleration libraries such as \texttt{PyTorch} \cite{PyTorch.software} and their extensions to enable multi-GPU and cluster support.  
    \item \textbf{Rapid inference,} thanks to the fast and parallelizable evaluation of existing sciML models. This is especially evident in operator learning of PDEs with random coefficients \cite{boulle2024mathematical,kovachki2024operator} where inference for an unseen realization only requires a single pass through the sciML model rather than completely rerunning a traditional solver. 
\end{enumerate}

\Subsection{Neural Network Architectures} 

Neural networks are by far the most popular class of sciML models, and they have been applied to countless PDEs across the literature. Below we describe a few landmark neural architectures. We refer the interested reader to reviews of physics informed machine learning \cite{karniadakis2021physics} and neural operators \cite{kovachki2023neural}. 
\begin{enumerate}
    \item \textbf{Physics Informed Neural Networks (PINNs),} which solve a PDE (with deterministic coefficients) by modeling the solution as a neural network and then training to minimize a physics-informed loss \cite{raissi2019physics}. Such a loss function is typically composed of PDE constraints and implemented using automatic-differentiation to infer derivatives of the PDE solution. 
    \item \textbf{Deep Operator Networks (DeepONets),} which model an operator through a basis expansion by combining basis coefficient predictions from a (branch) network (taking random coefficient realizations as inputs) with basis function predictions from a (trunk) network (taking the collocation points in the domain as inputs) \cite{Lu21Nature}. 
    \item \textbf{Fourier Neural Operators (FNOs),} which are resolution-invariant operator learning models that map inputs of any resolution to input Fourier coefficients, use a neural network to map input Fourier coefficients to output Fourier coefficients, and then infer the solution at an arbitrary resolution \cite{li2020fourier}.
\end{enumerate}
\Cref{fig:poolfire} considers radiative transfer in a pool fire and compares solutions from a traditional finite volume solver against predictions from a DeepONet model. Notice the DeepONet is able to capture the physical behavior of the fire and provide homogeneous radiation predictions thanks to its mesh-free training and inference. This problem is detailed in \Cref{sec:RTE_DeepONet}.

\begin{figure}[!ht]
    \centering
    \includegraphics[width=1\linewidth, clip=True]{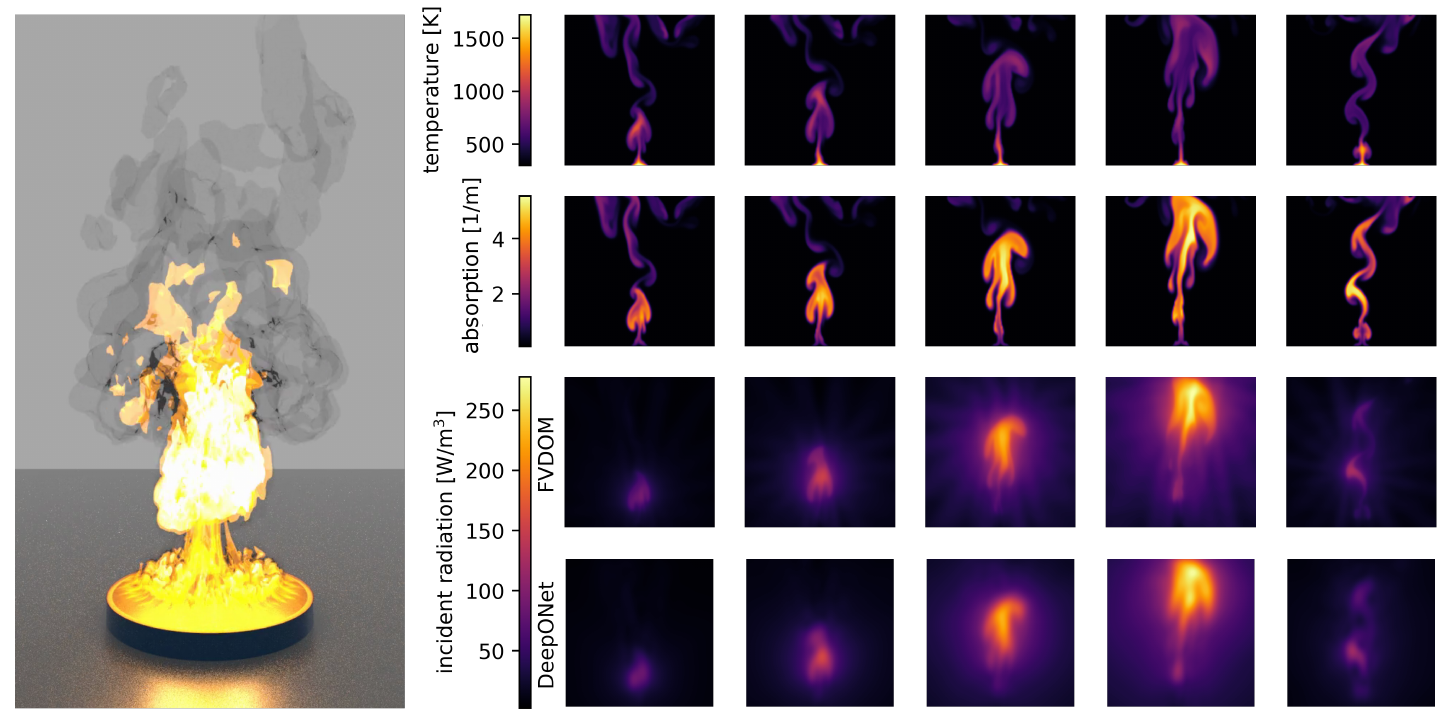}
    \caption{Predicted incident radiation from unknown temperature and absorption coefficients in a pool fire as described in \Cref{sec:RTE_DeepONet}. We compare DeepONet solutions against those from a reference Finite Volume Discrete-Ordinate Method (FVDOM) solver. Notice that FVDOM suffers from the ray effect where predicted radiation around the fire has greater intensity along the fixed discretization directions. The mesh-free nature of DeepONet fitting and inference enables more homogeneous predictions which more closely align with physical experimentation.}
    \label{fig:poolfire}
\end{figure}

\Subsection{GP Methods}

GP methods for sciML offer strong theoretical foundations and convergence guarantees \cite{owhadi.book_operator_adapted_wavelets}. Such guarantees are less readily available for neural architectures, although recent progress has been made in this direction \cite{kovachki2024operator,lanthaler2024discretization,grohs2025theory,reinhardt2024statistical}. 

\begin{enumerate}
    \item \textbf{GPs methods for PDEs,} may solve a PDE (with deterministic coefficients) by optimizing values and derivative observations at collocation points (subject to PDE constrains) to minimize the kernel interpolant's RKHS norm \cite{chen.learning_nonlinear_PDE_GP,chen.sparse_cholesky_PDE_kernel_methods}. 
    \item \textbf{GP operator learning,} often fits a vector-valued GP to map observations of the input to observations of the output \cite{kadri2016operator,nelsen2021random,batlle.operator_learning_kernel_methods}. GPs are also used to map the input observations to the input function and to map the output observations to the output function, both of which can be done in a resolution invariant manner. 
    \item \textbf{Hybrid GP and neural network operator learning,} may set the prior mean of a GP to a neural network to leverage both the strong theoretical backing of GPs and the great flexibility of neural architectures \cite{mora2025operator,owhadi2023ideas,owhadi2019kernel}.
\end{enumerate}

\Subsection{Use of Low-Discrepancy Points in SciML}

Low-discrepancy (LD) points are widely used as space filling designs for training neural networks, even if their explicit use is not an emphasized in many publications. There have been a few recent works explicitly showing the advantage of training neural sciML models using LD collocation points \cite{longo2021higher,chen2021quasi,keller2025regularity}. Moreover, as referenced in \Cref{sec:apps_fast_gps}, fast GPs using LD points and shift-invariant kernels have been used for solving a variety of PDEs with random coefficients \cite{kaarnioja.kernel_interpolants_lattice_rkhs,kaarnioja.kernel_interpolants_lattice_rkhs_serendipitous,sorokin.gp4darcy}. This year, papers have come out showing certain LD point sets are also optimally space filling (quasi-uniform) \cite{dick2025quasi_digital_nets,dick2025quasi_lattices}. Such LD points could pair nicely with the space filling terms in Sobolev inequalities which are commonly used in proving convergence rates for GP methods \cite{chen.learning_nonlinear_PDE_GP,chen.sparse_cholesky_PDE_kernel_methods,kadri2016operator,nelsen2021random,batlle.operator_learning_kernel_methods}. 

\Subsection{A Word of Caution on Using SciML Methods}

We emphasize that caution should be taken when using sciML to replace traditional numerical PDE solvers. While theoretical guarantees exist for GP methods, and recent work has found similar guarantees for neural architectures, these assurances are often asymptotic and generally weaker than those available for traditional methods such as finite difference methods, finite volume methods, or finite element methods. Moreover, the design choice of sciML models is often rather heuristic and typically found using a guess-and-check method. While prior knowledge of a simulation may be encoded into say the GP kernel or the neural network activation functions, such choices present additional difficulties and opportunities for misspecification compared to more robust traditional methods.

Finally, we note that sciML models generally require expensive training, achieve significantly lower accuracy, and are not necessarily faster to evaluate compared to traditional solvers. While traditional solvers often require constructing and solving large linear systems, training sciML models often require an unknown number of parameter-optimization steps which can scale poorly, especially for larger models. There is also a trade-off for sciML models between accuracy and inference speed. In settings where inference speed is paramount, sciML models can be developed which are capable of almost instantaneous inference but achieve low-accuracy predictions. However, increasing the complexity of a sciML model, for example by increasing the neural network depth, may not necessarily achieve desirable accuracy. Large scale models that are capable of comparable accuracy often end up emulating the underlying numerical solvers and thus may not provide the desired accelerated inference. In other words, there is no free lunch with sciML models, but they still provide value, especially in settings where low accuracy is acceptable in exchange for rapid inference. 

\Chapter{Contributions and Outline} 

This chapter highlights our contributions across Quasi-Monte Carlo, fast Gaussian process regression, and scientific machine learning ranging from theory to implementations to applications. First, \Cref{sec:contrib_list} will list our contributions with citations to our work. Then, \Cref{sec:qmcpy_features} will detail our \texttt{QMCPy} software contribution. Next, \Cref{sec:fastgps_features} will detail our \texttt{FastGPs} software contribution. Finally, \Cref{sec:apps_overview} will detail our novel applications in the fields of uncertainty quantification and scientific machine learning. The links within these sections will act as an outline for the remainder of this thesis.

\Section{List of Contributions} \label{sec:contrib_list}

The following list details our contributions with citations to our work. 
\begin{enumerate}
    \item Developed \texttt{QMCPy} (\url{https://qmcsoftware.github.io/QMCSoftware/}) \cite{choi.QMC_software,choi.challenges_great_qmc_software,hickernell.qmc_what_why_how,sorokin.2025.ld_randomizations_ho_nets_fast_kernel_mats,sorokin.MC_vector_functions_integrals}, an open-source Python library of algorithms from across the Quasi-Monte Carlo community in a unified framework. The features of \texttt{QMCPy} are detailed in \Cref{sec:qmcpy_features}. 
    \item Built \texttt{LDData} (\url{https://github.com/QMCSoftware/LDData}), a repository for lattice generating vectors, digital net generating matrices, and quality low-discrepancy point sets. This data is stored in new standardized formats and an interface to these formats was implemented into \texttt{QMCPy}. \Cref{sec:rldseqs_acm_toms} provides examples using \texttt{LDData}.
    \item Implemented into \texttt{QMCPy} the first Python interface to support lattices, higher-order digital nets, higher-order digital net scrambling with either linear matrix scrambling (LMS) or nested uniform scrambling (NUS), and Halton scrambling with either LMS or NUS \cite{sorokin.2025.ld_randomizations_ho_nets_fast_kernel_mats}. \Cref{sec:rldseqs_acm_toms} details all the randomized LD sequences supported by \texttt{QMCPy}. 
    \item Developed algorithms to estimate and quantify error for functions of multiple integrals approximated using Quasi-Monte Carlo \cite{sorokin.MC_vector_functions_integrals}. Our \texttt{QMCPy} implementation adaptively select the number of points required to meet a user-specified error tolerance on the combined quantity of interest. This work is detailed in \Cref{sec:stop_crit_vectorziation} along with a number of examples. 
    \item Derived a new set of digitally-shift-invariant (DSI) kernels whose corresponding RKHSs contain smooth functions \cite{sorokin.2025.ld_randomizations_ho_nets_fast_kernel_mats}. This includes a new order $4$ smoothness kernel whose form has not appeared elsewhere in the literature. DSI kernels are treated in \Cref{sec:dsi_kernels}.
    \item Implemented into \texttt{QMCPy} the first Python interface to shift-invariant (SI) and digitally-shift-invariant (DSI) kernels of varying smoothness \cite{sorokin.2025.ld_randomizations_ho_nets_fast_kernel_mats}. Such kernels are required to enable fast kernel computations. SI kernels are discussed in \Cref{sec:si_kernels} while DSI kernels are discussed in \Cref{sec:dsi_kernels}. 
    \item Implemented into \texttt{QMCPy} the most efficient Python routines available to perform the fast Fourier transform in bit-reversed order (FFTBR), the inverse FFTBR (IFFTBR), and the fast Walsh--Hadamard transform (FWHT) \cite{sorokin.2025.ld_randomizations_ho_nets_fast_kernel_mats}. Such transforms are required to enable fast kernel computations as detailed in \Cref{sec:fast_kernel_methods}. 
    \item Developed \texttt{FastGPs} (\url{https://alegresor.github.io/fastgps/}) \cite{sorokin.fastgps_probnum25}, an open-source Python library for fast Gaussian process regression algorithms pairing low-discrepancy points with (digitally)-shift-invariant kernels. The features of \texttt{FastGPs} are detailed in  \Cref{sec:fastgps_features}.
    \item Implemented into \texttt{FastGPs} the first open-source code for fast Gaussian process regression with support for both the lattice with shift-invariant kernel version and digital net with digitally-shift-invariant-kernel version \cite{sorokin.fastgps_probnum25}. Fast GPs are detailed in \Cref{sec:fast_gps}.
    \item Derived and analyzed novel fast multitask GPs (\Cref{sec:fmtgps}) and fast derivative-informed GP (\Cref{sec:gps_deriv_informeds}) \cite{sorokin.fastgps_probnum25}. 
    \item Applied Quasi-Monte Carlo, fast Gaussian process regression, and/or scientific machine learning to a number of applications with collaborators from academia, national labs, and industry \cite{sorokin.adaptive_prob_failure_GP,sorokin.gp4darcy,sorokin.RTE_DeepONet,bacho.CHONKNORIS,sorokin.FastBayesianMLQMC,sorokin.sigopt_mulch,gjergo.GalCEM1,GalCEM.software,gjergo2025massive}. \Cref{sec:apps_overview} provides an overview of these applications as well as links to sections detailing specific projects.
\end{enumerate}

\Section{\texttt{QMCPy} Software for Quasi-Monte Carlo Methods} \label{sec:qmcpy_features}

While Quasi-Monte Carlo methods have been extensively studied, their implementations are often scattered or even missing in some popular programming languages \cite{choi.QMC_software}. \texttt{QMCPy} implements algorithms from across the QMC community into a unified and accessible Python framework for both theoreticians and practitioners. Below we summarize the supported features in \texttt{QMCPy} and provide pointers to related codes:
\begin{enumerate}
    \item \textbf{Randomized LD Sequences} (\Cref{sec:rldseqs_acm_toms}). High quality pseudo-random number generators (PRNGs) have been extensively developed \cite{lecuyer.random_number_generation_book}, and are readily available in nearly all popular programming languages. Generators and randomization routines for LD point sets have been generally less readily available than their PRNG counterparts. \texttt{QMCPy} supports the point sets and randomization routines described below. These features are also supported in the comprehensive Java software \texttt{SSJ}\footnote{\url{https://simul.iro.umontreal.ca/ssj/}} \cite{lecuyer.ssj_software}. The C++ \texttt{LatNet Builder}\footnote{\label{url:latnetbuilder}\url{https://github.com/umontreal-simul/latnetbuilder}} software \cite{LatNetBuilder.software} and the multi-language \texttt{Magic Point Shop (MPS)}\footnote{\label{url:mps}\url{https://people.cs.kuleuven.be/~dirk.nuyens/qmc-generators/}} \cite{kuo.application_qmc_elliptic_pde} both provide search routines for finding good lattice generating vectors and digital net generating matrices. \texttt{QMCPy} integrates with the new \texttt{LDData} repository which contains a variety of these pregenerated vectors and matrices in standardized formats. \texttt{LDData} additionally includes popular choices from the websites of Frances Kuo on lattices\footnote{\url{https://web.maths.unsw.edu.au/~fkuo/lattice/index.html}} \cite{cools2006constructing,nuyens2006fast} and Sobol' points\footnote{\url{https://web.maths.unsw.edu.au/~fkuo/sobol/index.html}} (a special case of digital nets) \cite{joe2003remark,joe2008constructing}.
    \begin{enumerate}
        \item \textbf{Lattice Points} (\Cref{sec:lattices}). Rank-1 lattices may be generated in radical inverse or linear order, and may be randomized using shifts modulo one. Randomly shifted rank-1 lattices are also available in \texttt{MPS} and \texttt{GAIL}\footnote{\label{url:gail}\url{http://gailgithub.github.io/GAIL_Dev/}} (MATLAB's Guaranteed Automatic Integration Library) \cite{GAIL.software,hickernell2018monte}.
        \item \textbf{Digital Nets} (\Cref{sec:dnets}). Base $2$ digital nets, including higher-order digital nets, may be generated in either radical inverse or Gray code order, and may be randomized with linear matrix scrambling (LMS) \cite{owen.variance_alternative_scrambles_digital_net}, digital shifts, nested uniform scrambling (NUS) \cite{owen1995randomly}, and/or permutation scrambling. Early implementations of unrandomized digital sequences, including the Faure, Sobol', and Niederreiter--Xing constructions, can be found in \cite{fox1986algorithm,bratley1992implementation,bratley2003implementing,pirsic2002software}. Considerations for implementing scrambles were discussed in \cite{hong2003algorithm}. Support for combining LMS with digital shifts is also provided in MATLAB, \texttt{MPS}, and both the \texttt{PyTorch}\footnote{\url{https://pytorch.org/}} \cite{PyTorch.software} and \texttt{SciPy}\footnote{\label{url:scipy}\url{https://scipy.org/}} \cite{SciPy.software} Python packages.
        \item \textbf{Halton Points} (\Cref{sec:Halton}). As with digital nets, Halton point sets may be randomized with LMS, digital shifts, NUS, and/or permutation scrambles. The implementation of Halton point sets and randomizations have been treated in \cite{owen_halton,wang2000randomized}. The \texttt{QRNG}\footnote{\label{url:qrng}\url{https://cran.r-project.org/web/packages/qrng/qrng.pdf}} (Quasi-Random Number Generators) R package \cite{qrng.software} implements generalized Halton point sets \cite{faure2009generalized} which use optimized digital permutation scrambles; these are also supported in \texttt{QMCPy}.
    \end{enumerate}
    \item \textbf{Variable Transformations} (\Cref{sec:variable_transforms}). These define the distribution of stochasticity in the underlying problem and  automatically rewrite user-defined functions into QMC-compatible forms. The available transforms are mainly wrappers around distributions provided by the \texttt{SciPy} Python package.  
    \item \textbf{Error Estimators for (Q)MC and Multilevel (Q)MC} (\Cref{sec:qmc_stopping_crit}). We provide numerous adaptive error estimation algorithms which automatically select the number of points required for a (Q)MC approximation to be within user-specified error tolerances. QMC error estimation is treated more broadly in \cite{owen.error_QMC_review}, while \cite{clancy2014cost,adaptive_qmc} detail additional considerations for adaptive QMC algorithms. \texttt{QMCPy}'s adaptive error estimation procedures are described below. The single-level algorithms were also implemented in \texttt{GAIL}. 
    \begin{enumerate}
        \item \textbf{Monte Carlo with Independent Points} (\Cref{sec:stop_crit_iid}). Confidence intervals are derived from either a central limit theorem (CLT) heuristic or a guaranteed version of the CLT for functions with bounded kurtosis \cite{hickernell.MC_guaranteed_CI}.
        \item \textbf{QMC with Multiple Randomizations} (\Cref{sec:stop_crit_qmc_rep_student_t}). Student's-$t$ confidence intervals are evaluated based on independent mean estimates from independent randomizations of an LD point set, see \cite{lecuyer.RQMC_CLT_bootstrap_comparison} or \cite[Chapter 17]{owen.mc_book_practical}. 
        \item \textbf{QMC via Decay Tracking using a Single Randomized LD Sequence} (\Cref{sec:stop_crit_qmc_decay_tracking}). Guaranteed error bounds are available for cones of functions whose Fourier or Walsh coefficients decay predictably \cite{hickernell.adaptive_dn_cubature,adaptive_qmc,cubqmclattice,ding2018adaptive}. 
        \item \textbf{QMC via Bayesian Cubature using a Single Randomized LD Sequence} (\Cref{sec:stop_crit_qmc_fast_bayes}) Posterior credible intervals on the integral of a Gaussian process may be computed quickly by exploiting fast kernel computations \cite{rathinavel.bayesian_QMC_lattice,rathinavel.bayesian_QMC_sobol,rathinavel.bayesian_QMC_thesis}.
        \item \textbf{Multilevel Monte Carlo with Independent Points} (\Cref{sec:stop_crit_multilevel}). We have implemented standard multilevel Monte Carlo \cite{giles.MLMC_path_simulation,giles2015multilevel} and continuation multilevel Monte Carlo algorithms \cite{collier2015continuation}.
        \item \textbf{Multilevel QMC with Multiple Randomizations} (\Cref{sec:stop_crit_multilevel}). We have implemented standard multilevel QMC \cite{giles.mlqmc_path_simulation} and continuation multilevel QMC algorithms \cite{robbe2019multilevel}. 
    \end{enumerate}
    \item \textbf{Fast Kernel Computations} (\Cref{sec:fast_kernel_methods}). Matching a rank-1 lattice in radical inverse order to a shift-invariant (SI) RKHS kernel yields a Gram matrix diagonalizable by the fast Fourier transform (FFT) \cite{cooley1965algorithm} in bit-reversed order (FFTBR) and its inverse (IFFTBR). Similarly, matching a base $2$ digital net in radical inverse order to a digitally-shift-invariant (DSI) kernel yields a Gram matrix diagonalizable by the fast Walsh--Hadamard transform (FWHT) \cite{fino.fwht}. The  currently supported kernels and fast transforms are described below. \texttt{PyTorch} compatible versions of these methods are also maintained in order to enable GPU acceleration. 
    \begin{enumerate}
        \item \textbf{RKHS Kernels} (\Cref{sec:si_kernels,sec:dsi_kernels}). We provide SI kernels and DSI kernels, including those of higher-order smoothness. We also maintain an interface to a number of popular kernels, such as squared exponential, rational quadratic, and Mat\'ern kernels, whose implementations may also be found in the \texttt{GPyTorch} \cite{gardner.gpytorch_GPU_conjugate_gradient} and the \texttt{scikit-learn}\footnote{\label{url:scikitlearn}\url{https://scikit-learn.org/}} \cite{scikit-learn} Python packages among others. SI kernels of arbitrary smoothness are well known and can be computed based on the Bernoulli polynomials \cite{kaarnioja.kernel_interpolants_lattice_rkhs,kaarnioja.kernel_interpolants_lattice_rkhs_serendipitous,cools2021fast,cools2020lattice,sloan2001tractability,kuo2004lattice}. DSI kernels of order $1$ smoothness were derived in \cite{dick.multivariate_integraion_sobolev_spaces_digital_nets}. In \cite{sorokin.2025.ld_randomizations_ho_nets_fast_kernel_mats}, we derived new higher-order DSI kernels whose RKHSs contain smooth functions. Low-order kernel forms appeared in \cite{baldeaux.polylat_efficient_comp_worse_case_error_cbc} as worst-case error bounds on QMC rules using higher-order polynomial lattices, but there they were not interpreted as DSI kernels. The form of the order $4$ smoothness DSI kernel has not appeared elsewhere in the literature.
        \item \textbf{Fast Transforms} (\Cref{sec:fft_si_kernels_r1lattices,sec:fwht_dsi_kernels_dnb2s}). We provide interfaces to the FFTBR, IFFTBR, and FWHT algorithms which all have $\mathcal{O}(n \log n)$ complexity in the number of points $n$. Our FFTBR and IFFT implementations use the FFT routines in the \texttt{SciPy}. At the time of publishing, a significantly slower implementation of the FWHT is also available in the \texttt{SymPy}\footnote{\label{url:sympy}\url{https://www.sympy.org}} Python package \cite{10.7717/peerj-cs.103}. 
    \end{enumerate}
\end{enumerate}

\Section{\texttt{FastGPs} Software for Fast Gaussian Process} \label{sec:fastgps_features}

\texttt{FastGPs} implements fast Gaussian process (GP) regression methods which pair low-discrepancy (LD) points to special shift-invariant (SI) or digitally-shift-invariant (DSI) kernels, see the background in \Cref{sec:intro_fastgps}. Fast GPs are detailed in \Cref{sec:gps} which use the fast kernel computations described in \Cref{sec:fast_kernel_methods}. Below we summarize the supported features in \texttt{FastGPs}:

\begin{enumerate}
    \item \textbf{Kernel hyperparameter optimization,} with support for optimizing either the marginal log-likelihood (MLL) or generalized cross validation (GCV) loss.
    \item \textbf{Fast Bayesian cubature,} for uncertainty quantification in Quasi-Monte Carlo. 
    \item \textbf{Batched GPs,} for simultaneously modeling vector-output simulations.
    \item \textbf{GPU support,} enabled by the \texttt{PyTorch} stack. 
    \item \textbf{Flexible LD sequence and SI/DSI kernel specifications,} using \texttt{QMCPy}.
    \item \textbf{Efficient variance projections,} for non-greedy Bayesian optimization in multilevel Monte Carlo and multilevel Quasi-Monte Carlo.
    \item \textbf{Fast multitask GPs,} with support for different sample sizes and LD sequence randomizations for each task.
    \item \textbf{Derivative-informed GPs,} for simulations coupled with automatic differentiation. 
\end{enumerate}

\Section{Applications in Uncertainty Quantification and Scientific Machine Learning} \label{sec:apps_overview}

We have developed a number of applications across a variety of disciplines. Some of these projects are listed below along with the resulting citations and references. 

\begin{enumerate}
    \item \cite{sorokin.FastBayesianMLQMC} proposed fast Bayesian multilevel QMC without replications; see \Cref{sec:FastBayesianMLQMC} for details.
    \item \cite{sorokin.adaptive_prob_failure_GP} predicted the probability of system failure and quantified prediction error when using Gaussian processes models, see \Cref{sec:pfgpci} for details.
    \item \cite{sorokin.gp4darcy} modeled subsurface flow through porous media with multilevel fast Gaussian processes; see \Cref{sec:gp4darcy} for details.
    \item \cite{sorokin.RTE_DeepONet} modeled radiative transfer with deep operator networks (DeepONets); see \Cref{sec:RTE_DeepONet} for details.
    \item \cite{bacho.CHONKNORIS} presented the CHONKNORIS method for solving nonlinear parameterized PDEs to machine precision using scientific machine learning; see \Cref{sec:CHONKNORIS} for details.
    \item \cite{sorokin.sigopt_mulch} developed SigOpt Mulch, a tool for enhanced machine learning of gradient boosted trees by meta-learning priors for Bayesian optimization; this is not discussed further in this thesis.
    \item \cite{jain.bernstein_betting_confidence_intervals} studied betting confidence intervals for randomized Quasi-Monte Carlo, specifically looking at the trade-off between the size and number of randomizations of a low-discrepancy point set; this is not discussed further in this thesis.
    \item \cite{gjergo.GalCEM1,GalCEM.software,gjergo2025massive} delivered codes for modeling galactic chemical evolution; this is not discussed further in this thesis.
\end{enumerate}

\Chapter{Notation}

We will denote the set of positive integers by $\bbN = \{1,2,\dots\}$ and the set of nonnegative integers by $\bbN_0 = \bbN \cup \{0\}$. The set of all integers will be denoted by $\bbZ$, and the set of all integers excluding $0$ will be denoted by $\bbZ_0 = \bbZ \setminus \{0\}$. The reals and complex numbers will be denoted by $\bbR$ and $\bbC$ respectively. We will also use intuitive notations to further restrict sets, e.g., $\bbR_{>1/2} = \{x \in \bbR: x > 1/2\}$ or $\bbN_{\geq 2} = \{i \in \bbN: i \geq 2\}$. We will often use $\calX \subset \bbR^d$ to denote the domain of a function. 

Capital letters in sans-serif font will denote matrices, e.g., $\mK \in \bbR^{n \times n'}$ may denote an $n \times n'$ matrix of kernel evaluations at pairs of points. The transpose of a matrix $\mK$ is $\mK^\intercal$ and its conjugate is $\overline{\mK}$. The identity matrix is $\mI$. 

Lower-case letters in sans-serif font will denote digits in a base $b$ expansion, e.g., for $i \in \bbN_0$ we may write $i = \sum_{t \in \bbN_0} \mi_t b^t$. This may be combined with bold notation when denoting the vector of base $b$ digits, e.g., $\bmi = (\mi_0,\mi_1,\dots,\mi_{m-1})^\intercal$. Modulo is always taken elementwise, e.g., rank-1 lattices will use the notation $\bx \bmod 1 = (x_1 \bmod 1,\dots,x_d \bmod 1)^\intercal$ and digital nets will take all matrix operations to be carried modulo $b$. Permutations may be identified by vectors, e.g., a permutation $\pi: \{0,1,2\} \to \{0,1,2\}$ with $\pi(0) = 2$, $\pi(1) = 0$, and $\pi(2)=1$ may be denoted by $\pi = (2,0,1)$. 

For probability, we will let $\bbE$ denote the expectation, sometimes using notations like $\bbE_X$ or $\bbE_\bbG$ if we need to specify which random variable $X$ or measure $\bbG$ over which the expectation is taken. Similarly, $\bbV$ will denote the variance. The vertical bar notation is used for conditional random variables, e.g., $Z_1 | Z_2$ is $Z_1$ conditioned on $Z_2$. 

We will use $\odot$ to denote elementwise multiplication (the Hadamard product), $\oplus$ to denote digit-wise addition modulo $b$ for some base $b$, and $\ominus$ to denote digit-wise subtraction modulo $b$. The later two notations will be defined more precisely before they are used. We will use $\otimes$ to denote the Kronecker product where, for an $m \times n$ matrix $\mA$ and a $p \times q$ matrix $\mB$, the Kronecker product $\mA \otimes \mB$ is the $pm \times qn$ block matrix 
$$\mA \otimes \mB = \begin{pmatrix} a_{11} \mB & \cdots & a_{1n} \mB \\ \vdots & \ddots & \vdots \\ a_{m1} \mB & \cdots & a_{mn} \mB \end{pmatrix}.$$ 

We will use the shorthand $\partial_x f := \D f / \D x$ and often write $f^{(\bbeta)} := \partial_{x_1}^{\beta_1} \cdots \partial_{x_d}^{\beta_d} f(\bx)$ for $\bbeta \in \bbN_0^d$. Similarly, we may write $K^{(\bbeta,\bbeta')}(\bx,\bx') := \partial_{x_1}^{\beta_1} \cdots \partial_{x_d}^{\beta_d} \partial_{x_1'}^{\beta_1'} \cdots \partial_{x_d'}^{\beta_d'} K(\bx,\bx')$ for $\bbeta,\bbeta' \in \bbN_0^d$.   Notations such as $f^p(\bx)$ will be used instead of writing $(f(\bx))^p$. Putting a bar over a symbol will denote the complex conjugate, e.g., for $c = a+b\sqrt{-1} \in \bbC$ we have $\overline{c} = a-b\sqrt{-1}$. The complex conjugate is assumed to be applied elementwise when annotating vectors or matrices. Similarly, integrals of functions with multiple outputs are understood to be taken elementwise, e.g., if $\boldsymbol{f}: \calX \to \bbR^2$ then $\int \boldsymbol{f} := (\int f_1, \int f_2)$. We will often use $1_{C}(\bx)$ to denote an indicator which is $1$ if condition $C$ is satisfied and $0$ otherwise. For a symmetric non-negative kernel $K$ we will denote its corresponding RKHS by $H(K)$. 

\Chapter{Quasi-Monte Carlo Methods} \label{sec:qmc}

This chapter details Quasi-Monte Carlo (QMC) methods and our implementation into \texttt{QMCPy}. \Cref{sec:qmcpy_code_setup} will discussion installing and setting up \texttt{QMCPy} in order to follow along with the codes presented in this chapter. \Cref{sec:qmc_context} will then formulate QMC methods which we break into four components: Randomized low-discrepancy sequences (\Cref{sec:rldseqs_acm_toms}), automatic variable transformations (\Cref{sec:variable_transforms}), user-specified problems (\Cref{{sec:qmc_problems}}), and adaptive stopping criterion algorithms for error estimation (\Cref{sec:qmc_stopping_crit}). \Cref{sec:stop_crit_vectorziation} details our novel vectorized stopping criterion algorithms for functions of multiple expectations.

\Section{\texttt{QMCPy} Setup and Code Reproducibility} \label{sec:qmcpy_code_setup}

The features discussed in this chapter require our novel \texttt{QMCPy} package (\url{https://qmcsoftware.github.io/QMCSoftware/}) with version 2.0 or later, which relies on the \texttt{NumPy} and \texttt{SciPy} Python packages. \texttt{QMCPy} is readily installed using the command \texttt{pip install -U qmcpy}, after which we may import our required packages:

\begin{lstlisting}[style=Python]
import qmcpy as qp 
import numpy as np
import scipy.stats
\end{lstlisting}

Most of \texttt{QMCPy}'s randomized low-discrepancy sequence routines are wrappers around efficient C implementations in our \texttt{QMCToolsCL} package \\ (\url{https://qmcsoftware.github.io/QMCToolsCL/}). These low level routines could be used to rapidly build similar interfaces in other programming languages which support C extensions.

\Section{Context} \label{sec:qmc_context} 

This section provides known background on Monte Carlo (MC) and Quasi-Monte Carlo (QMC) methods. MC and QMC methods approximate a high-dimensional integral over the unit cube by the sample average of function evaluations at certain sampling locations:
\begin{equation}
    \mu := \bbE[f(\bX)] = \int_{[0,1]^d} f(\bx) \D \bx \approx \frac{1}{n} \sum_{i=0}^{n-1} f(\bx_i) =: \hmu, \qquad \bX \sim \calU[0,1]^d.
    \label{eq:mc_approx}
\end{equation}
Here $f: [0,1]^d \to \bbR$ is a given integrand and $\{\bx_i\}_{i=0}^{n-1} \in [0,1]^{n \times d}$ is a point set. For problems with non-uniform measures, one may apply a variable transformation as discussed in \Cref{sec:variable_transforms}. Classic MC methods choose the sampling locations to be independent and identically distributed (IID) $d$-dimensional standard uniforms $\bx_0,\dots,\bx_{n-1} \simiid \calU[0,1]^d$. IID-MC methods for \eqref{eq:mc_approx} have a root mean squared error (RMSE) of $\calO(n^{-1/2})$.

QMC methods \cite{niederreiter.qmc_book,dick.digital_nets_sequences_book,kroese.handbook_mc_methods,dick2022lattice,lemieux2009monte,sloan1994lattice,dick.high_dim_integration_qmc_way} replace IID point sets with LD point sets which more evenly cover the unit cube $[0,1]^d$. For integrands with bounded variation, plugging LD point sets into \eqref{eq:mc_approx} yields a worst-case error rate of $\calO(n^{-1+\delta})$ with $\delta>0$ arbitrarily small. Some popular LD point sets are plotted in \Cref{fig:pointsets} including randomized rank-1 lattices, base $2$ digital nets (including higher-order versions), and Halton point sets.

Randomized Quasi-Monte Carlo (RQMC) uses randomized LD point sets to give improved convergence rates and enable practical error estimation. Specifically, if we again assume the integrand has bounded variation, then certain RQMC methods can achieve a RMSE of $\calO(n^{-3/2+\delta})$. For rank-1 lattices, randomization is typically done using random shifts modulo $1$. For digital nets, randomization is typically done using nested uniform scrambling (NUS) or the cheaper combination of linear matrix scrambling (LMS) with digital shifts / permutations \cite{MATOUSEK1998527,owen.variance_alternative_scrambles_digital_net,owen_halton,owen.gain_coefficients_scrambled_halton}.

Higher-order LD point sets were designed to yield faster convergence for integrands with additional smoothness. For integrands with square integrable mixed partial derivatives up to order $\alpha>1$, plugging higher-order digital nets into $\hmu$ \eqref{eq:mc_approx} yields a worst-case error rate of $\calO(n^{-\alpha+\delta})$ \cite{dick.walsh_spaces_HO_nets,dick.qmc_HO_convergence_MCQMC2008,dick.decay_walsh_coefficients_smooth_functions,goda.ho_qmc_recent_advances}. RQMC using higher-order digital nets with higher-order NUS or LMS has been shown to achieve an RMSE of order $\calO(n^{-\alpha-1/2+\delta})$ \cite{dick.higher_order_scrambled_digital_nets}. Higher-order digital nets with NUS and with LMS plus digital shifts are also shown in \Cref{fig:pointsets}.  

\begin{figure}[!ht]
    \centering
    \includegraphics[width=1\textwidth]{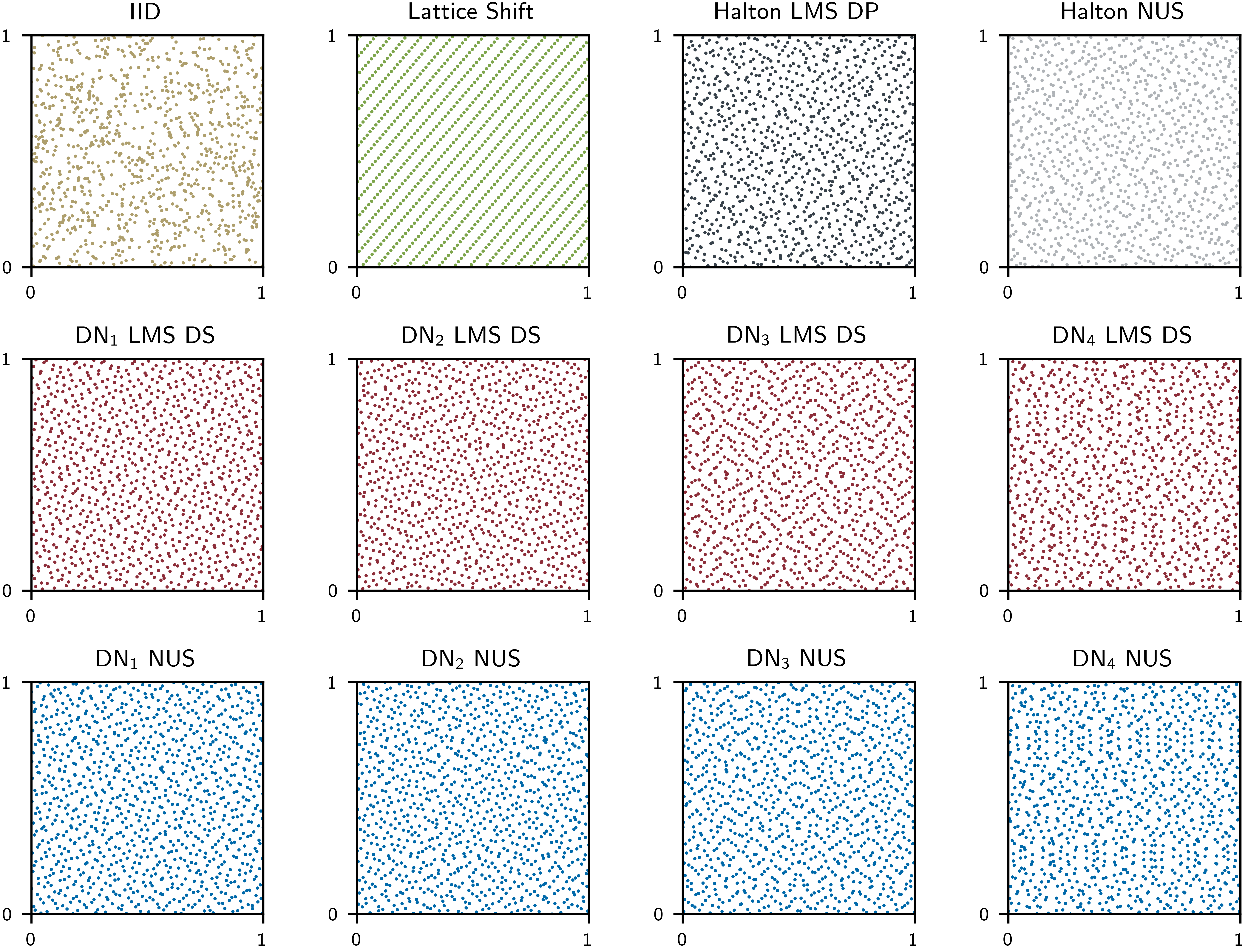}
    \caption{An independent identically distributed (IID) point set and low-discrepancy (LD) point sets of size $n=2^{10}=1024$. Notice the LD points more evenly fill the space than IID points, which leads to faster convergence of Quasi-Monte Carlo methods compared to IID Monte Carlo methods. Randomized rank-1 lattices, digital nets in base $2$, and Halton points are shown. Randomizations include shift modulo 1, linear matrix scramblings (LMS), digital shifts (DS), digital permutations (DP), and nested uniform scramblings (NUS). Digital interlacing of order $\alpha$ is used to create higher-order digital nets in base $2$ (DN${}_\alpha$).}
    \label{fig:pointsets}
\end{figure}

\Section{Randomized Low-Discrepancy Sequences} \label{sec:rldseqs_acm_toms}

This section details randomized low-discrepancy sequences formulated throughout the literature. Our novel contribution is to provide a unified, accessible, and efficient implementation of these sequences and their randomizations in \texttt{QMCPy}. We will detail rank-1 lattices (\Cref{sec:lattices}), digital nets (\Cref{sec:dnets}), and Halton sequences (\Cref{sec:Halton}) before running a series of numerical experiments (\Cref{sec:rldseqs_numerical_experiments}). 

Let us begin with some notation. For a fixed prime base $b \in \bbN$, write $i \in \bbN_0$ as $i = \sum_{t=0}^\infty \mi_t b^t$ so $\mi_t$ is the $t^\mathrm{th}$ digit in the base $b$ expansion of $i$. We denote the vector of the first $m$ base $b$ digits of $i$ by
$$\bD_m(i) = (\mi_0,\mi_1,\dots,\mi_{m-1})^\intercal.$$
For $\bmi = \bD_m(i)$, we denote the radical inverse of $i$ by
$$F_m(\bmi) = \sum_{t=1}^m \mi_{t-1} b^{-t} \in [0,1).$$
Finally, let
\begin{equation}
    v(i) = \sum_{t=1}^\infty \mi_{t-1} b^{-t}
    \label{eq:v}
\end{equation}
so that $v(i) = F_m(\bD_m(i))$ when $i < b^m$, which is always the case in this chapter. The van der Corput sequence in base $b$ is $\{v(i)\}_{i \geq 0}$.

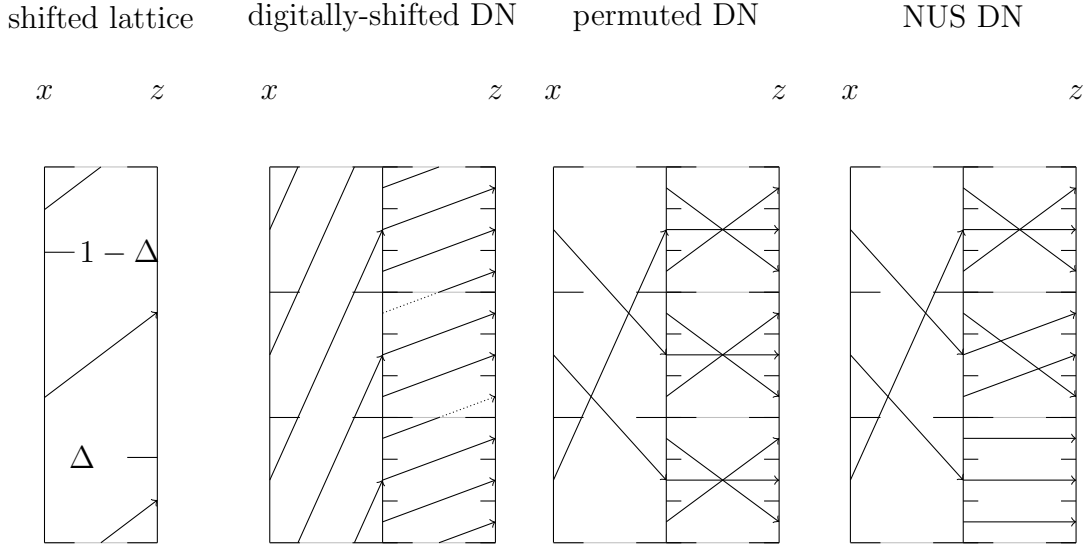
\begin{figure}[!ht]
    \centering
    \newcommand{\scale}{5}
    \newcommand{\axsep}{1.5}
    \newcommand{\thisDelta}{0.227}
    \begin{subfigure}[t]{.2\textwidth}
    \begin{tikzpicture}
        \draw (0,\scale+1) node{$x$};
        \draw (\axsep,\scale+1) node{$z$};
        \draw[-,color=lightgray] (0,0) -- (\axsep,0);
        \draw[-,color=lightgray] (0,\scale) -- (\axsep,\scale);
        \draw[-] (0,0) -- (0,\scale);
        \draw[-] (\axsep,0) -- (\axsep,\scale); 
        \draw[-] (0,0) -- (.4,0);
        \draw[-] (0,\scale) -- (.4,\scale);
        \draw[-] (\axsep-.4,0) -- (\axsep,0);
        \draw[-] (\axsep-.4,\scale) -- (\axsep,\scale);
        \draw (\axsep/2,\scale+2) node{shifted lattice};
        \draw (1,\scale-\scale*\thisDelta) node{$1-\Delta$};
        \draw (\axsep-1,\scale*\thisDelta) node{$\Delta$};
        \draw[-] (0,\scale-\scale*\thisDelta) -- (.4,\scale-\scale*\thisDelta);
        \draw[-] (\axsep-.4,\scale*\thisDelta) -- (\axsep,\scale*\thisDelta);
        \draw[->] (0,\scale/2-\scale*\thisDelta/2) -- (\axsep,\scale-\scale/2+\scale*\thisDelta/2);
        \draw[-] (0,\scale-\scale*\thisDelta/2)-- (\axsep/2,\scale);
        \draw[->] (\axsep/2,0)-- (\axsep,\scale*\thisDelta/2);
    \end{tikzpicture}
    \end{subfigure}
    \begin{subfigure}[b]{.25\textwidth}
    \begin{tikzpicture}
        \draw (0,\scale+1) node{$x$};
        \draw (2*\axsep,\scale+1) node{$z$};
        \draw[-] (0,0) -- (0,\scale);
        \draw[-] (\axsep,0) -- (\axsep,\scale); 
        \draw[-] (2*\axsep,0) -- (2*\axsep,\scale); 
        \draw[-,color=lightgray] (0,0) -- (\axsep,0);
        \draw[-,color=lightgray] (0,\scale) -- (\axsep,\scale);
        \foreach \x in {0,...,3}{
            \draw[-,color=lightgray] (\axsep,\scale*\x/3) -- (2*\axsep,\scale*\x/3);
            \draw[-] (0,\scale*\x/3) -- (.4,\scale*\x/3);
            \draw[-] (\axsep-.4,\scale*\x/3) -- (\axsep+.4,\scale*\x/3);
            \draw[-] (2*\axsep-.4,\scale*\x/3) -- (2*\axsep,\scale*\x/3);
        }
        \foreach \x in {0,...,9}{
            \draw[-] (\axsep,\scale*\x/9) -- (\axsep+.2,\scale*\x/9);
            \draw[-] (2*\axsep-.2,\scale*\x/9) -- (2*\axsep,\scale*\x/9);
        }
        \draw (\axsep,\scale+2) node{digitally-shifted DN};
        \draw[->] (0,\scale*1.5/9)-- (\axsep,\scale*7.5/9);
        \draw[-] (0,\scale*4.5/9)-- (\axsep*3/4,\scale);
        \draw[->] (\axsep*3/4,0)-- (\axsep,\scale*1.5/9);
        \draw[-] (0,\scale*7.5/9)-- (\axsep/4,\scale);
        \draw[->] (\axsep/4,0)-- (\axsep,\scale*4.5/9);
        \draw[->] (\axsep,\scale/18) -- (2*\axsep,\scale*3/18);
        \draw[->] (\axsep,\scale*3/18) -- (2*\axsep,\scale*5/18);
        \draw[-] (\axsep,\scale*5/18) -- (\axsep+\axsep/2,\scale/3);
        \draw[->] (\axsep+\axsep/2,0) -- (2*\axsep,\scale/18);
        \draw[->] (\axsep,\scale/3+\scale/18) -- (2*\axsep,\scale/3+\scale*3/18);
        \draw[->] (\axsep,\scale/3+\scale*3/18) -- (2*\axsep,\scale/3+\scale*5/18);
        \draw[-,densely dotted] (\axsep,\scale/3+\scale*5/18) -- (\axsep+\axsep/2,\scale/3+\scale/3);
        \draw[->,densely dotted] (\axsep+\axsep/2,\scale/3) -- (2*\axsep,\scale/3+\scale/18);
        \draw[->] (\axsep,\scale*2/3+\scale/18) -- (2*\axsep,\scale*2/3+\scale*3/18);
        \draw[->] (\axsep,\scale*2/3+\scale*3/18) -- (2*\axsep,\scale*2/3+\scale*5/18);
        \draw[-] (\axsep,\scale*2/3+\scale*5/18) -- (\axsep+\axsep/2,\scale*2/3+\scale/3);
        \draw[->] (\axsep+\axsep/2,\scale*2/3) -- (2*\axsep,\scale*2/3+\scale/18);
    \end{tikzpicture}
    \end{subfigure}
    \begin{subfigure}[b]{.25\textwidth}
    \begin{tikzpicture}
        \draw (0,\scale+1) node{$x$};
        \draw (2*\axsep,\scale+1) node{$z$};
        \draw[-] (0,0) -- (0,\scale);
        \draw[-] (\axsep,0) -- (\axsep,\scale); 
        \draw[-] (2*\axsep,0) -- (2*\axsep,\scale);
        \draw[-,color=lightgray] (0,0) -- (\axsep,0);
        \draw[-,color=lightgray] (0,\scale) -- (\axsep,\scale); 
        \foreach \x in {0,...,3}{
            \draw[-,color=lightgray] (\axsep,\scale*\x/3) -- (2*\axsep,\scale*\x/3);
            \draw[-] (0,\scale*\x/3) -- (.4,\scale*\x/3);
            \draw[-] (\axsep-.4,\scale*\x/3) -- (\axsep+.4,\scale*\x/3);
            \draw[-] (2*\axsep-.4,\scale*\x/3) -- (2*\axsep,\scale*\x/3);
        }
        \foreach \x in {0,...,9}{
            \draw[-] (\axsep,\scale*\x/9) -- (\axsep+.2,\scale*\x/9);
            \draw[-] (2*\axsep-.2,\scale*\x/9) -- (2*\axsep,\scale*\x/9);
        }
        \draw (\axsep,\scale+2) node{permuted DN};
        \draw[->] (0,\scale*1/6) -- (\axsep,\scale*5/6);
        \draw[->] (0,\scale*3/6) -- (\axsep,\scale*1/6);
        \draw[->] (0,\scale*5/6) -- (\axsep,\scale*3/6);
        \draw[->] (\axsep,\scale*1/18) -- (2*\axsep,\scale*5/18);
        \draw[->] (\axsep,\scale*3/18) -- (2*\axsep,\scale*3/18);
        \draw[->] (\axsep,\scale*5/18) -- (2*\axsep,\scale*1/18);
        \draw[->] (\axsep,\scale*7/18) -- (2*\axsep,\scale*11/18);
        \draw[->] (\axsep,\scale*9/18) -- (2*\axsep,\scale*9/18);
        \draw[->] (\axsep,\scale*11/18) -- (2*\axsep,\scale*7/18);
        \draw[->] (\axsep,\scale*13/18) -- (2*\axsep,\scale*17/18);
        \draw[->] (\axsep,\scale*15/18) -- (2*\axsep,\scale*15/18);
        \draw[->] (\axsep,\scale*17/18) -- (2*\axsep,\scale*13/18);
    \end{tikzpicture}
    \end{subfigure}
    \begin{subfigure}[b]{.25\textwidth}
    \begin{tikzpicture}
        \draw (0,\scale+1) node{$x$};
        \draw (2*\axsep,\scale+1) node{$z$};
        \draw[-] (0,0) -- (0,\scale);
        \draw[-] (\axsep,0) -- (\axsep,\scale); 
        \draw[-] (2*\axsep,0) -- (2*\axsep,\scale); 
        \draw[-,color=lightgray] (0,0) -- (\axsep,0);
        \draw[-,color=lightgray] (0,0+\scale) -- (\axsep,0+\scale);
        \foreach \x in {0,...,3}{
            \draw[-,color=lightgray] (\axsep,\scale*\x/3) -- (2*\axsep,\scale*\x/3);
            \draw[-] (0,\scale*\x/3) -- (.4,\scale*\x/3);
            \draw[-] (\axsep-.4,\scale*\x/3) -- (\axsep+.4,\scale*\x/3);
            \draw[-] (2*\axsep-.4,\scale*\x/3) -- (2*\axsep,\scale*\x/3);
        }
        \foreach \x in {0,...,9}{
            \draw[-] (\axsep,\scale*\x/9) -- (\axsep+.2,\scale*\x/9);
            \draw[-] (2*\axsep-.2,\scale*\x/9) -- (2*\axsep,\scale*\x/9);
        }
        \draw (\axsep,\scale+2) node{NUS DN};
        \draw[->] (0,\scale*1/6) -- (\axsep,\scale*5/6);
        \draw[->] (0,\scale*3/6) -- (\axsep,\scale*1/6);
        \draw[->] (0,\scale*5/6) -- (\axsep,\scale*3/6);
        \draw[->] (\axsep,\scale*1/18) -- (2*\axsep,\scale*1/18);
        \draw[->] (\axsep,\scale*3/18) -- (2*\axsep,\scale*3/18);
        \draw[->] (\axsep,\scale*5/18) -- (2*\axsep,\scale*5/18);
        \draw[->] (\axsep,\scale*7/18) -- (2*\axsep,\scale*9/18);
        \draw[->] (\axsep,\scale*9/18) -- (2*\axsep,\scale*11/18);
        \draw[->] (\axsep,\scale*11/18) -- (2*\axsep,\scale*7/18);
        \draw[->] (\axsep,\scale*13/18) -- (2*\axsep,\scale*17/18);
        \draw[->] (\axsep,\scale*15/18) -- (2*\axsep,\scale*15/18);
        \draw[->] (\axsep,\scale*17/18) -- (2*\axsep,\scale*13/18);
    \end{tikzpicture}
    \end{subfigure}
    \caption{Low-discrepancy randomization routines in dimension $d=1$. Each vertical line is a unit interval with $0$ at the bottom and $1$ at the top. A given interval is partitioned at the horizontal ticks extending to the right, and then rearranged following the arrows to create the partition of the right interval as shown by the ticks extending to the left. For the shifted rank-1 lattice and digitally-shifted digital net, when the arrow hits a horizontal gray bar it is wrapped around to the next gray bar below. See for example the dotted line in the digitally-shifted digital net. The lattice shift is $\Delta = \thisDelta$. All digital nets (DNs) use base $b=3$ and $t_\mathrm{max}=2$ digits of precision. The digital shift is $\bDelta = (2,1)^\intercal$. Dropping the $j$ subscript for dimension, the digital permutations in the third panel are $\pi_1 = (2,0,1)$ and $\pi_2 = (2,1,0)$. Notice $\pi_1$ is equivalent to a digital shift by $2$, but $\pi_2$ cannot be written as a digital shift. The nested uniform scramble (NUS) has digital permutations $\pi = (2,0,1)$, $\pi_0 = (2,1,0)$, $\pi_1 = (0,1,2)$, and $\pi_2 = (1,2,0)$. Notice the permuted digital net in the third panel has permutations depending only on the current digit in the base $b$ expansion. In contrast, the full NUS scrambling in the fourth panel has permutations which depend on all previous digits in the base $b$ expansion.}
    \label{fig:ld_randomizations}
\end{figure}

\Subsection{Rank-1 Lattices} \label{sec:lattices}

Consider a fixed \emph{generating vector} $\bg \in \bbN^d$ and fixed prime base $b$. Then we define the \emph{lattice sequence} 
\begin{equation}
    \bz_i = v(i) \bg \bmod 1, \qquad i \geq 0.
    \label{eq:unrandomized_lattice}
\end{equation}
If $n=b^m$ for some $m \in \bbN_0$, then the lattice point set $\{\bz_i\}_{i=0}^{n-1}\in [0,1)^{n \times d}$ is equivalent to $\{\bg i/n \bmod 1\}_{i=0}^{n-1}$ where we say the former is in \emph{radical inverse order} while the latter is in \emph{linear order}. When $n$ is not of the form $b^m$, then linear and radical inverse order will not generate the same lattice point set. 

\begin{enumerate}
    \item \textbf{Shifted lattice.} For a shift $\bDelta \in [0,1)^d$, we define the shifted point set 
    $$\bx_i= (\bz_i + \bDelta) \bmod 1$$ 
\end{enumerate}
where $\bz_i$ denotes the unrandomized lattice point in \eqref{eq:unrandomized_lattice}. The shift operation modulo 1 is visualized in \Cref{fig:ld_randomizations}. Randomized lattices use $\bDelta \sim \calU[0,1]^d$. In the following code snippet we generate $R$ independently shifted lattices with shifts $\bDelta_1,\dots,\bDelta_R \simiid \calU[0,1]^d$. 

\begin{lstlisting}[style=Python]
lattice = qp.Lattice(
    dimension = 52,
    randomize = "shift", # for unrandomized set to None
    replications = 16, # R
    order = "radical inverse", # also supports "linear"
    seed = None, # pass integer seed for reproducibility
    generating_vector = "mps.exod2_base2_m20_CKN.txt")
x = lattice(2**16) # numpy.ndarray with shape (16, 65536, 52)
\end{lstlisting}

\noindent Here we have used a generating vector from \cite{cools2006constructing} which is stored in a standardized format in the \texttt{LDData} repository. Other generating vectors from \texttt{LDData} may be used by passing in a file name from \url{https://github.com/QMCSoftware/LDData/tree/main/lattice/} or by passing an explicit array. 

\Subsection{Digital Nets} \label{sec:dnets}

Consider \emph{generating matrices} $\mC_1,\dots\mC_d \in \{0,\dots,b-1\}^{t_\mathrm{max} \times m}$ where $t_\mathrm{max},m \in \bbN$ are fixed and $b$ is a given prime base. One may relax the assumption that $b \in \bbN$ is prime, but we do not consider that here or in the \texttt{QMCPy} implementation. The first $n=b^m$ points of a digital sequence form a \emph{digital net} $\{\bz_i\}_{i=0}^{b^m-1} \in [0,1)^{b^m \times d}$ where, in \emph{radical inverse order}, for $1 \leq j \leq d$ and $0 \leq i < b^m$ with base $b$ digit vector $\bmi = \bD_m(i)$, 
$$z_{ij} = F_{t_\mathrm{max}}(\bmz_{ij}) \qquad\text{with}\qquad \bmz_{ij} = \mC_j \bmi \bmod b.$$

\begin{enumerate}
    \item \textbf{Digitally-shifted digital net.} Similar to lattices, one may apply a shift $\mDelta \in [0,1)^{t_\mathrm{max} \times d}$ to the digital net to get a digitally-shifted digital net $\{\bx_i\}_{i=0}^{b^m-1}$ where 
    $$x_{ij} = F_{t_\mathrm{max}}((\bmz_{ij} + \bDelta_j ) \bmod b)$$
    and $\bDelta_j$ is the $j^\mathrm{th}$ column of $\mDelta$. Randomly shifted digital nets use $\mDelta \simiid \calU\{0,\dots,b-1\}^{t_\mathrm{max} \times d}$, i.e., each digit is chosen uniformly from $\{0,\dots,b-1\}$. 
    \item \textbf{Digitally permuted digital net.} In what follows we will denote permutations of $\{0,\dots,b-1\}$ by $\pi$. Suppose we are given a set of permutations
    $$\mPi = \{\pi_{j,t}: \quad 1 \leq j \leq d, \quad 0 \leq t < t_\mathrm{max}\}.$$
    Then we may construct the digitally permuted digital net $\{\bx_i\}_{i=0}^{b^m-1}$ where 
    $$x_{ij} = F_{t_\mathrm{max}}((\pi_{j,0}(\mz_{ij0}),\pi_{j,1}(\mz_{ij1}),\dots,\pi_{j,t_\mathrm{max}-1}(\mz_{ij(t_\mathrm{max}-1)}))^\intercal).$$
    Randomly permuted digital nets use independent permutations chosen uniformly over all permutations of $\{0,\dots,b-1\}$. 
    \item \textbf{Nested Uniform Scrambling (NUS).} In the context of digital nets, NUS is often called Owen scrambling for its conception in \cite{owen1995randomly}. As before, $\pi$ denotes permutations of $\{0,\dots,b-1\}$. Now suppose we are given a set of permutations 
    $$\calP = \{\pi_{j,\mv_1\cdots\mv_t}: \quad 1 \leq j \leq d, \quad 0 \leq t < t_\mathrm{max}, \quad \mv_k \in \{0,\dots,b-1\}, \quad 0 \leq k \leq t \}.$$
    Then a \emph{nested uniform scrambling} of a digital net is $\{\bx_i\}_{i=0}^{b^m-1}$ where
    $$x_{ij} = F_{t_\mathrm{max}}\left(\begin{pmatrix} \pi_{j,}(\mz_{ij0}) \\ \pi_{j,\mz_{ij0}}(\mz_{ij1}) \\ \pi_{j,\mz_{ij0}\mz_{ij1}}(\mz_{ij2}) \\ \vdots \\ \pi_{j,\mz_{ij0}\mz_{ij1}\cdots\mz_{ij(t_\mathrm{max}-2)}}(\mz_{ij(t_\mathrm{max}-1)}) \end{pmatrix} \right).$$ 
    As the number of elements in $\calP$ is  
    $$\lvert \calP \rvert = d(1+b+b^2+\dots+b^{t_\mathrm{max}-1}) = d \frac{b^{t_\mathrm{max}}-1}{b-1},$$ 
    our implementation cannot afford to generate all permutations a priori. Instead, permutations are generated and stored only as needed. As with digitally-permuted digital nets, NUS uses independent uniform random permutations.  
\end{enumerate}
The randomization routines described above are visualized in \Cref{fig:ld_randomizations}. 

\emph{Linear matrix scrambling} (LMS) is a computationally cheaper but less complete version of NUS which has proven sufficient for many practical problems. LMS uses scrambling matrices $\mS_1,\dots,\mS_d \in \{0,\dots,b-1\}^{t_\mathrm{max} \times t_\mathrm{max}}$ and sets the LMS generating matrices $\tmC_1,\dots,\tmC_d \in \{0,\dots,b-1\}^{t_\mathrm{max} \times m}$ so that 
$$\tmC_j = \mS_j \mC_j \bmod b$$
for $j=1,\dots,d$. Following \cite{owen.variance_alternative_scrambles_digital_net}, let us denote elements in $\{1,\dots,b-1\}$ by $h$ and elements in $\{0,\dots,b-1\}$ by $g$. Then common structures for $\mS_j$ include 
$$\begin{pmatrix} h_{11} & 0 & 0 & 0 & \dots \\ g_{21} & h_{22} & 0 & 0 & \dots \\ g_{31} & g_{32} & h_{33} & 0 & \dots \\ g_{41} & g_{42} & g_{43} & h_{44} & \dots \\ \vdots & \vdots & \vdots & \vdots & \ddots \end{pmatrix}, \begin{pmatrix} h_1 & 0 & 0 & 0 & \dots \\ g_2 & h_1 & 0 & 0 & \dots \\ g_3 & g_2 & h_1 & 0 & \dots \\ g_4 & g_3 & g_2 & h_1 & \dots \\ \vdots & \vdots & \vdots & \vdots & \ddots \end{pmatrix}, \begin{pmatrix} h_1 & 0 & 0 & 0 & \dots \\ h_1 & h_2 & 0 & 0 & \dots \\ h_1 & h_2 & h_3 & 0 & \dots \\ h_1 & h_2 & h_3 & h_4 & \dots \\ \vdots & \vdots & \vdots & \vdots & \ddots \end{pmatrix}$$
which corresponds to Matou\v{s}ek's linear scrambling \cite{MATOUSEK1998527}, Tezuka's $i$-binomial scrambling \cite{tezuka2002randomization}, and Owen's striped LMS \cite{owen.variance_alternative_scrambles_digital_net} (not to be confused with NUS which is often called Owen scrambling). Random LMS chooses $g$ and $h$ values all independently and uniformly from $\{1,\dots,b-1\}$ and $\{0,\dots,b-1\}$ respectively. \cite{owen.variance_alternative_scrambles_digital_net} analyzes these scramblings and their connection to NUS.
 
\emph{Digital interlacing} enables the construction of higher-order digital nets. For integer order $\alpha \geq 1$, interlacing $\mA_1,\mA_2,\dots \in \{0,\dots,b-1\}^{t_\mathrm{max} \times m}$ gives $\widehat{\mA}_1,\widehat{\mA}_2,\dots \in \{0,\dots,b-1\}^{\alpha t_\mathrm{max} \times m}$ satisfying $\widehat{\mA}_{jtk} = A_{\widehat{j},\widehat{t},k}$ where $\widehat{j} = \alpha (j-1) + (t \bmod \alpha) + 1$ and $\widehat{t} = \lfloor t / \alpha \rfloor$ for $j \geq 1$ and $0 \leq t < \alpha t_\mathrm{max}$ and $1 \leq k \leq m$. For example, with $m=2$, $t_\mathrm{max}=2$, and $\alpha=2$ we have
$$\mA_1 = \begin{pmatrix} a_{101} & a_{102} \\ a_{111} & a_{112} \end{pmatrix} \;\;\; \mA_2 = \begin{pmatrix} a_{201} & a_{202} \\ a_{211} & a_{212} \end{pmatrix} \;\;\; \mA_3 = \begin{pmatrix} a_{301} & a_{302} \\ a_{311} & a_{312} \end{pmatrix} \;\;\; \mA_4 = \begin{pmatrix} a_{401} & a_{402} \\ a_{411} & a_{412} \end{pmatrix}$$
$$\widehat{\mA}_1 = \begin{pmatrix} a_{101} & a_{102} \\ a_{201} & a_{202} \\ a_{111} & a_{112} \\ a_{211} & a_{212} \end{pmatrix} \qquad \widehat{\mA}_2 = \begin{pmatrix} a_{301} & a_{302} \\ a_{401} & a_{402} \\ a_{311} & a_{312} \\ a_{411} & a_{412}\end{pmatrix}.$$
Without scrambling, higher-order digital nets may be directly generated from $\widehat{\mC}_1$, $\dots$, $\widehat{\mC}_d$, the interlaced generating matrices resulting from interlacing $\mC_1,\dots,\mC_{\alpha d}$. Higher-order NUS \cite{dick.higher_order_scrambled_digital_nets} requires generating a digital net from $\mC_1,\dots,\mC_{\alpha d}$, applying NUS to the resulting $\alpha d$-dimensional digital net, and then interlacing the resulting digit matrices $\{\bmz_{i,1}^\intercal\}_{i=0}^{b^m-1},\dots,\{\bmz_{i,\alpha d}^\intercal\}_{i=0}^{b^m-1} \in \{0,\dots,b-1\}^{t_\mathrm{max} \times m}$. For higher-order LMS, one applies LMS to the generating matrices $\mC_1,\dots,\mC_{\alpha d}$ to get $\tmC_1,\dots,\tmC_{\alpha d}$, then interlaces $\tmC_1,\dots,\tmC_{\alpha d}$ to get $\widehat{\tmC}_1,\dots,\widehat{\tmC}_d$, then generates the digital net from $\widehat{\tmC}_1,\dots,\widehat{\tmC}_d$. As we show in the numerical experiments in \Cref{sec:rldseqs_numerical_experiments}, LMS is significantly faster than NUS (especially for higher-order nets) while still achieving higher-order rates of RMSE convergence. 

A subtle difference between the above presentation and practical implementation is that $t_\mathrm{max}$ may change with randomization in practice. For example, suppose we are given generating matrices $\mC_1,\dots,\mC_d \in \{0,\dots,b-1\}^{32 \times32}$ but would like the shifted digital net to have $64$ digits of precision. Then we should generate $\bDelta \in \{0,\dots,b-1\}^{t_\mathrm{max} \times d}$ with $t_\mathrm{max}=64$ and treat $\mC_j$ as $t_\mathrm{max} \times t_\mathrm{max}$ matrices with appropriate rows and columns zeroed out. 

Gray code ordering of digital nets enables computing the next point $\bx_{i+1}$ from $\bx_i$ by only adding a single column from each generating matrix. Specifically, the $q^\mathrm{th}$ column of each generating matrix gets digitally added to the previous point where $q-1$ is the index of the only digit to be changed in Gray code ordering. Gray code orderings for $b=2$ and $b=3$ are shown in \Cref{tab:graycode}. 

\begin{table}[!ht]
    \caption{Gray code order for bases $b=2$ and $b=3$. In Gray code order only one digit is incremented or decremented by $1$ (modulo $b$) when $i$ is incremented by $1$.}
    \centering
    \begin{tabular}{r | r l | r l }
        $i$ & $i_2$ & Gray code  $i_2$ & $i_3$ & Gray code $i_3$ \\
        \hline 
        $0$ & $0000_2$ & $0000_2$ & $00_3$ & $00_3$ \\
        $1$ & $0001_2$ & $0001_2$ & $01_3$ & $01_3$ \\ 
        $2$ & $0010_2$ & $0011_2$ & $02_3$ & $02_3$ \\ 
        $3$ & $0011_2$ & $0010_2$ & $10_3$ & $12_3$ \\ 
        $4$ & $0100_2$ & $0110_2$ & $11_3$ & $11_3$ \\ 
        $5$ & $0101_2$ & $0111_2$ & $12_3$ & $10_3$ \\ 
        $6$ & $0110_2$ & $0101_2$ & $20_3$ & $20_3$ \\ 
        $7$ & $0111_2$ & $0100_2$ & $21_3$ & $21_3$ \\ 
        $8$ & $1000_2$ & $1000_2$ & $22_3$ & $22_3$ \\
        \hline 
    \end{tabular}
    \label{tab:graycode} 
\end{table}

The following code generates $\alpha=2$ higher-order digital nets in base $2$ with  $R$ independent LMS and digital shift combinations. 

\begin{lstlisting}[style=Python]
dnb2 = qp.DigitalNetB2(
    dimension = 52, 
    randomize = "LMS DS", # Matousek's LMS & a digital shift
    # other options ["NUS", "DS", "LMS", None]
    t = 64, # number of LMS bits, i.e., rows in S_j
    alpha = 2, # interlacing factor for higher-order nets
    replications = 16, # R
    order = "radical inverse", # also supports "Gray code"
    seed = None, # pass integer seed for reproducibility
    generating_matrices = "joe_kuo.6.21201.txt")
x = dnb2(2**16) # numpy.ndarray with shape (16, 65536, 52)
\end{lstlisting}

\noindent Here we have used a set of Sobol' generating matrices from Joe and Kuo\footnote{The ``new-joe-kuo-6.21201'' direction numbers from \\ \url{https://web.maths.unsw.edu.au/~fkuo/sobol/index.html}} \cite{joe2008constructing} which are stored in a standardized format in the \texttt{LDData} repository. Other generating matrices from \texttt{LDData} may be used by passing in a file name from \url{https://github.com/QMCSoftware/LDData/blob/main/dnet/} or by passing in an explicit array. 

\Subsection{Halton Sequences} \label{sec:Halton}

The digital sequences described in the previous section used a fixed prime base $b$. One may allow each dimension $j \in \{1,\dots,d\}$ to have its own prime base $b_j$. The most popular of such constructions is the Halton sequence which sets $b_j$ to the $j^\mathrm{th}$ prime, sets $\mC_j$ to the identity matrix, and sets $t_\mathrm{max} = m$. This enables the simplified construction of Halton points $\{\bx_i\}_{i=0}^{n-1}$ via
$$\bx_i = (v_{b_1}(i),\dots,v_{b_d}(i))^\intercal$$
where we have added a subscript to $v$ in \eqref{eq:v} to denote the base dependence. 

Almost all the methods described for digital sequences are immediately applicable to Halton sequences after accounting for the differing bases across dimensions. However, digital interlacing is not generally applicable when the bases differ. Halton with random starting points has also been explored in \cite{wang2000randomized}, although we do not consider this here. The following code generates a Halton point set with $R$ independent LMS and digital permutation combinations. Specifically, we generate $R$ independent LMS-Halton point sets and then apply an independent permutation scramble to each.

\begin{lstlisting}[style=Python]
halton = qp.Halton(
    dimension = 52, 
    randomize = "LMS PERM", # Matousek's LMS & a digital perm
    # or ["LMS DS", "LMS", "PERM", "DS", "NUS", "QRNG", None]
    t = 64, # number of LMS digits, i.e., rows in S_j
    replications = 16, # R
    seed = None) # pass integer seed for reproducibility
x = halton(2**10) # numpy.ndarray with shape (16, 1024, 52)
\end{lstlisting}

\noindent The \texttt{"QRNG"} randomization  follows the \texttt{QRNG} software package \cite{qrng.software} in generating a generalized Halton point set \cite{faure2009generalized} using an optimized,  deterministic set of permutation scrambles and random digital shifts. 

\Subsection{Numerical Experiments} \label{sec:rldseqs_numerical_experiments}

These numerical experiments were carried out on the CPUs of a 2023 MacBook Pro with an M3 Max processor. \Cref{fig:timing} compares the wall-clock time required to generate point sets and perform fast transforms for high dimensions and/or a high number of randomizations. \Cref{sec:fast_kernel_methods} will detail the fast transforms including the FFT in bit-reversed order (FFTBR), its inverse (IFFTBR), and the fast Walsh--Hadamard transform (FWHT); see also \Cref{fig:fast_transforms}. Following the theory, our implementation scales linearly in the number of dimensions and randomizations. For large numbers of randomizations, our vectorized \texttt{QMCPy} generators significantly outperform the looped implementations in \texttt{SciPy} and \texttt{PyTorch}. For fast transforms, our FFTBR and IFFTBR implementations are the same speed as the underlying FFT and IFFT algorithms in \texttt{SciPy}. Our FWHT implementation is significantly faster than the \texttt{SymPy} version, especially when applying the FWHT to multiple sequences simultaneously. IID points, randomly shifted lattices, and digital nets with LMS and digital shifts (including higher-order versions) are the fastest sequences to generate. Digital nets in base $b=2$ exploit Gray code order, integer storage of bit-vectors, and exclusive or (XOR) operations to perform digital addition. Halton point sets are slower to generate as they cannot use these exploits. NUS, especially higher-order versions, are significantly slower to generate than LMS randomizations. Moreover, higher-order LMS with digital shifts is sufficient to achieve higher-order RMSE convergence as we show in the next experiment.  

\begin{figure}[!ht]
    \centering
    \includegraphics[width=1\textwidth]{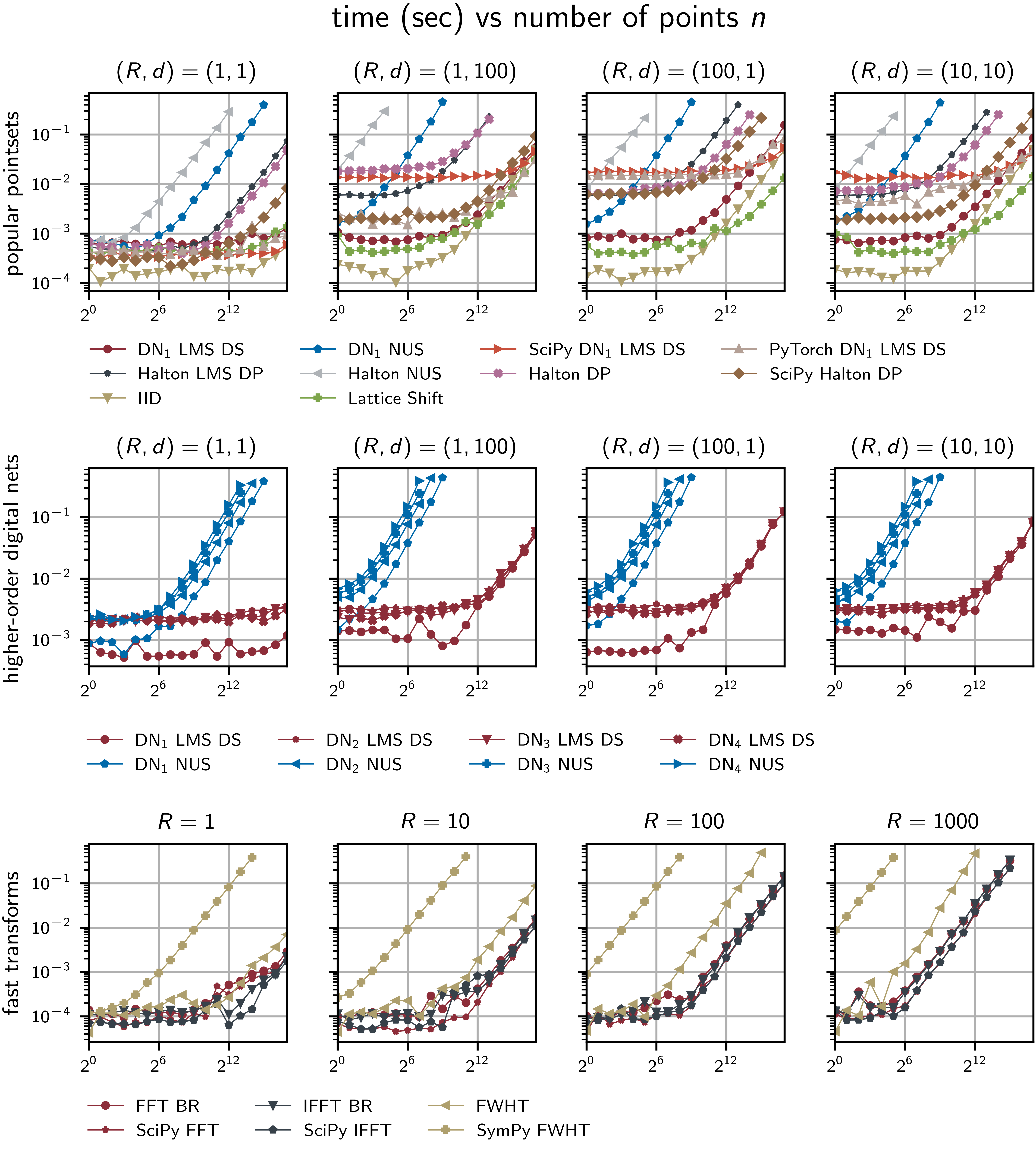}
    \caption{Comparison of time required to generate point sets and perform fast transforms. For IID and randomized LD generators, we vary the number of randomizations $R$, the number of dimensions $d$, and the sequence size $n$. Timings include initialization, and a new point set is generating for each plotted $(R,n,d)$. Fast transforms are applied to $R$ sequences of size $n$ in a vectorized fashion.}
    % \Description[]{}
    \label{fig:timing}
\end{figure}

\Cref{fig:convergence} plots the RMSE convergence of RQMC methods applied to the integrands described below.
\begin{enumerate}
    \item \textbf{Simple function, $d=1$,} has $f(x) = x e^x-1$. This was used in \cite{dick.higher_order_scrambled_digital_nets} where higher-order digital net scrambling was first proposed.
    \item \textbf{Simple function, $d=2$,} has $f(\bx) = x_2 e^{x_1 x_2}/(e-2)-1$. This was also considered in \cite{dick.higher_order_scrambled_digital_nets}. 
    \item \textbf{Oakley \& O'Hagan, $d=2$,} has $f(\bx) = g((\bx-1/2)/50)$ for $g(\bt) = 5+t_1+t_2+2\cos(t_1)+2\cos(t_2)$, see \cite{oakley2002bayesian} or the VLSE\footnote{\label{url:VLSE}\url{https://www.sfu.ca/~ssurjano/uq.html}} (Virtual Library of Simulation Experiments).
    \item \textbf{G-Function, $d=3$,} has $f(\bx) = \prod_{j=1}^d \frac{\lvert 4x_j-2\rvert-a_j}{1+a_j}$ with $a_j = (j-2)/2$ for $1 \leq j \leq d$, see  \cite{crestaux2007polynomial,marrel2009calculations} or the VLSE. 
    \item \textbf{Oscillatory Genz, $d=3$,} has $f(\bx) = \cos\left(-\sum_{j=1}^d c_j x_j \right)$ with coefficients of the third kind $c_j = 4.5 \tc_j/\sum_{j=1}^d \tc_j$ where $\tc_j = \exp\left(j \log\left(10^{-8}\right)/d\right)$. This is a common test function for uncertainty quantification which is available in the \texttt{Dakota}\footnote{\label{url:dakota}\url{https://dakota.sandia.gov/}} software \cite{adams2020dakota} among others.  
    \item \textbf{Corner-peak Genz, $d=3$,} has $f(\bx) = \left(1+\sum_{j=1}^d c_j x_j\right)^{-(d+1)}$ with coefficients of the second kind $c_j = 0.25 \tc_j/\sum_{j=1}^d \tc_j$ where $\tc_j = 1/j^2$. This is also available in \texttt{Dakota}. 
\end{enumerate}
For each problem, the RMSE of the (Q)MC estimator $\hmu$ in \eqref{eq:mc_approx} is approximated using $300$ independent randomizations of an IID or randomized LD point sets from \texttt{QMCPy}. IID points consistently achieve the theoretical $\calO(n^{1/2})$ convergence rate. For shifted rank-1 lattices, we periodized the integrand using a baker transform which does not change the mean, i.e., we use $\tilde{f}(\bx) = f(1-2\lvert \bx-1/2\rvert)$ in place of $f(\bx)$ in \eqref{eq:mc_approx}. Shifted lattices consistently attained RMSEs like $\mathcal{O}(n^{-1})$, except for on the $d=3$ G-Function where higher-order convergence was unexpectedly observed. Our base $2$ digital nets with LMS and DS, including higher-order versions with higher-order scrambling, consistently achieved the lowest RMSEs and best rates of convergence. For the $d=1$ integrand, we are able to realize the theoretical rate of $\calO(n^{-\min\{\alpha,\talpha\}-1/2+\delta})$ where $\alpha$ is the higher-order digital interlacing of the net, $\talpha$ is the smoothness of the integrand, and $\delta>0$ is arbitrarily small. For the $d=2$ and $d=3$-dimensional integrands we were also able to achieve higher-order convergence, but with varying rates of success.

\begin{figure}[!ht]
    \centering
    \includegraphics[width=1\textwidth]{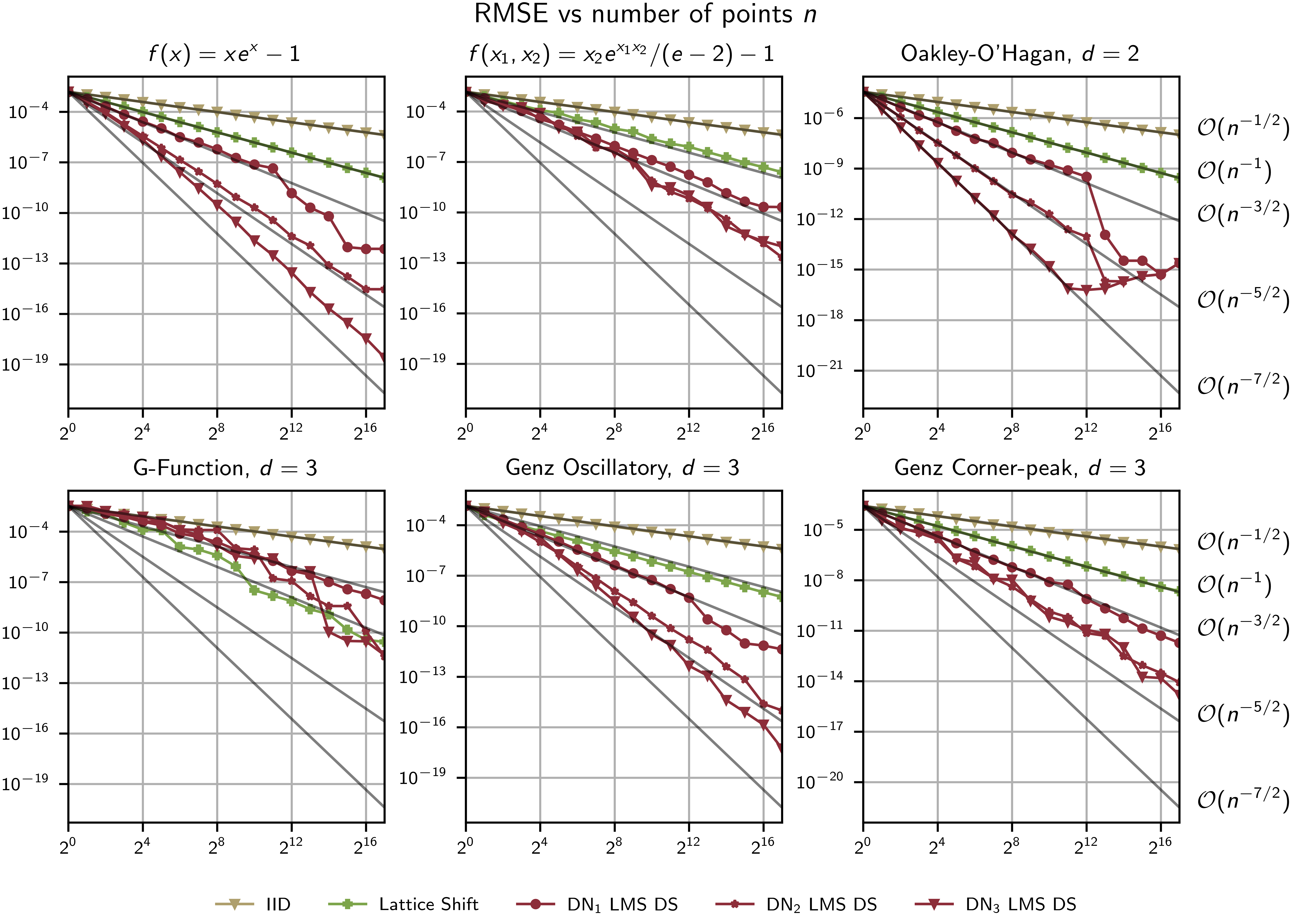}
    \caption{The root mean squared error (RMSE) of the randomized QMC (RQMC) estimate for a few different integrands. Higher-order digital nets with linear matrix scrambling and digital shifts achieve higher-order convergence and significantly outperform both IID points and randomly shifted lattices.}
    % \Description[]{}
    \label{fig:convergence}
\end{figure}

\Section{Variable Transformations} \label{sec:variable_transforms}

Here we will review well known results from probability which enable us to rewrite our problem into a QMC compatible form.  Often the QMC problem in \eqref{eq:mc_approx} may be written as 
\begin{equation}
    \mu = \bbE[g(\bT)] = \bbE[f(\bX)], \qquad \bX \sim \calU[0,1]^d
    \label{eq:mc_approx_g_T}
\end{equation}
where $\bT$ is a $d$-dimensional random variable with joint density $\lambda$. Since QMC nodes have low-discrepancy with the standard uniform distribution $\bX \sim \calU[0,1]^d$, we perform an invertible variable transformation $\bPsi: [0,1]^d \to \calM$ satisfying $\mathrm{range}(\bT) \subseteq \calM$. This implies
$\mu = \bbE[f(\bX)]$ where  
$$
f(\bX) = g(\bPsi(\bX)) \frac{\lambda (\bPsi(\bX))}{\nu(\bPsi(\bX))},
$$
and $\nu$ is the density of $\bPsi(\bX)$. Specifically, $\nu(\bPsi(\bX))= \lvert \bJ_\bPsi(\bX) \rvert^{-1}$ where $\lvert \bJ_\bPsi(\bX) \rvert$ is the determinant of the Jacobian of $\bPsi$. The above formulation is equivalent to importance sampling by a distribution with density $\nu$. Optimally, one would choose $\nu \propto g\lambda$ so that $f$ is constant, but this is typically infeasible in practice due to our limited knowledge of $g$. Choosing $\bPsi$ so that $\bT \sim \bPsi(\bX)$ gives the simplification $f(\bX) = g(\bPsi(\bX))$. The next few subsections review some measures and transforms available in \texttt{QMCPy}. 

\Subsection{Uniform Measure} 

For $\bT \sim \calU[\bl,\bu]$ with lower and upper bounds $\bl,\bu \in \bbR^d$ satisfying $\bl \leq \bu$ elementwise we set 
$$\bPsi(\bx) = \bl + (\bu - \bl) \odot \bx$$
with $\odot$ denoting the Hadamard (elementwise) product. 

\Subsection{Gaussian Measure} 

For a Gaussian $\bT \sim \calN(\bM,\mK=\mL \mL^\intercal)$ with mean $\bM \in \bbR^d$ and covariance $\mK \in \bbR^{d \times d}$ we set 
$$\bPsi(\boldsymbol{x}) = \bM+\mL\Phi^{-1}(\bx)$$
where $\Phi^{-1}$ is the inverse CDF of a standard Gaussian distribution applied elementwise. QMC is more effective if the variation of $f$ is concentrated in the lower indexed coordinates.  For finance examples where the outcome depends more on the gross behavior of the underlying stochastic process, rather than the high frequency behavior, it makes sense to choose the transformation accordingly. For example, in \Cref{fig:asian_option_pda_vs_chol} we see that if  $\bT$ contains observations from a Brownian motion, then an eigendecomposition of $\mK = \mL \mL^\intercal$ shows greater convergence gains compared to the Cholesky decomposition \cite{AcwBroGla97}.

\begin{figure}[!ht]
	\centering
    \includegraphics[width=1\textwidth]{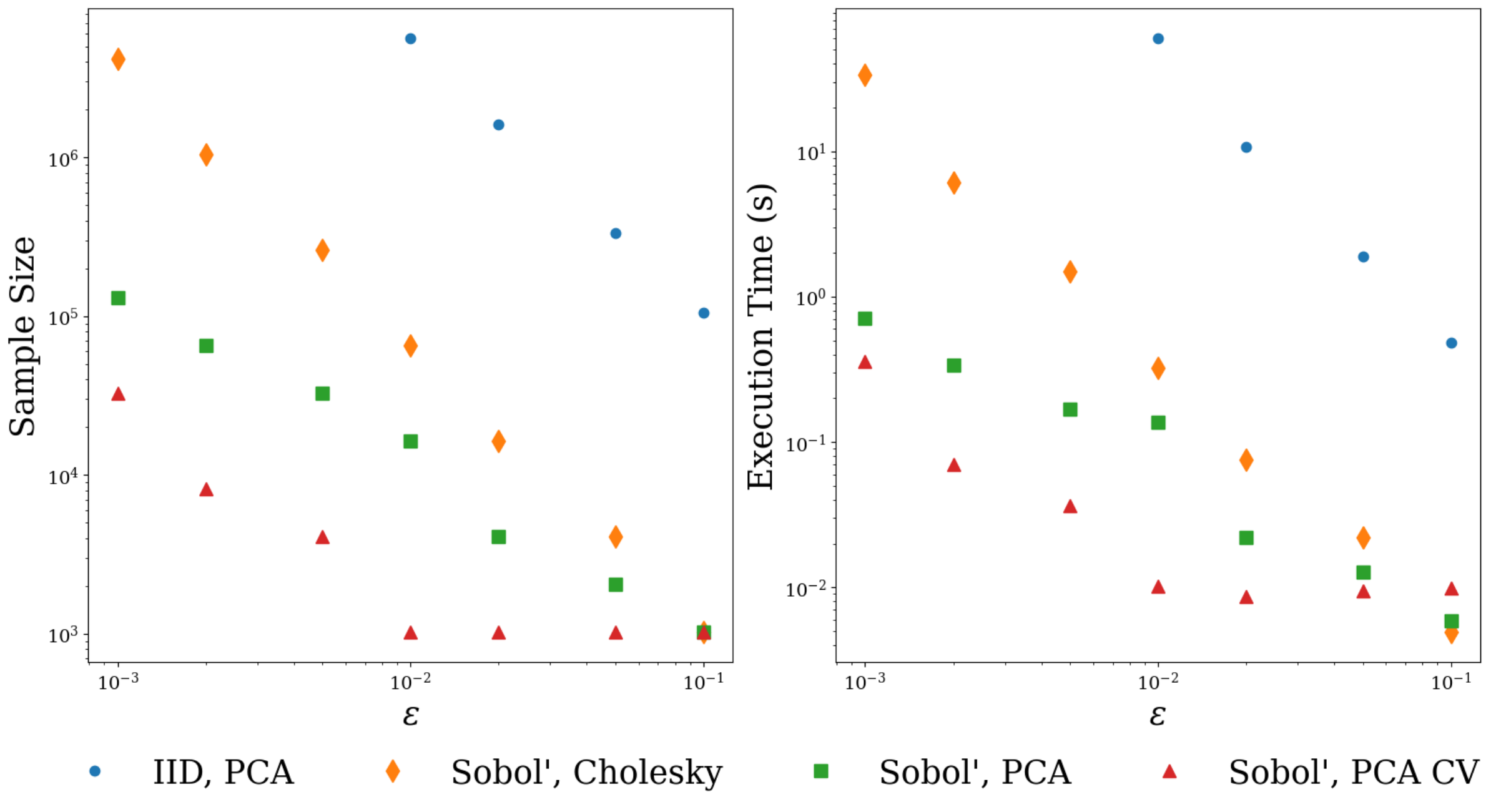}
	\caption{Comparison of IID Monte Carlo and QMC stopping criteria for estimating the fair price of an Asian option with arithmetic mean to within an error tolerance $\varepsilon$. For QMC with Sobol' points, using a principal component analysis (PCA) decomposition of the Brownian motion provides substantial savings compared to the Cholesky decomposition. Using an Asian option with a geometric mean control variate (CV) provides additional savings. \label{fig:asian_option_pda_vs_chol}}
\end{figure}

\Subsection{Brownian Motion} \label{sec:brownain_motion} 

A Brownian motion with an initial value $B_0$, drift $\gamma$, and diffusion $\sigma^2$ observed at times $0 \leq \tau_1 < \tau_2 < \dots < \tau_d$ is a multivariate Gaussian,
$$(B(\tau_1),\dots,B(\tau_d))^\intercal \sim \calN(\bM,\mK),$$
with respective mean and covariance
$$\bM = B_0 + \gamma \btau \qqtqq{and} \mK = \sigma^2 \left(\min(\tau_j,\tau_{j'})\right)_{j,j'=1}^{d}.$$

\Subsection{Other Measures} 

\texttt{QMCPy} supports a number of other measures, including, but not limited to:
\begin{enumerate}
    \item \textbf{\texttt{SciPy}'s stats distributions,} including the ability to specify a multivariate measure with independent marginals. 
    \item \textbf{Geometric Brownian motion,} which is commonly used for financial modeling. 
    \item \textbf{Lebesgue measures,} which requires importance sampling.
\end{enumerate}

\Section{Problems} \label{sec:qmc_problems} 

Each problem must specify both the true measure $\bT$ and integrand $g$ as written in \eqref{eq:mc_approx_g_T}. This section describes a few of the problems supported in \texttt{QMCPy}. \Cref{sec:financia_options} comprises the majority of this section which will discuss the well-studied financial option pricing problems, and then \Cref{sec:other_problems} will briefly mention support for other user-specified problems and frameworks. 

\Subsection{Financial Options} \label{sec:financia_options}

Define the geometric Brownian motion with start price $S_0$, drift $\gamma$, and volatility $\sigma$ observed at times $0 < \tau_1 < \tau_2 < \dots < \tau_d$ to be 
$$\bS(\bT) = S_0 \exp((\gamma-\sigma^2/2)\btau+\sigma \bT)$$
where $\bT \sim \calN(\bzero,\mK)$ is a standard Brownian motion so $\mK = (\min(\tau_j,\tau_{j'}))_{j,j'=1}^d$, i.e., a Brownian motion as defined in \Cref{sec:brownain_motion} with initial value $B_0=0$, drift $\gamma=0$, and diffusion $\sigma=1$. 

For a financial option exercised at the final monitoring time, the discounted payoff with interest rate $r >0$ is
$$g(\bT) = P(\bS(\bT))e^{-r \tau_d}.$$
Various payoff functions $P$ for different options are defined below, often with respect to some strike price $K$ and/or the value of the path at exercise time which we will denote by $S_{-1} := S_d$. 
\begin{enumerate}
    \item \textbf{European options,} where the call and put options have respective payoffs 
    $$P(\bS) = \max\{S_{-1}-K,0\} \qqtqq{and} P(\bS) = \max\{K-S_{-1},0\}.$$
    \item \textbf{Asian options,} whose value is the average of an asset path across time. We use the trapezoidal rule to approximate either the arithmetic mean by
    $$A(\bS) = \frac{1}{d}\left[\frac{1}{2} S_0 + \sum_{j=1}^{d-1} S_j + \frac{1}{2} S_{-1}\right]$$
    or the geometric mean by
    $$A(\bS) = \left[\sqrt{S_0} \prod_{j=1}^{d-1} S_j \sqrt{S_{-1}}\right]^{1/d}.$$
    Asian call and put option then have respective payoffs
    $$P(\bS) = \max\{A(\bS)-K,0\} \qqtqq{and} P(\bS) = \max\{K-A(\bS),0\}.$$
    \Cref{fig:asian_option_pda_vs_chol} gave an example of Asian option pricing using the automatic stopping criterion discussed in the next section.
    \item \textbf{Barrier options,} are either ``in'' options, which are activated when the path crosses the barrier $B$, or ``out'' options, which are activated only if the path never crosses the barrier $B$. An ``up'' Barrier option satisfies $S_0<B$ while a ``down'' option satisfies $S_0>B$, both indicating the direction of the barrier from the start price. Barrier ``up-in'' call and put option have respective payoffs
    \begin{align*}
        P(\bS) &= \begin{cases} \max\{S_{-1}-K,0\}, & \text{any } \bS \geq B \\ 0, & \mathrm{otherwise} \end{cases} \qquad\text{and} \\
        P(\bS) &= \begin{cases} \max\{K-S_{-1},0\}, & \text{any } \bS \geq B \\ 0, & \mathrm{otherwise} \end{cases}.
    \end{align*}
    Barrier ``up-out'' call and put options have respective payoffs
    \begin{align*}
        P(\bS) &= \begin{cases} \max\{S_{-1}-K,0\}, & \text{all } \bS < B \\ 0, & \mathrm{otherwise} \end{cases} \qquad\text{and} \\
        P(\bS) &= \begin{cases} \max\{K-S_{-1},0\}, & \text{all } \bS < B \\ 0, & \mathrm{otherwise} \end{cases}.
    \end{align*}
    
    Barrier ``down-in'' call and put option have respective payoffs
    \begin{align*}
        P(\bS) &= \begin{cases} \max\{S_{-1}-K,0\}, & \text{any } \bS \leq B \\ 0, & \mathrm{otherwise} \end{cases} \qquad\text{and} \\
        P(\bS) &= \begin{cases} \max\{K-S_{-1},0\}, & \text{any } \bS \leq B \\ 0, & \mathrm{otherwise} \end{cases}.
    \end{align*}
    
    Barrier ``down-out'' call and put option have respective payoffs
    \begin{align*}
        P(\bS) &= \begin{cases} \max\{S_{-1}-K,0\}, & \text{all } \bS > B \\ 0, & \mathrm{otherwise} \end{cases} \qquad\text{and} \\
        P(\bS) &= \begin{cases} \max\{K-S_{-1},0\}, & \text{all } \bS > B \\ 0, & \mathrm{otherwise} \end{cases}.
    \end{align*}
    \item \textbf{Lookback options,} where the call and put options have respective payoffs
    $$P(\bS) = S_{-1}-\min(\bS), \qquad P(\bS) = \max(\bS)-S_{-1}.$$
    \item \textbf{Digital options,} with payout $\rho$ have respective call and put payoffs
    $$P(\bS) = \begin{cases} \rho, & S_{-1} \geq K \\ 0, & \mathrm{otherwise} \end{cases}, \qquad P(\bS) =  \begin{cases} \rho, & S_{-1} \leq K \\ 0, & \mathrm{otherwise} \end{cases}.$$
\end{enumerate}
Multilevel versions of these options are also supported.

\Subsection{Other Problems} \label{sec:other_problems}

Below we describe some other problem specification methods in \texttt{QMCPy} 
\begin{enumerate}
    \item \textbf{Custom problems,} where the user supplies the integrand $g$ and true measure $\bT$ explicitly. 
    \item \textbf{Problems from \texttt{UM-Bridge}\footnote{\url{https://um-bridge-benchmarks.readthedocs.io/en/docs/}} \cite{umbridge.software},} the ``UQ and Model Bridge'' software, which uses containerization to make numerical simulations accessible across any programming language. \texttt{UM-Bridge} also supports HPC load balancing and numerous scheduler backends. 
\end{enumerate}

\Section{Stopping Criteria} \label{sec:qmc_stopping_crit} 

This section discusses existing IID-MC and QMC error approximation methods. In \texttt{QMCPy}, we have used these error approximations to automatically determine the smallest number of points required for an approximation to meet user-specified error tolerances. \Cref{SoRa_table:qmcpy_sc} and \Cref{SoRa_table:qmcpy_sc_cont} compare features of the currently available adaptive stopping criteria in \texttt{QMCPy}, which will be discussed in the following subsections. Afterwards, \Cref{sec:stop_crit_vectorziation} will discuss our novel contribution to support bounding functions of multiple expectations. 

Here we will detail error approximation methods which determine bounds $[\mu^-,\mu^+]$ on a scalar mean $\mu$ in \eqref{eq:mc_approx} which hold with uncertainty less than some threshold $\alpha \in (0,1)$ (not to be confused with $\alpha$ the smoothness parameter). Specifically, given sampling nodes $\{\boldsymbol{x}_i\}_{i=0}^{n-1}$ and corresponding function evaluations $\{f(\boldsymbol{x}_i)\}_{i=0}^{n-1}$, we discuss methods for determining bounds $-\infty \leq \mu^- \leq \mu^+ \leq \infty$ so that $\mu \in [\mu^-,\mu^+]$ with probability greater than or equal to $1-\alpha$. 

\begin{table}[!ht]
    \caption{A comparison of adaptive stopping criteria classes in the \texttt{QMCPy} library. Additional details are given in \Cref{SoRa_table:qmcpy_sc_cont}. \emph{Point sets} indicate classes of compatible sequences in \texttt{QMCPy}. For example, \texttt{CubQMCRepStudentT} is compatible with any low-discrepancy (\texttt{LD}) sequence including base 2 digital nets (\texttt{DigitalNetB2}) and integration lattices (\texttt{Lattice}). However, \texttt{CubQMCNetG} is only compatible with \texttt{DigitalNetB2} sequences and will not work with \texttt{Lattice} or other LD sequences. \emph{Vectorized} algorithms are capable of simultaneously approximating multiple expectations or even functions of multiple expectations, see \Cref{sec:stop_crit_vectorziation}. The \texttt{GAIL} MATLAB library \cite{GAIL.software} also implements the single-level algorithms.}
    \centering
    \begin{tabular}{r c c c c }
        \texttt{QMCPy} Class & Point sets & Guaranteed & Vectorized & Multilevel\\
        \hline
        \texttt{CubMCCLT} & \texttt{IID} & & & \\
        \texttt{CubMCCLTVec} & \texttt{IID} & & $\checkmark$ & \\
        \texttt{CubMCG} & \texttt{IID} & $\checkmark$ & & \\
        \texttt{CubQMCRepStudentT} & \texttt{LD} & & $\checkmark$ & \\
        \texttt{CubQMCNetG} & \texttt{DigitalNetB2} & $\checkmark$ & $\checkmark$ &  \\
        \texttt{CubQMCLatticeG} & \texttt{Lattice} & $\checkmark$ & $\checkmark$ &  \\
        \texttt{CubQMCBayesNetG} & \texttt{DigitalNetB2} & $\checkmark$ & $\checkmark$ & \\
        \texttt{CubQMCBayesLatticeG} & \texttt{Lattice} & $\checkmark$ & $\checkmark$ & \\
        \texttt{CubMLMC} & \texttt{IID} & & & $\checkmark$ \\
        \texttt{CubMLMCCont} & \texttt{IID} & & & $\checkmark$ \\
        \texttt{CubMLQMC} & \texttt{LD} & & & $\checkmark$ \\
        \texttt{CubMLQMCCont} & \texttt{LD} & & & $\checkmark$ \\
        \hline 
    \end{tabular}
    \label{SoRa_table:qmcpy_sc}
\end{table}

\begin{table}[!ht]
    \caption{A comparison of adaptive stopping criteria classes in the \texttt{QMCPy} library. This is a continuation of \Cref{SoRa_table:qmcpy_sc}. \emph{Bounds} specify the method of error estimation as discussed in the corresponding subsections. Deterministic bounds hold with probability $1$, i.e., for any $\alpha \in (0,1)$. Probabilistic and Bayesian bounds are tailored to the choice of uncertainty $\alpha$. \emph{Replications} indicate the number of randomizations one must apply to an LD point set; this is not applicable (NA) for IID points.}
    \centering
    \begin{tabular}{r c c c c }
        \texttt{QMCPy} Class & Bounds & Replications & Section & References \\
        \hline
        \texttt{CubMCCLT} & Probabilistic & NA& \Cref{sec:stop_crit_iid} & \cite{hickernell.MC_guaranteed_CI} \\
        \texttt{CubMCCLTVec} & Probabilistic & NA & \Cref{sec:stop_crit_iid} & \cite{hickernell.MC_guaranteed_CI} \\
        \texttt{CubMCG} & Probabilistic & NA & \Cref{sec:stop_crit_iid} & \cite{hickernell.MC_guaranteed_CI} \\
        \texttt{CubQMCRepStudentT} & Probabilistic & $>1$ & \Cref{sec:stop_crit_qmc_rep_student_t} & \cite{lecuyer.RQMC_CLT_bootstrap_comparison,owen.mc_book_practical} \\
        \texttt{CubQMCNetG} & Deterministic & $1$ & \Cref{sec:stop_crit_qmc_decay_tracking} & \cite{hickernell.adaptive_dn_cubature} \\
        \texttt{CubQMCLatticeG} & Deterministic & $1$ & \Cref{sec:stop_crit_qmc_decay_tracking} & \cite{cubqmclattice} \\
        \texttt{CubQMCBayesNetG} & Bayesian & $1$ & \Cref{sec:stop_crit_qmc_fast_bayes} & \cite{rathinavel.bayesian_QMC_sobol} \\
        \texttt{CubQMCBayesLatticeG} & Bayesian & $1$ & \Cref{sec:stop_crit_qmc_fast_bayes} & \cite{rathinavel.bayesian_QMC_lattice} \\ 
        \texttt{CubMLMC} & Probabilistic & NA & \Cref{sec:stop_crit_multilevel} & \cite{giles.MLMC_path_simulation,giles2015multilevel} \\
        \texttt{CubMLMCCont} & Probabilistic & NA & \Cref{sec:stop_crit_multilevel} & \cite{collier2015continuation} \\
        \texttt{CubMLQMC} & Probabilistic & $>1$ & \Cref{sec:stop_crit_multilevel} & \cite{giles.mlqmc_path_simulation} \\
        \texttt{CubMLQMCCont} & Probabilistic & $>1$ & \Cref{sec:stop_crit_multilevel} & \cite{robbe2019multilevel} \\
        \hline
    \end{tabular}
    \label{SoRa_table:qmcpy_sc_cont}
\end{table}

\Subsection{IID Monte Carlo} \label{sec:stop_crit_iid}

When $\{\boldsymbol{x}_i\}_{i=0}^{n-1}$ are IID and $f$ has a finite variance, the central limit theorem (CLT) may provide a heuristic $1-\alpha$ confidence interval for $\mu$ by setting $\mu^\pm = \hat{\mu} \pm Z_{\alpha/2}\sigma/\sqrt{n}$. Here $Z_{\alpha/2}$ is the inverse CDF of a standard normal distribution at $1-\alpha/2$, and $\hat{\mu}$ is the sample average of function evaluations as in \eqref{eq:mc_approx}. The variance 
\begin{equation}
    \sigma^2 = \bbE[(f(\bX)-\mu)^2], \qquad \bX \sim \calU[0,1]^d
    \label{eq:var}
\end{equation}
may be approximated by the unbiased estimator
\begin{equation}
    \hat{\sigma}^2 = \frac{1}{n-1}\sum_{i=0}^{n-1} (f(\boldsymbol{x}_i)-\hat{\mu})^2,
    \label{eq:sample_standard_deviation_unbiased}
\end{equation}
perhaps multiplied by an inflation factor $C^2>1$ for a more conservative estimate. The resulting  heuristic bounds on $\mu$ are 
\begin{equation}
    \mu^\pm = \hat{\mu} \pm CZ_{\alpha/2} \hat{\sigma} / \sqrt{n}
    \label{eq:mu_pm_clt}
\end{equation}

These quantities are used in the following stopping criteria in different ways  
\begin{enumerate}
    \item \texttt{CubMCCLTVec}, iteratively doubles the samples size until the bounds in \eqref{eq:mu_pm_clt} are sufficiently tight. \Cref{sec:stop_crit_vectorziation} details how this simple algorithm may be used for vectorized approximation and bounding functions of multiple expectations.
    \item \texttt{CubMCCLT}, is a two-step method which first takes a pilot sample of size $n_0$ in order to compute an initial sample mean $\hat{\mu}_0$ and unbiased sample standard $\hat{\sigma}$ in \eqref{eq:sample_standard_deviation_unbiased}. Then, for absolute and relative error tolerances $\varepsilon_\mathrm{abs},\varepsilon_\mathrm{rel} > 0$, one may set the combined tolerance to $\varepsilon = \max\{\varepsilon_\mathrm{abs},\varepsilon_\mathrm{rel} \hat{\mu}_0\}$. A final sample of $n$ new IID points is chosen to satisfy $CZ_{\alpha/2} \hat{\sigma} / \sqrt{n} < \varepsilon$, so we use the smallest acceptable sample size $n = \lceil (C Z_{\alpha/2} \hat{\sigma} / \varepsilon)^2 \rceil$.
    \item \texttt{CubMCG} \cite{hickernell.MC_guaranteed_CI,jiang2016guaranteed}, extends \texttt{CubMCCLT} to accommodate finite $n$ and provide bounds that are guaranteed to satisfy the uncertainty threshold. Specifically, their algorithm holds for the cone of functions with modified kurtosis $\bbE[(f(\bX)-\mu)^4]/\sigma^4$ less than some parameter $\tkappa_\mathrm{max}$. As before, $\bX \sim \calU[0,1]^d$, the true mean is $\mu$ from \eqref{eq:mc_approx}, and $\sigma^2$ is the variance in \eqref{eq:var}. This two-step method relies on the Berry--Esseen and Chebyshev inequalities to compute conservative fixed width confidence intervals for functions in the cone. Note that the resulting bounds from this algorithm are different from the CLT bounds in \eqref{eq:mu_pm_clt}. \cite{hickernell.MC_guaranteed_CI} originally proposed the \texttt{CubMCG} algorithm for absolute error tolerances while \cite{jiang2016guaranteed} provides additional details and an extension to also accommodate relative error tolerances.  
\end{enumerate}
The two-stage sampling techniques of \texttt{CubMCCLT} and \texttt{CubMCG} make them incompatible with the vectorized approximation extension discussed in \Cref{sec:stop_crit_vectorziation}.

\Subsection{QMC with Replications and Student Confidence Intervals} \label{sec:stop_crit_qmc_rep_student_t}

This method, called \texttt{CubQMCRepStudentT} utilizes IID randomizations of an LD sequence and then derives bounds based on the IID sample averages. Specifically, suppose $\{\bx_{1i}\}_{i=0}^{n-1}$, $\dots$, $\{\bx_{Ri}\}_{i=0}^{n-1}$ denote $R$ IID randomizations of an LD point set where typically $R$ is small, e.g., $R=15$. Then, following  \cite{lecuyer.RQMC_CLT_bootstrap_comparison} and \cite[Chapter 17]{owen.mc_book_practical}, the RQMC estimate
\begin{equation}
    \hmu = \frac{1}{R} \sum_{r=1}^R \hat{\mu}_r \qquad \text{where} \quad \hat{\mu}_r = \frac{1}{n} \sum_{i=0}^{n-1} f(\bx_{ri})
    \label{eq:RQMC}
\end{equation}
gives a practical $100(1-\alpha)\%$ confidence interval
\begin{equation}
    \mu^\pm = \hmu \pm C t_{R-1,\alpha/2} \hsigma/\sqrt{R}
    \label{eq:rqmc_student_bounds}
\end{equation}
 for $\mu$. Here
\begin{equation}
    \hsigma^2 = \frac{1}{R-1} \sum_{r=1}^R (\hmu_r - \hmu)^2,
    \label{eq:sample_var_replicated_qmc}
\end{equation}
$t_{R-1,\alpha/2}$ is the $\alpha/2$ quantile of a Student's-$t$ distribution with $R-1$ degrees of freedom, and $C>1$ is again an inflation factor for the standard deviation. An adaptive version of this algorithm iteratively doubles the samples size $n$ until the bounds in \eqref{eq:rqmc_student_bounds} are sufficiently tight. \Cref{sec:stop_crit_vectorziation} details how this algorithm may be used for vectorized approximation and bounding functions of multiple expectations.

In the following code, we use \texttt{QMCPy} to build an RQMC estimate to the mean of the corner-peak Genz function in $d=50$ dimensions defined as 
$$f(\bx) = \left(1+\sum_{j=1}^d c_j x_j\right)^{-(d+1)}$$
with coefficients of the second kind $c_j = 0.25 \tc_j/\sum_{j=1}^d \tc_j$ where $\tc_j = 1/j^2$. Using $R$ replications and a fixed number of points $n$, we construct a Student's-$t$ confidence interval as in \eqref{eq:RQMC} and the discussion thereafter. \texttt{SciPy} is used to compute the necessary quantile. 

\begin{lstlisting}[style=Python]
import scipy.stats 
def gen_corner_peak_2(x):
    d = x.shape[-1] # x.shape=(...,d), e.g., (n,d) or (R,n,d) 
    c_tilde = 1/np.arange(1,d+1)**2
    c = 0.25*c_tilde/np.sum(c_tilde)
    y = (1+np.sum(c*x,axis=-1))**(-(d+1)) 
    return y # y.shape=(...,n), e.g., (n,) or (R,n)
R = 10 # number of randomizations
n = 2**15 # number of points 
d = 50 # dimension
dnb2 = qp.DigitalNet(
    dimension = d,
    replications = R,
    seed = 7,
    alpha = 3)
x = dnb2(n) # x.shape=(R,n,d)
y = gen_corner_peak_2(x) # y.shape=(R,n) 
muhats_r = np.mean(y,axis=1) # muhats_r.shape=(R,)
muhat = np.mean(muhats_r) # muhat is a scalar 
print(muhat)
""" 0.014936813948394042 """
alpha = 0.01 # uncertainty level
t_star = -scipy.stats.t.ppf(alpha/2,df=R-1) # quantile
sigmahat = np.std(muhats_r,ddof=1) # unbiased estimate
std_error = t_star*sigmahat/np.sqrt(R)
print(std_error)
""" 5.247445301861484e-07 """
conf_int = [muhat-std_error,muhat+std_error]
\end{lstlisting}

The next snippet shows how to use an adaptive version of this algorithm which automatically selects the number of points $n$ required to meet a user-specified absolute error tolerance.

\begin{lstlisting}[style=Python]
R = 10
d = 50 
dnb2 = qp.DigitalNet(
    dimension = d,
    replications = R,
    seed = 7,
    alpha = 3)
true_measure = qp.Uniform(dnb2, lower_bound=0, upper_bound=1)
integrand = qp.CustomFun(
    true_measure = true_measure,
    g = gen_corner_peak_2)
# equivalent to 
# integrand = qp.Genz(
#     dnb2,
#     kind_func = "CORNER PEAK",
#     kind_coeff = 2)
qmc_algo = qp.CubQMCRepStudentT(integrand, abs_tol=1e-4)
solution,data = qmc_algo.integrate() # run adaptive QMC 
print(solution)
""" 0.014950908095474802 """
conf_int = [data.comb_bound_low,data.comb_bound_high]
std_error = (conf_int[1]-conf_int[0])/2
print(std_error)
""" 2.7968149935497788e-05 """
print(data)
"""
Data (Data)
    solution        0.015
    comb_bound_low  0.015
    comb_bound_high 0.015
    comb_bound_diff 5.59e-05
    n               10240
    n_rep           2^(10)
    time_integrate  0.019
CubQMCRepStudentT (AbstractStoppingCriterion)
    alpha           0.010
    abs_tol         1.00e-04
    rel_tol         0
    n_init          2^(8)
DigitalNetB2 (AbstractLDDiscreteDistribution)
    d               50
    replications    10
    randomize       LMS DS
    gen_mats_source joe_kuo.6.21201.txt
    order           RADICAL INVERSE
    t               63
    alpha           3
    entropy         7
"""
\end{lstlisting}

\Subsection{Fast Decay Tracking QMC Without Replications} \label{sec:stop_crit_qmc_decay_tracking} 

The \texttt{CubQMCLatticeG} and \texttt{CubQMCNetG} algorithms track the decay of basis coefficients using only a single randomized LD sequence, see \cite{adaptive_qmc} for a unified treatment. These algorithms provide deterministic bounds $\mu^\pm$ on $\mu$ for functions in a cone parameterized by the decay rate of the complex exponential Fourier coefficients for integration lattices \cite{cubqmclattice} or the Walsh coefficients for digital sequences \cite{hickernell.adaptive_dn_cubature}. Again, an adaptive version of this algorithm iteratively doubles the samples size until the bounds are sufficiently tight. \Cref{sec:stop_crit_vectorziation} details how this algorithm may be used for vectorized approximation and bounding functions of multiple expectations. \Cref{fig:qmc_convergence} and \Cref{fig:asian_option_pda_vs_chol} show results when applying this algorithm to the Keister integrand and an Asian option pricing example respectively.  

\Subsection{Fast Bayesian QMC Without Replications} \label{sec:stop_crit_qmc_fast_bayes} 

The algorithms \texttt{CubQMCBayesLatticeG} and \texttt{CubQMCBayesNetG} take a Bayesian approach to error estimation in assuming the integrand is a realization of a Gaussian process. Since the integral of a Gaussian process is still Gaussian, one may compute posterior credible interval bounds on the true mean. As with \texttt{CubQMCLatticeG} and \texttt{CubQMCNetG}, only a single randomized LD sequence is required, and this algorithm may be made adaptive by doubling the sample size until the bounds are sufficiently tight. \Cref{sec:stop_crit_vectorziation} details how this algorithm may be used for vectorized approximation and bounding functions of multiple expectations.

Utilizing special kernels matched to LD sequences enables the Gaussian process to be fit at $\mathcal{O}(n \log n)$ cost, including hyperparameter estimation. \cite{rathinavel.bayesian_QMC_thesis} gives a unified treatment of these algorithms for either lattices with shift-invariant kernels \cite{rathinavel.bayesian_QMC_lattice} or digital nets with digitally-shift-invariant kernels \cite{rathinavel.bayesian_QMC_sobol}. These special pairings of points and kernels will be further discussed in \Cref{sec:fast_kernel_methods}, and their applications to fast Gaussian process regression will be discussed in \Cref{sec:fast_gps}.

\Subsection{Multilevel (Q)MC} \label{sec:stop_crit_multilevel}

Here we will give a quick overview of the Multilevel Monte Carlo (MLMC) and Multilevel Quasi-Monte Carlo (MLQMC) algorithms implemented into \texttt{QMCPy} and point to a few extensions throughout the literature. \Cref{sec:FastBayesianMLQMC} will discuss our new algorithm which uses fast Bayesian cubature for MLQMC without replications for a fixed number of levels. The standard MLMC and MLQMC with replications algorithms will be fully detailed there. Extending our new fast Bayesian MLQMC algorithms to support adaptive multilevel schemes and developing an implementation into \texttt{QMCPy} are areas for future work.

MLMC and MLQMC exploit cheap low-fidelity approximations to $f$ to accelerate standard IID-MC and QMC methods. Specifically, suppose we have access to low-fidelity surrogates $f_\ell$ to the full fidelity-model $f_L := f$ where $1 \leq \ell \leq L$. The basis of ML(Q)MC methods is to expand the target expectation into the telescoping sum 
\begin{equation}
    \mu = \bbE[f_L(\bX)] = \sum_{\ell = 1}^L \bbE\left[f_\ell(\bX) - f_{\ell-1}(\bX)\right] = \sum_{\ell=1}^L \bbE[Y_\ell(\bX)]
    \label{eq:mlmc_telescoping_sum}
\end{equation} 
where $Y_\ell = f_\ell-f_{\ell-1}$, and we use the convention $f_0(\bX) = 0$. Each of the $\bbE[Y_\ell(\bX)]$ are approximated using independent IID-MC or randomized QMC estimators. Assume a single evaluation of $f_\ell$ costs $C_\ell$, and we are given a target budget or error tolerance. 

Multilevel IID-MC estimates to $\bbE[Y_\ell(\bX)]$ with $n_\ell$ samples on level $\ell$ will optimally set $n_\ell$ proportional to  $\sigma/\sqrt{C_\ell}$ where $\sigma_\ell^2 = \bbV[Y_\ell(\bX)]$ is the variance of the $\ell^\mathrm{th}$ difference. Approximation of this variance and the resulting algorithm are detailed in \cite{giles.MLMC_path_simulation,giles2015multilevel} and implemented into the \texttt{CubMLMC} class. \cite{collier2015continuation} algorithmically enhances \texttt{CubMLMC} to a continuation multilevel IID-MC algorithm called \texttt{CubMLMCCont}. This algorithm intelligently trades off bias and variance estimates for a decreasing sequence of error tolerances while exploiting structure in higher tolerance approximations to accelerate lower tolerance solutions. 

MLQMC methods typically use independent replicated QMC estimates on each level, e.g., $R=8$ independent randomizations of a low-discrepancy point set on each of the levels. See \Cref{sec:stop_crit_qmc_rep_student_t} for a discussion of the replicated QMC estimate for a single-level problem. For MLQMC, one iteratively doubles $n_\ell$, the number of points in the low-discrepancy sequence, on the level with the maximum error-to-work ratio $\hsigma_\ell^2/(R n_\ell C_\ell)$ where $\hsigma_\ell$ is the variance of the combined approximate on level $\ell$ analogous to \eqref{eq:sample_var_replicated_qmc}. The resulting \texttt{CubMLQMC} algorithm is detailed in \cite{giles.mlqmc_path_simulation}, with \cite{robbe2019multilevel} providing the continuation version \texttt{CubMLQMCCont}.  

Often we deal with a maximum level $L$ which is not fixed. For example, option pricing is an infinite-dimensional problem where the true expected value depends on the asset being monitored at an infinite number of time steps. The algorithms and references in this subsection have implemented schemes which adapt the number of levels $L$ to account for potential truncation errors. Multilevel algorithms are not compatible with the vectorized approximation techniques discussed in the next section. 

A number of ML(Q)MC extensions exist. One idea for infinite level ML(Q)MC problems is that randomly selecting the level at which to sample leads to unbiased estimates \cite{rhee2015unbiased}. For simulations which also evaluate all lower fidelity levels when queried, samples may be recycled using the unbiased random level idea in a scheme known as recycled ML(Q)MC \cite{kumar2017multigrid,robbe2019recycling}. Another extension is multi-index (Q)MC which considers problems for which multiple fidelity parameters exist \cite{robbe.multi_index_qmc}. For example, \cite{robbe2016dimension} considers a heat exchanger modeled by a numerical PDE solver with one fidelity parameter controlling the spatial discretization and another controlling the time step size. \cite{robbe2019multilevel} provides a unified view of ML(Q)MC methods and their extensions. 

\Section{Vectorized Stopping Criterion for Functions of Multiple Expectations} \label{sec:stop_crit_vectorziation}

\begin{quotation}
    This section summarizes \cite{sorokin.MC_vector_functions_integrals}, a paper I published with Jagadeeswaran Rathinavel. This is our novel contribution to the literature on adaptive (Q)MC algorithms.
\end{quotation}

This section will describe (single-level) adaptive IID-MC and QMC methods which propagate bounds on scalar expectations to bounds on functions of multiple expectations. Specifically, we describe extensions of the vectorized-compatible (Q)MC error bounding algorithms in \Cref{SoRa_table:qmcpy_sc} to support adaptive sampling for satisfying general error criteria on functions of a common array of expectations. Although several functions involving multiple expectations are being evaluated, only one random sequence is required, albeit sometimes of larger dimension than the underlying randomness.  In this section only, we will use bold notation to denote multidimensional arrays, i.e., tensors. 

For many problems, an array quantity of interest (QOI) $\boldsymbol{s} \in \mathbb{R}^{\boldsymbol{d}_{\boldsymbol{s}}}$ may be formulated as a function of an array mean $\boldsymbol{\mu} = \mathbb{E}[\boldsymbol{f}(\boldsymbol{X})] \in \mathbb{R}^{\boldsymbol{d}_{\boldsymbol{\mu}}}$ where $\bX \sim \calU[0,1]^d$ as before. Here $\boldsymbol{\mu}$ is a multidimensional array, which we simply call an array, with shape vector $\boldsymbol{d}_{\boldsymbol{\mu}}$, e.g., $\boldsymbol{\mu} \in \mathbb{R}^{(2,3)}$ indicates $\boldsymbol{\mu}$ is a $2 \times 3$ matrix. Similarly, we allow $\boldsymbol{s}$ to be an array with shape $\boldsymbol{d}_{\boldsymbol{s}}$. The integrand is now $\boldsymbol{f}: [0,1]^{d} \to \mathbb{R}^{\boldsymbol{d}_{\boldsymbol{\mu}}}$. 

The QOI array is formulated from the mean array via a function $\boldsymbol{C}: \mathbb{R}^{\boldsymbol{d}_{\boldsymbol{\mu}}} \to \mathbb{R}^{\boldsymbol{d}_{\boldsymbol{s}}}$ so that $\boldsymbol{s} = \boldsymbol{C}(\boldsymbol{\mu})$.
Example QOI arrays include
\begin{itemize}
    \item an $a \times b$ mean matrix where $\boldsymbol{C}$ is the identity and $\boldsymbol{d}_{\boldsymbol{s}} = \boldsymbol{d}_{\boldsymbol{\mu}} = (a,b)$,
    \item a Bayesian posterior mean where $s = C(\mu_1,\mu_2) = \mu_1/\mu_2$, $d_s = 1$, and $d_{\boldsymbol{\mu}} = 2$, 
    \item $c$ closed and total sensitivity indices requiring $\boldsymbol{d}_{\boldsymbol{s}} = (2,c)$ and $\boldsymbol{d}_{\boldsymbol{\mu}} = (2,3,c)$ to formulate $s_{ij} = C_{ij}(\boldsymbol{\mu}) =  \mu_{i3j}/(\mu_{i2j}-\mu_{i1j}^2)$ for $i \in \{1,2\}$, $j \in \{1,\dots,c\}$.
\end{itemize}
These examples will be detailed in \Cref{sec:vsc_numerical_experiments_v}.

The main result of this section is the development of \Cref{SoRa_algo:MCStoppingCriterion}, a generalization of \cite{adaptive_qmc}, which
\begin{enumerate}
    \item produces bounds $[\boldsymbol{s}^-,\boldsymbol{s}^+]$ on QOI $\boldsymbol{s}$ which hold with elementwise uncertainty below a user-specified threshold array $\boldsymbol{\alpha}^{(\boldsymbol{s})} \in (0,1)^{\boldsymbol{d}_{\boldsymbol{s}}}$, 
    \item computes an optimal QOI approximation $\hat{\boldsymbol{s}}$ based on bounds $[\boldsymbol{s}^-,\boldsymbol{s}^+]$, a user-specified error metric, and user-specified error tolerance,
    \item repeats with increasing sample sizes until the stopping criterion is satisfied. 
\end{enumerate}
The algorithm utilizes existing (Q)MC methods that, given an appropriate set of sampling nodes and their corresponding function evaluations, produce bounds $[\boldsymbol{\mu}^-,\boldsymbol{\mu}^+]$ on the mean $\boldsymbol{\mu}$ that hold with elementwise uncertainty below a derived threshold array $\boldsymbol{\alpha}^{(\boldsymbol{\mu})} \in (0,1)^{\boldsymbol{d}_{\boldsymbol{\mu}}}$. Many such algorithms were given in \Cref{sec:qmc_stopping_crit}. A dependency mapping from $\boldsymbol{s}$ to $\boldsymbol{\mu}$ is used to derive $\boldsymbol{\alpha}^{(\boldsymbol{\mu})}$ from $\boldsymbol{\alpha}^{(\boldsymbol{s})}$. Interval arithmetic functions are used to propagate mean bounds $[\boldsymbol{\mu}^-,\boldsymbol{\mu}^+]$ to QOI bounds $[\boldsymbol{s}^-,\boldsymbol{s}^+]$. These interval arithmetic functions are derived from $\boldsymbol{C}$ and problem specific QOI restrictions. When existing approximations to QOI are sufficiently accurate, the dependency function may tell the algorithm that certain outputs of $\boldsymbol{f}$ are not necessary to evaluate, a principal we call \emph{economic evaluation}. Our implementations in \texttt{QMCPy} incorporate shared samples, parallel computation, and economic evaluation to efficiently find bounds and approximations satisfying flexible user specifications.

The remainder of this section is organized as follows. In \Cref{sec:vsc_comb_sol_approx} we consider the case where $\boldsymbol{\mu}$ is an array and $s$ is a scalar. Here we describe how to set array $\boldsymbol{\alpha}^{(\boldsymbol{\mu})}$ based on scalar $\alpha^{(s)}$ and how to propagate bounds $[\boldsymbol{\mu}^-,\boldsymbol{\mu}^+]$ on $\boldsymbol{\mu}$ to bounds $[s^-,s^+]$ on $s$ so that both hold with uncertainty below $\alpha^{(s)}$. \Cref{sec:vsc_opt_comb_sol_sc} derives a stopping criterion for adaptive sampling and optimal approximation $\hat{s}$ of scalar QOI $s$. Both the approximation and stopping criterion are based on $[s^-,s^+]$, a user-specified error metric, and a user-specified error threshold. Considerations for extending to array QOI $\boldsymbol{s}$, including a generalized method for setting $\boldsymbol{\alpha}^{(\boldsymbol{\mu})}$ and a strategy for economic evaluation, are discussed in \Cref{sec:vsc_vectorized_implementation} before presenting the unifying \Cref{SoRa_algo:MCStoppingCriterion}. \Cref{sec:vsc_numerical_experiments_v} gives examples from machine learning and sensitivity analysis.   

\Subsection{Bounds on a scalar QOI} \label{sec:vsc_comb_sol_approx}

Here we discuss how to compute bounds $[s^-,s^+]$ on scalar QOI $s = C(\boldsymbol{\mu})$ so that 
\begin{equation}
    P(s \in [s^-,s^+]) \geq 1-\alpha^{(s)}
    \label{SoRa_eq:scalar_comb_desired_ineq}
\end{equation}
where $\alpha^{(s)} \in (0,1)$ is an uncertainty threshold on the QOI bounds. Here $C: \mathbb{R}^{\boldsymbol{d}_{\boldsymbol{\mu}}} \to \mathbb{R}$ combines the array mean $\boldsymbol{\mu}$ into a scalar QOI $s$. First, we discuss how to set the array of mean uncertainty thresholds $\boldsymbol{\alpha}^{(\boldsymbol{\mu})} \in (0,1)^{\boldsymbol{d}_{\boldsymbol{\mu}}}$ so the resulting mean bounds $[\boldsymbol{\mu}^-,\boldsymbol{\mu}^+]$ contain $\boldsymbol{\mu}$ with uncertainty below $\alpha^{(s)}$. Then we discuss how the user may utilize $C$ to define \emph{bound functions} $C^-,C^+: \mathbb{R}^{\boldsymbol{d}_{\boldsymbol{\mu}}} \times \mathbb{R}^{\boldsymbol{d}_{\boldsymbol{\mu}}} \to \mathbb{R}$ so that setting $s^- = C^-(\boldsymbol{\mu}^-,\boldsymbol{\mu}^+)$ and $s^+ = C^+(\boldsymbol{\mu}^-,\boldsymbol{\mu}^+)$ ensures \eqref{SoRa_eq:scalar_comb_desired_ineq} is satisfied.

Let $N = \lvert \boldsymbol{d}_{\boldsymbol{\mu}} \rvert$ be the number of elements in an $\mathbb{R}^{\boldsymbol{d}_{\boldsymbol{\mu}}}$ array and set each element of $\boldsymbol{\alpha}^{(\boldsymbol{\mu})}$ to the constant $\alpha^{(s)}/ N$. Then Boole's inequality \cite{boole1847mathematical} implies that if $[\boldsymbol{\mu}^-,\boldsymbol{\mu}^+]$ are chosen so that
\begin{equation}
    P(\mu_{\boldsymbol{k}} \in [\mu_{\boldsymbol{k}}^-,\mu_{\boldsymbol{k}}^+]) \geq 1-\alpha_{\boldsymbol{k}}^{(\boldsymbol{\mu})}  \qquad \text{for all }\boldsymbol{1} \leq \boldsymbol{k} \leq \boldsymbol{d}_{\boldsymbol{\mu}},
    \label{SoRa_eq:indv_prob_bounds}
\end{equation}
then  
\begin{equation}
    P(\boldsymbol{\mu} \in [\boldsymbol{\mu}^-,\boldsymbol{\mu}^+]) \geq 1-\alpha^{(s)}.
    \label{SoRa_eq:indv_prob_bounds_all}
\end{equation}
The bounds in \eqref{SoRa_eq:indv_prob_bounds} may be found using any of the vectorized-compatible (Q)MC methods in \Cref{SoRa_table:qmcpy_sc}.

To propagate bounds $[\boldsymbol{\mu}^-,\boldsymbol{\mu}^+]$ on mean $\boldsymbol{\mu}$ to bounds $[s^-,s^+]$ on QOI $s$, the user must define functions $C^-$ and $C^+$ using interval arithmetic \cite{interval_analysis} and problem specific knowledge. These functions must ensure $s \in [C^-(\boldsymbol{\mu}^-,\boldsymbol{\mu}^+),C^+(\boldsymbol{\mu}^-,\boldsymbol{\mu}^+)]$ whenever $\boldsymbol{\mu} \in [\boldsymbol{\mu}^-,\boldsymbol{\mu}^+]$. Without problem specific knowledge, one may set 
\begin{equation}
    s^- = C^-(\boldsymbol{\mu}^-,\boldsymbol{\mu}^+) = \min_{\boldsymbol{\mu} \in [\boldsymbol{\mu}^-,\boldsymbol{\mu}^+]} C(\boldsymbol{\mu}), \qquad 
    s^+= C^+(\boldsymbol{\mu}^-,\boldsymbol{\mu}^+) = \max_{\boldsymbol{\mu} \in [\boldsymbol{\mu}^-,\boldsymbol{\mu}^+]} C(\boldsymbol{\mu}).
    \label{SoRa_eq:C_minus_C_plus}
\end{equation}
\Cref{SoRa_table:elementary_ops_Cpm} provides examples of such interval arithmetic functions for some basic operations. Problem specific knowledge may be used to further shrink the naive bounds in bounds \eqref{SoRa_eq:C_minus_C_plus}. For example, if $s$ is a probability then $0 \leq s^- \leq s^+ \leq 1$ may be encoded into $C^-$ and $C^+$. See the sensitivity index computation in \Cref{sec:vsc_numerical_experiments_v} for a more nuanced example.

In all, given $\alpha^{(s)}$, we may set $\boldsymbol{\alpha}^{(\boldsymbol{\mu})}=\alpha^{(s)}/N$ elementwise then use scalar (Q)MC  algorithms to find $[\boldsymbol{\mu}^-,\boldsymbol{\mu}^+]$ satisfying \eqref{SoRa_eq:indv_prob_bounds} so that \eqref{SoRa_eq:indv_prob_bounds_all} holds. Then setting $[s^-,s^+]$ via \eqref{SoRa_eq:C_minus_C_plus}, potentially combined with problem specific knowledge, guarantees \eqref{SoRa_eq:scalar_comb_desired_ineq} holds.

\begin{table}[!ht]
    \caption{Interval arithmetic functions for elementary operations.}
    \centering
    \resizebox{\columnwidth}{!}{%
    \begin{tabular}{r  c  c}
        $s=C(\boldsymbol{\mu})$ & $s^- = C^-(\boldsymbol{\mu}^-,\boldsymbol{\mu}^+)$ & $s^+ = C^+(\boldsymbol{\mu}^-,\boldsymbol{\mu}^+)$ \\
        \hline
        $\mu_1+\mu_2$ & $\mu_1^-+\mu_2^-$ & $\mu_1^++\mu_2^+$ \\
        $\mu_1-\mu_2$ & $\mu_1^--\mu_2^+$ & $\mu_1^+-\mu_2^-$ \\
        $\mu_1 \cdot \mu_2$ & $\min(\mu_1^-\mu_2^-,\mu_1^-\mu_2^+,\mu_1^+\mu_2^-,\mu_1^+\mu_2^+)$ & $\max(\mu_1^-\mu_2^-,\mu_1^-\mu_2^+,\mu_1^+\mu_2^-,\mu_1^+\mu_2^+)$ \\
        $\mu_1 / \mu_2$ & $\begin{cases} -\infty, & 0 \in [\mu_2^-,\mu_2^+] \\ \min\left(\frac{\mu_1^-}{\mu_2^-},\frac{\mu_1^+}{\mu_2^-},\frac{\mu_1^-}{\mu_2^+},\frac{\mu_1^+}{\mu_2^+}\right), & 0 \notin [\mu_2^-,\mu_2^+] \end{cases}$ & $\begin{cases} \infty, & 0 \in [\mu_2^-,\mu_2^+] \\ \max\left(\frac{\mu_1^-}{\mu_2^-},\frac{\mu_1^+}{\mu_2^-},\frac{\mu_1^-}{\mu_2^+},\frac{\mu_1^+}{\mu_2^+}\right), & 0 \notin [\mu_2^-,\mu_2^+] \end{cases}$ \\
        $\min(\mu_1,\mu_2)$ & $\min(\mu_1^-,\mu_2^-)$ & $\min(\mu_1^+,\mu_2^+)$ \\
        $\max(\mu_1,\mu_2)$ & $\max(\mu_1^-,\mu_2^-)$ & $\max(\mu_1^+,\mu_2^+)$ \\
        \hline
    \end{tabular}
    }
    \label{SoRa_table:elementary_ops_Cpm}
\end{table}

\Subsection{Optimal Approximation of a Scalar QOI} \label{sec:vsc_opt_comb_sol_sc}

Here we derive a stopping criterion for adaptive sampling and an optimal approximation $\hat{s}$ of a scalar QOI $s$. Let $h_\varepsilon: \mathbb{R} \to \mathbb{R}^+$ be an error metric dependent on some error tolerance $\varepsilon$ so that the stopping criterion is met if and only if the QOI approximation $\hat{s}$ satisfies 
\begin{equation}
    \lvert s-\hat{s} \rvert \leq h_\varepsilon(s) \qquad \forall\; s \in [s^-,s^+].
    \label{SoRa_eq:sc_raw}
\end{equation}
\Cref{SoRa_thm:shat_opt} determines the optimal $\hat{s}$ and an equivalent condition to \eqref{SoRa_eq:sc_raw} when $h_\varepsilon$ is a metric map, i.e., Lipschitz continuous with Lipschitz constant at most $1$. Some compatible error metric options are
\begin{subequations}
\begin{align}
    h_\varepsilon(s) &= \max\left(\varepsilon^\text{abs},\lvert s \rvert \varepsilon^\text{rel} \right) \quad &&\text{absolute or relative error satisfied,} \label{SoRa_eq:h_abs_or_rel} \\
    h_\varepsilon(s) &= \min\left(\varepsilon^\text{abs},\lvert s \rvert \varepsilon^\text{rel} \right) \quad &&\text{absolute and relative error satisfied.} \label{SoRa_eq:h_abs_and_rel}
\end{align}
\end{subequations}
\begin{theorem} \label{SoRa_thm:shat_opt}
    Suppose that  $h_\varepsilon$ satisfies the metric map condition
    \begin{equation}
        \lvert h_\varepsilon(s_1) - h_\varepsilon(s_2) \rvert \leq \lvert s_1 - s_2 \rvert \qquad \forall\; s_1,s_2 \in \mathbb{R}.
        \label{SoRa_eq:metric_map_cond}
    \end{equation}
    Then error criterion  \eqref{SoRa_eq:sc_raw} holds if and only if 
    \begin{equation}
        s^+-s^- \leq h_\varepsilon(s^-)+h_\varepsilon(s^+).
        \label{SoRa_eq:sc}
    \end{equation}
    Furthermore, the choice of 
    \begin{equation}
        \hat{s} = \frac{1}{2}\left[s^-+s^++h_\varepsilon(s^-)-h_\varepsilon(s^+)\right]
        \label{SoRa_eq:shat_opt}
    \end{equation}
    minimizes $\sup_{s \in [s^-,s^+]} \lvert s - \hat{s} \rvert -h_\varepsilon(s)$ for any choice of $s^{\pm}$ with $s^- < s^+$.
\end{theorem}

\begin{proof}
    Define $g(s,\hat{s})=\lvert s - \hat{s} \rvert -h_\varepsilon(s)$. From \eqref{SoRa_eq:metric_map_cond}, it follows that if  $s^- \leq s \leq \hat{s}$ then $g(s^-,\hat{s})-g(s,\hat{s}) \geq 0$, and if $\hat{s} \leq s \leq s^+$ then $g(s^+,\hat{s})-g(s,\hat{s})  \geq 0$. This means that $g(\cdot,\hat{s})$ attains its maximum at either $s^-$ or $s^+$ so that
    \begin{equation*}
        \max_{s \in [s^-,s^+]} g(s,\hat{s}) = \max_{s \in \{s^-,s^+\}} g(s,\hat{s}).
    \end{equation*}
    
    Next, we find the optimal choice of $\hat{s}$.  The function $g(s^-,\cdot)$ is monotonically decreasing to the left of  $s^-$ and monotonically increasing to the right of $s^-$. Similarly, $g(s^+,\cdot)$ is monotonically decreasing to the left of $s^+$ and monotonically increasing to the right of $s^+$. This means that the optimal choice of $\hat{s}$ to minimize $\max_{s \in \{s^-,s^+\}} g(s,\hat{s})$ lies in $[s^-,s^+]$ and satisfies $g(s^-,\hat{s}) = g(s^+,\hat{s})$, that is, 
    $$\hat{s} - s^- - h_\varepsilon(s^-) = s^+ - \hat{s} - h_\varepsilon(s^+).$$
    Solving for the optimal value of $\hat{s}$ leads to \eqref{SoRa_eq:shat_opt}.
    
    For this optimal $\hat{s}$, 
    $$2 \max_{s \in [s^-,s^+]} g(s,\hat{s}) =  s^+  -  s^-  - h_\varepsilon(s^-) - h_\varepsilon(s^+).$$
    The error criterion is equivalent to $\max_{s \in [s^-,s^+]} g(s,\hat{s}) \le 0 $.  This can only hold under condition  \eqref{SoRa_eq:sc}. 
\end{proof}

\Subsection{Adaptive Algorithm with Extension to Array QOI} \label{sec:vsc_vectorized_implementation}

Up to this point, we have assumed an array mean $\boldsymbol{\mu}$ was used to compute a scalar QOI $s$. We now relax this assumption to enable approximation of array QOI $\boldsymbol{s}$. The optimal approximation $\hat{\boldsymbol{s}}$ and stopping criterion may still be computed by elementwise application of \eqref{SoRa_eq:shat_opt} and \eqref{SoRa_eq:sc}. 

For some integrands $\boldsymbol{f}$ it is possible to avoid evaluating particular integrand outputs when all affected QOI have already been sufficiently approximated. In such cases, the user may enable economic evaluation by defining a dependency function $\boldsymbol{D}: \{\text{True},\text{False}\}^{\boldsymbol{d}_{\boldsymbol{s}}} \to \{\text{True},\text{False}\}^{\boldsymbol{d}_{\boldsymbol{\mu}}}$ which maps stopping flags on QOI to stopping flags on means. The latter indicates which outputs the integrand is required to compute in the next iteration.  We say (QOI) index $\boldsymbol{1} \leq \boldsymbol{l} \leq \boldsymbol{d}_{\boldsymbol{s}}$ depends on (mean) index $\boldsymbol{1} \leq \boldsymbol{k} \leq \boldsymbol{d}_{\boldsymbol{\mu}}$ if the $\boldsymbol{k}^\text{th}$ entry is $\text{True}$ in the output of evaluating $\boldsymbol{D}$ at the multidimensional array with only the $\boldsymbol{l}^\text{th}$ entry set to $\text{True}$.
%Take a simple example where $\boldsymbol{s}=\boldsymbol{\mu}$, so $\boldsymbol{C}$ and $\boldsymbol{D}$ are the identity function. At some iteration suppose the QOI flag at index $\boldsymbol{1} \leq \boldsymbol{l} \leq \boldsymbol{d}_{\boldsymbol{s}}$ is $\text{True}$, indicating the QOI, and therefore mean, at index $\boldsymbol{l}$ has been sufficiently approximated. Then in the next iteration, the function does not need to compute outputs at index $\boldsymbol{l}$.

Moreover, $\boldsymbol{D}$ is used to compute mean uncertainty levels $\boldsymbol{\alpha}^{(\boldsymbol{\mu})}$ from combined uncertainty levels $\balpha^{(\bs)}$ in the spirit of Boole's inequality as done in \Cref{sec:vsc_comb_sol_approx}. The idea is to ensure that each element of $\boldsymbol{\alpha}^{(\boldsymbol{s})}$ is greater than the sum of elements in $\boldsymbol{\alpha}^{(\boldsymbol{\mu})}$ with dependent indices. Specifically, let $\boldsymbol{N} \in \mathbb{N}^{\boldsymbol{d}_{\boldsymbol{s}}}$ contain the number of dependent mean for each QOI. That is, for every $\boldsymbol{1} \leq \boldsymbol{l} \leq \boldsymbol{d}_{\boldsymbol{s}}$, $N_{\boldsymbol{l}}$ is the number of indices dependent on $\boldsymbol{l}$. For every $\boldsymbol{1} \leq \boldsymbol{k} \leq \boldsymbol{d}_{\boldsymbol{\mu}}$, if $\boldsymbol{l}$ is dependent on $\boldsymbol{k}$ then $\alpha_{\boldsymbol{l}}^{(\boldsymbol{s})}/N_{\boldsymbol{l}}$ is a candidate for $\alpha_{\boldsymbol{k}}^{(\boldsymbol{\mu})}$. We then set $\alpha_{\boldsymbol{k}}^{(\boldsymbol{\mu})}$ to the minimum among all candidates for $\alpha_{\boldsymbol{k}}^{(\boldsymbol{\mu})}$, assuming the candidate set is not empty.

While not theoretically required, our implementation practically requires that each index $\boldsymbol{1} \leq \boldsymbol{k} \leq \boldsymbol{d}_{\boldsymbol{\mu}}$ be a dependency of exactly one index $\boldsymbol{1} \leq \boldsymbol{l} \leq \boldsymbol{d}_{\boldsymbol{s}}$. To illustrate this requirement and the previously discussed dependency structure, let us consider the simple example with QOI $s_1 = \mu_1 + \mu_2$ and $s_2 = \mu_1 + \mu_3$. Suppose after some iteration that $s_1$ is sufficiently approximated and $s_2$ is not. Since $\mu_1$ is a dependency of $s_2$, we would like to continue sampling for $\mu_1$ to get a better approximation of $s_2$. However, changes in the bounds on $\mu_1$ will change the bounds on $s_1$ and may potentially make the approximation of $s_1$ become insufficient again. This out-of-sync nature of the sampling occurs because $\mu_1$ is a dependency of more than one QOI. To remedy this, let $\mu_4 = \mu_1$ and set $s_2 = \mu_4 + \mu_3$. Now each mean is a dependency of exactly one QOI as required. In practice this remedy amounts to copying integrand outputs at index $1$ into index $4$, thus increasing storage requirements in favor of potentially avoiding evaluating integrand outputs at index $2$ or $3$. This dependency structure is illustrated below with dependency function $D(b_1,b_2) = (b_1,b_1,b_2,b_2)$ where $b_1,b_2 \in \{\text{True},\text{False}\}$. 

\begin{figure}[!ht]
    \centering
\begin{tikzpicture}[main/.style = {draw, circle}] 
    \node at (0,0)  [main,line width=3pt] (1) {$\mu_1$}; 
    \node at (2,0)  [main,line width=3pt] (2) {$\mu_2$};
    \node at (4,0)  [main,line width=3pt] (3) {$\mu_3$};
    \node at (6,0)  [main,line width=3pt] (4) {$\mu_4$};
    \node at (1,1) [main,line width=3pt] (5) {$s_1$};
    \node at (5,1) [main,line width=3pt] (6) {$s_2$};
    \draw[->,line width=3pt] (5) -- (1);
    \draw[->,line width=3pt] (5) -- (2);
    \draw[->,line width=3pt] (6) -- (3);
    \draw[->,line width=3pt] (6) -- (4);
\end{tikzpicture}
\end{figure}

\Cref{SoRa_algo:MCStoppingCriterion} details the adaptive procedure developed throughout this chapter. Notice that the implementation does not require specifying $\boldsymbol{C}$ despite its use in deriving the necessary $\boldsymbol{C}^-$ and $\boldsymbol{C}^+$ inputs. The cost of this algorithm is concentrated on evaluating the function at an IID or LD sequence. In practice, the run time may be reduced through parallel and/or economic evaluation.

\begin{algorithm}[!ht]
    \fontsize{12}{10}\selectfont
    \caption{Adaptive (Quasi-)Monte Carlo for Array QOI}
    \label{SoRa_algo:MCStoppingCriterion}
    \begin{algorithmic}
        \Require $\boldsymbol{f}: (0,1)^d \to \mathbb{R}^{\boldsymbol{d}_{\boldsymbol{\mu}}}$, the integrand where $\boldsymbol{\mu} = \mathbb{E}[\boldsymbol{f}(\boldsymbol{X})]$ for $\boldsymbol{X} \sim \mathcal{U}[0,1]^d$.
        \Require $\boldsymbol{C}^-,\boldsymbol{C}^+: \mathbb{R}^{\boldsymbol{d}_{\boldsymbol{\mu}}} \times \mathbb{R}^{\boldsymbol{d}_{\boldsymbol{\mu}}} \to \mathbb{R}^{\boldsymbol{d}_{\boldsymbol{s}}}$, generalization of \eqref{SoRa_eq:C_minus_C_plus} so $\boldsymbol{\mu} \in [\boldsymbol{\mu}^-,\boldsymbol{\mu}^+]$ implies $\boldsymbol{s} \in [\boldsymbol{C}^-(\boldsymbol{\mu}^-,\boldsymbol{\mu}^+),\boldsymbol{C}^+(\boldsymbol{\mu}^-,\boldsymbol{\mu}^+)]$.
        \Require $\boldsymbol{\alpha}^{(\boldsymbol{s})} \in (0,1)^{\boldsymbol{d}_{\boldsymbol{s}}}$, the desired uncertainty thresholds on QOI bounds so the returned $[\boldsymbol{s}^-,\boldsymbol{s}^+]$ will satisfy $P(s_{\boldsymbol{l}} \in [s_{\boldsymbol{l}}^-,s_{\boldsymbol{l}}^+]) \geq 1-\alpha^{(\boldsymbol{s})}_{\boldsymbol{l}}$ for any $\boldsymbol{1} \leq \boldsymbol{l} \leq \boldsymbol{d}_{\boldsymbol{s}}$.
        \Require $h_{\boldsymbol{l},\varepsilon_l}: \mathbb{R} \to \mathbb{R}^+$ for $\boldsymbol{1} \leq \boldsymbol{l} \leq \boldsymbol{d}_{\boldsymbol{l}}$, see \eqref{SoRa_eq:h_abs_or_rel} or \eqref{SoRa_eq:h_abs_and_rel} for examples. Stopping flag at index $\boldsymbol{l}$ is set to $\text{True}$ when $\lvert s_{\boldsymbol{l}} - \hat{s}_{\boldsymbol{l}} \rvert \leq h_{\boldsymbol{l},\varepsilon_l}(s_{\boldsymbol{l}})$ for all $s_{\boldsymbol{l}} \in [s_{\boldsymbol{l}}^-,s_{\boldsymbol{l}}^+]$.
        \Require $\boldsymbol{D}: \{\text{True},\text{False}\}^{\boldsymbol{d}_{\boldsymbol{s}}} \to \{\text{True},\text{False}\}^{\boldsymbol{d}_{\boldsymbol{\mu}}}$, maps stopping flags on QOI $\boldsymbol{s}$ to stopping flags on mean $\boldsymbol{\mu}$. 
        \Require A scalar (Q)MC algorithm capable of producing bounds on a mean which holds with uncertainty below a specified threshold. See the vectorized-compatible (Q)MC methods in \Cref{SoRa_table:qmcpy_sc}.
        \Require A generator of IID or LD sequences compatible with the scalar MC or QMC algorithm. 
        \Require $m_1 \in \mathbb{N}$, where $2^{m_1}$ is the initial number of samples.
        \State $n_\text{start} \gets 1$ \Comment{lower index in node sequence(s)}
        \State $n_\text{end} \gets 2^{m_1}$ \Comment{upper index in node sequence(s)}
        \State $\boldsymbol{b}^{(\boldsymbol{\mu})} \gets \text{False}^{\boldsymbol{d}_{\boldsymbol{\mu}}}$ \Comment{stopping flags on the mean}
        \State $\boldsymbol{b}^{(\boldsymbol{s})} \gets \text{False}^{\boldsymbol{d}_{\boldsymbol{s}}}$ \Comment{stopping flags on QOI}
        \State Set $\boldsymbol{\alpha}^{(\boldsymbol{\mu})}$ based on $\boldsymbol{\alpha}^{(\boldsymbol{s})}$ using $\boldsymbol{D}$ \Comment{See the discussion in \Cref{sec:vsc_vectorized_implementation}.}
        \While{$b_{\boldsymbol{l}}^{(\boldsymbol{s})} = \text{False}$ for some $\boldsymbol{1} \leq \boldsymbol{l} \leq \boldsymbol{d}_{\boldsymbol{s}}$} \Comment{A QOI is not sufficiently bounded}
            \State Generate nodes from the IID or LD sequence(s) from index $n_\text{start}$ to $n_\text{end}$
            \State Evaluate $\boldsymbol{f}$ at the new nodes where $\boldsymbol{b}^{(\boldsymbol{\mu})}=\textbf{\text{False}}$ \Comment{may be done in parallel}
            \State Update $\boldsymbol{\mu}^-$ and $\boldsymbol{\mu}^+$ using the scalar (Q)MC algorithm where $\boldsymbol{b}^{(\boldsymbol{\mu})}=\textbf{\text{False}}$
            \State $[\boldsymbol{s}^-,\boldsymbol{s}^+] \gets \left[\boldsymbol{C}^-(\boldsymbol{\mu}^-,\boldsymbol{\mu}^+),\boldsymbol{C}^+(\boldsymbol{\mu}^-,\boldsymbol{\mu}^+)\right]$ 
            \State $b^{(\boldsymbol{s})}_{\boldsymbol{l}} \gets \text{Bool}\left(s_{\boldsymbol{l}}^+-s_{\boldsymbol{l}}^- < h_{\boldsymbol{l},\varepsilon_l}(s_{\boldsymbol{l}}^-)+h_{\boldsymbol{l},\varepsilon_l}(s_{\boldsymbol{l}}^+)\right),\quad \boldsymbol{1} \leq \boldsymbol{l} \leq \boldsymbol{d}_{\boldsymbol{s}}$ \Comment{\eqref{SoRa_eq:sc}}
            \State $\boldsymbol{b}^{(\boldsymbol{\mu})} \gets \boldsymbol{D}\left(\boldsymbol{b}^{(\boldsymbol{s})}\right)$
            \State $n_\text{start} \gets n_\text{end}+1$
            \State $n_\text{end} \gets 2n_\text{start}$
        \EndWhile
        \State $\hat{s}_{\boldsymbol{l}} \gets \frac{1}{2}[s_{\boldsymbol{l}}^-+s_{\boldsymbol{l}}^++h_{\boldsymbol{l},\varepsilon_l}(s_{\boldsymbol{l}}^-)-h_{\boldsymbol{l},\varepsilon_l}(s_{\boldsymbol{l}}^+)], \quad\boldsymbol{1} \leq \boldsymbol{l} \leq \boldsymbol{d}_{\boldsymbol{s}}$ \Comment{\eqref{SoRa_eq:shat_opt}}
        \\ \Return $\hat{\boldsymbol{s}},[\boldsymbol{s}^-,\boldsymbol{s}^+]$
    \end{algorithmic}
\end{algorithm}

\Subsection{Numerical Experiments} \label{sec:vsc_numerical_experiments_v}

We exemplify the capabilities of \Cref{SoRa_algo:MCStoppingCriterion} on problems from machine learning and global sensitivity analysis. As a simple example, let us compute the expected displacement and stress on a cantilever beam
\begin{equation*}
    \tD(T_1,T_2,T_3) = \frac{4l^3}{T_1 w t} \sqrt{\frac{T_2^2}{t^4} + \frac{T_3^2}{w^4}} \qqtqq{and} \tS(T_1,T_2,T_3) = 600\left(\frac{T_2}{wt^2} + \frac{T_3}{w^2t}\right)
    \label{SoRa_eq:cantilever_beam}
\end{equation*}
respectively as defined in \cite{VLSE}. Here $l=100$, $w=4$, $t=2$ are constants and $T_1$, $T_2$, $T_3$ are independent Gaussian random variables with mean and standard deviation as defined as in the \Cref{tab:cantilever_beam} below.
\begin{table}[!ht]
    \caption{Variable descriptions for the cantilever beam function.}
    \centering
    \begin{tabular}{r c c l}
        Gaussian & mean & standard deviation & description \\ 
        \hline
        $T_1$ & $2.9 \times 10^7$ & $1.45 \times 10^6$ & Young's modulus of beam material \\
        $T_2$ & $500$ & $100$ & horizontal load on beam \\
        $T_3$ & $1000$ & $100$ & vertical load on beam \\
        \hline 
    \end{tabular}
    \label{tab:cantilever_beam}
\end{table}

The following code snippet approximates $s_1 = \mu_1 = \mathbb{E}\left[\tD(\bT)\right]$ and $s_2 = \mu_2 = \mathbb{E}\left[\tS(\bT)\right]$ where $\boldsymbol{C}^-$, $\boldsymbol{C}^+$, and $\boldsymbol{D}$ are defaulted to the identity functions. Notice that \texttt{CustomFun} is initialized with $g=\left(\tD,\tS\right)$ as a function of the Gaussian random vector $\bT$ rather than the transformed integrand $f$, which is a function of $\boldsymbol{X} \sim \mathcal{U}(0,1)^3$. Here the default absolute or relative error metric \eqref{SoRa_eq:h_abs_or_rel} is used.

\begin{lstlisting}[style=Python]
def cantilever_beam_function(T,compute_flags): # T is (n, 3)
    Y = np.zeros((2,len(T)),dtype=float) # (n, 2)
    l,w,t = 100,4,2
    T1,T2,T3 = T[:,0],T[:,1],T[:,2]
    if compute_flags[0]: # compute D tilde
        Y[0] = 4*l**3/(T1*w*t)*np.sqrt(T2**2/t**4+T3**2/w**4)
    if compute_flags[1]: # compute S tilde
        Y[1] = 600*(T2/(w*t**2)+T3/(w**2*t))
    return Y
true_measure = qp.Gaussian(
    sampler = qp.DigitalNetB2(dimension=3,seed=7),
    mean = [2.9e7,500,1000],
    covariance = np.diag([(1.45e6)**2,(100)**2,(100)**2]))
integrand = qp.CustomFun(true_measure,
    g = cantilever_beam_function,
    dimension_indv = 2)
qmc_stop_crit = qp.CubQMCNetG(integrand,
    abs_tol = 1e-3,
    rel_tol = 1e-6)
solution,data = qmc_stop_crit.integrate()
print(data)
"""
Data (Data)
    solution        [2.426e+00 3.750e+04]
    comb_bound_diff [0.002 0.041]
    comb_flags      [ True  True]
    n_total         2^(18)
    n               [  1024 262144]
    time_integrate  0.060
"""
\end{lstlisting}

\Subsubsection{Vectorized Acquisition Functions for Bayesian Optimization}

Bayesian optimization (BO) is a sequential optimization technique that attempts to find the global maximum of a black box function $\varphi: (0,1)^{\nu} \to \mathbb{R}$. It is assumed that $\varphi$ is expensive to evaluate, so we must strategically select sampling locations that maximize some utility or acquisition function. At a high level, BO 
\begin{enumerate}
    \item Iteratively samples $\varphi$ at locations maximizing the acquisition function,
    \item Updates a Gaussian process surrogate based on these new observations,
    \item Updates an acquisition function based on the updated surrogate,
    \item repeats until the budget for sampling $\varphi$ has expired.
\end{enumerate}
Bayesian optimization is detailed in \cite{snoek2012practical} while Gaussian process regression is given individual treatment in \Cref{sec:gps} and more generally in \cite{rasmussen.gp4ml}

Concretely, suppose we have already sampled $\varphi$ at $\boldsymbol{z}_1,\dots,\boldsymbol{z}_{N} \in [0,1]^{\nu}$ to collect data $\mathcal{D}=\{(\boldsymbol{z}_i,y_i)\}_{i=1}^N$ where $y_i = \varphi(\boldsymbol{z}_i)$. BO may then fit a Gaussian process surrogate to data $\mathcal{D}$. The next $d$ sampling locations may then be chosen to maximize an acquisition function $\alpha: [0,1]^{(d,\nu)} \to \mathbb{R}$ which takes a matrix whose rows are the next sampling locations to a payoff value. Specifically, we set $\boldsymbol{z}_{N+1}, \dots, \boldsymbol{z}_{N+d}$ to be the rows of $\mathop{\text{argmax}}_{\boldsymbol{Z} \in [0,1]^{(d,\nu)}}\alpha(\boldsymbol{Z})$. Many acquisition functions may be expressed as an expectation of the form $\alpha(\boldsymbol{Z}) = \mathbb{E}\left[a(\boldsymbol{y}) \mid \boldsymbol{y} \sim \mathcal{N}\left(\boldsymbol{M},\boldsymbol{K}\right)\right]$ where $\boldsymbol{M} \in \mathbb{R}^{d}$ and  $\boldsymbol{K} \in \mathbb{R}^{(d,d)}$ are respectively the posterior mean and covariance of the Gaussian process at points $\boldsymbol{Z}$. Here we focus on the $q$-Expected Improvement ($q$EI) acquisition function which uses $a(\boldsymbol{y}) = \max_{1 \leq i \leq d} (y_i - y^*)_+$ where $y^*= \max\left(y_1,\dots,y_N\right)$ is the current maximum and $(\cdot)_+ = \max(\cdot,0)$. 

Suppose we choose the argument maximum  from among a finite set of $d$-sized batches $\boldsymbol{Z}_1,\dots,\boldsymbol{Z}_{d_{\boldsymbol{\mu}}} \in [0,1]^{(d,\nu)}$ so that $\boldsymbol{z}_{N+1}, \dots,\boldsymbol{z}_{N+d}$ are set to be the rows of 
$$\mathop{\text{argmax}}_{\boldsymbol{Z} \in \{\boldsymbol{Z}_1,\dots,\boldsymbol{Z}_{d_{\boldsymbol{\mu}}}\}}\alpha(\boldsymbol{Z}).$$
We may vectorize the acquisition function computations to 
$$s_i = \mu_i = \alpha(\boldsymbol{Z}_i) = \mathbb{E}\left[a\left(\boldsymbol{A}_i\boldsymbol{\Phi}^{-1}(\boldsymbol{X})+\boldsymbol{M}_i\right)\right]$$
for $i=1,\dots,d_\mu$ where $\boldsymbol{X} \sim \mathcal{U}(0,1)^d$ and $\boldsymbol{\Phi}^{-1}$ is the inverse CDF of the standard Gaussian taken elementwise. Now $\boldsymbol{M}_i$ and $\boldsymbol{K}_i = \boldsymbol{A}_i\boldsymbol{A}_i^\intercal$ are the posterior mean and covariance respectively of the Gaussian process at $\boldsymbol{Z}_i$ so that $\boldsymbol{A}_i\boldsymbol{\Phi}^{-1}(\boldsymbol{X})+\boldsymbol{M}_i \sim \mathcal{N}\left(\boldsymbol{M}_i,\boldsymbol{K}_i\right)$ for $i=1,\dots,d_{\boldsymbol{\mu}}$.

Since the quantity of interest is simply the vector of expectations, one may set $\boldsymbol{C}$, $\boldsymbol{C}^-$, $\boldsymbol{C}^+$, and $\boldsymbol{D}$ to appropriate identity functions. The process described above is visualized in \Cref{SoRa_fig:bo_qei} for $\nu=1$ and $d=2$. In this example, it may be more intuitive to make $\boldsymbol{d}_{\boldsymbol{\mu}}$ have length $2$ so the matrix of means reflects the grid of white dots in the right panel of \Cref{SoRa_fig:bo_qei}. 

Notice the optimal next points trade off exploiting parts of the domain where performant samples have already been observed and  exploring parts of the domain where the Gaussian process has large uncertainty.  As $d$ and/or $\nu$ grow, the number of candidates $d_\mu$ required for the discrete argument maximum to be a good approximation of the continuous argument maximum grows rapidly, thus rendering such non-greedy acquisition functions intractable for even moderate values of $d$ or $\nu$.

\begin{figure}[!ht]
    \centering
    \includegraphics[width=.9\textwidth]{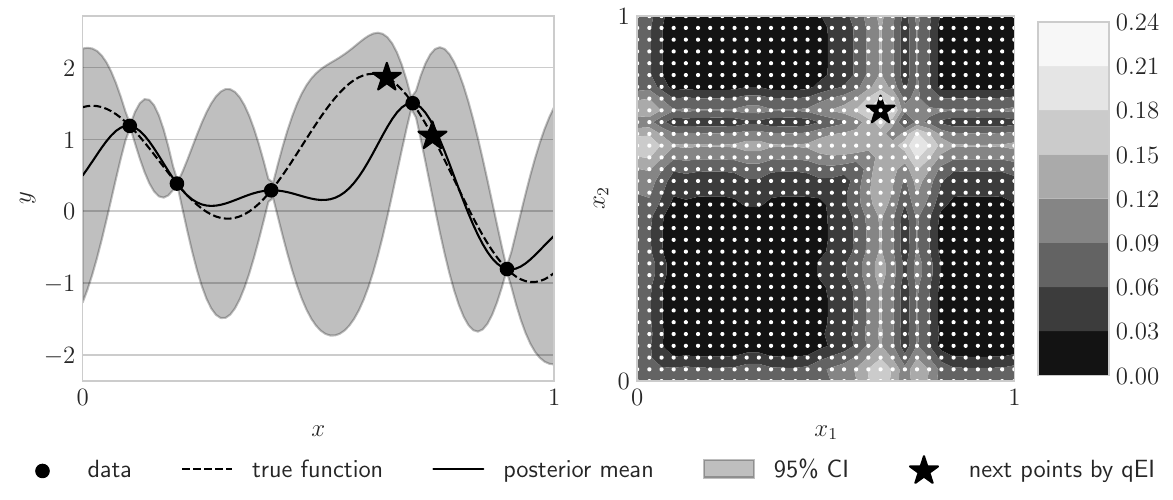}
    \caption{First, the true function has been sampled at the data points shown in the left panel. Next, a Gaussian process is fit to the data points to approximate the true function. The posterior mean and $95\%$ confidence interval (CI) of the Gaussian process are shown in the left panel. With $d=2$, a fine grid of candidates is chosen in $[0,1]^{2}$ and depicted in the right panel. The presented (Q)MC algorithm is then used to approximate the acquisition function value at each of the candidate grid points. These approximations are made into a contour plot in the right panel. The discrete argument maximum among these approximations on the fine grid is the next size $d$ batch of points by $q$-expected Improvement ($q$EI). These next points for sequential optimization are visualized in both the right and left panels by the black stars.}
    \label{SoRa_fig:bo_qei}
\end{figure}

\Subsubsection{Bayesian Posterior Mean}

The Bayesian framework combines prior knowledge of random parameters $\boldsymbol{\Theta} \in \mathbb{R}^{d_{\boldsymbol{s}}}$ with observational data and a likelihood function $\rho$ to construct a model-aware posterior distribution on $\boldsymbol{\Theta}$. Suppose we have a dataset of observations $\boldsymbol{y} = (y_1,\dots,y_{N})$ taken at IID locations $\boldsymbol{z}_1,\dots,\boldsymbol{z}_{N}$ respectively. Then Bayes' rule may be used to write the posterior density of $\boldsymbol{\Theta}$ as 
$$P\left(\boldsymbol{\theta} \mid \boldsymbol{y} \right) = \frac{P(\boldsymbol{y} \mid \boldsymbol{\theta}) P(\boldsymbol{\theta})}{P\left(\boldsymbol{y}\right)} = \frac{\prod_{i=1}^{N} \rho(y_i \mid \boldsymbol{\theta}) P(\boldsymbol{\theta})}{\mathbb{E}\left[\prod_{i=1}^{N} \rho(y_i \mid \boldsymbol{\theta})\right]}.$$
Here the expectation is taken with respect to the prior distribution on $\boldsymbol{\Theta}$ with density $P(\boldsymbol{\theta})$, and $P\left(\boldsymbol{y} \mid \boldsymbol{\theta} \right)$ is the likelihood density which factors into the product of likelihoods $\rho(y_i \mid \boldsymbol{\theta})$ since the observations are IID. 

A useful quantity of interest is the posterior mean of $\boldsymbol{\Theta}$. In this example, the QOI is the posterior mean $\boldsymbol{s}$ which may be written as the ratio of expectations via $\boldsymbol{s} = \mathbb{E}\left[\boldsymbol{\Theta} \mid \boldsymbol{y}\right] = \mathbb{E}\left[\boldsymbol{\Theta} \; \prod_{i=1}^{N} \rho(y_i \mid \boldsymbol{\Theta})\right]/\mathbb{E}\left[\prod_{i=1}^{N} \rho(y_i \mid \boldsymbol{\Theta})\right]$. As before, the expectations are taken with respect to the prior distribution on $\boldsymbol{\Theta}$. In the framework of this section, $\boldsymbol{\mu} \in \mathbb{R}^{(2, d_{\boldsymbol{s}})}$ where for $k=1,\dots,d_{\boldsymbol{s}}$ we have 
$$\mu_{0k} = \mathbb{E}\left[\Theta_k \prod_{i=1}^{N} \rho(y_i \mid \boldsymbol{\Theta})\right], \quad \mu_{1k} = \mathbb{E}\left[\prod_{i=1}^{N} \rho(y_i \mid \boldsymbol{\Theta})\right], \quad \text{and} \quad s_k = \frac{\mu_{0k}}{\mu_{1k}}.$$
Defining $\boldsymbol{C}^-$ and $\boldsymbol{C}^+$ follow from vectorizing the quotient forms in \Cref{SoRa_table:elementary_ops_Cpm} while the dependency function $\boldsymbol{D}: \{\text{True},\text{False}\}^{d_{\boldsymbol{s}}} \to \{\text{True},\text{False}\}^{(2, d_{\boldsymbol{s}})}$ is defined by stacking the row vectors of QOI flags on top of itself. 

While $\mu_{1k}$ is the same for all $1 \leq k \leq d_{\boldsymbol{s}}$, we opt to keep separate denominator estimates for each coefficient. Maintaining only a single estimate of the denominator would open up the possibility that a coefficient is sufficiently approximated in one iteration and then insufficiently approximated the next due to an unfavorable update in the denominator. In this way, maintaining separate estimates of the denominator favors economic evaluation in favor of increased storage.

Consider Bayesian logistic regression as a concrete example. Here the dataset contains observations $y_i \in \{0,1\}$ at locations $\boldsymbol{z}_i = \left(z_{i1},\dots,z_{i(d_{\boldsymbol{s}}-1)},1\right)$ for $i=1,\dots,N$ where the last value is $1$ to accommodate an intercept term. The sigmoid likelihood function $\rho(y_i = 1 \mid \boldsymbol{\theta}) = \frac{\exp(\boldsymbol{\theta}.\boldsymbol{z}_i)}{1+\exp(\boldsymbol{\theta}.\boldsymbol{z}_i)}$ may be used to yield
$$P(\boldsymbol{y} \mid \boldsymbol{\theta}) = \prod_{i=1}^N \left(\frac{\exp(\boldsymbol{\theta}.\boldsymbol{z}_i)}{1+\exp(\boldsymbol{\theta}.\boldsymbol{z}_i)}\right)^{y_i} \left(1-\frac{\exp(\boldsymbol{\theta}.\boldsymbol{z}_i)}{1+\exp(\boldsymbol{\theta}.\boldsymbol{z}_i)}\right)^{1-y_i} = \prod_{i=1}^N \frac{\exp(\boldsymbol{\theta}.\boldsymbol{z}_i)^{y_i}}{1+\exp(\boldsymbol{\theta}.\boldsymbol{z}_i)}.$$

The following code performs Bayesian logistic regression in \texttt{QMCPy} on the Haberman's Survival Dataset retrieved from the UCI Machine Learning Repository \cite{uci_ml_repo}. Here we use a normal prior $\boldsymbol{\Theta} \sim \mathcal{N}(\boldsymbol{M},\boldsymbol{K})$ so that
$$P(\boldsymbol{\theta}) = \frac{\exp\left(-(\boldsymbol{\theta}-\boldsymbol{M})^\intercal\boldsymbol{K}^{-1}(\boldsymbol{\theta}-\boldsymbol{M})/2\right)}{\sqrt{(2\pi)^d\lvert \det(\boldsymbol{K})\rvert }}.$$
This example uses the absolute and relative error metric \eqref{SoRa_eq:h_abs_and_rel}. Notice that the number of samples required to approximate each coefficient is different. 

\begin{lstlisting}[style=Python]
blr = qp.BayesianLRCoeffs(
    sampler = qp.DigitalNetB2(4,seed=1),
    feature_array = xt, # np.ndarray of shape (n, d-1)
    response_vector = yt, # np.ndarray of shape (n,)
    prior_mean = 0, # normal prior mean = (0, 0, ..., 0)
    prior_covariance = 5) # normal prior covariance = 5I
qmc_sc = qp.CubQMCNetG(blr,
    abs_tol = .1,
    rel_tol = .5,
    error_fun = "BOTH",
    n_limit=2**18)
blr_coefs,blr_data = qmc_sc.integrate()
print(blr_data)
"""
Data (Data)
solution        [-0.128  0.084 -0.091  0.09 ]
comb_bound_low  [-0.198  0.069 -0.14   0.074]
comb_bound_high [-0.105  0.129 -0.075  0.138]
comb_bound_diff [0.092 0.06  0.065 0.065]
comb_flags      [ True  True  True  True]
n_total         2^(17)
n               [[  1024   1024   2048 131072]
                    [  1024   1024   2048 131072]]
time_integrate  0.351
"""
\end{lstlisting}

\Subsubsection{Sensitivity Indices}

Sensitivity analysis quantifies how uncertainty in a function output may be attributed to subsets of function inputs. Functional ANOVA (analysis of variance) decomposes a function $\varphi \in L^2(0,1)^\nu$ into the sum of orthogonal functions $(\varphi_u)_{u \subseteq {1:\nu}}$. Here $1:\nu=\{1,\dots,\nu\}$ denotes the set of all dimensions and $\varphi_u \in L^2(0,1)^{\lvert u \rvert}$ denotes a sub-function dependent only on inputs $\boldsymbol{x}_u = (x_j)_{j \in u}$ where $\lvert u \rvert$ is the cardinality of $u$. By construction, these sub-functions sum to the objective function so that $\varphi(\boldsymbol{x}) = \sum_{u \subseteq 1:\nu} \varphi_u(\boldsymbol{x}_u)$ \cite[Appendix A]{owen.mc_book_practical}. 
The orthogonality of sub-functions enables the variance of $\varphi$ to be decomposed into the sum of variances of sub-functions. Specifically, denoting the variance of $\varphi$ by $\sigma^2$, we may write $\sigma^2 = \sum_{u \subseteq 1:\nu} \sigma^2_u$ where $\sigma^2_u$ is the variance of sub-function $\varphi_u$. The sub-variance $\sigma_u$ quantifies the variance of $\varphi$ attributable to inputs $u \subseteq 1:\nu$.  The \emph{closed and total Sobol' indices}
\begin{align*}
    \underline{\tau}_u^2 &= \sum_{v \subseteq u} \sigma^2_v = \int_{[0,1]^{2\nu}} f(\boldsymbol{x})[f(\boldsymbol{x}_{u_j},\boldsymbol{z}_{-{u_j}})-f(\boldsymbol{z})]\mathrm{d}\boldsymbol{x}\mathrm{d}\boldsymbol{z} \quad \text{and}  \\ 
    \overline{\tau}_u^2 &= \sum_{v \cap u \neq \emptyset} \sigma^2_v = \frac{1}{2}\int_{[0,1]^{2\nu}} [f(\boldsymbol{z})-f(\boldsymbol{x}_u,\boldsymbol{z}_{-{u_j}})]^2\mathrm{d}\boldsymbol{x}\mathrm{d}\boldsymbol{z}
    \label{SoRa_eq:sobol_indices}
\end{align*}
quantify the variance attributable to subsets of $u$ and subsets containing $u$ respectively. Here the notation $(\boldsymbol{x}_{u},\boldsymbol{z}_{-u})$ denotes a point where the value at index $1 \leq j \leq \nu$ is $x_j$ if $j \in u$ and $z_j$ otherwise. The \emph{closed and total sensitivity indices} $\underline{s}_u = \underline{\tau}_u^2/\sigma^2$ and $\overline{s}_u = \overline{\tau}_u^2/\sigma^2$ respectively normalize the Sobol' indices to quantify the proportion of variance explained by a given subset of inputs. 

Suppose one is interested in computing the closed and total sensitivity indices of $\varphi$ at $u_1,\dots,u_c \subseteq 1:\nu$. Then we may choose the mean $\boldsymbol{\mu} \in \mathbb{R}^{(2, 3, c)}$ so that $\boldsymbol{\mu}_1,\boldsymbol{\mu}_2 \in \mathbb{R}^{(3,c)}$ contain values for the closed and total sensitivity indices respectively. Specifically,  $\boldsymbol{\mu}_{11},\boldsymbol{\mu}_{21} \in \mathbb{R}^c$ contain the closed and total Sobol' indices respectively while $\boldsymbol{\mu}_{i2}, \boldsymbol{\mu}_{i3} \in \mathbb{R}^c$ contain first and second moments respectively for any $i \in \{1,2\}$. For the QOI $\boldsymbol{s} \in \mathbb{R}^{(2, c)}$, we set $\boldsymbol{s}_1, \boldsymbol{s}_2 \in \mathbb{R}^c$ to contain the closed and total sensitivity indices respectively. 

Bounds may be propagated via $\boldsymbol{C}^-,\boldsymbol{C}^+:\mathbb{R}^{(2, 3, c)} \to \mathbb{R}^{(2, c)}$ defined for $i \in \{1,2\}$ and $j \in \{1,\dots,c\}$  by  
\begin{align*}
    C_{ij}^-(\boldsymbol{\mu}^-,\boldsymbol{\mu}^+) 
    %= \text{clip}\left(\min_{\boldsymbol{\mu} \in [\boldsymbol{\mu}^-,\boldsymbol{\mu}^+]} \frac{\mu_{i1j}}{\mu_{i3j}-\mu_{i2j}^2}\right) \\
    = \begin{cases} 
        \text{clip}\left(\min\left(\frac{\mu_{i1j}^-}{\mu_{i3j}^+-\left(\mu_{i2j}^-\right)^2},\frac{\mu_{i1j}^-}{\mu_{i3j}^+-\left(\mu_{i2j}^+\right)^2}\right)\right), & \mu_{i3j}^- - \left(\mu_{i2j}^\pm\right)^2 >0 \\
        0, &\text{else}
     \end{cases} %\\
    % C_{ij}^+(\boldsymbol{\mu}^-,\boldsymbol{\mu}^+) 
    % &= \text{clip}\left(\max_{\boldsymbol{\mu} \in [\boldsymbol{\mu}^-,\boldsymbol{\mu}^+]} \frac{\mu_{i1j}}{\mu_{i3j}-\mu_{i2j}^2}\right) \\
    % &= \begin{cases} 
    %     \text{clip}\left(\max\left(\frac{\mu_{i1j}^+}{\mu_{i3j}^--\left(\mu_{i2j}^-\right)^2},\frac{\mu_{i1j}^+}{\mu_{i3j}^--\left(\mu_{i2j}^+\right)^2}\right)\right), & \mu_{i3j}^- - \left(\mu_{i2j}^\pm\right)^2 >0 \\
    %     1, &\text{else}
    % \end{cases}
\end{align*}
with $C^+_{ij}(\boldsymbol{\mu}^-,\boldsymbol{\mu}^+)$ defined similarly and where $\text{clip}(\cdot) = \min(1,\max(0,\cdot))$ restricts values between 0 and 1. Above we have encoded the facts that sensitivity indices are between $0$ and $1$, the variance of $\varphi$ is non-negative, and Sobol' indices are non-negative. The dependency function $\boldsymbol{D}:\{\text{True},\text{False}\}^{(2, c)} \to \{\text{True},\text{False}\}^{(2, 3, c)}$ may be defined by broadcasting shapes so that for any $(1,1,1) \leq (i,j,k) \leq (2,3,c)$ we have $D_{ikj}(\boldsymbol{b}^{(\boldsymbol{s})}) = b_{ij}^{(\boldsymbol{s})}$. 

The \texttt{QMCPy} implementation further generalize to allow array objective functions $\boldsymbol{\varphi}: (0,1)^\nu \to \mathbb{R}^{\tilde{\boldsymbol{d}}_{\boldsymbol{\mu}}}$ so $\boldsymbol{d}_{\boldsymbol{\mu}} = (2,3,c,\tilde{\boldsymbol{d}}_{\boldsymbol{\mu}})$ and $\boldsymbol{d}_{\boldsymbol{s}} = (2,c,\tilde{\boldsymbol{d}}_{\boldsymbol{\mu}})$. Here the notation of nested vectors indicates that, for example, that  $(2,c,(5,6)) =(2,c,5,6)$. Also, notice that $d = 2 \nu$ in general. That is, the dimension of the node sequence is twice the size of the input dimension to $\varphi$.

Sensitivity indices present an illustrative case for computational complexity. Suppose the QMC algorithm takes $2^m$ total samples to accurately approximate all closed and total sensitivity indices for $u_1,\dots,u_c \subseteq 1:\nu$. Then the computational cost is $\$(\boldsymbol{\varphi})(2+c)2^m$ since every time our sensitivity index function is evaluated at $(\boldsymbol{x},\boldsymbol{z}) \in [0,1]^{2\nu}$ we must evaluate the users objective function at $\boldsymbol{x}$, $\boldsymbol{z}$, and $(\boldsymbol{x}_{u_j},\boldsymbol{z}_{-{u_j}})$ for $j=1,\dots,c$. If a user is only interested in approximating singleton sensitivity indices, $u_j = \{j\}$ for $j=1,\dots,\nu$, then it is possible to reduce the cost from $\$(\boldsymbol{\varphi})(2+\nu)2^m$ to $\$(\boldsymbol{\varphi})2^{m+1}$ using order $1$ replicated designs \cite{alex2008comparison,tissot2015randomized}. Such designs have been extended to  digital sequences in \cite{galquin.sobol_seq_replicated_designs_construction} and utilized for sensitivity index approximation in \cite{rugama.sobol_indices_error_estimation}.

A first example computes sensitivity indices of the Ishigami function \cite{ishigami1990importance} $g(\bT) = (1+bT_3^4)\sin(T_1)+a\sin^2(T_2)$ where $\bT \sim \mathcal{U}(-\pi,\pi)^3$ and $a=7$, $b=0.1$ as in \cite{crestaux2007polynomial,marrel2009calculations}. \Cref{SoRa_fig:ishigami} visualizes the resulting optimal approximations and QOI bounds which capture the exact sensitivity indices of the Ishigami function. 

\begin{figure}[!ht]
    \centering
    \includegraphics[width=.8\textwidth]{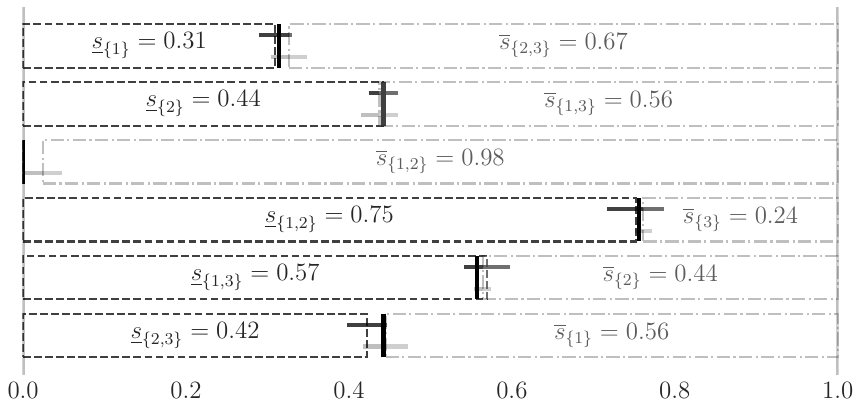}
    \caption{Approximate closed and total sensitivity indices for the Ishigami function illustrating the relationship $\underline{s}_u + \overline{s}_{u^c} = 1$ for all $u \subseteq 1:d$. In each row, the closed sensitivity index bar is  extended to the right from $0$ while the total sensitivity index bar is extended to the left from $1$. The bars should meet at the heavy vertical line for the analytic quantity of interest (QOI) $\underline{s}_u=1-\overline{s}_{u^c}$. The darker and lighter horizontal lines within each row depict the bounds for the closed and total sensitivity indices respectively. The heavy vertical line crossing both horizontal lines in each row indicates the true QOI is indeed captured in the bounds.}
    \label{SoRa_fig:ishigami}
\end{figure}

In another example, we compute sensitivity indices of a neural network classifier \cite{he2015delving} for the Iris dataset \cite{uci_ml_repo}. This example was inspired by a similar experiment in \cite{hoyt2021efficient}. The dataset consists of attributes sepal length (SL), sepal width (SW), petal length (PL), and petal width (PW), all in centimeters, from which an Iris is to be classified as either the \emph{setosa}, \emph{versicolor}, or \emph{virginica} species. We begin by fitting a neural network classifier that takes in input features and outputs a size $3$ vector of probabilities for each species summing to $1$. Taking the argument maximum among these three probabilities gives a species prediction. On a held out portion of the dataset, the neural network attains 98\% classification accuracy and may therefore be deemed a high quality surrogate for the true relation between input features and species classification. 

Our problem is to quantify, for each species, the variability in the classification probability attributed to a set of inputs. In other words, we would like to compute the sensitivity indices for each species probability. 
In the code below we compute all sensitivity indices with the help of the scikit-learn package \cite{scikit-learn} for splitting off a holdout dataset and training the neural network classifier. 
Here $\boldsymbol{d}_{\boldsymbol{\mu}} = (2,3,14,3)$ and $\boldsymbol{d}_{\boldsymbol{s}} = (2,14,3)$ since we have $3$ species classes, $14$ sensitivity indies of interest, and we are computing both the closed and total sensitivity indices. \Cref{SoRa_fig:nn_si} visualizes closed sensitivity index approximations. 

\begin{lstlisting}[style=Python]
mlpc = MLPClassifier(random_state=7,max_iter=1024).fit(xt,yt)
yhat = mlpc.predict(xv)
print("accuracy: %.1f%%"%(100*(yv==yhat).mean()))
# accuracy: 98.0%
sampler = qp.DigitalNetB2(4,seed=7)
true_measure =  qp.Uniform(sampler,
    lower_bound = xt.min(0),
    upper_bound = xt.max(0))
fun = qp.CustomFun(
    true_measure = true_measure,
    g = lambda x: mlpc.predict_proba(x).T,
    dimension_indv = 3)
si_fun = qp.SensitivityIndices(fun,indices="all")
qmc_algo = qp.CubQMCNetG(si_fun,abs_tol=.005)
nn_sis,nn_sis_data = qmc_algo.integrate()
nn_sis.shape
"""
(2, 14, 3)
"""
\end{lstlisting}

\begin{figure}[!ht]
    \centering
    \includegraphics[width=.8\textwidth]{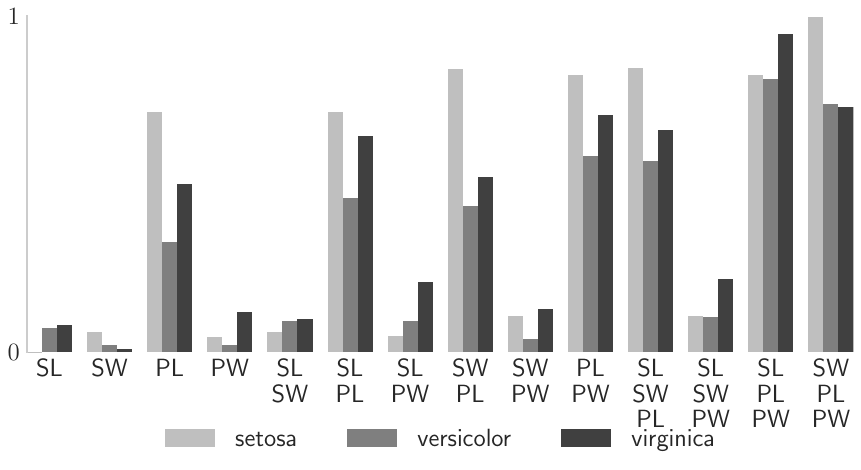}
    \caption{Closed sensitivity indices for neural network classification probability of each Iris species.}
    \label{SoRa_fig:nn_si}
\end{figure}

\Chapter{Fast Kernel Methods} \label{sec:fast_kernel_methods}

This chapter will discuss fast kernel methods enabled by pairing certain low-discrepancy sequences with special kernel forms. We will begin with some motivation and background in \Cref{sec:context_fast_kernel_methods}. Then, \Cref{sec:kernel_matrix_structure} will detail the kernel matrix structures which will arise when using certain low-discrepancy sequences paired with (digitally-)shift-invariant kernels. Specifically, we will use the rank-1 lattices from \Cref{sec:lattices} and base $2$ digital nets from \Cref{sec:dnets}. \Cref{sec:kernel_forms} will then detail the special multivariate kernel forms we use which are built from the univariate shift-invariant (SI) kernels presented in \Cref{sec:si_kernels} or the univariate digitally-shift-invariant (DSI) kernels presented in \Cref{sec:dsi_kernels}

\Section{Context} \label{sec:context_fast_kernel_methods}

Let us begin by defining Hermitian positive definite/semi-definite kernels and matrices. We will denote the domain by $\calX \subseteq \bbR^d$, and the fast kernel computations we present will use kernels defined over $\calX = [0,1)^d$. 

\begin{definition}[HPD and HPSD kernels] \label{def:spd_SPSD_kernels}
    A complex-valued kernel $K: \calX \times \calX \to \bbC$ is HPD (Hermitian positive definite) if it is conjugate symmetric,
    $$K(\bx,\bx') = \overline{K(\bx',\bx)} \qquad \forall\; \bx,\bx' \in \calX,$$
    and positive definite,
    \begin{equation}
        \sum_{i,i'=0}^{n-1} c_i \overline{c_{i'}} K(\bx_i,\bx_{i'}) > 0 \qquad \forall\; n \in \bbN, \quad c_0,\dots,c_{n-1} \in \bbC, \quad \bx_0,\dots,\bx_{n-1} \in \calX.
        \label{eq:spd_k_func_cond}
    \end{equation}
    A complex-valued kernel $K: \calX \times \calX \to \bbC$ is HPSD (Hermitian positive semi-definite) if $K$ is conjugate symmetric and \eqref{eq:spd_k_func_cond} holds with the positive constraint relaxed to a non-negative constraint. 
\end{definition}
\begin{definition}[HPD and HPSD matrices] \label{def:spd_SPSD_matrices}
    A complex-valued matrix $\mK \in \bbC^{n \times n}$ is HPD (Hermitian positive definite) if it is conjugate symmetric,
    $$\mK = \overline{\mK}^\intercal,$$
    and positive definite,
    \begin{equation}
        \overline{\bc}^\intercal \mK \bc > 0, \qquad \forall\; \bc \in \bbC^n.
        \label{eq:spd_k_mat_cond}
    \end{equation}
    A matrix $\mK \in \bbC^{n \times n}$ is HPSD (Hermitian positive semi-definite) if $\mK$ is conjugate symmetric and \eqref{eq:spd_k_mat_cond} holds with the positive constraint relaxed to a non-negative constraint. 
\end{definition}
Clearly all HPD kernels are HPSD kernels and all HPD matrices are HPSD matrices. When the kernel $K$ is real-valued HPSD or HPD we may call it SPD (symmetric positive definite) or SPSD (symmetric positive semi-definite) respectively. Similarly, when a matrix $\mK$ is real-valued HPSD or HPD we may call it SPD or SPSD respectively. Going forward, when we refer to a kernel it is assumed to be HPSD or SPSD appropriately. The following theorem from \cite{aronszajn.theory_of_reproducing_kernels} connects HPSD/SPSD kernels and RKHSs (reproducing kernel Hilbert spaces). 
\begin{theorem}[Moore--Aronszajn]
    \label{thm:moore_aronszajn}
    If $K$ is a HPSD or SPSD kernel, then there is a unique Hilbert space $H(K)$ for which $K$ is the reproducing kernel. 
\end{theorem}

We will motivate our kernel computations with the following two examples. For these, suppose $H$ is an RKHS with SPD kernel $K: [0,1)^d \times [0,1)^d \to \bbR$, $\{\bx_i\}_{i=0}^{n-1} \subset [0,1]^d$ is a point set, and $\mK = (K(\bx_i,\bx_{i'}))_{i,i'=0}^{n-1}$ is a SPD kernel matrix. 

\begin{enumerate}
    \item \textbf{Discrepancy Computation} The error of the Quasi-Monte Carlo approximation in \eqref{eq:mc_approx} can be bounded by the product of two terms: a measure of variation of the function $f$ and a measure of the discrepancy of the point set $\{\bx_i\}_{i=0}^{n-1}$.  The most well known bound of this type is the Koksma--Hlawka inequality \cite{hickernell.generalized_discrepancy_quadrature_error_bound,dick.high_dim_integration_qmc_way,hickernell1999goodness,niederreiter.qmc_book}. The discrepancy of $\{\bx_i\}_{i=0}^{n-1}$ is frequently computed when designing and evaluating low-discrepancy point sets, see \cite{rusch2024message} for a newer idea in this area where good point sets are generated using neural networks trained with a discrepancy-based loss function.  
    
    Let us consider the more general setting of approximating the mean $\mu$ in \eqref{eq:mc_approx} by a weighted cubature rule $\sum_{i=0}^{n-1} \omega_i f(\bx_i)$. If we assume $f \in H$, then the discrepancy in the error bound of such a cubature rule takes the form
    $$\int_{[0,1]^d} \int_{[0,1]^d} K(\bx,\bx') \D \bx \D \bx' - 2 \sum_{i=0}^{n-1} \omega_i \int_{[0,1]^d} K(\bx',\bx_i) \D \bx' + \sum_{i,i'=0}^{n-1} \omega_i \omega_{i'} K(\bx_i,\bx_{i'})$$
    following \cite{hickernell.generalized_discrepancy_quadrature_error_bound}. Evaluating the discrepancy above requires computing the matrix-vector product $\mK \bomega$ where $\bomega = \{\omega_i\}_{i=0}^{n-1}$. Moreover, this discrepancy is minimized for a given point set by setting the weights to $\sbomega = \mK^{-1} \bK$ where $\bK = \left(\int_{[0,1]^d} K(\bx,\bx_i) \D \bx\right)_{i=0}^{n-1}$.
    \item \textbf{Kernel Interpolation} Suppose we would like to approximate $f: [0,1)^d \to \bbR$ given observations $\by = \{y_i\}_{i=0}^{n-1}$ of $f$ at $\{\bx_i\}_{i=0}^{n-1}$ satisfying $y_i = f(x_i)$ for $0 \leq i < n$. Then an optimal kernel interpolant approximates $f$ by $\hf \in H$ where 
    $$\hf(\bx) = \sum_{i=0}^{n-1} \omega_i K(\bx,\bx_i)$$
    and $\bomega = \mK^{-1} \by$. The above kernel interpolant may be reinterpreted as the posterior mean in Gaussian process regression, see \cite{rasmussen.gp4ml} or \Cref{sec:gps}. Fitting a Gaussian process regression model often includes optimizing a kernel's hyperparameters, which may also be done by computing  $\mK \bomega$ and $\mK^{-1} \by$.
\end{enumerate}

Underlying these problems, and many others, is the requirement to compute the matrix-vector product $\mK \by$ or solve the linear system $\mK^{-1}\by$. The standard cost of these computations are $\calO(n^2)$ and $\calO(n^3)$ respectively, typically using a Cholesky decomposition and back-substitutions to solving the linear system. One method to reduce these high costs is to induce structure into $\mK$. \cite{zeng.spline_lattice_digital_net,zeng.spline_lattice_error_analysis} proposed structure-inducing methods which use certain kernels $K$ and points $\{\bx_i\}_{i=0}^{n-1}$ which enable both computations to be done at only $\calO(n \log n)$ cost. Such pairings were recently studied in the context of fast Bayesian cubature \cite{rathinavel.bayesian_QMC_lattice,rathinavel.bayesian_QMC_sobol,rathinavel.bayesian_QMC_thesis}. At least two such pairings exist:
\begin{enumerate}
    \item When a rank-1 lattice point set $\{\bx_i\}_{i=0}^{n-1}$ in linear order is paired with a shift-invariant (SI) kernel, $\mK$ is circulant and thus diagonalizable by the fast Fourier transform (FFT) and inverse FFT (IFFT). Using lattices in the extensible radical inverse order requires using the bit-reversed fast transforms FFTBR and IFFTBR. \Cref{sec:fft_si_kernels_r1lattices} will detail the permuted circulant structures and the FFTBR/IFFTBR, then \Cref{sec:kernel_forms} and \Cref{sec:si_kernels} will study the commonly used class of SI kernels.
    \item When a base $2$ digital net $\{\bx_i\}_{i=0}^{n-1}$ in the extensible radical inverse order is paired with a digitally-shift-invariant (DSI) kernel, $\mK$ becomes RSBT (recursive symmetric block Toeplitz) and thus diagonalizable by the fast Walsh--Hadamard transform (FWHT). \Cref{sec:fwht_dsi_kernels_dnb2s} will detail RSBT structures and the FWHT, then \Cref{sec:kernel_forms} and \Cref{sec:dsi_kernels} will study DSI kernels, including new higher-order smoothness forms derived in this thesis.
\end{enumerate}
\Cref{tab:com_kernel_costs} compares these fast kernel methods against standard techniques in terms of both storage and computational costs. \Cref{fig:fast_transforms} gives schematics of the FFTBR and FWHT algorithms which enable these fast kernel methods.

\begin{table}[!ht]
    \caption{Comparison of storage and cost requirements for kernel methods. When evaluating (forming) $\mK$, we assume the cost of evaluating the kernel is $\calO(d)$. Factorization of the symmetric positive semi-definite Gram matrix $\mK$ is the cost of computing the eigendecomposition or Cholesky decomposition. The costs of matrix-vector multiplication $\mK \by$, solving a linear system $\mK^{-1} \by$ (when $\mK$ is symmetric positive definite), and computing the determinant $\lvert \mK \rvert$ are the costs after performing the decomposition. Both storage and kernel computation costs are greatly reduced by pairing certain low-discrepancy point sets with special shift-invariant (SI) or digitally-shift-invariant (DSI) kernels. These fast algorithms rely on the fast Fourier transform in bit-reversed order (FFTBR), its inverse (IFFTBR), and the fast Walsh--Hadamard transform (FWHT).}
    \centering
    \small
    \begin{tabular}{ccccc} 
        $\{\bx_i\}_{i=0}^{n-1}$ points  & $K$ structure & factor $\mK$ method & $\mK$ storage & form $\mK$ cost \\ 
        \hline 
        any & general SPD & Cholesky  & $\calO(n^2)$ & $\calO(n^2 d)$\\
        rank-1 lattice & SPD SI & (I)FFTBR & $\calO(n)$ & $\calO(n d)$\\ 
        base $2$ digital net & SPD DSI & FWHT & $\calO(n)$ & $\calO(n d)$ \\
        \hline 
        \hline
        $\{\bx_i\}_{i=0}^{n-1}$ & factor $\mK$ cost & $\mK \by$ cost & $\mK^{-1} \by$ cost  & $\lvert \mK \rvert$ cost  \\ 
        \hline 
        any  & $\calO(n^3)$ & $\calO(n^2)$ & $\calO(n^2)$ & $\calO(n^2)$ \\
        rank-1 lattice  & $\calO(n \log n)$ & $\calO(n \log n)$ & $\calO(n \log n)$ & $\calO(n \log n)$ \\ 
        base $2$ digital net & $\calO(n \log n)$ & $\calO(n \log n)$ & $\calO(n \log n)$ & $\calO(n \log n)$ \\
        \hline 
    \end{tabular}
    \label{tab:com_kernel_costs}
\end{table}

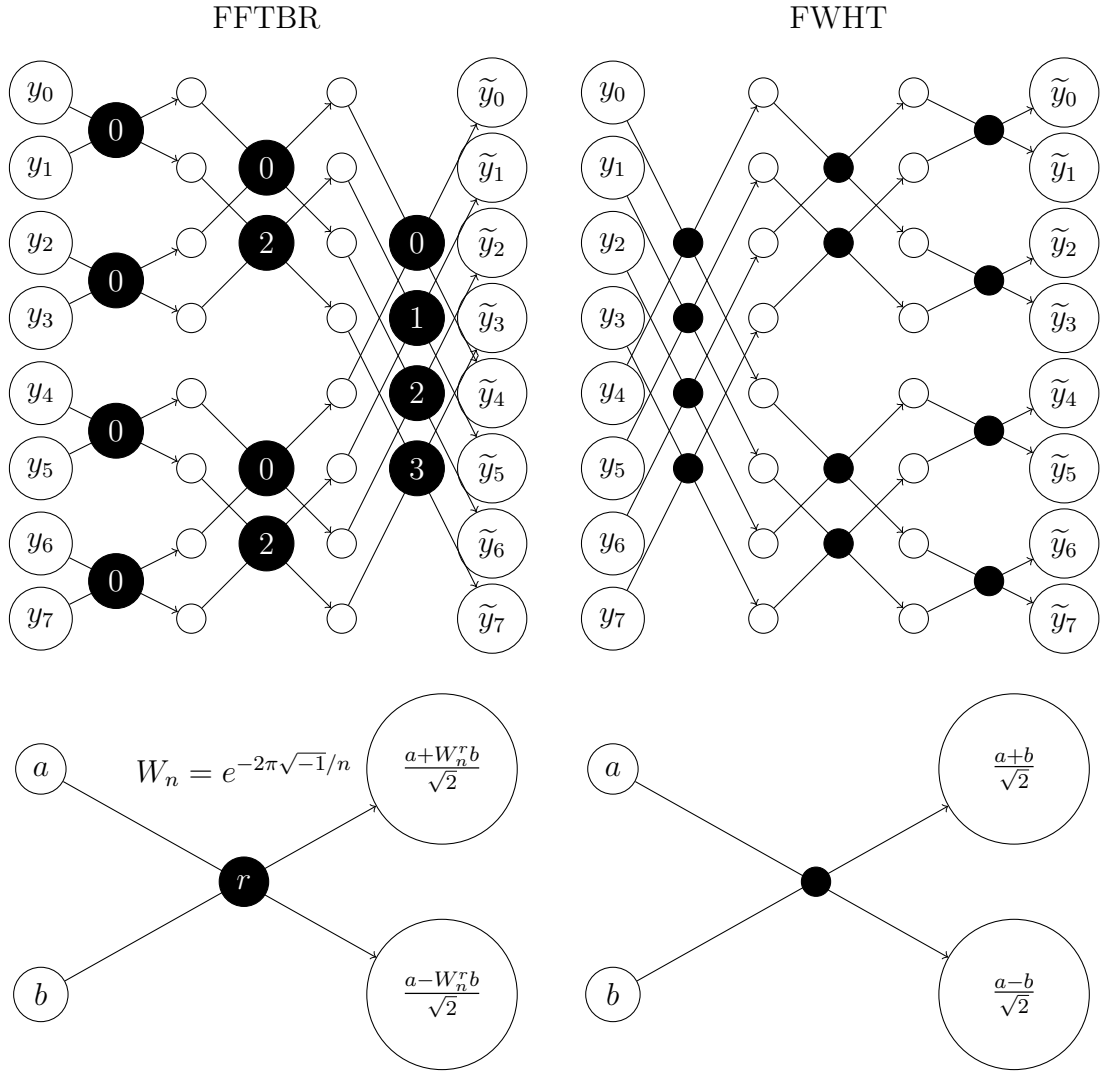
\begin{figure}[!ht]
    \centering
    \newcommand{\h}{1}
    \newcommand{\w}{2}
    \newcommand{\y}{4}
    \newcommand{\z}{6}
    \newcommand{\p}{5}
    \begin{subfigure}[t]{.49\textwidth}
    \begin{tikzpicture}
        % top left axis
        \draw (\z/2,\p+8*\h) node{FFTBR};
        \draw ( 0,\p+7*\h) node[draw,circle](l00){$y_0$}; %y_0
        \draw ( 0,\p+6*\h) node[draw,circle](l10){$y_1$}; %y_4
        \draw ( 0,\p+5*\h) node[draw,circle](l20){$y_2$}; % y_2
        \draw ( 0,\p+4*\h) node[draw,circle](l30){$y_3$}; % y_6
        \draw ( 0,\p+3*\h) node[draw,circle](l40){$y_4$}; % y_1
        \draw ( 0,\p+2*\h) node[draw,circle](l50){$y_5$}; % y_5
        \draw ( 0,\p+1*\h) node[draw,circle](l60){$y_6$}; % y_6
        \draw ( 0,\p+0*\h) node[draw,circle](l70){$y_7$}; % y_7
        \draw (\w,\p+7*\h) node[draw,circle](l01){};
        \draw (\w,\p+6*\h) node[draw,circle](l11){};
        \draw (\w,\p+5*\h) node[draw,circle](l21){};
        \draw (\w,\p+4*\h) node[draw,circle](l31){};
        \draw (\w,\p+3*\h) node[draw,circle](l41){};
        \draw (\w,\p+2*\h) node[draw,circle](l51){};
        \draw (\w,\p+1*\h) node[draw,circle](l61){};
        \draw (\w,\p+0*\h) node[draw,circle](l71){};
        \draw (\y,\p+7*\h) node[draw,circle](l02){};
        \draw (\y,\p+6*\h) node[draw,circle](l12){};
        \draw (\y,\p+5*\h) node[draw,circle](l22){};
        \draw (\y,\p+4*\h) node[draw,circle](l32){};
        \draw (\y,\p+3*\h) node[draw,circle](l42){};
        \draw (\y,\p+2*\h) node[draw,circle](l52){};
        \draw (\y,\p+1*\h) node[draw,circle](l62){};
        \draw (\y,\p+0*\h) node[draw,circle](l72){};
        \draw (\z,\p+7*\h) node[draw,circle](l03){$\ty_0$};
        \draw (\z,\p+6*\h) node[draw,circle](l13){$\ty_1$};
        \draw (\z,\p+5*\h) node[draw,circle](l23){$\ty_2$};
        \draw (\z,\p+4*\h) node[draw,circle](l33){$\ty_3$};
        \draw (\z,\p+3*\h) node[draw,circle](l43){$\ty_4$};
        \draw (\z,\p+2*\h) node[draw,circle](l53){$\ty_5$};
        \draw (\z,\p+1*\h) node[draw,circle](l63){$\ty_6$};
        \draw (\z,\p+0*\h) node[draw,circle](l73){$\ty_7$};
        % first layer 
        \draw[->] (l00) -- (l11); 
        \draw[->] (l10) -- (l01); 
        \draw[->] (l20) -- (l31); 
        \draw[->] (l30) -- (l21); 
        \draw[->] (l40) -- (l51); 
        \draw[->] (l50) -- (l41); 
        \draw[->] (l60) -- (l71); 
        \draw[->] (l70) -- (l61);
        \draw (\w/2,\p+0.5*\h) node[draw,circle,fill=black,text=white]{$0$};
        \draw (\w/2,\p+2.5*\h) node[draw,circle,fill=black,text=white]{$0$};
        \draw (\w/2,\p+4.5*\h) node[draw,circle,fill=black,text=white]{$0$};
        \draw (\w/2,\p+6.5*\h) node[draw,circle,fill=black,text=white]{$0$};
        % second layer 
        \draw[->] (l01) -- (l22); 
        \draw[->] (l11) -- (l32); 
        \draw[->] (l21) -- (l02); 
        \draw[->] (l31) -- (l12); 
        \draw[->] (l41) -- (l62); 
        \draw[->] (l51) -- (l72); 
        \draw[->] (l61) -- (l42); 
        \draw[->] (l71) -- (l52);
        \draw (\w/2+\y/2,\p+1*\h) node[draw,circle,fill=black,text=white]{$2$};
        \draw (\w/2+\y/2,\p+2*\h) node[draw,circle,fill=black,text=white]{$0$};
        \draw (\w/2+\y/2,\p+5*\h) node[draw,circle,fill=black,text=white]{$2$};
        \draw (\w/2+\y/2,\p+6*\h) node[draw,circle,fill=black,text=white]{$0$};
        % third layer 
        \draw[->] (l02) -- (l43); 
        \draw[->] (l12) -- (l53); 
        \draw[->] (l22) -- (l63); 
        \draw[->] (l32) -- (l73); 
        \draw[->] (l42) -- (l03); 
        \draw[->] (l52) -- (l13); 
        \draw[->] (l62) -- (l23); 
        \draw[->] (l72) -- (l33);
        \draw (\y/2+\z/2,\p+2*\h) node[draw,circle,fill=black,text=white]{$3$};
        \draw (\y/2+\z/2,\p+3*\h) node[draw,circle,fill=black,text=white]{$2$};
        \draw (\y/2+\z/2,\p+4*\h) node[draw,circle,fill=black,text=white]{$1$};
        \draw (\y/2+\z/2,\p+5*\h) node[draw,circle,fill=black,text=white]{$0$};
        % combination diagram 
        \draw (0,3*\h) node[draw,circle](wit){$a$};
        \draw (0,0) node[draw,circle](wib){$b$};
        \draw (\y/3+2*\z/3,3*\h) node[draw,circle,minimum size=2cm](wot){$\frac{a+W_n^rb}{\sqrt{2}}$};
        \draw (\y/3+2*\z/3,0) node[draw,circle,minimum size=2cm](wob){$\frac{a-W_n^rb}{\sqrt{2}}$};
        \draw[->] (wit) -- (wob); 
        \draw[->] (wib) -- (wot); 
        \draw (.45*\z,1.5*\h) node[draw,circle,fill=black,text=white]{$r$};
        \draw (.45*\z,3*\h) node{$W_n = e^{-2 \pi \sqrt{-1}/n}$};
    \end{tikzpicture}
    \end{subfigure}
    \begin{subfigure}[t]{.49\textwidth}
    \begin{tikzpicture}
        % top left axis
        \draw (\z/2,\p+8*\h) node{FWHT};
        \draw ( 0,\p+7*\h) node[draw,circle](l00){$y_0$};
        \draw ( 0,\p+6*\h) node[draw,circle](l10){$y_1$};
        \draw ( 0,\p+5*\h) node[draw,circle](l20){$y_2$};
        \draw ( 0,\p+4*\h) node[draw,circle](l30){$y_3$};
        \draw ( 0,\p+3*\h) node[draw,circle](l40){$y_4$};
        \draw ( 0,\p+2*\h) node[draw,circle](l50){$y_5$};
        \draw ( 0,\p+1*\h) node[draw,circle](l60){$y_6$};
        \draw ( 0,\p+0*\h) node[draw,circle](l70){$y_7$};
        \draw (\w,\p+7*\h) node[draw,circle](l01){};
        \draw (\w,\p+6*\h) node[draw,circle](l11){};
        \draw (\w,\p+5*\h) node[draw,circle](l21){};
        \draw (\w,\p+4*\h) node[draw,circle](l31){};
        \draw (\w,\p+3*\h) node[draw,circle](l41){};
        \draw (\w,\p+2*\h) node[draw,circle](l51){};
        \draw (\w,\p+1*\h) node[draw,circle](l61){};
        \draw (\w,\p+0*\h) node[draw,circle](l71){};
        \draw (\y,\p+7*\h) node[draw,circle](l02){};
        \draw (\y,\p+6*\h) node[draw,circle](l12){};
        \draw (\y,\p+5*\h) node[draw,circle](l22){};
        \draw (\y,\p+4*\h) node[draw,circle](l32){};
        \draw (\y,\p+3*\h) node[draw,circle](l42){};
        \draw (\y,\p+2*\h) node[draw,circle](l52){};
        \draw (\y,\p+1*\h) node[draw,circle](l62){};
        \draw (\y,\p+0*\h) node[draw,circle](l72){};
        \draw (\z,\p+7*\h) node[draw,circle](l03){$\ty_0$};
        \draw (\z,\p+6*\h) node[draw,circle](l13){$\ty_1$};
        \draw (\z,\p+5*\h) node[draw,circle](l23){$\ty_2$};
        \draw (\z,\p+4*\h) node[draw,circle](l33){$\ty_3$};
        \draw (\z,\p+3*\h) node[draw,circle](l43){$\ty_4$};
        \draw (\z,\p+2*\h) node[draw,circle](l53){$\ty_5$};
        \draw (\z,\p+1*\h) node[draw,circle](l63){$\ty_6$};
        \draw (\z,\p+0*\h) node[draw,circle](l73){$\ty_7$};
        % first layer 
        \draw[->] (l00) -- (l41); 
        \draw[->] (l10) -- (l51); 
        \draw[->] (l20) -- (l61); 
        \draw[->] (l30) -- (l71); 
        \draw[->] (l40) -- (l01); 
        \draw[->] (l50) -- (l11); 
        \draw[->] (l60) -- (l21); 
        \draw[->] (l70) -- (l31);
        \draw (\w/2,\p+2*\h) node[draw,circle,fill=black,text=white]{};
        \draw (\w/2,\p+3*\h) node[draw,circle,fill=black,text=white]{};
        \draw (\w/2,\p+4*\h) node[draw,circle,fill=black,text=white]{};
        \draw (\w/2,\p+5*\h) node[draw,circle,fill=black,text=white]{};
        % second layer 
        \draw[->] (l01) -- (l22); 
        \draw[->] (l11) -- (l32); 
        \draw[->] (l21) -- (l02); 
        \draw[->] (l31) -- (l12); 
        \draw[->] (l41) -- (l62); 
        \draw[->] (l51) -- (l72); 
        \draw[->] (l61) -- (l42); 
        \draw[->] (l71) -- (l52);
        \draw (\w/2+\y/2,\p+1*\h) node[draw,circle,fill=black,text=white]{};
        \draw (\w/2+\y/2,\p+2*\h) node[draw,circle,fill=black,text=white]{};
        \draw (\w/2+\y/2,\p+5*\h) node[draw,circle,fill=black,text=white]{};
        \draw (\w/2+\y/2,\p+6*\h) node[draw,circle,fill=black,text=white]{};
        % third layer 
        \draw[->] (l02) -- (l13); 
        \draw[->] (l12) -- (l03); 
        \draw[->] (l22) -- (l33); 
        \draw[->] (l32) -- (l23); 
        \draw[->] (l42) -- (l53); 
        \draw[->] (l52) -- (l43); 
        \draw[->] (l62) -- (l73); 
        \draw[->] (l72) -- (l63);
        \draw (\y/2+\z/2,\p+0.5*\h) node[draw,circle,fill=black,text=white]{};
        \draw (\y/2+\z/2,\p+2.5*\h) node[draw,circle,fill=black,text=white]{};
        \draw (\y/2+\z/2,\p+4.5*\h) node[draw,circle,fill=black,text=white]{};
        \draw (\y/2+\z/2,\p+6.5*\h) node[draw,circle,fill=black,text=white]{};
        % combination diagram 
        \draw (0,3*\h) node[draw,circle](wit){$a$};
        \draw (0,0) node[draw,circle](wib){$b$};
        \draw (\y/3+2*\z/3,3*\h) node[draw,circle,minimum size=2cm](wot){$\frac{a+b}{\sqrt{2}}$};
        \draw (\y/3+2*\z/3,0) node[draw,circle,minimum size=2cm](wob){$\frac{a-b}{\sqrt{2}}$};
        \draw[->] (wit) -- (wob); 
        \draw[->] (wib) -- (wot); 
        \draw (.45*\z,1.5*\h) node[draw,circle,fill=black,text=white]{};
    \end{tikzpicture}
    \end{subfigure}
    \caption{Schematics of the fast Fourier transform in bit-reversed order (FFTBR) and fast Walsh--Hadamard transform (FWHT). FFTBR is performed via an FFT decimation in time algorithm without the initial reordering of inputs. The inverse FFTBR (IFFTBR) may be performed by propagating $\{\ty_i\}_{i=0}^{2^m-1}$ to $\{y_i\}_{i=0}^{2^m-1}$ in the other direction through the FFTBR algorithm.}
    \label{fig:fast_transforms}
\end{figure}  

\Section{Kernel Matrix Structures} \label{sec:kernel_matrix_structure}

This section will discuss the nice structures which appear in the kernel matrix of pairwise kernel evaluations at two sets of points when special kernels and point sets are used. The structures we present are generalizations of structures studied elsewhere in the literature where we additionally support kernel matrices formed from low-discrepancy sequences of potentially different sizes and with potentially different randomizations. This will later be exploited in \Cref{sec:fast_gps} for fast Gaussian process regression and in \Cref{sec:fmtgps} fast multitask Gaussian process regression where kernel matrices with different sequences of different sizes will be used for each task. 

Throughout this section, let $\mI_r$ denote the size $2^r \times 2^r$ identity matrix for $r \in \bbN_0$. We will begin with a condition characterizing the fast transform matrices which we will later see are eigenvectors of nicely structured kernel matrices.

\begin{condition} \label{cond:fast_transform}
    A sequence of matrices $\left(\mV_m\right)_{m \in \bbN_0}$ satisfy this condition if for every $m \in \bbN_0$:
    \begin{enumerate}
        \item $\mV_m \in \bbC^{2^m \times 2^m}$.
        \item $\mV_m$ is symmetric, i.e., $\mV_m^\intercal = \mV_m$, and unitary, i.e., $\overline{\mV_m} \mV_m = \mI_m$.
        \item The zeroth column of $\mV_m$ is the constant vector $1/\sqrt{2^m}$.
        \item Computing $\mV_m \by$ costs $\calO(2^m m)$ for any $\by \in \bbC^{2^m}$. 
        \item For some $\bomega_m \in \bbR^{2^m}$, 
        $$\overline{\mV_{m+1}} = \frac{1}{\sqrt{2}}\begin{pmatrix} \overline{\mV_m} & \diag(\bw_m) \overline{\mV_m} \\ \overline{\mV_m} & - \diag(\bw_m) \overline{\mV_m} \end{pmatrix}.$$
    \end{enumerate}
\end{condition}

When point sets of different sizes are used, the resulting kernel matrices have blocks with nice structure. Here we prove a useful property for such kernel matrices.

\begin{lemma} \label{lemma:remove_krons}
    For $(\mV_m)_{m \in \bbN_0}$ satisfying \Cref{cond:fast_transform} and any $p,q \in \bbN_0$:
    \begin{enumerate}
        \item If $p \geq q$ and 
        $$\mK := (\mI_{p-q} \otimes \mV_q) \mLambda \overline{\mV_q} \in \bbR^{2^p \times 2^q}$$
        for some $2^{p-q} \times 1$ block matrix $\mLambda$ with $2^q \times 2^q$ diagonal blocks, then 
        $$\mK = \mV_p \tmLambda \overline{\mV_q}$$
        for some different $2^{p-q} \times 1$ block matrix $\tmLambda$ with $2^q \times 2^q$ diagonal blocks. 
        \item If $p \leq q$ and 
        $$\mK := \mV_p \mLambda (\mI_{q-p} \otimes \overline{\mV_q}) \in \bbR^{2^p \times 2^q}$$
        for some $1 \times 2^{q-p}$ block matrix $\mLambda$ with $2^p \times 2^p$ diagonal blocks, then 
        $$\mK = \mV_p \tmLambda \overline{\mV_q}$$
        for some different $1 \times 2^{q-p}$ block matrix $\tmLambda$ with $2^p \times 2^p$ diagonal blocks.
    \end{enumerate}  
\end{lemma}
\begin{proof}
    Without loss of generality, assume $p \geq q$. The result follows if we can show $\overline{\mV_p} (\mI_{p-q} \otimes \mV_q)$ is a $2^{p-q} \times 2^{p-q}$ block matrix with $2^q \times 2^q$ diagonal blocks as this would imply 
    $$\mK = \mV_p \overline{\mV_p} (\mI_{p-q} \otimes \mV_q) \mLambda \overline{\mV_q} = \mV_p \tmLambda \overline{\mV_q}$$
    for some different $2^{p-q} \times 1$ block matrix $\tmLambda$ with $2^q \times 2^q$ diagonal blocks. We show this by induction. This clearly holds when $p=q$. Suppose this holds for all $p \in \{q,\dots,m\}$ with $m>q$, and denote the $2^{m-q} \times 2^{m-q}$ block matrix with $2^q \times 2^q$ diagonal blocks by $\mD = \overline{\mV_m} (\mI_{m-q} \otimes \mV_q)$. Then \Cref{cond:fast_transform} implies 
    \begin{align*}
        \overline{\mV_{m+1}} (\mI_{m+1-q} \otimes \mV_q) &= \frac{1}{\sqrt{2}} \begin{pmatrix} \overline{\mV_m} & \diag(\bw_m) \overline{\mV_m} \\ \overline{\mV_m} & - \diag(\bw_m) \overline{\mV_m} \end{pmatrix} \begin{pmatrix} \mI_{m-q} \otimes \mV_q & \\ & \mI_{m-q} \otimes \mV_q \end{pmatrix} \\
        &= \frac{1}{\sqrt{2}}  \begin{pmatrix} \mD & \diag(\bomega_m) \mD \\ \mD & - \diag(\bomega_m) \mD \end{pmatrix}
    \end{align*}
    which is a $2^{m+1-q} \times 2^{m+1-q}$ block matrix with $2^q \times 2^q$ diagonal blocks, so  the result follows.
\end{proof} 

Now, let us characterize a special collection of two point sets, a kernel, and a set of fast transform matrices. 

\begin{condition} \label{cond:special_pairing}
    A pair of sequences $(\bx_i)_{i \in \bbN_0},(\bx_i')_{i \in \bbN_0} \subset \calX$, a SPSD kernel $K: \calX \times \calX \to \bbR$, and a set of matrices $\left(\mV_m\right)_{m \in \bbN_0}$ satisfy this condition if \Cref{cond:fast_transform} holds, and for any $p,q,r,c \in \bbN_0$ it holds that  
    $$\left(K(\bx_i,\bx_k')\right)_{i=r2^p,k=c2^q}^{(r+1)2^p,(c+1)2^q} = \mV_p \mLambda \overline{\mV_q} \in \bbR^{2^p \times 2^q}$$
    for some $2^{p-\min(p,q)} \times 2^{q-\min(p,q)}$ block matrix $\mLambda$ with $2^{\min(p,q)} \times 2^{\min(p,q)}$ diagonal blocks. Note that $\mLambda$ implicitly depends on $p$, $q$, $r$, and $c$ while $\mV_p$ only depends on $p$. 
\end{condition}

The following theorem shows how to exploit such collections for fast kernel computations. 

\begin{theorem} \label{thm:fast_eigenvalue_computation} 
    Under \Cref{cond:special_pairing}, suppose
    $$\mK = \mV_p \mLambda \overline{\mV_q} \in \bbR^{2^p \times 2^q}$$
    for some $2^{p-\min(p,q)} \times 2^{q-\min(p,q)}$ block matrix $\mLambda$ with $2^{\min(p,q)} \times 2^{\min(p,q)}$ diagonal blocks. Then $\mK$ is fully specified by the $2^{\max(p,q)}$ entries of $\mLambda$ across all diagonal blocks. Moreover, these entries $\blambda \in \bbC^{2^{\max(p,q)}}$ can be computed at $\calO(2^{\max(p,q)} \max(p,q))$ cost with only $\calO(2^{\max(p,q)})$ storage using the following formulas. 
    \begin{enumerate}
        \item If $p \geq q$, then $\blambda = \sqrt{2^q} \; \overline{\mV_p} \mK_{:,0}$.
        \item If $q \geq p$, then $\blambda = \sqrt{2^p} \; \mV_q \mK_{0,:}^\intercal$.
    \end{enumerate} 
    Here $\mK_{:,0}$ and $\mK_{0,:}$ are the zeroth column and row of $\mK$ respectively.
\end{theorem}
\begin{proof}
    If $p \geq q$,
    $$\blambda = \mLambda \bone_q = \sqrt{2^q} \;\overline{\mV_p} \mV_p \mLambda \overline{\bv_1} = \sqrt{2^q} \;\overline{\mV_p} \mK_{:,0}.$$
    If $q \geq p$, then 
    $$\blambda^\intercal = (\bone_p \mLambda)^\intercal = \sqrt{2^p} (\bv_1^\intercal \mLambda \overline{\mV_q} \mV_q)^\intercal  = \sqrt{2^p} (\mK_{0,:} \mV_q)^\intercal = \sqrt{2^p} \mV_q \mK_{0,:}^\intercal.$$
\end{proof}

\begin{corollary} \label{corr:fast_gram_matrix_comps}
    Under \Cref{cond:special_pairing}, if $\mK = \mV_m \mLambda \mV_m \in \bbR^{2^m \times 2^m}$ then one can compute $\blambda = \sqrt{2^m}\; \overline{\mV_m} \mK_{:,0}$ at $\calO(2^m m)$ cost and with $\calO(2^m)$ storage. One may also perform a matrix-vector multiplication, solve a linear system (assuming $\mK$ SPD), and compute the determinant of $\mK$ at $\calO(2^mm)$ cost using 
    \begin{equation*}
        \mK \by = \mV_m(\tby \odot \blambda), \qquad \mK^{-1} \by = \mV_m (\tby ./ \blambda), \qquad \lvert \mK \rvert = \prod_{i=0}^{2^m-1} \lambda_i
    \end{equation*}
    where $\tby = \mV_m \tby$,  $\odot$ denotes the elementwise (Hadamard) product, and $./$ denotes the elementwise division.
\end{corollary}

\Subsection{SI Kernels, Rank-1 Lattices, and the Fast Fourier Transform (FFT)} \label{sec:fft_si_kernels_r1lattices}

\begin{definition} \label{def:shift_invar}
    A function $K: [0,1)^d \times [0,1)^d \to \bbR$ is shift-invariant (SI) if for any $\bx,\bx' \in [0,1)^d$
    $$K(\bx,\bx') = Q(\bx - \bx' \bmod 1)$$
    for some $Q: [0,1)^d \to \bbR$. 
\end{definition}

Suppose $n=2^m$ for some $m \geq 0$ and let $W_m = \exp(-2 \pi \sqrt{-1}/2^m)$. Then we will denote the scaled Fourier matrix by $\overline{\mF_m} = \{W_m^{ij}/\sqrt{2^m}\}_{i,j=0}^{2^m-1}$ which satisfies $\mF_m \overline{\mF_m} = \mI_m$. We will call $\overline{\mF_m}$ and $\mF_m$ the DFT (discrete Fourier transform) and IDFT (inverse DFT) matrices respectively. Multiplication by $\overline{\mF_m}$ and $\mF_m$ can each be performed at $\calO(2^m m) = \calO(n \log n)$ cost using the FFT (fast Fourier transform) and IFFT (inverse FFT) respectively \cite{cooley1965algorithm}. This section will use shifted rank-1 lattices which were detailed in \Cref{sec:lattices}. 

\begin{lemma} \label{lemma:si_lattice_linear_circulant}
    The following holds for any shifted rank-1 lattices in linear order \\ $(\bx_i)_{i \in \bbN_0} := ((\bg i/2^m+\bDelta) \bmod 1)_{i \in \bbN_0}$ and $(\bx_i')_{i \in \bbN_0} := ((\bg i/2^m+\bDelta') \bmod 1)_{i \in \bbN_0}$  with common generating vector $\bg \in \bbN^d$ and any shifts $\bDelta,\bDelta' \in [0,1)^d$. When $K$ is a SPSD SI kernel, the following eigendecomposition holds
    $$(K(\bx_i,\bx_k'))_{i,k=0}^{2^m-1} = \mF_m \mLambda \overline{\mF_m}.$$
\end{lemma} 
\begin{proof} 
    For this proof only, denote $\{\bx-\bx'\}_p := (\bx-\bx') \bmod p$. Then with $n=2^m$, 
    \begin{align*}
        \{\bx_i - \bx_k'\}_1 &= \{\{\bg i/n+\bDelta\}_1-\{\bg k/n+\bDelta'\}_1\}_1 \\
        & = \{\{\bg\{i-k\}_n/n + \bDelta\}_1-\{\bg 0/n+\bDelta'\}_1\}_1 = \{\bx_{\{i-k\}_n} - \bx_0'\}_1.
    \end{align*}
    Since $K$ is SI, we see $K(\bx_i,\bx_k') = K(\bx_{\{i-k\}_n},\bx_0')$, so the matrix is circulant. The eigendecomposition of a circulant matrix in terms of the DFT and IDFT matrices is well known, see for example \cite{gray2006toeplitz}.
\end{proof}

\Cref{fig:idea_points_kernels_structures} visualizes lattices in linear order, SI kernels, and the resulting circulant kernel matrices. Let $R_m(i)$ reverse the first $m$ bits of $0 \leq i < 2^m$ in base $b=2$, i.e., if $i=\sum_{k=0}^{m-1} \mi_k 2^k$ then $R(i) = \sum_{k=0}^{m-1} \mi_{m-k-1} 2^k$. Notice that $v(i)=R_m(i)/n$ where $v(i)$ is the van der Corput sequence defined in \eqref{eq:v}. Let $\mP_m$ be a permutation matrix with the property that $\mP_m^\intercal (y_i)_{i=0}^{2^m-1} = (y_{R(i)})_{i=0}^{2^m-1}$ bit-reverses the ordering. Let us denote $\overline{\mV_m} := \overline{\mF_m} \mP_m^\intercal$ and $\mV_m = \mP_m \mF_m$, which we will call the DFTBR (DFT in bit-reversed order) and inverse IDFTBR (inverse DFTBR) matrices respectively. Multiplication by $\overline{\mV_m}$ may be done by performing a bit-reversal at $\calO(2^m)$ cost and then applying the FFT algorithm at costs $\calO(2^m m)$. This algorithm, which we call the FFTBR (FFT bit-reversed), is visualized in \Cref{fig:fast_transforms}. Similarly, multiplication by $\mV_m$ may be done by applying the IFFT algorithm at $\calO(2^m m)$ cost and then performing a bit-reversal at $\calO(2^m)$ cost. We call this algorithm the IFFTBR (IFFT bit-reversed). Better yet, the FFTBR algorithm is equivalent to a decimation-in-time FFT without the initial bit-reversal step, and likewise for the IFFTBR. 

\begin{lemma}
    The DFTBR matrices $\left(\mV_m\right)_{m \in \bbN_0}$ satisfy \Cref{cond:fast_transform} with $\bw_m = \left(\exp(-\pi \sqrt{-1}i/2^m)\right)_{i=0}^{2^m-1}$.
\end{lemma}
\begin{proof}
    The first four properties immediately follow from the prior definitions and discussion. For the last property, notice that $\overline{\mV_{m+1}} = \left(W_{m+1}^{i R_{m+1}(j)}/\sqrt{2^{m+1}}\right)_{i,j=0}^{2^{m+1}-1}$. For $0 \leq j < 2^m$ we have $R_{m+1}(j) = 2 R_m(j)$ so $W^{i R_{m+1}(j)}_{m+1}= W^{i R_m(j)}_m$. For $2^m \leq j < 2^{m+1}$ we have $R_{m+1}(j) = 2R_m(j-2^m)+1$ so $W^{i R_{m+1}(j)}_{m+1} = W_m^{i R(j-2^m)} W_{m+1}^i$. For $0 \leq i < 2^m$ we have $W_{m+1}^{2^m+i} = -W_{m+1}^i$. Notice that $\bw_{m+1}$ is the $\alpha=2$ interlacing of $\bw_m$ and $W_{m+1} \bw_m$. 
\end{proof}

\begin{lemma} \label{lemma:si_gram_mat_struct}
    The following holds for any shifted rank-1 lattices in radical inverse order $(\bx_i)_{i \in \bbN_0} := ((\bg v(i)+\bDelta) \bmod 1)_{i \in \bbN_0}$ and $(\bx_i')_{i \in \bbN_0} := ((\bg v(i)+\bDelta') \bmod 1)_{i \in \bbN_0}$ with common generating vector $\bg \in \bbN^d$ and any shifts $\bDelta,\bDelta' \in [0,1)^d$. When $K$ is a SPSD SI kernel we have the following eigendecomposition
    $$(K(\bx_i,\bx_k'))_{i,k=0}^{2^m-1} = \mV_m \mLambda \overline{\mV_m}.$$
\end{lemma}
\begin{proof} 
    \Cref{lemma:si_lattice_linear_circulant} and the definition of $\mP_m$ directly imply $\mP_m^\intercal \mK \mP_m = \mF_m \mLambda \overline{\mF_m}$ from which the theorem follows.  
\end{proof} 

\begin{lemma} \label{lemma:lattice_si_blocked_mat_kron}
    The following holds for any shifted rank-1 lattices in radical inverse order $(\bx_i)_{i \in \bbN_0} := ((\bg v(i)+\bDelta) \bmod 1)_{i \in \bbN_0}$ and $(\bx_i')_{i \in \bbN_0} := ((\bg v(i)+\bDelta') \bmod 1)_{i \in \bbN_0}$  with common generating vector $\bg \in \bbN^d$ and any shifts $\bDelta,\bDelta' \in [0,1)^d$. For any $m,r,c \in \bbN_0$, when $K$ is a SPSD SI kernel we have the following eigendecomposition
    $$(K(\bx_i,\bx_k'))_{i=r2^m,k=c2^m}^{(r+1)2^m-1,(c+1)2^m-1} = \mV_m \mLambda \overline{\mV_m}$$
\end{lemma}
\begin{proof} 
    For $r 2^m \leq i < (r+1)2^m$ we have $v(i) = v(i-r2^m) + v(r2^m)$, so $(\bx_i)_{i=r2^m}^{(r+1)2^m-1} = ((\bg v(i)+\hbDelta) \bmod 1)_{i=0}^{2^m-1}$ where $\hbDelta = (\bDelta + \bg v(r2^m)) \bmod 1$, i.e., $(\bx_i)_{i=r2^m}^{(r+1)2^m-1}$ is $(\bx_i)_{i=0}^{2^m-1}$ after a shift by $\hbDelta$ and modulo 1. Similarly, for $c 2^m \leq i < (c+1)2^m$ we have $v(i) = v(i-c2^m) + v(c2^m)$, so $(\bx_i')_{i=c2^m}^{(c+1)2^m-1} = ((\bg v(i)+\tbDelta) \bmod 1)_{i=0}^{2^m-1}$ where $\tbDelta = (\bDelta' + \bg v(c2^m)) \bmod 1$, i.e., $(\bx_i')_{i=c2^m}^{(c+1)2^m-1}$ is $(\bx_i')_{i=0}^{2^m-1}$ after a shift by $\tbDelta$ and modulo 1. The result then follows from \Cref{lemma:si_gram_mat_struct}.   
\end{proof} 

\begin{theorem} \label{thm:lattice_si_special_pairing}
    \Cref{cond:special_pairing} is satisfied for DFTBR matrices $\left(\mV_m\right)_{m \in \bbN_0}$, a SPSD SI kernel $K: [0,1)^d \times [0,1)^d \to \bbR$, and any shifted rank-1 lattices in radical inverse order $(\bx_i)_{i \in \bbN_0} := ((\bg v(i)+\bDelta) \bmod 1)_{i \in \bbN_0}$ and $(\bx_i') := ((\bg v(i)+\bDelta') \bmod 1)_{i \in \bbN_0}$  with common generating vector $\bg \in \bbN^d$ and any shifts $\bDelta,\bDelta' \in [0,1)^d$.
\end{theorem}
\begin{proof} 
    This is an immediate consequence of combining  \Cref{lemma:lattice_si_blocked_mat_kron} with \Cref{lemma:remove_krons}.
\end{proof}

\Subsection{DSI Kernels, Digital Nets, and the Fast Walsh--Hadamard Transform (FWHT)} \label{sec:fwht_dsi_kernels_dnb2s}

\begin{definition}
    For a fixed prime base $b$, denote the digit-wise addition and subtraction of $x,x' \in [0,1)$ by 
    $$x \oplus x' = \sum_{\ell \geq 0} ((\mx_\ell + \mx'_\ell) \bmod b) 2^{-\ell-1} \qqtqq{and} x \ominus y = \sum_{\ell \geq 0} ((\mx_\ell - \mx'_\ell) \bmod b) 2^{-\ell-1}$$
    where $x = \sum_{\ell \geq 0} \mx_\ell 2^\ell$ and $x' = \sum_{\ell \geq 0} \mx_\ell' 2^\ell$. When applied to vectors, these operations are applied elementwise, i.e., for any $\bx,\bx' \in [0,1)^d$ 
    $$\bx \oplus \bx' = (x_1 \oplus x_1', \dots, x_d \oplus x_d') \qqtqq{and} \bx \ominus \bx' = (x_1 \ominus x_1', \dots, x_d \ominus x_d').$$
\end{definition}

We technically require the base $b$ expansions of $x,x' \in [0,1)$ do not end in infinite trails of the digit $b-1$. This will be assumed throughout, and in not an issue for our finite precision implementation. 

\begin{definition} \label{def:dig_shift_invar}
    A function $K:[0,1)^d \times [0,1)^d \to \bbR$ is digitally-shift-invariant (DSI) if for any $\bx,\bx' \in [0,1)^d$
    $$K(\bx,\bx') = Q(\bx \ominus \bx')$$
    for some $Q: [0,1)^d \to \bbR$. 
\end{definition}

Let $n=2^m$ where $m \geq 0$. The $2^m \times 2^m$ scaled Hadamard matrix $\mV_m$ may be defined recursively starting from $\mV_0 = (1)$ and using the relation   
$$\mV_{m+1} = \frac{1}{\sqrt{2}} \begin{pmatrix} \mV_m, & \mV_m \\ \mV_m & -\mV_m \end{pmatrix}.$$
Multiplication by the scaled Hadamard matrix $\mV_m$ can be performed in $\calO(2^m m)$ computations using the fast Walsh--Hadamard transform (FWHT) \cite{fino.fwht} visualized in \Cref{fig:fast_transforms}. 

\begin{lemma} 
    The real-valued scaled Hadamard matrices $(V_m)_{m \in \bbN_0}$ satisfy \\ \Cref{cond:fast_transform} with $\bomega_m = \bone_m$ the constant length $2^m$ vector of ones.
\end{lemma}
\begin{proof}
    These properties immediately follow from the prior definitions and discussion.
\end{proof}

This section will use the digitally-shifted digital nets discussed in \Cref{sec:dnets}. Here we will only consider base $b=2$ digital nets for which digit-wise addition and subtraction are equivalent and may be computed quickly using efficient XOR (exclusive or) operations. The theory should readily extend to digital nets with arbitrary prime base $b$.

Let us denote the generating matrices by $\mG \in [0,1)^{d \times \infty}$, where the generating matrix $\mC_j \in \{0,1\}^{t_{\max} \times m}$ for $j \in \{1,\dots,d\}$ presented in \Cref{sec:dnets} is now identified with the $j^\mathrm{th}$ row of $\mG$ by letting $m \to \infty$ and writing each column of $\mC_j$ in a floating point representation. We will denote the $\ell^\mathrm{th}$ column of $\mG$ by $\bg_\ell$ for $\ell \in \bbN_0$, e.g., we denote the zeroth column of $\mG$ by $\bg_0$. 

Note that $\mG$ may be the result of applying linear matrix scrambling (LMS) to some original generating matrices, but the following theory does not permit different LMSs for each sequence, i.e., the same $\mG$ must be used for each sequence. Higher-order digital nets and higher-order scrambling with LMS is permitted under the same restrictions. Nested uniform scrambling (NUS) is not permitted.

Recall that we write the binary expansion of $i \in \bbN_0$ as $i = \sum_{\ell \in \bbN_0} \mi_\ell 2^\ell$.
\begin{lemma} \label{lemma:dsi_gram_mat_struct}
    The following holds for any digitally-shifted base $2$ digital nets in radical inverse order $(\bx_i)_{i \in \bbN_0} := (\oplus_{\ell \in \bbN_0} \mi_\ell \bg_\ell \oplus \bDelta)_{i \in \bbN_0}$ and $(\bx_i')_{i \in \bbN_0} := (\oplus_{\ell \in \bbN_0} \mi_\ell \bg_\ell \oplus \bDelta')_{i \in \bbN_0}$ with common generating matrices $\mG \in [0,1)^{d \times \infty}$ and any shifts $\bDelta,\bDelta' \in [0,1)^d$. When $K$ is a SPSD DSI kernel we have the following eigendecomposition
    \begin{equation}
        (K(\bx_i,\bx_k'))_{i,k=0}^{2^m-1} = \mV_m \mLambda \mV_m.
        \label{eq:RSBT}
    \end{equation}
\end{lemma}
\begin{proof} 
    This clearly holds for $m=0$. Assume this holds for $m \in \{0,\dots,j\}$ with $j>0$, and let us denote
    $$\mK := (K(\bx_i,\bx_k'))_{i,k=0}^{2^{j+1}-1} = \begin{pmatrix} \mK_{00} & \mK_{01} \\ \mK_{10} & \mK_{11} \end{pmatrix}.$$
    The eigendecomposition $\mK_{00} = \mV_j \mLambda_{00} \mV_j$ holds by assumption. Notice that 
    \begin{align*}
        (\oplus_{\ell \in \bbN_0} \mi_\ell \bg_\ell \oplus \bDelta)_{i=2^j}^{2^{j+1}-1} &= (\oplus_{\ell \in \bbN_0} \mi_\ell \bg_\ell \oplus \hbDelta)_{i=0}^{2^j-1}, \qquad \hbDelta = \bDelta \oplus \bg_j, \qquad\mathrm{and} \\
        (\oplus_{\ell \in \bbN_0} \mi_\ell \bg_\ell \oplus \bDelta')_{i=2^j}^{2^{j+1}-1} &= (\oplus_{\ell \in \bbN_0} \mi_\ell \bg_\ell \oplus \tbDelta)_{i=0}^{2^j-1}, \qquad \tbDelta = \bDelta' \oplus \bg_j. 
    \end{align*}
    Then the assumption and the digitally-shift-invariant property imply that $\mK_{01} = \mK_{10} = \mV_j \mLambda_{01} \mV_j$ and $\mK_{11} = \mK_{00}$. Therefore, 
    \begin{align*}
        \mV_{j+1} \mK \mV_{j+1} &= \frac{1}{2} \begin{pmatrix} \mV_j, & \mV_j \\ \mV_j & -\mV_j \end{pmatrix} \begin{pmatrix} \mV_j \mLambda_{00} \mV_j, & \mV_j \mLambda_{01} \mV_j \\ \mV_j \mLambda_{01} \mV_j & \mV_j \mLambda_{00} \mV_j \end{pmatrix} \begin{pmatrix} \mV_j, & \mV_j \\ \mV_j & -\mV_j \end{pmatrix} \\
        &= \begin{pmatrix} \mLambda_{00} + \mLambda_{10} & \\ & \mLambda_{00} - \mLambda_{01} \end{pmatrix},
    \end{align*}
    which completes the proof by induction. A more detailed discussion and proof can be found in \cite[Section 5.3.2]{rathinavel.bayesian_QMC_thesis}.  
\end{proof} 

Kernel matrices taking the form of \eqref{eq:RSBT} are said to be RSBT (recursive symmetric block Toeplitz). \Cref{fig:idea_points_kernels_structures} visualizes base 2 digital nets in radical inverse order, DSI kernels, and the resulting RSBT kernel matrices. 

\begin{lemma} \label{lemma:dnet_dsi_blocked_mat_kron}
    The following holds for any digitally-shifted base $2$ digital nets in radical inverse order $(\bx_i)_{i \in \bbN_0} := (\oplus_{\ell \in \bbN_0} \mi_\ell \bg_\ell \oplus \bDelta)_{i \in \bbN_0}$ and $(\bx_i')_{i \in \bbN_0} := (\oplus_{\ell \in \bbN_0} \mi_\ell \bg_\ell \oplus \bDelta')_{i \in \bbN_0}$ with common generating matrices $\mG \in [0,1)^{d \times \infty}$ and any shifts $\bDelta,\bDelta' \in [0,1)^d$. For any $m,r,c \in \bbN_0$, when $K$ is a SPSD DSI kernel we have the following eigendecomposition
    $$(K(\bx_i,\bx_k'))_{i=r2^m,k=c2^m}^{(r+1)2^m-1,(c+1)2^m-1} = \mV_m \mLambda \overline{\mV_m}$$
\end{lemma}
\begin{proof} 
    We have 
    \begin{align*}
        (\oplus_{\ell \in \bbN_0} \mi_\ell \bg_\ell \oplus \bDelta)_{i=r2^m}^{(r+1)2^m-1} &= (\oplus_{\ell \in \bbN_0} \mi_\ell \bg_\ell \oplus \hbDelta)_{i=0}^{2^m-1}, \quad  \hbDelta = \oplus_{\ell \in \bbN_0} \mj_\ell \bg_\ell \oplus \bDelta, \quad j = r2^m, \\
        (\oplus_{\ell \in \bbN_0} \mi_\ell \bg_\ell \oplus \bDelta')_{i=c2^m}^{(c+1)2^m-1} &= (\oplus_{\ell \in \bbN_0} \mi_\ell \bg_\ell \oplus \tbDelta)_{i=0}^{2^m-1}, \quad  \tbDelta = \oplus_{\ell \in \bbN_0} \mj_\ell \bg_\ell \oplus \bDelta', \quad j = c2^m.
    \end{align*}
    The result then follows from \Cref{lemma:dsi_gram_mat_struct}.
\end{proof} 

\begin{theorem} \label{thm:dnet_dsi_special_pairing}
    \Cref{cond:special_pairing} is satisfied for scaled Hadamard matrices $\left(\mV_m\right)_{m \in \bbN_0}$, a SPSD DSI kernel $K: [0,1)^d \times [0,1)^d \to \bbR$, and any digitally-shifted base $2$ digital nets in radical inverse order  $(\bx_i)_{i \in \bbN_0} := (\oplus_{\ell \in \bbN_0} \mi_\ell \bg_\ell \oplus \bDelta)_{i \in \bbN_0}$ and $(\bx_i')_{i \in \bbN_0} := (\oplus_{\ell \in \bbN_0} \mi_\ell \bg_\ell \oplus \bDelta')_{i \in \bbN_0}$ with common generating matrices $\mG \in [0,1)^{d \times \infty}$ and any shifts $\bDelta,\bDelta' \in [0,1)^d$.
\end{theorem}
\begin{proof} 
    This is an immediate consequence of combining \Cref{lemma:dnet_dsi_blocked_mat_kron} with \Cref{lemma:remove_krons}.
\end{proof} 

\begin{figure}[!ht]
    \centering
    \includegraphics[width=1\textwidth]{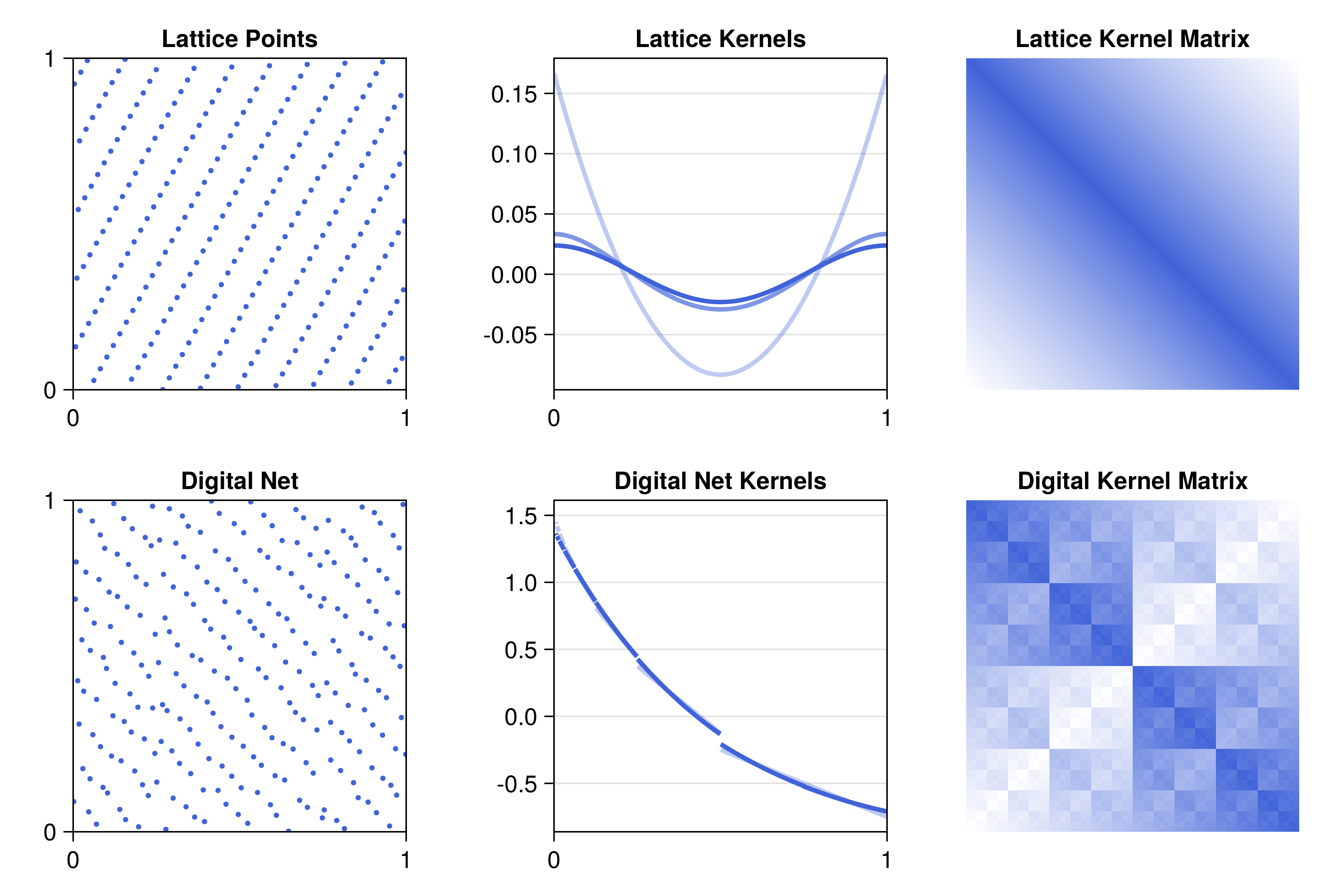} 
    \caption{Lattice points in linear order paired with shift-invariant (SI) kernels produce circulant kernel (Gram) matrices. Base $2$ digital nets in radical inverse order paired with digitally-shift-invariant (DSI) kernels produce RSBT (recursive symmetric block Toeplitz) Gram matrices. Circulant and RSBT matrices have eigendecompositions with nicely structured eigenvector matrices which enable fast Gram matrix computations.}
    \label{fig:idea_points_kernels_structures}
\end{figure}

\Section{Kernel Forms} \label{sec:kernel_forms}

Here we will discuss how the special multivariate kernels we use for fast kernel computations are built from univariate kernel forms. The resulting mixture-product kernels and product kernels have been considered elsewhere in the literature. We provide code examples using these multivariate kernels for the fast kernel matrix computations discussed in \Cref{sec:kernel_matrix_structure}. Univariate SI kernels are discussed in \Cref{sec:si_kernels} and DSI kernels in \Cref{sec:dsi_kernels}.

\Subsection{Multivarite Kernel Decompositions} 

In this section we will discuss kernels of varying smoothness which are SI (\Cref{def:shift_invar}) or DSI (\Cref{def:dig_shift_invar}). Let us motivate our kernels forms with the following functional ANOVA decomposition \cite[Appendix A]{owen.mc_book_practical}. 

\begin{remark}
    If $\bbV[f(\bX)] < \infty$ for $\bX \sim \calU([0,1]^d)$, then 
    $$f(\bx) = \sum_{\fu \subseteq \{1,\dots,d\}} f_\fu(\bx), \quad f_\fu(\bx) = \int_{[0,1]^d} f(\bx) \D \bx_{- \fu} - \sum_{\fv \subsetneq \fu} f_\fv(\bx), \quad f_\emptyset = \int_{[0,1]^d} f(\bx) \D \bx$$
    where $f_\fu$ only depends on $\bx_\fu := (x_j)_{j \in \fu}$, not $\bx_{-\fu} := (x_j)_{j \notin \fu}$.  
\end{remark}

Recall sums and products of kernels are still kernels \cite{aronszajn.theory_of_reproducing_kernels}, see also \cite[Chapter 2]{duvenaud2014automatic} for a discussion on expressing structure through combining kernels. Our SI and DSI kernels will be products and sums of one-dimensional SI and DSI kernels. Specifically, for some compatible smoothness parameters $\balpha \in \bbR_{>0}$, some weights $(\tEta_\fu)_{\fu \subseteq \{1,\dots,d\}} \subset \bbR_{\geq 0}$, and a global scaling parameter $\gamma \in \bbR_{>0}$, we will consider multivariate \emph{mixture-product kernels} of the form 
\begin{equation}
    K(\bx,\bx') = \gamma \sum_{\fu \subseteq \{1,\dots,d\}} \tEta_\fu \prod_{j \in \fu} K_{\alpha_j}(x_j,x_j')
    \label{eq:mixture_product_kernel}
\end{equation}
where the product over the empty set is taken to be $1$, and we require $\tEta_\emptyset = 1$. 

\emph{Product kernels} are a special case of mixture-product kernels \eqref{eq:mixture_product_kernel} where $\tEta_\fu = \prod_{j \in \fu} \eta_j$ for some per-dimension weights $\{\eta_j\}_{j=1}^d$  (which we may also call lengthscales) so that
\begin{equation}
    K(\bx,\bx') = \gamma \prod_{j=1}^d \left[1+\eta_j K_{\alpha_j}(x_j,x_j')\right].
    \label{eq:product_kernels}
\end{equation}
Product kernels, while less expressive, only cost $\calO(d)$ to evaluate compared to the $\calO(2^d)$ cost for mixture-product kernels, assuming the one-dimensional kernels cost $\calO(1)$ to evaluate. \cite{kaarnioja.kernel_interpolants_lattice_rkhs_serendipitous} proposes serendipitous weights, a special form of product weights, which showed unexpectedly strong performance for fast lattice based kernel interpolation. 

The next two lemmas will consider derivatives and integrals of mixture-product and product kernels respectively. Such derivatives and integrals will appear later when considering derivative-informed Gaussian process regression (\Cref{sec:gps_deriv_informeds}) and Bayesian cubature (\Cref{sec:mtgp_bayesian_cubature} and \Cref{sec:fmtgp_bayesian_cubature}). For $\bbeta,\bbeta' \in \bbN_0^d$ we will denote the support by $S(\bbeta+\bbeta') = \{j \in \{1,\dots,d\}: \beta_j+\beta_j' > 0\}$ with cardinality $\lvert S(\bbeta+\bbeta') \rvert$. 
\begin{lemma}
    For compatible $\bbeta,\bbeta' \in \bbN_0^d$ and any $\bx,\bx' \in [0,1)^d$, the mixture-product kernels \eqref{eq:mixture_product_kernel} satisfy 
    \begin{align*}
        K^{(\bbeta,\bbeta')}&(\bx,\bx') = \gamma \left(\prod_{j \in S(\bbeta+\bbeta')} K_{\alpha_j}^{(\beta_j,\beta_j')}(x_j,x_j')\right) \\
        &\cdot\left(\sum_{\fu \subseteq \{1,\dots,d\} \setminus S(\bbeta+\bbeta')} \tEta_{\fu \cup S(\bbeta+\bbeta')} \prod_{j \in \fu} K_{\alpha_j}(x_j,x_j')\right), \\
        \int_{[0,1]^d} K^{(\bbeta,\bbeta')}&(\bx,\bx') \D \bx = \gamma \left(\prod_{j \in S(\bbeta+\bbeta')} \int_0^1 K_{\alpha_j}^{(\beta_j,\beta_j')}(x_j,x_j') \D x_j\right) \\
        &\cdot\left(\sum_{\fu \subseteq \{1,\dots,d\} \setminus S(\bbeta+\bbeta')} \tEta_{\fu \cup S(\bbeta+\bbeta')} \prod_{j \in \fu} \int_0^1 K_{\alpha_j}(x_j,x_j') \D x_j\right), \\
        \int_{[0,1]^{2d}} K^{(\bbeta,\bbeta')}&(\bx,\bx') \D \bx \D \bx' = \gamma \left(\prod_{j \in S(\bbeta+\bbeta')} \int_0^1 \int_0^1 K_{\alpha_j}^{(\beta_j,\beta_j')}(x_j,x_j') \D x_j \D x_j'\right) \\
        &\cdot\left(\sum_{\fu \subseteq \{1,\dots,d\} \setminus S(\bbeta+\bbeta')} \tEta_{\fu \cup S(\bbeta+\bbeta')} \prod_{j \in \fu} \int_0^1 \int_0^1 K_{\alpha_j}(x_j,x_j') \D x_j \D x_j'\right)
    \end{align*}
    whenever the appropriate kernel derivatives exist. Assuming the cost of evaluating the one-dimensional kernels and their integrals are all $\calO(1)$, then the cost of evaluating any of the above expressions is $\calO(\lvert S(\bbeta+\bbeta') \rvert + 2^{d-\lvert S(\bbeta+\bbeta') \rvert})$. 
\end{lemma}

\begin{lemma}
    For compatible $\bbeta,\bbeta' \in \bbN_0^d$ and any $\bx,\bx' \in [0,1)^d$, the product kernels \eqref{eq:product_kernels} satisfy 
    \begin{align*}
        K^{(\bbeta,\bbeta')}(\bx,\bx') =& \gamma \prod_{j \in \{1,\dots,d\}} \left(1_{\beta_j+\beta_j'=0} + \eta_j K_{\alpha_j}^{(\beta_j,\beta_j')}(x_j,x_j')\right), \\
        \int_{[0,1]^d} K^{(\bbeta,\bbeta')}(\bx,\bx') \D \bx =& \gamma \prod_{j \in \{1,\dots,d\}} \left(1_{\beta_j+\beta_j'=0} + \eta_j \int_0^1 K_{\alpha_j}^{(\beta_j,\beta_j')}(x_j,x_j') \D x_j \right), \\
        \int_{[0,1]^{2d}} K^{(\bbeta,\bbeta')}(\bx,\bx') \D \bx \D \bx' =& \gamma \prod_{j \in \{1,\dots,d\}} \left(1_{\beta_j+\beta_j'=0} + \eta_j \int_0^1 \int_0^1 K_{\alpha_j}^{(\beta_j,\beta_j')}(x_j,x_j') \D x_j \D x_j'\right)
    \end{align*}
    where $1_{\beta_j+\beta_j'=0}$ is $1$ if $\beta_j+\beta_j'=0$ and $0$ otherwise. 
\end{lemma}

\Subsection{Examples of Fast Kernel Computations} 

The following codes perform fast kernel computations with the kernel (Gram) matrix 
$$\mK = (K(\bx_i,\bx_{i'}))_{i,i'=0}^{2^m-1}.$$
Specifically, we evaluate the kernel matrix product $\mK \by = \bu$ and solve the linear system $\mK \bv = \by$ for $\bv$, each at $\calO(n \log n)$ cost and requiring only $\calO(n)$ storage following \Cref{corr:fast_gram_matrix_comps}. We also give an example of how to compute $\tby_\mathrm{full} = \overline{\mV_{m+1}} \begin{pmatrix} \by \\ \by_\mathrm{new} \end{pmatrix}$ given $\tby = \overline{\mV_m} \by$ for special fast transform matrices satisfying \Cref{cond:fast_transform}. This is useful for efficient updates after doubling the sample size while retaining the original points. 

First, we consider an SI product kernel $K$ paired with a shifted rank-1 lattice in radical inverse order $(\bx_i)_{i=0}^{2^{m-1}}$.
\begin{lstlisting}[style=Python]
d = 3 # dimension 
m = 10 # will generate 2^m points
n = 2**m # number of points
lattice = qp.Lattice(d) # defaults to radical inverse order
kernel = qp.KernelShiftInvar(
    d, # dimension 
    alpha = [1,2,3], # smoothness parameters
    lengthscales = [1, 1/2, 1/4]) # product weights
x = lattice(n_min=0,n_max=n) # shape=(n, d) lattice points
y = np.random.rand(n) # shape=(n,) random uniforms
# fast matrix multiplication and linear system solve
k1 = kernel(x,x[0]) # shape=(n,) first column of Gram matrix
lam = np.sqrt(n)*qp.fftbr(k1) # vector of eigenvalues
yt = qp.fftbr(y)
u = qp.ifftbr(yt*lam) # fast matrix multiplication 
v = qp.ifftbr(yt/lam) # fast linear system solve
# efficient fast transform updates
ynew = np.random.rand(n) # shape=(n,) new random uniforms
omega = qp.omega_fftbr(m) # shape=(n,)
ytnew = qp.fftbr(ynew) # shape=(n,)
ytfull1 = yt+omega*ytnew # shape=(n,) first half of ytfull
y2full2 = yt-omega*ytnew # shape=(n,) second half of ytfull
ytfull = np.hstack([ytfull1,y2full2])/np.sqrt(2) # shape=(2n,)
\end{lstlisting}
Similarly, we consider a DSI product kernel $K$ paired with a shifted base $2$ digital net in radical inverse order $(\bx_i)_{i=0}^{2^m-1}$.
\begin{lstlisting}[style=Python]
d = 3 # dimension 
m = 10 # will generate 2^m points
n = 2**m # number of points
dnb2 = qp.DigitalNetB2(d) # defaults to radical inverse order
kernel = qp.KernelDigShiftInvar(
    d, # dimension 
    t = dnb2.t, # bits in integer representation of points
    alpha = [1,2,3], # smoothness parameters
    lengthscales = [1, 1/2, 1/4]) # product weights
x = dnb2(n_min=0,n_max=n) # shape=(n, d) digital net
y = np.random.rand(n) # shape=(n,) random uniforms
# fast matrix multiplication and linear system solve
k1 = kernel(x,x[0]) # shape=(n,) first column of Gram matrix
lam = np.sqrt(n)*qp.fwht(k1) # vector of eigenvalues
yt = qp.fwht(y)
u = qp.fwht(yt*lam) # fast matrix multiplication 
v = qp.fwht(yt/lam) # fast linear system solve
# efficient fast transform updates
ynew = np.random.rand(n) # shape=(n,) new random uniforms
omega = qp.omega_fwht(m) # shape=(n,)
ytnew = qp.fwht(ynew) # shape=(n,)
ytfull1 = yt+omega*ytnew # shape=(n,) first half of ytfull
y2full2 = yt-omega*ytnew # shape=(n,) second half of ytfull
ytfull = np.hstack([ytfull1,y2full2])/np.sqrt(2) # shape=(2n,)
\end{lstlisting}

The following subsections will define various univariate SI (\Cref{sec:si_kernels}) and DSI (\Cref{sec:dsi_kernels}) kernels $K_\alpha$ which are used to construct the multivariate kernels above. While the univariate kernels we present are SPSD, the multivariate mixture-product or product kernels will be SPD as we always add $\gamma>0$ (corresponding to the $\fu = \emptyset$ term in \eqref{eq:mixture_product_kernel}). Again, we will consider derivatives and integrals of univariate kernels which will be useful later when considering derivative-informed Gaussian process regression (\Cref{sec:gps_deriv_informeds}) and Bayesian cubature (\Cref{sec:mtgp_bayesian_cubature} and \Cref{sec:fmtgp_bayesian_cubature}). Moreover, the univariate SI and DSI kernels we present will all satisfy 
$$\int_0^1 K_\alpha^{(\beta,\beta')}(x,x') \D x' = \int_0^1 \int_0^1 K_\alpha^{(\beta,\beta')}(x,x') \D x \D x' = 0$$
for all $x \in [0,1)$ and all compatible $\alpha \in \bbR_{>0}$ and $\beta,\beta' \in \bbN_0$.

\Section{Shift-Invariant (SI) Kernels} \label{sec:si_kernels}

Here we present weighted Korobov spaces which have been extensively studied throughout the literature \cite{kaarnioja.kernel_interpolants_lattice_rkhs,kaarnioja.kernel_interpolants_lattice_rkhs_serendipitous,cools2021fast,cools2020lattice,sloan2001tractability,kuo2004lattice}. To the best of our knowledge, the kernel derivative forms have not appeared elsewhere in the literature. 

\Subsection{Fourier series of smooth functions}

Assume $f:[0,1] \to \bbR$ has an absolutely convergent Fourier series 
$$f(x) = \sum_{k \in \bbZ} \hf(k) e^{2 \pi \sqrt{-1} k x}, \qquad \hf(k) = \int_0^1 f(x) e^{-2 \pi \sqrt{-1} k x} \D x$$
where $\hf(k)$ are Fourier coefficients. For $\beta \in \bbN_0$, if $f^{(\beta)}(x)$ has an absolutely convergent Fourier series and $f^{(\beta')}$ is periodic for all $\beta' \in \{1,\dots,\beta-1\}$, then 
$$f^{(\beta)}(x) = \sum_{k \in \bbZ} \hf(k) (2 \pi \sqrt{-1} k)^\beta  e^{2 \pi \sqrt{-1} k x}.$$
This implies the following relationship between Fourier coefficients of the function and Fourier coefficients of its derivative:
$$\widehat{f^{(\beta)}}(k)= (2 \pi \sqrt{-1} k)^\beta \hf(k).$$

\Subsection{Periodic Sobolev Spaces}

\begin{definition} \label{def:si_kernel}
    Suppose $\alpha \in \bbR_{>1/2}$. We define the weight $\tr_\alpha: \bbZ_0 \to \bbR$ by
    \begin{equation*}
        \tr_\alpha^2(k) := k^{-2\alpha},
    \end{equation*}
    and define our SPSD SI kernel $\tK_\alpha: [0,1] \times [0,1] \to \bbR$ to be
    $$\tK_\alpha(x,x') = \sum_{k \in \bbZ_0} \tr_\alpha^2(k) e^{2 \pi \sqrt{-1} k(x - x')} = \sum_{k \in \bbZ_0} \tr_\alpha^2(k) \cos(2 \pi k(x - x')).$$ 
\end{definition}

\begin{theorem}
    For $\alpha \in \bbN$ suppose the real-valued functions $f$ and $g$ have absolutely convergent Fourier series and both $f^{(\beta)}$ and $g^{(\beta)}$ are periodic for all $\beta \in \{1,\dots,\alpha-1\}$. Then for the inner product 
    \begin{equation}
        \langle f,g \rangle_{\tK_\alpha} :=  \sum_{k \in \bbZ_0} \tr_\alpha(k)^{-2} \hf_k \overline{\hg_k}
        \label{eq:inner_prod_fourier}
    \end{equation}
    we have
    $$\langle f,g \rangle_{\tK_\alpha} = (-1)^\alpha (2 \pi)^{-2\alpha} \int_0^1 f^{(\alpha)}(x) g^{(\alpha)}(x) \D x.$$
\end{theorem}
\begin{proof}
    \begin{align*}
        \langle f,g \rangle_{\tK_\alpha} &= \sum_{k \in \bbZ_0} k^{2\alpha} \hf(k) \overline{\hg(k)} \\
        &= \sum_{k \in \bbZ_0} (2 \pi \sqrt{-1})^{-2\alpha} \widehat{f^{(\alpha)}(k)} \overline{\widehat{g^{(\alpha)}}(k)} \\
        &= (-1)^\alpha (2 \pi)^{-2 \alpha} \sum_{k \in \bbZ_0} \widehat{f^{(\alpha)}}(k) \overline{\widehat{g^{(\alpha)}}(k)}
    \end{align*}
    and 
    \begin{align*} 
        \int_0^1 f^{(\alpha)}(x) g^{(\alpha)}(x) \D x &= \int_0^1 \left(\sum_{k \in \bbZ_0} \widehat{f^{(\alpha)}}(k) e^{-2 \pi k x}\right)\left(\sum_{h \in \bbZ_0} \overline{\widehat{g^{(\alpha)}}(h) e^{-2 \pi k x}} \right) \D x \\
        &= \sum_{k,h \in \bbZ_0} \widehat{f^{(\alpha)}}(k) \overline{\widehat{g^{(\alpha)}}(h)} \int_0^1 e^{2 \pi i x (k-h)} \D x \\
        &= \sum_{k,h \in \bbZ_0} \widehat{f^{(\alpha)}}(k) \overline{\widehat{g^{(\alpha)}}(h)} \delta_{k,h} \\
        &= \sum_{k \in \bbZ_0} \widehat{f^{(\alpha)}}(k) \overline{\widehat{g^{(\alpha)}}(k)}.
    \end{align*}
\end{proof}

\begin{remark}
    For $\balpha \in \bbN^d$, the RKHS with mixture-product kernels \eqref{eq:mixture_product_kernel} formed from univariate SI kernels $\tK_{\alpha_j}$ is a weighted periodic unanchored Sobolev space of dominating smoothness of order $\balpha$ with RKHS norm 
    $$\lVert f \rVert^2 := \sum_{\fu \subseteq \{1,\dots,d\}} \frac{1}{(2 \pi)^{2 \lvert \fu \rvert}\gamma_\fu} \int_{[0,1]^{\lvert \fu \rvert}} \llvert \int_{[0,1]^{s - \lvert \fu \rvert}} \left(\prod_{j \in \fu} \frac{\partial^{\alpha_j}}{\partial y_j^{\alpha_j}}\right) f(\by) \D \by_{- \fu} \rrvert^2 \D \by_\fu.$$
    Here $f:[0,1]^d \to \bbR$, $\by_\fu = (\by_j)_{j \in \fu}$, $\by_{-\fu} := (y_j)_{j \in \{1,\dots,d\} \setminus \fu}$, and $\lvert \fu \rvert$ is the cardinality of $\fu \subseteq \{1,\dots,d\}$. This is a special case of the weighted Korobov space which allows for real smoothness parameters $\balpha \in \bbR_{>1/2}$ as we have discussed throughout this section. 
\end{remark}

\begin{theorem}
    \label{thm:K_lattice_beta_kappa}
    Let $\alpha \in \bbR_{>1/2}$. Suppose $\beta,\beta' \in \bbN_0$ satisfy $2\alpha-\beta-\beta'>1$. Then, $\tK_\alpha^{(\beta,\beta')}: [0,1) \times [0,1) \to \bbR$ is also a SPSD SI kernel which may be written as
    $$\tK_\alpha^{(\beta,\beta')}(x,y) = (2 \pi \sqrt{-1})^{\beta + \beta'} (-1)^{\beta'} \tK_{\alpha-\beta/2-\beta'/2}(x,y).$$
    Moreover, for any $x \in [0,1)$,
    $$\int_0^1 \tK_\alpha^{(\beta,\beta')}(x,x') \D x' = \int_0^1 \int_0^1 \tK_\alpha^{(\beta,\beta')}(x,x') \D x \D x' = 0.$$
\end{theorem}
\begin{proof}
    $\tK_\alpha^{(\beta,\beta')}$ is real-valued since $\tK_\alpha$ is real-valued. The first assertion follows from 
    $$\tK_\alpha^{(\beta,\beta')}(x,y) = \sum_{k \in \bbZ_0} (2 \pi \sqrt{-1} k)^\beta r(k) e^{2 \pi \sqrt{-1} k x}  \overline{r(k) (2 \pi \sqrt{-1} k)^{\beta'} e^{2 \pi \sqrt{-1} k y}}.$$
    The integral forms follow from the fact that for $k \in \bbZ_0$ we have $\int_0^1 e^{2 \pi \sqrt{-1} k x} = 0$.   
\end{proof}

\begin{theorem} \label{thm:si_kernel_bpoly_form}
    For $\alpha \in \bbN$ and $\beta,\beta' \in \bbN_0$ satisfying $2\alpha-\beta-\beta'>1$, we have that for any $x,x' \in [0,1)$
    $$\tK_\alpha^{(\beta,\beta')}(x,x') = \frac{(-1)^{\alpha+\beta'+1}(2 \pi)^{2 \alpha}}{(2 \alpha - \beta - \beta')!} B_{2 \alpha - \beta - \beta'}(x - x' \bmod 1)$$
    where $B_p$ is the $p^\mathrm{th}$ Bernoulli polynomial.
\end{theorem}

\begin{theorem} \label{thm:msdiff_Kcheckhat}
    For $\alpha>\bbR_{>3/2}$, both of the following limits exist and are finite:
    $$\lim_{h \to 0} \frac{\tK_\alpha(x,x)-\tK_\alpha(x+h,x)}{h^2},$$
    $$\lim_{h_1,h_2 \to 0} \frac{\tK_\alpha(x+h_1,x+h_2)-\tK_\alpha(x+h_1,x) - \tK_\alpha(x+h_2,x) + \tK_\alpha(x,x)}{h_1h_2}.$$
\end{theorem}
\begin{proof}
    We have 
    \begin{align*}
        \lim_{h \to 0} &\frac{\tK_\alpha(x,x)-\tK_\alpha(x+h,x)}{h^2} = \lim_{h \to 0} \frac{\tK_\alpha(0,0)-\tK_\alpha(h,0)}{h^2} \\
        &= \lim_{h \to 0} \frac{1}{h^2} \sum_{k \in \bbZ_0} \frac{1-\cos(2 \pi k h)}{k^{2 \alpha}} \\
        &= \sum_{k \in \bbZ_0} \frac{1}{k^{2\alpha}} \lim_{h \to 0} \frac{1-\cos(2 \pi k h)}{h^2} \\
        &= 2 \pi^2 \sum_{k \in \bbZ_0} \frac{1}{k^{2 (\alpha-1)}}
    \end{align*}
    and 
    \begin{align*}
        \lim_{h_1,h_2 \to 0} &\frac{\tK_\alpha(x+h_1,x+h_2)-\tK_\alpha(x+h_1,x) - \tK_\alpha(x+h_2,x) + \tK_\alpha(x,x)}{h_1h_2} \\
        &= \lim_{h_1,h_2 \to 0} \frac{\tK_\alpha(h_1-h_2,0)-\tK_\alpha(h_1,0) - \tK_\alpha(h_2,0) + \tK_\alpha(0,0)}{h_1h_2} \\
        &= \lim_{h_1,h_2 \to 0} \frac{1}{h_1h_2} \sum_{k \in \bbZ_0} \frac{\cos(2 \pi k (h_1-h_2)) - \cos(2 \pi k h_1) - \cos(2 \pi k h_2) + 1}{k^{2 \alpha}} \\
        &= \sum_{k \in \bbZ_0} \frac{1}{k^{2\alpha}} \lim_{h_1,h_2 \to 0} \frac{\cos(2 \pi k (h_1-h_2)) - \cos(2 \pi k h_1) - \cos(2 \pi k h_2) + 1}{h_1h_2} \\
        &= 4 \pi^2 \sum_{k \in \bbZ_0} \frac{1}{k^{2(\alpha-1)}}.
    \end{align*}
    We have passed the limits into the sum because the sum is absolutely convergent when $\alpha \in \bbR_{>1/2}$. Moreover, both final expressions are proportional to $\sum_{k \in \bbZ_0} \frac{1}{k^{2(\alpha-1)}}$, which is finite whenever $\alpha>3/2$.
\end{proof}

\Section{Digitally-Shift-Invariant (DSI) Kernels} \label{sec:dsi_kernels}

Here we will present classes of DSI kernels. An existing class of low-order smoothness DSI kernels is studied in \cite{dick.multivariate_integraion_sobolev_spaces_digital_nets}. Here we derive new higher-order smoothness DSI kernels which exploit the decay of Walsh coefficients of smooth functions as studied in \cite{dick.decay_walsh_coefficients_smooth_functions}. The forms of our new higher-order smoothness DSI kernels have appeared in worst case error bounds for Quasi-Monte Carlo methods using polynomial lattice rules \cite{baldeaux.polylat_efficient_comp_worse_case_error_cbc}, but there they were not interpreted kernels. We also derive a new expression for a certain $\alpha=4$ higher-order smoothness DSI kernels which has not appeared elsewhere in the literature.

In what follows, we will assume $b \in \bbN$ is a fixed prime base. Let $\beta(x)$ be the power of $b^{-1}$ of the first non-zero digit in the base $b$ expansion of $x \in [0,1)$, i.e., 
$$\beta(x) := - \lfloor \log_b(x) \rfloor.$$

\begin{definition} \label{def:walsh_funcs}
    For $k \in \bbN_0$ with $k = \sum_{\ell \in \bbN_0} \mk_\ell b^\ell$ and any $x \in [0,1)$ with $x = \sum_{\ell \in \bbN} \mx_\ell b^{-\ell} \in [0,1)$, define the $k^\mathrm{th}$ Walsh function 
    $$\wal_k(x) = \exp\left(\frac{2 \pi \sqrt{-1}}{b} \sum_{\ell \in \bbN_0} \mk_\ell \mx_{\ell+1}\right).$$
\end{definition}

\begin{lemma} \label{lemma:walsh_func_properties}
    Let $k,h \in \bbN_0$ and $x,y \in [0,1)$, with the restriction that if $x,y$ are not $b$-adic rationals then $x \oplus y$ and $x \ominus y$ are not allowed to be a $b$-adic rationals. Then
    \begin{enumerate}
        \item $\wal_k(0) = 1.$
        \item \begin{alignat*}{2}
            \wal_k(x) \wal_h(x) &= \wal_{k \oplus h}(x), \qquad \wal_k(x)\overline{\wal_h(x)} &&= \wal_{k \ominus h}(x), \\
            \wal_k(x)\wal_k(y) & = \wal_k(x \oplus y), \qquad \wal_k(x) \overline{\wal_k(y)} &&= \wal_k(x \ominus y).
        \end{alignat*}
        \item
        $$\int_0^1 \wal_k(x) \D x = \begin{cases} 1, & k = 0 \\ 0, & k > 0 \end{cases}.$$
        \item $\left\{\wal_k: k \in \bbN_0\right\}$ is a complete orthonormal system in $\calL_2([0,1))$, the space of square integrable functions.
        \item For any $f \in \calL_2([0,1))$ we have  
        $$\int_0^1 f(\sigma) \D \sigma = \int_0^1 f(x \oplus \sigma) \D \sigma = \int_0^1 f(x \ominus \sigma) \D \sigma.$$
        \item $$\sum_{k=0}^{b^a-1} \wal_k(x) = \begin{cases} b^a, & a \leq \beta(x) - 1, \\ 0, & \mathrm{otherwise} \end{cases}.$$ 
    \end{enumerate}
\end{lemma}
\begin{proof}
    These properties are well known and stated in many papers, e.g., \cite[Proposition 3]{baldeaux.HO_nets_RKHS}. Most proofs are straightforward, while others can be found in \cite{baldeaux.polylat_efficient_comp_worse_case_error_cbc} or \cite{chrestenson.class_generalized_walsh_functions} or \cite{walsh.closed_set_normal_orthogonal_functionss}. 
\end{proof} 

For $k \in \bbN$ write $k = \sum_{\ell=1}^{\# k} \mk_{a_\ell} b^{a_\ell}$ where $a_1 > \dots > a_{\#k} \geq 0$ and $\mk_{a_\ell} \in \{1,\dots,b-1\}$ for $1 \leq \ell \leq \#k$. Here $\#k$ is the number of non-zero elements in the base $b$ expansion of $k$. For $k=0$, set $\#k=0$ and set the empty sum to $0$ as well. Moreover, we may set $k' = k-\mk_{a_1} b^{a_1}$. 

\begin{definition} \label{def:dick_weight_func} 
    We define the Dick weight functions for $\alpha,k \in \bbN_0$ as 
    $$\mu_\alpha(k) = \sum_{\ell=1}^{\min(\alpha,\#k)} (a_\ell+1), \qquad \rmu_\alpha(k) = \mu_\alpha(k) + \max(\alpha-\#k,0)(a_{\#k}+1)$$
    where $\mu_\alpha(0) = \mu_0(k) = \rmu_\alpha(0) = \rmu_0(k) = 0$. These are used extensively in \cite{dick.decay_walsh_coefficients_smooth_functions} which provides many of the inequalities used in the following subsections. 
\end{definition}

\Subsection{RKHS Preliminaries for Higher-Order Smoothness DSI Kernels} 

For $x \in [0,1)$ and $k \in \bbN_0$ define the Walsh coefficient
$$\hf(k) = \int_0^1 f(x) \overline{\wal_k(x)} \D x.$$ 
If $f$ has an absolutely convergent Walsh series representation then 
$$f(x) = \sum_{k \in \bbN_0} \hf(k) \wal_k(x).$$
For any kernel $K: [0,1) \times [0,1) \to \bbR$, we will denote its Walsh coefficient at $k,h \in \bbN_0$ (when it exists) by
$$\calK(k,h) := \int_0^1 \int_0^1 K(x,x') \overline{\wal_k(x)} \wal_h(x') \D x' \D x.$$ 
We will carry accents through this notation as well, e.g., $\tcalK(k,h)$ is the Walsh coefficient corresponding to $\tK$. 

\begin{definition} \label{def:Q_DSI_kernel_from_Q}
    Following \cite[Defintion 6]{dick.multivariate_integraion_sobolev_spaces_digital_nets}, for any SPSD kernel $K: [0,1) \times [0,1) \to \bbR$, we define the associated SPSD DSI kernel for any $x,x' \in [0,1)$ to be  
    $$K^\mathrm{DSI}(x,x') := \int_0^1 K(x \oplus \sigma, x' \oplus \sigma) \D \sigma.$$
\end{definition}
 
\begin{lemma} \label{lemma:Q_DSI_kernel_form}
    For $K^\mathrm{DSI}$ in \Cref{def:Q_DSI_kernel_from_Q}, if $\calK(k,0) = \calK(0,k) = 0$ for all $k \in \bbN_0$, then
    $$K^\mathrm{DSI}(x,x') = \sum_{k \in \bbN} \calK(k,k) \wal_k(x \ominus x'), \qquad (\calK(k,k))_{k \in \bbN} \subset \bbR.$$ 
\end{lemma} 
\begin{proof}
    Following \cite[Lemma 5]{dick.multivariate_integraion_sobolev_spaces_digital_nets}, we have that 
    \begin{align*}
        K^\mathrm{DSI}(x,x') &= \int_0^1 K(x \oplus \sigma, x' \oplus \sigma) \D \sigma \\
        &= \int_0^1 \sum_{k,h \in \bbN} \calK(k,h) \wal_k(x \oplus \sigma) \overline{\wal_h(x' \oplus \sigma)} \D \sigma  \\ 
        &= \sum_{k,h \in \bbN} \calK(k,h) \wal_k(x) \overline{\wal_k(x')} \int_0^1 \wal_k(\sigma) \overline{\wal_h(\sigma)} \D \sigma  \\
        &= \sum_{k \in \bbN} \calK(k,k) \wal_k(x \ominus x')
    \end{align*}
    where $\calK(k,k) \in \bbR$ follows from $K(x,x') = \overline{K(x',x)}$. 
\end{proof} 

\begin{lemma} \label{lemma:RKHS_feature_map_form} 
    If a SPSD DSI kernel $K: [0,1) \times [0,1) \to \bbR$ can be written as 
    $$K(x,x') = \sum_{k \in \bbN} r^2(k) \wal_k(x \ominus x') \qquad \forall\; x,x' \in [0,1)$$
    for some $r: \bbN \to \bbC$, then the RKHS of $K$ can be written as 
    $$H(K) = \left\{f: [0,1) \to \bbR \; | \quad \; \exists\; (w_k)_{k \in \bbN} \in \ell^2(\bbC), \quad f(x) =  \sum_{k \in \bbN} w_k r(k) \wal_k(x) \right\}.$$
\end{lemma} 

\begin{theorem} \label{thm:RKHS_inclusions_mercer}
    Suppose we are given a SPSD kernel $K$ and construct $K^\mathrm{DSI}$ as in \Cref{def:Q_DSI_kernel_from_Q}. Moreover, suppose we construct the kernel
    $$R(x,y) = \sum_{k \in \bbN} r^2(k) \wal_k(x \ominus x')$$
    by choosing $(r(k))_{k \in \bbN} \subset \bbR_{>0}$ so that $R$ is SPSD. 
    \begin{enumerate}
        \item If there is some $C \in \bbR_{\geq 0}$ such that $\calK(k,k) \leq C r^2(k)$ for every $k \in \bbN$, then $H(K^\mathrm{DSI}) \subseteq H(R)$.
        \item If there is some $C \in \bbR_{\geq 0}$ such that $\calK(k,k) \geq C r^2(k) $ for every $k \in \bbN$, then $H(R) \subseteq H(K^\mathrm{DSI})$.
    \end{enumerate} 
\end{theorem}
\begin{proof}
    By \Cref{lemma:RKHS_feature_map_form} we have 
    $$H(R) = \left\{f: [0,1) \to \bbR \;|\; \exists\; (w_k)_{k \in \bbN} \in \ell^2(\bbC), f(x) = \sum_{k \in \bbN} w_k r(k) \wal_k(x)\right\} \qquad\text{and}$$
    $$H(K^\mathrm{DSI}) = \left\{f: [0,1) \to \bbR \;|\; \exists\; (v_k)_{k \in \bbN} \in \ell^2(\bbC), f(x) = \sum_{k \in \bbN} v_k \sqrt{\calK(k,k)} \wal_k(x)\right\}$$
    where $\ell^2(\bbC)$ is the space of sequences $(x_k)_{k \geq 0} \subset \bbC$ satisfying $\sum_{k \geq 0} \lvert x_k \rvert^2 < \infty$. 
    
    For the first claim, if $f \in H(K^\mathrm{DSI})$ then for some $(v_k)_{k \in \bbN} \subset \ell^2(\bbC)$ we have 
    $$f(x) = \sum_{k \in \bbN} v_k \sqrt{\calK(k,k)} \wal_k(x).$$
    Define $w_k = v_k \sqrt{\calK(k,k)}/r(k)$ so that 
    $$f(x) = \sum_{k \in \bbN} w_k r(k) \wal_k(x).$$ By the assumption, there is some $C \in \bbR_{\geq 0}$ such that $\calK(k,k) \leq C r^2(k)$ which implies $w_k \leq \sqrt{C} v_k$ for all $k \in \bbN_0$. This shows $(w_k) \in \ell^2(\bbC)$ and thus $f \in H(R)$. 

    For the second claim, if $f \in H(R)$ then for some $(w_k)_{k \in \bbN} \subset \ell^2(\bbC)$ we have 
    $$f(x) = \sum_{k \in \bbN} w_k r(k) \wal_k(x).$$
    Define $v_k = w_k r(k)/\sqrt{\calK(k,k)}$ so that 
    $$f(x) = \sum_{k \in \bbN} v_k \sqrt{\calK(k,k)} \wal_k(x).$$ By the assumption, there is some $C \in \bbR_{\geq 0}$ such that $ r^2(k) \leq C \calK(k,k)$ which implies $v_k \leq \sqrt{C} w_k$ for all $k \in \bbN$. This shows $(v_k) \in \ell^2(\bbC)$ and thus $f \in H(K^\mathrm{DSI})$.
\end{proof}

\Subsection{DSI Kernels for Smooth Periodic Functions} \label{sec:dsi_kernels_smooth_periodic_functions} 

Recall that for $\alpha \in \bbR_{>1/2}$, the SPSD kernel $\tK_\alpha$ in \Cref{def:si_kernel} corresponds to the weighted Korobov RKHS of smooth periodic functions. 

\begin{lemma} \label{lemma:baldeaux2009_lemma19}
    For every $\alpha \in \bbN$  we have $\tcalK_\alpha(k,0) = \tcalK_\alpha(0,k) = 0$ for all $k \in \bbN_0$ and there is some $\rC_\alpha \in \bbR_{>0}$ such that for any $k \in \bbN$
    $$\tcalK_\alpha(k,k) \leq \rC_\alpha b^{-2\rmu_\alpha(k)}.$$
\end{lemma}
\begin{proof}
    This is a weaker version of \cite[Lemma 19]{baldeaux.HO_nets_RKHS} which gives explicit bounds on $\lvert \tcalK_\alpha(k,h) \rvert$ for any $k,h \in \bbN$. 
\end{proof}

\begin{lemma}  \label{lemma:dick2008thm16}
    For $\alpha \in \bbN_{\geq 2}$, $f \in H\left(\tK_\alpha\right)$, and $k \in \bbN$ there is some $\rrC_\alpha \in \bbR_{>0}$ such that 
    $$\left\lvert \hf(k) \right\rvert \leq \rrC_\alpha b^{-\rmu_\alpha(k)}.$$
\end{lemma}
\begin{proof}
    \cite[Theorem 16]{dick.decay_walsh_coefficients_smooth_functions} gives an explicit constant $\rrC_\alpha$.
\end{proof}

\begin{definition} \label{def:RKHS_DSI_periodic_weak}
    Fix $\alpha \in \bbN_{\geq 2}$. We define the weight $\rr_\alpha: \bbN \to \bbR$ by 
    $$\rr_\alpha^2(k) = b^{-2\rmu_\alpha(k)},$$
    and define our SPSD DSI kernel $\rK_\alpha: [0,1) \times [0,1) \to \bbR$ to be 
    $$\rK_\alpha(x,x') = \sum_{k \in \bbN} \rr_\alpha^2(k) \wal_k(x \ominus x').$$
\end{definition}

\begin{definition} \label{def:RKHS_DSI_periodic_strong}
    Fix $\alpha \in \bbN_{\geq 2}$. We define the weight $\rrr_\alpha: \bbN \to \bbR$ by 
    $$\rrr_\alpha^2(k) = b^{-\rmu_\alpha(k)}$$
    and define our SPSD DSI kernel $\rrK_\alpha: [0,1) \times [0,1) \to \bbR$ to be 
    $$\rrK_\alpha(x,x') = \sum_{k \in \bbN} \rrr_\alpha^2(k) \wal_k(x \ominus x').$$
\end{definition}

\begin{theorem} \label{thm:RKHS_DSI_periodic_inclusions}
    For any $\alpha \in \bbN_{ \geq 2}$ we have 
    $$H\left(\tK_\alpha^\mathrm{DSI}\right) \subseteq H(\rK_\alpha) \subseteq H(\rrK_\alpha)$$
    and
    $$H\left(\tK_\alpha\right) \subseteq H(\rrK_\alpha).$$
\end{theorem}
\begin{proof}
    We have $H\left(\tK_\alpha^\mathrm{DSI}\right) \subseteq H(\rK_\alpha)$ using \Cref{thm:RKHS_inclusions_mercer} combined with the assertion in \Cref{lemma:baldeaux2009_lemma19} which says $\tcalK_\alpha(k,k) \leq \rC_\alpha \rr_\alpha^2(k)$. We have $H(\rK_\alpha) \subseteq H(\rrK_\alpha)$ using \Cref{thm:RKHS_inclusions_mercer} combined with the fact that $\rr_\alpha^2(k) \leq \rrr_\alpha^2(k)$.
    
    To show $H\left(\tK_\alpha\right) \subseteq H(\rrK_\alpha)$, assume $f \in H\left(\tK_\alpha\right)$. Then, since $\int_0^1 f(x) \D x = 0$, we have that for any $x \in [0,1)$
    $$f(x) = \sum_{k \in \bbN} \hf(k) \wal_k(x) = \sum_{k \in \bbN} w_k \rrr_\alpha(k) \wal_k(x)$$ 
    for $w_k  = \hf(k) / \rrr_\alpha(k)$. Then by \Cref{lemma:RKHS_feature_map_form} it is sufficient to show $(w_k)_{k \in \bbN} \in \ell^2(\bbC)$. To see this, use \Cref{lemma:dick2008thm16} to say   
    $$\sum_{k \in \bbN} \lvert w_k \rvert^2 = \sum_{k \in \bbN} \frac{\lvert \hf(k) \rvert^2}{\rrr_\alpha^2(k)} \leq \rrC_\alpha^2 \sum_{k \in \bbN}  \frac{b^{-2\rmu_\alpha(k)}}{b^{-\rmu_\alpha(k)}} = \rrC_\alpha^2 \sum_{k \in \bbN} b^{-\rmu_\alpha(k)} < \infty$$
    where the last inequality above follows from \cite[Remark 19]{dick.decay_walsh_coefficients_smooth_functions}. 
\end{proof}

\begin{theorem} \label{thm:integrals_dsi_kernels_smooth_periodic}
    For any $\alpha \in \bbN_{ \geq 2}$ and $x \in [0,1)$ we have
    \begin{align*}
        \int_0^1 \rK_\alpha(x,x') \D x' &= \int_0^1 \int_0^1 \rK_\alpha(x,x') \D x \D x' = 0, \\
        \int_0^1 \rrK_\alpha(x,x') \D x' &= \int_0^1 \int_0^1 \rrK_\alpha(x,x') \D x \D x' = 0.
    \end{align*}
\end{theorem}
\begin{proof}
    This is an immediate consequence of the third property in \Cref{lemma:walsh_func_properties}. 
\end{proof}

\Subsection{DSI Kernels for Smooth Functions} \label{sec:dsi_kernels_smooth_functions}

\begin{definition} \label{def:sobolev_space_H_alpha}
    For $\alpha \in \bbN$ define the RKHS $H\left(\cK_\alpha\right)$ of real-valued functions $f:[0,1) \to \bbR$ with inner product 
    $$\langle f,g \rangle_{\cK_\alpha} = \sum_{l=1}^{\alpha-1} \int_0^1 f^{(l)}(x) \D x \int_0^1 g^{(l)}(x) \D x + \int_0^1 f^{(\alpha)}(x) g^{(\alpha)}(x) \D x$$
    and SPSD kernel defined for any $x,x' \in [0,1)$ as 
    $$\cK_\alpha(x,x') = \sum_{l=1}^\alpha \frac{B_l(x) B_l(x')}{(l!)^2} + \frac{(-1)^{\alpha+1}}{(2\alpha)!} B_{2\alpha}((x-x') \bmod 1)$$
    where $B_p$ is the $p^\mathrm{th}$ Bernoulli polynomial. 
\end{definition}

\begin{remark}
    $H\left(\cK_\alpha\right)$ is an RKHS of smooth, not necessarily periodic, functions. Moreover, $H\left(\cK_\alpha\right)$ contains the RKHS of smooth periodic functions $H\left(\tK_\alpha\right)$ with $\tK_\alpha$ given in \Cref{def:si_kernel}. Specifically, for $\alpha \in \bbN$,
    $$H\left(\tK_\alpha\right) = \left\{f \in H\left(\cK_\alpha\right): \int_0^1 f^{(l)}(x) \D x = 0, \quad \forall\; 0 \leq l < \alpha\right\}.$$
\end{remark}

\begin{lemma}
    For $\alpha=1$ and any $x,x' \in [0,1)$ we have
    $$\cK_1^\mathrm{DSI}(x,x') =  \begin{cases} \frac{1}{6}, & x = y \\ \frac{1}{6} - \frac{\lvert \mx_\beta - \my_\beta \rvert (b-\lvert \mx_\beta - \my_\beta \rvert)}{b^{\beta+1}}, & x \neq y \qtq{with} \beta = \beta(x \oplus x') \end{cases}.$$
\end{lemma}
\begin{proof}
    This is shown in \cite[Appendix B]{dick.multivariate_integraion_sobolev_spaces_digital_nets}.
\end{proof}

\begin{corollary}
    For $\alpha=1$ and $b=2$ we have 
    $$\cK_1^\mathrm{DSI}(x,x') =  \begin{cases} \frac{1}{6}, & x = y \\ \frac{1}{6} - 2^{-1-\beta(x \oplus x')}, & x \neq y \end{cases}.$$
\end{corollary}

\begin{lemma} \label{lemma:baldeaux2009_proposition20}
    For every $\alpha \in \bbN$ we have $\widecheck{\calK}_\alpha(k,0) = \widecheck{\calK}_\alpha(0,k) = 0$ for all $k \in \bbN_0$ and there is some $\dC_\alpha \in \bbR_{>0}$ such that for any $k \in \bbN$
    $$\widecheck{\calK}_\alpha(k,k) \leq \dC_\alpha b^{-2\mu_\alpha(k)}.$$
\end{lemma}
\begin{proof}
    This is a weaker version of \cite[Proposition 20]{baldeaux.HO_nets_RKHS} which gives explicit bounds on $\lvert \widecheck{\calK}_\alpha(k,h) \rvert$ for any $k,h \in \bbN$. 
\end{proof}

\begin{lemma} \label{lemma:dick2008thm14}
    For $\alpha \in \bbN_{ \geq 2}$, $f \in H\left(\cK_\alpha\right)$, and $k \in \bbN$ there is some $\ddC_\alpha < \infty$ independent of $k$ such that 
    $$\lvert \hf(k) \rvert \leq \ddC_\alpha b^{-\mu_\alpha(k)}.$$   
\end{lemma}
\begin{proof}
    \cite[Theorem 14]{dick.decay_walsh_coefficients_smooth_functions} tells us that for any $k \in \bbN$
    \begin{equation}
        \label{eq:sobolev_upper_bound_C}
        \lvert \hf(k) \rvert \leq \sum_{l=\#k}^\alpha C'_l b^{-\rmu_l(k)} + C''_\alpha b^{-\rmu_\alpha(k)}
    \end{equation}
    where for $\#k>\alpha$ the empty sum $\sum_{l=\#k}^\alpha$ is defined as $0$ and
    \begin{align*}
        C'_l &:= \left\lvert \int_0^1 f^{(l)}(x) \D x \right\rvert \frac{1}{\left(2\sin(\pi/b)\right)^l} \left(1+\frac{1}{b}+\frac{1}{b(b+1)}\right)^{\max(0,l-2)}, \\
        C''_\alpha &:= \int_0^1 \left\lvert f^{(\alpha)}(x) \right\rvert \D x \frac{2}{\left(2\sin(\pi/b)\right)^\alpha} \left(1+\frac{1}{b}+\frac{1}{b(b+1)}\right)^{\alpha-2} \left(3+\frac{2}{b}+\frac{2b-1}{b-1}\right).
    \end{align*}
    If $f \in H\left(\cK_\alpha\right)$ then $\left(\int_0^1 f^{(l)}(x) \D x \right)^2 < \infty$ for $0 \leq l \leq \alpha-1$ and $\int_0^1 \left(f^{(\alpha)}(x)\right)^2 \D x < \infty$. Thus, $C'_l < \infty$ for $0 \leq l < \alpha$ and $C''_\alpha < \infty$ follows from the Cauchy--Schwarz inequality. Let $C_\alpha''' := \max\left(C'_0,\dots,C'_{\alpha-1},C'_\alpha+C''_\alpha\right)$ so that for any $k \in \bbN_0$ we have 
    $$\lvert \hf(k) \rvert \leq C_\alpha' \left[\sum_{l=\#k}^{\alpha-1} b^{-\rmu_l(k)} + b^{-\rmu_\alpha(k)}\right].$$
    Now, notice that $\rmu_l(k)$ is non-decreasing in $l$ for every $k \in \bbN_0$, so setting $\ddC_\alpha = (\alpha+1)C_\alpha'''$ and recalling $\mu_\alpha(k) = \rmu_{\min(\alpha,\#k)}(k)$ yields the result. 
\end{proof}

\begin{definition} \label{def:RKHS_DSI_weak}
    Fix $\alpha \in \bbN_{\geq 2}$. We define the weight $\dr_\alpha: \bbN \to \bbR$ by 
    $$\dr_\alpha^2(k) = b^{-2\mu_\alpha(k)},$$
    and define our SPSD DSI kernel $\dK_\alpha: [0,1) \times [0,1) \to \bbR$ to be 
    $$\dK_\alpha(x,x') = \sum_{k \in \bbN} \dr_\alpha^2(k) \wal_k(x \ominus x').$$
\end{definition}

\begin{definition} \label{def:RKHS_DSI_strong}
    Fix $\alpha \in \bbN_{\geq 2}$. For $\tau \in \bbR_{<1}$, and $\zeta \in \bbR_{\leq 1}$ we define the weight $\ddr_{\alpha,\tau,\zeta}: \bbN \to \bbR$ by 
    $$\ddr_{\alpha,\tau,\zeta}^2(k) = b^{-\mu_\alpha(k) - \tau \mu_1(k') - \zeta \mu_{\alpha-2}(k'')},$$
    and define our SPSD DSI kernel $K_{\alpha,\tau,\zeta}: [0,1) \times [0,1) \to \bbR$ to be
    \begin{equation}
        \ddK_{\alpha,\tau,\zeta}(x,x') = \sum_{k \in \bbN} \ddr_{\alpha,\tau,\zeta}^2(k) \wal_k(x \ominus x').
        \label{eq:K_net}
    \end{equation}
\end{definition}

\begin{theorem} \label{thm:RKHS_DSI_inclusions}
    For $\alpha \geq 2$ and any $\tau \in \bbR_{<1}$ and $\zeta \in \bbR_{\leq 1}$ we have 
    $$H\left(\cK_\alpha^\mathrm{DSI}\right) \subseteq H(\dK_\alpha) \subseteq H\left(\ddK_{\alpha,\tau,\zeta}\right)$$
    and
    $$H\left(\cK_\alpha\right) \subseteq H\left(\ddK_{\alpha,\tau,\zeta}\right).$$
    Moreover, for any  $\tau_1,\tau_2 \in \bbR_{<1}$ and $\zeta_1,\zeta_2 \in \bbR_{\leq 1}$ satisfying $\tau_1 \geq \tau_2$ and $\zeta_1 \geq \zeta_2$ we have 
    $$H\left(\ddK_{\alpha,\tau_1,\zeta_1}\right) \subseteq H\left(\ddK_{\alpha,\tau_2,\zeta_2}\right).$$
\end{theorem}
\begin{proof}
    We have $H\left(\cK_\alpha^\mathrm{DSI}\right) \subseteq H(\dK_\alpha)$ using \Cref{thm:RKHS_inclusions_mercer} combined with the assertion in \Cref{lemma:baldeaux2009_proposition20} which says $\widecheck{\calK}_\alpha(k,k) \leq \dC_\alpha \dr_\alpha^2(k)$. We have $H(\dK_\alpha) \subseteq H(\ddK_{\alpha,\tau,\zeta})$ using \Cref{thm:RKHS_inclusions_mercer} combined with the fact that $\dr_\alpha^2(k) \leq \ddr_{\alpha,\tau,\zeta}^2(k)$. 

    To show $H\left(\cK_\alpha\right) \subseteq H\left(\ddK_{\alpha,\tau,\zeta}\right)$, assume $f \in H\left(\cK_\alpha\right)$. Then, since \\ $\int_0^1 f(x) \D x = 0$, we have that for any $x \in [0,1)$ 
    $$f(x) = \sum_{k \in \bbN} \hf(k) \wal_k(x) = \sum_{k \in \bbN} w_k \ddr_{\alpha,\tau,\zeta}(k) \wal_k(x)$$
    for $w_k = \hf(k) / \ddr_{\alpha,\tau,\zeta}(k)$. Then by \Cref{lemma:RKHS_feature_map_form} it is sufficient to show $(w_k)_{k \in \bbN} \in \ell^2(\bbC)$. To see this, use \Cref{lemma:dick2008thm14} to say 
    \begin{align*}
        \sum_{k \in \bbN} \lvert w_k \rvert^2 &= \sum_{k \in \bbN} \frac{\lvert \hf(k) \rvert^2}{\ddr_{\alpha,\tau,\zeta}^2(k)} \\
        &\leq C_\alpha^2 \sum_{k \in \bbN} \frac{b^{-2\mu_\alpha(k)}}{b^{-\mu_\alpha(k) - \tau \mu_1(k') - \zeta \mu_{\alpha-2}(k'')}} \\
        &= C_\alpha^2 \sum_{k \in \bbN} b^{-\mu_\alpha(k) + \tau \mu_1(k') + \zeta \mu_{\alpha-2}(k'')} \\
        &< \infty
    \end{align*}
    where the last inequality follows from \cite[Remark 19]{dick.decay_walsh_coefficients_smooth_functions}. 

    We have $H\left(\ddK_{\alpha,\tau_1,\zeta_1}\right) \subseteq H\left(\ddK_{\alpha,\tau_2,\zeta_2}\right)$ due to \Cref{thm:RKHS_inclusions_mercer} and the inequality $\ddr_{\alpha,\tau_1,\zeta_1} \leq \ddr_{\alpha,\tau_2,\zeta_2}$. 
\end{proof}

\begin{remark}
    \cite[Corollary 3.11]{dick.walsh_spaces_HO_nets} provides a much more general result for weighted function spaces $\calE_\alpha$. There it is shown that $H\left(\cK_\alpha\right) \subseteq \calE_\alpha$ for certain parameterizations of $\calE_\alpha$. Moreover, the error rate of Quasi-Monte Carlo methods for $f \in \calE_\alpha$ using higher-order digital nets is optimal.
\end{remark} 

\begin{definition} \label{def:RKHS_DSI_low_alpha}
    The following kernel was studied in \cite{dick.multivariate_integraion_sobolev_spaces_digital_nets} and corresponds to the scramble invariant kernel in \cite{yue2002discrepancy}. Fix $\alpha \in \bbR_{>0}$. We define the weight $\dddr_\alpha: \bbN \to \bbR_{>0}$ by 
    $$\dddr_\alpha^2(k) = b^{-(\alpha+1) (\mu_1(k)-1)},$$
    and define our SPSD DSI kernel $\dddK_\alpha: [0,1) \times [0,1) \to \bbR$ to be  
    $$\dddK_\alpha(x,x') = \sum_{k \in \bbN} \dddr_\alpha^2(k) \wal_k(x \ominus x')= \begin{cases} \upsilon(\alpha), & x=x' \\ \upsilon(\alpha) - b^{-\alpha(\beta(x \ominus x')-1)}\left(\upsilon(\alpha)+1\right), & x \neq x' \end{cases}$$
    where 
    $$\upsilon(\alpha) = \frac{b^{\alpha+1}(b-1)}{b^{\alpha+1}-b}.$$
\end{definition}

\begin{theorem} \label{thm:integrals_dsi_kernels_smooth}
    For any $x \in [0,1)$ we have 
    \begin{align*}
        \int_0^1 \dK_\alpha(x,x') \D x' &= \int_0^1 \int_0^1 \dK_\alpha(x,x') \D x \D x' = 0 \quad \forall\; \alpha \in \bbR_{\geq 2}, \\
        \int_0^1 \ddK_{\alpha,\tau,\zeta}(x,x') \D x' &= \int_0^1 \int_0^1 \ddK_{\alpha,\tau,\zeta}(x,x') \D x \D x' = 0 \quad \forall\; \alpha \in \bbR_{\geq 2}, \\
        \int_0^1 \dddK_{\alpha,\tau,\zeta}(x,x') \D x' &= \int_0^1 \int_0^1 \dddK_{\alpha,\tau,\zeta}(x,x') \D x \D x' = 0 \quad \forall\; \alpha \in \bbR_{\geq 0}, \;\; \tau \in \bbR_{<1} \;\; \zeta \in \bbR_{\leq 1}.
    \end{align*}
\end{theorem}
\begin{proof}
    This is an immediate consequence of the third property in \Cref{lemma:walsh_func_properties}. 
\end{proof}

\Subsection{Computable DSI Kernels in Base 2}

This section explores analytical forms of $\ddK_{\alpha,0,0}$ in base $b=2$. 

\begin{definition}
    For $\alpha \in \bbN_{\geq 2}$ and $x \in [0,1)$ denote 
    \begin{equation}
        \omega_\alpha(x) := \sum_{k \in \bbN} b^{-\mu_\alpha(k)} \wal_k(x)
        \label{eq:omega_x_alpha}
    \end{equation}
    so that for $x,x' \in [0,1)$, 
    $$\ddK_{\alpha,0,0}(x,x') = \omega_\alpha(x \ominus x').$$
\end{definition}

\begin{remark}
    Explicit expressions of $\omega_2(x)$ and $\omega_3(x)$ for any prime $b \in \bbN$ are given in \cite[Theorem 3]{baldeaux.polylat_efficient_comp_worse_case_error_cbc}. 
\end{remark}

\begin{remark}
    For any $\alpha \in \bbN_{\geq 2}$ and prime $b \in \bbN$, \cite[Theorem 2]{baldeaux.polylat_efficient_comp_worse_case_error_cbc} shows how to compute $\omega_\alpha(x)$ in at most $\calO(\alpha \#x)$ operations where $x$ has at most $\#x$ digits in its base $b$ representation. 
\end{remark}

In the following lemma we derive Walsh coefficients of low-order monomials.

\begin{lemma}\label{lemma:walsh_low_order_monomials}
    Fix $b=2$. Let $f_p(x) := x^p$. When $k \in \bbN$ write 
    $$k = 2^{a_1}+\dots+2^{a_{\#k}}$$
    where $a_1 > a_2 > \dots > a_{\#k} \geq 0$. Then we have 
    \begin{align*}
        \hf_1(k) &= \begin{cases} 1/2, & k = 0 \\ -2^{-a_1-2}, & k=2^{a_1} \\ 0, & \mathrm{otherwise} \end{cases}, \\
        \hf_2(k) &= \begin{cases} 1/3, & k = 0 \\ -2^{-a_1-2}, & k=2^{a_1} \\ 2^{-a_1-a_2-3}, & k = 2^{a_1}+2^{a_2} \\ 0, & \mathrm{otherwise} \end{cases}, \\
        \hf_3(k) &= \begin{cases} 1/4, & k=0 \\ -2^{-a_1-2} + 2^{-3a_1-5}, & k = 2^{a_1} \\ 3 \cdot 2^{-a_1-a_2-4}, & k=2^{a_1}+2^{a_2} \\ -3 \cdot 2^{-a_1-a_2-a_3-5}, & k=2^{a_1}+2^{a_2}+2^{a_3} \\ 0, & \mathrm{otherwise} \end{cases}.
    \end{align*}
\end{lemma}
\begin{proof}
    The forms for $\hf_1$ and $\hf_2$ follow from \cite[Example 14.2, Example 14.3]{dick.digital_nets_sequences_book}. 
    For $k=0$ and any $x \in [0,1)$ we have $\wal_0(x) = 1$, so  
    $$\hf_3(x) = \int_0^1 x^3 \D x = 1/4.$$ 
    Assume $k \in \bbN$ going forward. For $k=2^{a_1}+k'$ with $0 \leq k' < 2^{a_1}$, \cite[Equation 3.6]{fine.walsh_functions} implies
    $$J_k(x) := \int_0^x \wal_k(t) \D t = 2^{-a_1-2} \left[\wal_{k'}(x) - \sum_{r=1}^\infty 2^{-r} \wal_{2^{a_1+r}+k}(x)\right].$$
    Using integration by parts and the fact that $J_k(0) = J_k(1) = 0$ we have
    \begin{align*}
        \hf_3(k) &= \int_0^1 x^3 \wal_k(x) \D x 
        = \left[x^3 J_k(x) \right]_{x=0}^{x=1} - 3 \int_0^1 x^2 J_k(x) \D x \\
        &= -3*2^{-a_1-2} \left[\hf_2(k') - \sum_{r=1}^\infty 2^{-r} \hf_2(2^{a_1+r}+k)\right].
    \end{align*}
    If $\#k=1$, i.e., $k=2^{a_1}$, then 
    \begin{align*}
        \hf_3(k) &= -3*2^{-a_1-2} \left[\hf_2(0) - \sum_{r=1}^\infty 2^{-r} \hf_2(2^{a_1+r}+2^{a_1})\right] \\
        &= -3*2^{-a_1-2} \left[\frac{1}{3} - \sum_{r=1}^\infty 2^{-r} 2^{-(a_1+r)-a_1-3}\right] \\
        &= 2^{-3a_1-5} - 2^{-a_1-2}.
    \end{align*}
If $\#k=2$ then 
    $$\hf_3(k) = -3*2^{-a_1-2} \hf_2(2^{a_2}) = 3 \cdot 2^{-a_1-a_2-4}.$$
    If $\#k=3$ then 
    $$\hf_3(k) = -3*2^{-a_1-2} \hf_2(2^{a_2}+2^{a_3}) = -3 \cdot 2^{-a_1-a_2-a_3-5}.$$
    If $\#k>3$ then $\hf_3(k)=0$.
\end{proof}

The following theorem gives some explicit forms of $\omega_\alpha$ in base $b=2$.  The $\alpha \in \{2,3\}$ forms are due to \cite{baldeaux.polylat_efficient_comp_worse_case_error_cbc}. These are also useful in the computation of the worst-case error of QMC with higher-order polynomial-lattice rules as in \cite{baldeaux.polylat_efficient_comp_worse_case_error_cbc}.

\begin{theorem} \label{thm:explicit_DSI_low_order_forms}
    Fix the base $b=2$. Let $\beta(x) = - \lfloor \log_2(x) \rfloor$ and for $\nu \in \bbN$ define $t_\nu(x) = 2^{-\nu \beta(x)}$ where for $x=0$ we set $\beta(x) = t_\nu(x) = 0$. Then 
    \begin{align*}
        \omega_2(x) =& -1+-\beta(x) x + \frac{5}{2}\left[1-t_1(x)\right], \\
        \omega_3(x) =& -1+\beta(x)x^2-5\left[1-t_1(x)\right]x+\frac{43}{18}\left[1-t_2(x)\right], \\
        \omega_4(x) =& -1 -\frac{2}{3}\beta(x)x^3+5\left[1-t_1(x)\right]x^2 - \frac{43}{9}\left[1-t_2(x)\right]x \\
        &+\frac{701}{294}\left[1-t_3(x)\right]+\beta(x)\left[\frac{1}{48} S(x) - \frac{1}{42}\right],
    \end{align*}
    where 
    \begin{equation} \label{eq:K4_sum_term}
        S(x) = \sum_{a=0}^\infty \frac{\wal_{2^a}(x)}{2^{3a}}.
    \end{equation}
\end{theorem}
\begin{proof}
    Write  
    $$\omega_\alpha(x) = \sum_{1 \leq \nu < \alpha} s_\nu(x) + \ts_\alpha(x)$$  
    where $s_\nu$ sums over all $k$ with $\#k = \nu$ and $\ts_\alpha$ sums over all $k$ with $\#k \geq \alpha$. In \cite[Corollary 1]{baldeaux.polylat_efficient_comp_worse_case_error_cbc} it was shown that 
    \begin{align*}
        s_1(x) &= -2x+1, \\
        s_2(x) &= 2x^2 - 2x + \frac{1}{3}, \\
        \ts_2(x) &=  \left[2-\beta(x)\right]x + \frac{1}{2}\left[1-5t_1(x)\right], \\
        \ts_3(x) &= -\left[2-\beta(x)\right] x^2 - \left[1-5t_1(x)\right]x + \frac{1}{18}\left[1-43t_2(x)\right]
    \end{align*}
    from which $\omega_2$ and $\omega_3$ follow. We now find expressions for $s_3$ and $\ts_4$ from which $\omega_4$ follows. 

    Assume sums over $a_i$ are over $\bbN_0$ unless otherwise restricted. \Cref{lemma:walsh_low_order_monomials} gives
    \begin{align*}
        x =& \frac{1}{2} - \sum_{a_1} \frac{\wal_{2^{a_1}}(x)}{2^{a_1+2}}, \\
        x^2 =& \frac{1}{3} - \sum_{a_1}\frac{\wal_{2^{a_1}}(x)}{2^{a_1+2}} + \sum_{a_1>a_2} \frac{\wal_{2^{a_1}+2^{a_2}}(x)}{2^{a_1+a_2+3}}, \\
        x^3 =& \frac{1}{4} - \sum_{a_1} \frac{\wal_{2^{a_1}}(x)}{2^{a_1+2}} + \frac{3}{2}\sum_{a_1>a_2} \frac{\wal_{2^{a_1}+2^{a_2}}(x)}{2^{a_1+a_2+3}} \\
        &- \frac{3}{2} \sum_{a_1>a_2>a_3} \frac{\wal_{2^{a_1}+2^{a_2}+2^{a_3}}(x)}{2^{a_1+a_2+a_3+4}} + \sum_{a_1} \frac{\wal_{2^{a_1}}(x)}{2^{3a_1+5}}
    \end{align*}
    so
    $$x^3-\frac{3}{2}x^2+\frac{1}{2}x = \frac{1}{32} \sum_{a_1} \frac{\wal_{2^{a_1}}(x)}{2^{3a_1}} - \frac{3}{4} \sum_{a_1>a_2>a_3} \frac{\wal_{2^{a_1}+2^{a_2}+2^{a_3}}(x)}{2^{a_1+a_2+a_3+3}}$$
    and 
    $$s_3(x) = \sum_{a_1>a_2>a_3} \frac{\wal_{2^{a_1}+2^{a_2}+2^{a_3}}(x)}{2^{a_1+a_2+a_3+3}} = -\frac{4}{3} x^3+ 2x^2 -\frac{2}{3}x + \frac{1}{24} \sum_{a_1} \frac{\wal_{2^{a_1}}(x)}{2^{3a_1}}.$$
    Now,
    \begin{align*}
        \ts_4(x) &= \sum_{\substack{a_1>a_2>a_3>a_4 \\ 0 \leq k < 2^{a_4}}} \frac{\wal_{2^{a_1}+2^{a_2}+2^{a_3}+2^{a_4}+k}(x)}{2^{a_1+a_2+a_3+a_4+4}} \\
        &= \sum_{a_1>a_2>a_3>a_4} \frac{\wal_{2^{a_1}+2^{a_2}+2^{a_3}+2^{a_4}}(x)}{2^{a_1+a_2+a_3+a_4+4}} \sum_{0 \leq k < 2^{a_4}} \wal_k(x).
    \end{align*}
    If $x=0$, then 
    $$\ts_4(0) = \sum_{a_1 > a_2 > a_3 > a_4} \frac{1}{2^{a_1+a_2+a_3+4}} = \frac{1}{294}.$$ 
    Going forward, assume $x \in (0,1)$ so $\beta(x) = - \lfloor \log_2(x) \rfloor$ is finite. Recall from the last property of \Cref{lemma:walsh_func_properties} that  
    $$\sum_{0 \leq k < 2^{a_4}} \wal_k(x) = \begin{cases} 2^{a_4}, & a_4 \leq \beta(x)-1 \\ 0, & a_4 > \beta(x)-1 \end{cases}.$$
    Moreover, since $\beta(x)$ is the index of the first $1$ in the base $2$ expansion of $x$, when $a_4 < \beta(x)-1$ we have $\wal_{2^{a_4}}(x) = (-1)^{\mx_{a_4+1}} = 1$ and when $a_4 = \beta(x)-1$ we have $\wal_{2^{a_4}}(x) = -1$. This implies 
    \begin{align*}
        \ts_4(x) &= \sum_{\substack{a_1>a_2>a_3>a_4 \\ \beta(x)-1 \geq a_4}} \frac{\wal_{2^{a_1}+2^{a_2}+2^{a_3}+2^{a_4}}(x)}{2^{a_1+a_2+a_3+4}} \\
        &= \sum_{\substack{a_1>a_2>a_3>a_4 \\ \beta(x)-1 > a_4}} \frac{\wal_{2^{a_1}+2^{a_2}+2^{a_3}}(x)}{2^{a_1+a_2+a_3+4}} - \sum_{a_1>a_2 > a_3 > \beta(x)-1} \frac{\wal_{2^{a_1}+2^{a_2}+2^{a_3}}(x)}{2^{a_1+a_2+a_3+4}} \\
        &=: T_1 - T_2.
    \end{align*}
    
    The first term is 
    \begin{align*}
        T_1 &= \sum_{\beta(x)-1 > a_4} \bigg(\sum_{a_1>a_2>a_3} \frac{\wal_{2^{a_1}+2^{a_2}+2^{a_3}}(x)}{2^{a_1+a_2+a_3+4}} -  \sum_{a_4  \geq a_1>a_2>a_3} \frac{1}{2^{a_1+a_2+a_3+4}} \\ & \qquad\qquad  - \sum_{a_1 > a_4 \geq a_2 > a_3} \frac{\wal_{2^{a_1}}(x)}{2^{a_1+a_2+a_3+4}} - \sum_{a_1 > a_2 > a_4 \geq a_3} \frac{\wal_{2^{a_1}+2^{a_2}}(x)}{2^{a_1+a_2+a_3+4}} 
        \bigg) \\
        &=: \sum_{\beta(x)-1 > a_4} \left[V_1(a_4)-V_2(a_4) - V_3(a_4) - V_4(a_4) \right].
    \end{align*}
    Clearly $V_1(a_4) = s_3(x)/2$ and $V_2$ is easily computed. Now
    \begin{align*}
        V_3(a_4) &= \left(\sum_{a_4 \geq a_2 > a_3} \frac{1}{2^{a_2+a_3+3}}\right)\left(\sum_{a_1>a_4} \frac{\wal_{2^{a_1}}(x)}{2^{a_1+1}}\right) \\
        &= \left(\sum_{a_4 \geq a_2 > a_3} \frac{1}{2^{a_2+a_3+3}}\right)\left(s_1(x) - \sum_{a_4 \geq a_1} \frac{1}{2^{a_1+1}} \right)
    \end{align*}
    and 
    \begin{align*}
        &V_4(a_4) = \left(\sum_{a_4 \geq a_3} \frac{1}{2^{a_3+2}}\right)\left(\sum_{a_1>a_2 > a_4} \frac{\wal_{2^{a_1}+2^{a_2}}(x)}{2^{a_1+a_2+2}}\right) \\
        &= \left(\sum_{a_4 \geq a_3} \frac{1}{2^{a_3+2}}\right)\left(\sum_{a_1>a_2} \frac{\wal_{2^{a_1}+2^{a_2}}(x)}{2^{a_1+a_2+2}} - \sum_{a_4 \geq a_1 > a_2} \frac{1}{2^{a_1+a_2+2}} - \sum_{a_1 > a_4 \geq a_2} \frac{\wal_{2^{a_1}}(x)}{2^{a_1+a_2+2}}\right) \\
        &= \left(\sum_{a_4 \geq a_3} \frac{1}{2^{a_3+2}}\right)\\
        &\cdot \left(s_2(x) - \sum_{a_4 \geq a_1 > a_2} \frac{1}{2^{a_1+a_2+2}} - \left(s_1(x) - \sum_{a_4 \geq a_1} \frac{1}{2^{a_1+1}} \right)\left(\sum_{a_4 \geq a_2} \frac{1}{2^{a_2+1}}\right)\right).
    \end{align*}
    
    The second term is 
    \begin{align*}
        T_2 =& \sum_{a_1>a_2>a_3} \frac{\wal_{2^{a_1}+2^{a_2}+2^{a_3}}(x)}{2^{a_1+a_2+a_3+4}} - \sum_{\beta(x)-1 > a_1 > a_2 > a_3} \frac{1}{2^{a_1+a_2+a_3+4}} \\ & + \sum_{\beta(x)-1 > a_2 > a_3} \frac{1}{2^{\beta(x)+a_2+a_3+3}}  \\ & - \sum_{a_1 > \beta(x)-1 > a_2 > a_3} \frac{\wal_{2^{a_1}}(x)}{2^{a_1+a_2+a_3+4}} + \sum_{a_1 > \beta(x)-1 > a_3}  \frac{\wal_{2^{a_1}}(x)}{2^{a_1+\beta(x)+a_3+3}} \\ &- \sum_{a_1 > a_2 > \beta(x)-1 > a_3}\frac{\wal_{2^{a_1}+2^{a_2}}(x)}{2^{a_1+a_2+a_3+4}} + \sum_{a_1 > a_2 > \beta(x)-1} \frac{\wal_{2^{a_1}+2^{a_2}}(x)}{2^{a_1+a_2+\beta(x)+3}} \\
        &=: W_1-W_2+W_3-W_4+W_5-W_6+W_7.
    \end{align*}
    Clearly $W_1 = s_3(x)/2$ and both $W_2$ and $W_3$ are easily computed. Similarity in the next two sums gives
    \begin{align*}
        W_5 &- W_4 = \left(\sum_{\beta(x)-1 > a_3}\frac{1}{2^{\beta(x)+a_3+2}}-\sum_{\beta(x)-1>a_2>a_3} \frac{1}{2^{a_2+a_3+3}}\right) \left(\sum_{a_1 > \beta(x)-1} \frac{\wal_{2^{a_1}}(x)}{2^{a_1+1}} \right) \\
        &= \left(\sum_{\beta(x)-1 > a_3}\frac{1}{2^{\beta(x)+a_3+2}}-\sum_{\beta(x)-1>a_2>a_3} \frac{1}{2^{a_2+a_3+3}}\right) \\
        &\cdot \left(s_1(x) - \sum_{\beta(x)-1 > a_1} \frac{1}{2^{a_1+1}} + \frac{1}{2^{\beta(x)}}\right)
    \end{align*}
    Similarity in the final two sums gives  
    \begin{align*}
        W_7 - W_6 &= \left(\frac{1}{2^{\beta(x)+1}}-\sum_{\beta(x)-1 > a_3} \frac{1}{2^{a_3+2}}\right) \left(\sum_{a_1>a_2 > \beta(x)-1} \frac{\wal_{2^{a_1}+2^{a_2}}(x)}{2^{a_1+a_2+2}}\right)
    \end{align*}
    where 
    \begin{align*}
        \sum_{a_1>a_2 > \beta(x)-1}& \frac{\wal_{2^{a_1}+2^{a_2}}(x)}{2^{a_1+a_2+2}} = \sum_{a_1>a_2} \frac{\wal_{2^{a_1}+2^{a_2}}(x)}{2^{a_1+a_2+2}} - \sum_{\beta(x)-1 > a_1 > a_2} \frac{1}{2^{a_1+a_2+2}} \\ &+ \sum_{\beta(x)-1 > a_2} \frac{1}{2^{\beta(x)+a_2+1}} - \sum_{a_1 > \beta(x)-1 > a_2} \frac{\wal_{2^{a_1}}(x)}{2^{a_1+a_2+2}} + \sum_{a_1 > \beta(x)-1} \frac{\wal_{2^{a_1}}(x)}{2^{a_1+\beta(x)+1}} \\
        &= s_2(x) - \sum_{\beta(x)-1 > a_1 > a_2} \frac{1}{2^{a_1+a_2+2}} + \sum_{\beta(x)-1 > a_2} \frac{1}{2^{\beta(x)+a_2+1}}\\
        &+ \left(\frac{1}{2^{\beta(x)}} - \sum_{\beta(x)-1 > a_2} \frac{1}{2^{a_2+1}}\right)\left(s_1(x) - \sum_{\beta(x)-1 > a_1} \frac{1}{2^{a_1+1}} + \frac{1}{2^{\beta(x)}}\right).
    \end{align*}
    This implies 
    \begin{align*}
        \ts_4(x) &=  \frac{2}{3}\left(2-\beta(x)\right)x^3 + \left(1-5\ t_1(x)\right)x^2 - \frac{1}{9} \left(1 - 43 t_2(x)\right)x \\
        &-\frac{1}{48} \left(2-\beta(x)\right)\sum_{a_1} \frac{\wal_{2^{a_1}}(x)}{2^{3a_1}} -\frac{1}{294} \left(7\beta(x)+701 t_3(x)\right) +\frac{5}{98},
    \end{align*}
    from which the result follows.
\end{proof}

\Subsection{Derivatives of DSI Kernels in Base 2}

For $b=2$, let $\bbB$ denote the set of all dyadic rationals, i.e., $x \in [0,1)$ whose binary expansion has a finite number of ones. For $x \in \bbB$ let $\Omega(x)$ be the index of the last $1$ in that binary expansion of $x$, e.g., if $x=3/4=0.11_2$ then $\Omega(x)=2$.

\begin{lemma}\label{lemma:deriv_S}
    For base $b=2$, any dyadic rational $x \in \bbB$, and $S$ defined in \eqref{eq:K4_sum_term}, the partial derivatives satisfy
    \begin{equation}
        \partial^+ S(x)=0 \qqtqq{and} \partial^- S(x) = -\infty.
    \end{equation}
\end{lemma}
\begin{proof}
    Let $\Omega = \Omega(x)$ and recall $\beta(x) = - \lfloor \log_2(x) \rfloor$ is the index of the first $1$ in the binary expansion of $x$. For $h<2^{-\Omega}$ we have 
    \begin{equation}
        S(x+h)-S(x) = \sum_{\ell \geq \Omega} \left[\frac{(-1)^{\mh_{\ell+1}}}{2^{3\ell}} - \frac{1}{2^{3\ell}}\right].
    \end{equation}
    Clearly $\partial^+ S(x) \leq 0$. However, 
    \begin{align*}
        \partial^+ S(x) &\geq \lim_{h \to 0^+} \frac{1}{h} \sum_{\ell \geq \beta(h)-1} \left[\frac{-1}{2^{3\ell}} - \frac{-1}{2^{3\ell}}\right] \\
        &= - \lim_{h \to 0^+} \frac{1}{h} \sum_{\ell \geq \beta(h)-1} \frac{1}{2^{3\ell-1}} \\
        &= - \frac{2^7}{7} \lim_{h \to 0^+} \frac{2^{-3\beta(h)}}{h} \\
        &= 0,
    \end{align*}
    so $\partial^+ S(x) = 0$. 
    
    The left derivative, assuming $x \in \bbB \setminus \{0\}$, is 
    \begin{align*}
        \partial^- S(x) =& \lim_{h \to 0^+} \frac{S(x)-S(x-h)}{h} \\
        \leq& \lim_{h \to 0^+} \frac{1}{h} \bigg[\bigg(-\frac{1}{2^{3 (\Omega-1)}} + \sum_{\ell \geq \Omega} \frac{1}{2^{3\ell}}\bigg) - \bigg(\frac{1}{2^{3(\Omega-1)}}  - \sum_{\ell \geq \Omega} \frac{1}{2^{3\ell}} \bigg)\bigg] \\
        &= \lim_{h \to 0^+} \frac{1}{h}  \left(\sum_{\ell \geq \Omega} \frac{1}{2^{3\ell-1}} - \frac{1}{2^{3\Omega-4}}\right) \\
        &= - \frac{3}{7} 2^{5-3\Omega} \left(\lim_{h \to 0^+} \frac{1}{h}\right) \\
        &= - \infty.
    \end{align*}
\end{proof}

\begin{lemma}\label{lemma:partial_plus_widetildeK234_og}
    For base $b=2$ and any dyadic rational $x \in \bbB$ we have 
    $$\partial^+ \omega_4(x)= -2 (1+\omega_3(x)) \qquad\text{and}\qquad  \partial^+ \omega_3(x) = -2 (1+\omega_2(x)).$$
\end{lemma}
\begin{proof}
    For any $x \in \bbB \setminus \{0\}$ we can choose $h>0$ small enough so that $\beta(x) = \beta(x+h)$. Combined with \Cref{lemma:deriv_S}, this implies $\omega_4$ and $\omega_3$ are both right differentiable on $\bbB \setminus \{0\}$. The fact that $\partial^+ \omega_3(0) = -5$ follows from
    $$\lim_{h \to 0^+} \frac{\omega_3(h)-\omega_3(0)}{h} = \lim_{h \to 0^+} \frac{\beta(h) h^2 - 5 \left(1-t_1(h)\right) h + \frac{43}{18}\left(1-t_2(h)\right) - \frac{43}{18}}{h}.$$
    The fact that $\partial^+ \omega_4(0) = -43/9$ follows from
    \begin{align*}
        &\lim_{h \to 0^+} \frac{\omega_4(h)-\omega_4(0)}{h} \\
        &= \lim_{h \to 0^+} \frac{-\frac{2}{3} \beta(h) h^3 + 5\left(1-t_1(h)\right)x^2 -\frac{43}{9}\left(1-t_2(h)\right)h +\frac{701}{294} \left(1-t_3(h)\right) - \frac{701}{294}}{h}
    \end{align*}
    and 
    $$\lim_{h \to 0^+} \frac{\beta(h)}{h}\left(\frac{1}{48}S(x) - \frac{1}{42}\right) = 0$$
    where $S$ is given in \eqref{eq:K4_sum_term}. The above equation follows from writing 
    \begin{align*}
        \frac{1}{48}S(x) - \frac{1}{42} &= \frac{1}{48} \left(\sum_{\beta(h)-1 > a} \frac{1}{2^{3a}} - \frac{1}{2^{3(\beta(h)-1)}} + \sum_{a > \beta(h)-1} \frac{\wal_{2^{a}}(h)}{2^{3a}}\right) - \frac{1}{42} \\
        &\in \left[-\frac{8}{21} 2^{-3\beta(h)},-\frac{1}{3} 2^{-3\beta(h)}\right] 
    \end{align*}
    so that 
    $$\lim_{h \to 0^+} \frac{\beta(h)}{h}\left(\frac{1}{48}S(x) - \frac{1}{42}\right) \geq  -\frac{8}{21} \left(\lim_{h \to 0^+} \frac{\beta(h)}{h 2^{3\beta(h)}}\right)=0$$
    and 
    $$\lim_{h \to 0^+} \frac{\beta(h)}{h}\left(\frac{1}{48}S(x) - \frac{1}{42}\right) \leq -\frac{1}{3} \left(\lim_{h \to 0^+} \frac{\beta(h)}{h2^{3\beta(h)}}\right) = 0.$$
\end{proof}

\begin{lemma} \label{lemma:partial_digital_shift}
    For base $b=2$ and dyadic rationals $x,x' \in \bbB$, the right derivatives satisfy
    \begin{equation*}
        \partial_x^+ (x \oplus x') = \partial_{x'}^+ (x \oplus x')=1,
    \end{equation*}
    while the left derivatives satisfy 
    \begin{equation*}
        \partial_x^- (x \oplus x') = \partial_{x'}^- (x \oplus x') = \infty.
    \end{equation*}
\end{lemma}
\begin{proof}
    Let $\Omega = \max\{\Omega(x),\Omega(x')\}$. For $h<2^{-\Omega}$ we have $(x+h) \oplus x' = (x \oplus x') + h$ so the right partial derivative satisfies 
    \begin{equation}
        \partial_x^+ (x \oplus x') = \lim_{h \to 0^+} \frac{((x+h) \oplus x') - (x \oplus x')}{h} =1.
    \end{equation}
    However, for $h<2^{-\Omega}$ we have $(x-h) \oplus x' \geq 2^{-a(x)}$ so the left partial derivative is $\partial_x^- (x \oplus x') = \infty$. The same arguments can be used to show $\partial_{x'}^+ (x \oplus x') = 1$ while $\partial_{x'}^- (x \oplus x') = \infty$.
\end{proof}

\begin{corollary} \label{lemma:b2_alpha234_digital_dright_derivs}
    Fix the base $b=2$. For $\alpha \in \{3,4\}$ and dyadic rationals $x,x' \in \bbB$ the right derivative satisfies 
    $$\partial_x^+ \ddK_{\alpha,0,0}(x,x') = \partial_{x'}^+ \ddK_{\alpha,0,0}(x,x') = -2 \left(1+\ddK_{\alpha-1,0,0}(x,x')\right).$$
\end{corollary}
\begin{proof}
    This follows immediately from combining \Cref{lemma:partial_plus_widetildeK234_og} and \Cref{lemma:partial_digital_shift}.
\end{proof}

\begin{theorem} \label{thm:derivs_dddK}
    Fix the base $b=2$. For any $\alpha \in \bbR_{>0}$ and any dyadic rationals $x,x' \in \bbB$ the right derivative satisfies
    $$\partial_x^+ \dddK_\alpha(x,x') = \partial_{x'}^+ \dddK_\alpha(x,x') = 0.$$
\end{theorem}
\begin{proof}
    For $x\neq x'$ the claim is obvious since $\beta((x+h) \oplus x')=\beta(x \oplus x')=\beta(x \oplus (x'+h))$ for $h \in (0,1)$ small enough. When $x=x'$ we have 
    \begin{align*}
        \lim_{h \to 0^+} \frac{\dddK_\alpha(x+h,x)-\dddK_\alpha(x,x)}{h} &= \lim_{h \to 0^+} \frac{\dddK_\alpha(h,0)-\dddK_\alpha(0,0)}{h} \\
        &= -\left(\frac{2^{2\alpha+1}}{2^{\alpha+1}-2} + 2^\alpha\right) \lim_{h \to 0^+} \frac{2^{-\alpha\beta(h)}}{h}.
    \end{align*}
    Now, using the fact that $2^{-\beta(h)} \leq h \leq 2^{1-\beta(h)}$ whenever $h \in (0,1)$,  we have 
    \begin{align*}
        0 =  2^{-\alpha} \lim_{h \to 0^+} \frac{h^{\alpha}}{h} \leq \lim_{h \to 0^+} \frac{2^{-\alpha\beta(h)}}{h} \leq  \lim_{h \to 0^+} \frac{h^\alpha}{h} = 0,
    \end{align*}
    so $\partial_x^+ \dddK_\alpha(x,x') = 0$. Swapping $x$ and $x'$ in the above statements gives the proof for $\partial_{x'}^+ \dddK_\alpha(x,x') = 0$. 
\end{proof}

\begin{theorem} \label{thm:msdiff_K}
    For base $b=2$ and any dyadic rational $x \in \bbB$, the right limits satisfy
    \begin{align*}
        \lim_{h \to 0^+} \frac{\ddK_{2,0,0}(x,x) - \ddK_{2,0,0}(x+h,x)}{h^2} &= \infty, \\
        \lim_{h \to 0^+} \frac{\ddK_{3,0,0}(x,x) - \ddK_{3,0,0}(x+h,x)}{h^2} &= \infty, \\
        \lim_{h \to 0^+} \frac{\ddK_{4,0,0}(x,x) - \ddK_{4,0,0}(x+h,x)}{h^2} &= \infty.
    \end{align*}
\end{theorem}
\begin{proof}
    For $\alpha=2$, 
    \begin{align*}
        \lim_{h \to 0^+}& \frac{\ddK_{2,0,0}(x,x) - \ddK_{2,0,0}(x+h,x)}{h^2} = \lim_{h \to 0^+} \frac{\ddK_{2,0,0}(x \oplus x,0) - \ddK_{2,0,0}((x+h) \oplus x,0)}{h^2} \\
        &= \lim_{h \to 0^+} \frac{\ddK_{2,0,0}(0,0) - \ddK_{2,0,0}(h,0)}{h^2} \\
        &= \lim_{h \to 0^+} \frac{\beta(h)h+\frac{5}{2}2^{-\beta(h)}}{h^2} \\
        &\geq \lim_{h \to 0^+} \frac{\beta(h)}{h} \\
        &= \infty
    \end{align*}
    For $\alpha=3$, notice that $2^{-\beta(h)} \leq h$ whenever $h \in (0,1)$, so 
    \begin{align*}
        \lim_{h \to 0^+}& \frac{\ddK_{3,0,0}(x,x) - \ddK_{3,0,0}(x+h,x)}{h^2} = \lim_{h \to 0^+} \frac{\ddK_{3,0,0}(x \oplus x,0) - \ddK_{3,0,0}((x+h) \oplus x,0)}{h^2} \\
        &= \lim_{h \to 0^+} \frac{\ddK_{3,0,0}(0,0) - \ddK_{3,0,0}(h,0)}{h^2} \\
        &= \lim_{h \to 0^+} \frac{-\beta(h)h^2+5\left[1-2^{-\beta(h)}\right]h+\frac{43}{18} \cdot 2^{-2\beta(h)}}{h^2} \\
        &\geq \lim_{h \to 0^+} \left(5\left[1-2^{-\beta(h)}\right]h^{-1}-\beta(h)\right) \\
        &\geq \lim_{h \to 0^+} \left(5\left[1-h\right]h^{-1}-\beta(h)\right) \\
        &= \lim_{h \to 0^+} \left(5h^{-1}+\lfloor\log_2(h)\rfloor-5\right) \\
        &= \infty
    \end{align*}
    For $\alpha=4$, notice $S(h)/48 \leq 1/42$ so $\beta(h)\left[\frac{1}{48} S(h) - \frac{1}{42}\right]<0$ where $S$ is given in \eqref{eq:K4_sum_term}. Now, using the fact that $h/2 \leq 2^{-\beta(h)} \leq h$ whenever $h \in (0,1)$, we have
    \begin{align*}
        \lim_{h \to 0^+}& \frac{\ddK_{4,0,0}(x,x) - \ddK_{4,0,0}(x+h,x)}{h^2} = \lim_{h \to 0^+} \frac{\ddK_{4,0,0}(x \oplus x,0) - \ddK_{4,0,0}((x+h) \oplus x,0)}{h^2} \\
        &= \lim_{h \to 0^+} \frac{\ddK_{4,0,0}(0,0) - \ddK_{4,0,0}(h,0)}{h^2} \\
        &\geq \lim_{h \to 0^+} \frac{\frac{2}{3}\beta(h)h^3-5\left[1-2^{-\beta(h)}\right]h^2 + \frac{43}{9}\left[1-2^{-2\beta(h)}\right]h +\frac{701}{294}\left[2^{-3\beta(h)}\right]}{h^2} \\
        &= \lim_{h \to 0^+} \frac{-5\left[1-2^{-\beta(h)}\right]h^2 + \frac{43}{9}\left[1-2^{-2\beta(h)}\right]h +\frac{701}{294}\left[2^{-3\beta(h)}\right]}{h^2} \\
        &\geq \lim_{h \to 0^+} \left(\frac{43}{9}\left[1-2^{-2\beta(h)}\right]h^{-1}-5\left[1-2^{-\beta(h)}\right]\right) \\
        &\geq \lim_{h \to 0^+} \left(\frac{43}{9}\left[1-h^2\right]h^{-1}-5\left[1-h/2\right]\right) \\
        &= \infty
    \end{align*} 
\end{proof}

\begin{theorem} \label{thm:msdiff_Kdotdot}
    Fix the base $b=2$. For $\alpha \in \bbR_{>0}$ and any dyadic rational $x \in \bbB$, the right limit satisfies
    $$\lim_{h \to 0^+} \frac{\dddK_\alpha(x,x) - \dddK_\alpha(x+h,x)}{h^2}= \begin{cases} \infty, &\alpha \in (0,2) \\ \text{does not exist}, & \alpha=2 \\ 0, & \alpha>2\end{cases}.$$
    For $\alpha \in \bbR_{>2}$, the right limit satisfies 
    $$\lim_{h_1,h_2 \to 0+} \frac{\dddK_\alpha(x+h_1,x+h_2)-\dddK_\alpha(x+h_1,x) - \dddK_\alpha(x+h_2,x) + \dddK_\alpha(x,x)}{h_1h_2} = 0.$$
\end{theorem}
\begin{proof}
    For the first claim, 
    \begin{align*}
        \lim_{h \to 0^+}& \frac{\dddK_\alpha(x,x) - \dddK_\alpha(x+h,x)}{h^2} = \lim_{h \to 0^+} \frac{\dddK_\alpha(x \oplus x,0) - \dddK_\alpha((x+h) \oplus x,0)}{h^2} \\
        &= \lim_{h \to 0^+} \frac{\dddK_\alpha(0,0) - \dddK_\alpha(h,0)}{h^2} \\
        &= \left(\frac{2^{2\alpha+1}}{2^{\alpha+1}-2}+2^{\alpha}\right) \lim_{h \to 0^+} \frac{2^{-\alpha\beta(h)}}{h^2}.
    \end{align*}
    Recalling that $h/2 \leq 2^{-\beta(h)} \leq h$ whenever $h \in (0,1)$,
    $$\lim_{h \to 0^+}  \frac{2^{-\alpha\beta(h)}}{h^2} \in \left[2^{-\alpha}\lim_{h \to 0^+} \frac{h^{\alpha}}{h^2}, \lim_{h \to 0^+} \frac{h^{\alpha}}{h^2}\right] = \left[2^{-\alpha}\lim_{h \to 0^+} h^{\alpha-2}, \lim_{h \to 0^+} h^{\alpha-2}\right].$$
    This proves the $\alpha \in (0,2)$ and $\alpha>2$ claims. To see the $\alpha=2$ claim, take two subsequences $(h_p)_{p \in \bbN} := (1/2^p)_{p \in \bbN}$ and $(w_p)_{p \in \bbN} = (1/2^p + 1/2^{p+1})_{p \in \bbN}$ which are both monotonically decreasing to $0$ as $p \to \infty$. Then 
    $$\lim_{p \to \infty} \frac{2^{-2\beta(h_p)}}{h_p^2} = \lim_{p \to \infty} \frac{2^{-2p}}{2^{-2p}} = 1 \qqtqq{but} \lim_{p \to \infty} \frac{2^{-2\beta(w_p)}}{w_p^2} = \lim_{p \to \infty} \frac{2^{-2p}}{(2^{-p} + 2^{-p-1})^2} = \frac{4}{9},$$
    so the limit does not exist. 

    For the second claim, with $\alpha>2$, 
    \begin{align*}
        \lim_{h_1,h_2 \to 0+} &\frac{\dddK_\alpha(x+h_1,x+h_2)-\dddK_\alpha(x+h_1,x) - \dddK_\alpha(x+h_2,x) + \dddK_\alpha(x,x)}{h_1h_2} \\
        &=  \lim_{h_1,h_2 \to 0+} \frac{\dddK_\alpha(h_1 \oplus h_2)-\dddK_\alpha(h_1,0) - \dddK_\alpha(h_2,0) + \dddK_\alpha(0,0)}{h_1h_2} \\
        &= \left(\frac{2^{2\alpha+1}}{2^{\alpha+1}-2}+2^{\alpha}\right)\lim_{h_1,h_2 \to 0^+} \frac{2^{-\alpha\beta(h_1)}+2^{-\alpha\beta(h_2)}-2^{-\alpha\beta(h_1 \oplus h_2)}}{h_1 h_2}.
    \end{align*}
    Let us denote 
    $$L := \frac{2^{-\alpha\beta(h_1)}+2^{-\alpha\beta(h_2)}-2^{-\alpha\beta(h_1 \oplus h_2)}}{h_1 h_2}.$$ 
    Suppose $h_1,h_2 \in (0,1)$ and $\beta(h_1)=p$ and $\beta(h_2)=q$, i.e., $h_1 = 2^{-p}+\cdots$ and $h_2 = 2^{-q}+\cdots$. Notice that $\beta(h_1 \oplus h_2) \geq \min(p,q)$ so $(1-\alpha)\beta(h_1 \oplus h_2) \leq (1-\alpha)\min(p,q)$ and $2^{(1-\alpha)\beta(h_1 \oplus h_2)} \leq 2^{(1-\alpha) \min(p,q)}$. Therefore, 
    $$2^{-\alpha p} + 2^{-\alpha q} - 2^{-\alpha\beta(h_1 \oplus h_2)} \geq 2^{-\alpha p} + 2^{(1-\alpha)q} - 2^{-\alpha\min(p,q)} \geq 0,$$
    so $L \geq 0$. 
    We know $2^{-p} \leq h_1 \leq 2^{-p+1}$ and $2^{-q} \leq h_2 \leq 2^{-q+1}$. Let $\alpha = 2+\kappa$ for some $\kappa>0$. Then
    \begin{enumerate}
        \item If $p=q$, then 
        $$L \leq \frac{2^{-\alpha p}+2^{-\alpha p}}{2^{-p}2^{-p}} = 2^{(2-\alpha)p+1} = 2^{1-\kappa p},$$
        so $1-\kappa p \leq -\tau$ whenever $q=p \geq (1+\tau)/\kappa$, which ensures $L \leq 2^{-\tau}$
        \item If $p>q$, then 
        $$L \leq \frac{2^{-\alpha p}+2^{-\alpha q}-2^{-\alpha q}}{2^{-p}2^{-q}} = 2^{(1-\alpha)p+q} \leq 2^{(2-\alpha)q} = 2^{-\kappa q},$$
        so $-\kappa q \leq -\tau$ whenever $p > q \geq \tau/\kappa$, which ensures $L \leq 2^{-\tau}$. 
        \item If $q>p$, then 
        $$L \leq \frac{2^{-\alpha p}+2^{-\alpha q}-2^{-\alpha p}}{2^{-p}2^{-q}} = 2^{(1-\alpha)q+p} \leq 2^{(2-\alpha)p} = 2^{-\kappa p},$$
        so $-\kappa p \leq -\tau$ whenever $q > p \geq \tau/\kappa$, which ensures $L \leq 2^{-\tau}$. 
    \end{enumerate} 
    Let $\varepsilon>0$ and choose $\tau \in \bbN$ so that $2^{-\tau} < \varepsilon$. For $\delta = 2^{-(1+\tau)/\kappa}$, if $\sqrt{h_1^2+h_2^2} \leq \delta$ then $0<h_1 \leq \delta$ and $0<h_2 \leq \delta$ which ensures $p \geq (1+\tau)/\kappa$ and $q \geq (1+\tau)/\kappa$. Then all three of the above cases are satisfied which ensures $0 \leq L \leq 2^{-\tau}<\varepsilon$. 
\end{proof} 

\Chapter{Gaussian Processes} \label{sec:gps}

This section will discuss Gaussian process (GP) regression. We will begin with standard GPs in \Cref{sec:gps_standard} and then review some popular kernel forms in \Cref{sec:gp_kernels}, including summaries of the shift-invariant (SI) kernels from \Cref{sec:si_kernels} and digitally-shift-invariant (DSI) kernels from \Cref{sec:dsi_kernels}. Then, \Cref{sec:gp_rkhs_connections} will briefly mention a few connections between RKHS kernel interpolation and GP regression. Next, \Cref{sec:fast_gps} will discuss fast GPs using lattices with SI kernels or digital nets with DSI kernels using the fast kernel computations of \Cref{sec:fast_kernel_methods}. After discussing standard GPs and fast GPs, \Cref{sec:mtgps} and \Cref{{sec:fmtgps}} will detail multitask GPs and fast multitask GPs respectively with the latter constituting our novel contribution to the GP literature. \Cref{sec:gps_deriv_informeds} will discuss (fast) derivative-informed GPs as a special case of (fast) multitask GPs. Finally, \Cref{sec:gps_numerical_experiments} will present some numerical experiments using our \texttt{FastGPs} package.

\Section{Standard GPs} \label{sec:gps_standard}

Here we present some notation and background for standard GP regression in \Cref{sec:sgp_background} and then discuss kernel hyperparameter optimization in \Cref{sec:hpopt_sgps}. Standard GP regression and hyperparameter optimization has been thoroughly studied, see the cornerstone work \cite{rasmussen.gp4ml} or \cite{fasshauer.meshfree_approx_methods_matlab,craven1978smoothing,golub1979generalized,wahba1990spline}.

\Subsection{Background} \label{sec:sgp_background}

Let us assume $f$ is a GP on $\calX \subseteq \bbR^d$ with prior mean $M: \calX \to \bbR$ and a SPSD (covariance) kernel $K: \calX \times \calX \to \bbR$ (see \Cref{def:spd_SPSD_kernels}): 
\begin{align*}
    f &\sim \mathrm{GP}(M,K), \\
    \bbE[f(\bx)] &= M(\bx), \\
    \Cov[f(\bx),f(\bx')] &= K(\bx,\bx').
\end{align*}
The fast GPs we present use $\calX = [0,1)^d$, and we will always assume evaluating the kernel $K(\bx,\bx')$ costs $\calO(d)$. The parameterization of the prior mean $M$ may be quite flexible, e.g., a neural network. 

Suppose we have observed 
\begin{equation}
    \by = \boldsymbol{f}+\bvarepsilon, \qquad \boldsymbol{f} = (f(\bx_i))_{i=0}^{n-1}, \qquad (\bx_i)_{i=0}^{n-1} \subset \calX, \qquad \bvarepsilon \sim \calN(\bzero,\xi \mI).
\end{equation}
Here $\by$ are noisy function evaluations at collocation points $(\bx_i)_{i=0}^{n-1}$, and the zero-mean IID Gaussian noise $\bvarepsilon$ has common noise variance (nugget) $\xi>0$.

\begin{proposition}
    \label{proposition:gaussian_conditional_distribution}
    For arbitrary $T_1,T_2 \in \bbN$, suppose  $\ba_1 \in \bbR^{T_1}$ and $\ba_2 \in \bbR^{T_2}$ are Gaussian with 
    $$\begin{pmatrix} \ba_1 \\ \ba_2 \end{pmatrix} \sim \calN\left(\begin{pmatrix} \bmu_1 \\ \bmu_2 \end{pmatrix}, \begin{pmatrix} \mC_{11} & \mC_{12} \\ \mC_{12}^\intercal & \mC_{22} \end{pmatrix}\right)$$
    where $\bmu_1 \in \bbR^{T_1}, \bmu_2 \in \bbR^{T_2}$ and $\mC_{11} \in \bbR^{T_1 \times T_1},\mC_{12} \in \bbR^{T_1 \times T_2},\mC_{22} \in \bbR^{T_2 \times T_2}$. Then 
    $$\ba_2 \mid \ba_1 \sim \calN\left(\bmu_2 + \mC_{12}^\intercal \mC_{11}^{-1} (\ba_1 - \bmu_1), \mC_{22} - \mC_{12}^\intercal \mC_{11}^{-1} \mC_{12}\right).$$
\end{proposition}
\begin{proof}
    See \cite[Appendix A.2]{rasmussen.gp4ml}.
\end{proof}

The posterior distribution of $f$ given noisy observations $\by$ is also a Gaussian process with closed form posterior mean and variance: 
\begin{align}
    f | \by &\sim \mathrm{GP}(\bbE[f | \by],\Cov[f,f | \by]), \nonumber \\
    \bbE[f(\bx) | \by] &= M(\bx)+\bK^\intercal(\bx) \tmK^{-1}(\by-\bM), \label{eq:post_mean}\\
    \Cov[f(\bx),f(\bx') | \by] &= K(\bx,\bx') - \bK^\intercal(\bx) \tmK^{-1} \bK(\bx'). \label{eq:post_cov}
\end{align}
Here $\bK(\bx) := (K(\bx,\bx_i))_{i=0}^{n-1}$ is a vector of kernel evaluations, $\bM = (M(\bx_i))_{i=0}^{n-1}$ is a vector of prior mean evaluations, and $\tmK$ is the noisy version of the Gram matrix, i.e.,
$$\tmK = \mK + \xi \mI, \qquad \mK := (K(\bx_i,\bx_{i'}))_{i,i'=0}^{n-1}.$$
Notice the posterior covariance \eqref{eq:post_cov} only depends on the sampling locations $(\bx_i)_{i=0}^{n-1}$ and kernel $K$, not the prior mean $M$ or noisy evaluations $\by$. We will always assume a positive noise variance, $\xi>0$, so $\tmK$ is always a SPD matrix (see \Cref{def:spd_SPSD_matrices}). In practice, if observations are noise free we still set the nugget $\xi$ to a small positive number for numerical stability, e.g., $\xi = 10^{-8}$. 

The computational bottleneck for GPs is the requirement to solve the linear systems in the noisy Gram matrix $\tmK$ which appear in the equations for the posterior mean \eqref{eq:post_mean} and posterior covariance \eqref{eq:post_cov}. Additionally, one may need to evaluate the determinant $\lvert \tmK \rvert$ if performing kernel hyperparameter optimization as we discuss in the next section. To solve a linear system in $\tmK$, standard practice is to compute the Cholesky decomposition, which costs $\calO(n^3)$ computations and $\calO(n^2)$ storage, and then perform back-substitution solves at cost $\calO(n^2)$. Inference of the posterior mean \eqref{eq:post_mean} and posterior covariance \eqref{eq:post_cov} can then be performed at $\calO(n)$ and $\calO(n^2)$ cost respectively. \Cref{tab:com_kernel_costs} summarizes these costs and compares against the fast GP methods we will present in \Cref{sec:fast_gps}. 

\Subsection{Hyperparameter Optimization (HPOPT)} \label{sec:hpopt_sgps}

The kernel $K$ typically relies on hyperparameters $\btheta$, e.g., $\btheta = \{\gamma,\bEta\}$ for a global scale $\gamma$ and lengthscales $\bEta$, or $\btheta = \{\balpha,\gamma,\bEta\}$ to additionally optimize smoothness parameters $\balpha$. We consider two loss functions below: the negative marginal log-likelihood (NMLL) and the  generalized cross validation (GCV). Analysis of the NMLL and GCV losses was performed in \cite{rathinavel.bayesian_QMC_thesis} for both standard GPs and the fast GP variants we will present in \Cref{sec:fast_gps}. \cite{rathinavel.bayesian_QMC_thesis} also analyzes a full Bayesian loss which we do not consider here.

\Subsection{HPOPT: Negative Marginal Log-Likelihood} \label{sec:mll_sgp}

Perhaps the most popular optimization objective for GPs is the NMLL, which, up to an additive constant independent of $\btheta$, is proportional to 
\begin{equation}
    \calL_\mathrm{NMLL}(\btheta) := (\by-M(\bx))^\intercal \tmK^{-1} (\by-M(\bx)) + \log \lvert \tmK \rvert.
    \label{eq:mll}
\end{equation}
Here we have hidden the implicit dependence of $\tmK$ on kernel hyperparameters $\btheta$. If the prior mean is a constant $\tau \in \bbR$, i.e., $M(\bx) = \tau$ for all $\bx \in \calX$, then $\tau$ may be treated as an additional hyperparameter. The NMLL loss is then proportional to 
\begin{equation*} 
    \calL_\mathrm{NMLL}(\tau,\btheta) := (\by-\tau \bone)^\intercal \tmK^{-1} (\by-\tau \bone) + \log \lvert \tmK \rvert.
\end{equation*}
Setting the derivative with respect to $\tau$ equal to $0$ gives the optimal constant prior mean 
\begin{equation*}
    \stau = \frac{\bone^\intercal \tmK^{-1} \by}{\bone^\intercal \tmK^{-1} \bone}.
    \label{eq:optimal_tau_mll}
\end{equation*}

\Subsection{HPOPT: Generalized Cross Validation} \label{sec:gcv_sgp}

The generalized cross validation loss \cite{craven1978smoothing,golub1979generalized,wahba1990spline} is proportional to
\begin{equation}
    \calL_\mathrm{GCV}(\btheta) := \frac{(\by - M(\bx))^\intercal \tmK^{-2} (\by - M(\bx))}{\mathrm{trace}^2(\tmK^{-1})}
    \label{eq:gcv}
\end{equation}
where again we have hidden the implicit dependence of $\tmK$ on kernel hyperparameters $\btheta$. If the prior mean is a constant $\tau \in \bbR$, i.e., $M(\bx) = \tau$ for all $\bx \in \calX$, the GCV loss becomes proportional to
\begin{equation*}
    \calL_\mathrm{GCV}(\tau,\btheta) := \frac{(\by - \tau \bone)^\intercal \tmK^{-2} (\by - \tau \bone)}{\mathrm{trace}^2\left(\tmK^{-1}\right)}.
\end{equation*}
Setting the derivative with respect to $\tau$ equal to $0$ gives the optimal prior mean constant
\begin{equation*}
    \stau = \frac{\bone^\intercal \tmK^{-2} \by}{\bone^\intercal \tmK^{-2} \bone}.
    \label{eq:optimal_tau_gcv}
\end{equation*}

% \Subsection{HPOPT: Weighted Leave-One-Out Cross Validation} \label{sec:wloocv_sgp}

% The WLOOCV with weights $\bw = (w_i)_{i=0}^{n-1} \in \bbR^n$ is 
% \begin{equation}
%     \calL_\mathrm{WLOOCV}(\btheta) = \sum_{i=0}^{n-1} w_i (y_i - \bbE[f(\bx_i) | \by_{-i}])^2 = (\by - \bM)^\intercal  \mC (\by - \bM)
%     \label{eq:wloocv}
% \end{equation}
% where $\by_{-i}$ excludes the $i^\mathrm{th}$ component and $\mC = \tmK^{-1} \mD \tmK^{-1}$ with 
% $$\mD = \mathrm{diag}((w_i/(\mA_{ii})^2)_{i=0}^{n-1})$$
% where $\mA = \tmK^{-1}$. Again we have hidden the implicit dependence of $\tmK$ on kernel hyperparameters $\btheta$. Note that GCV loss in \Cref{sec:gcv_sgp} is equivalent to setting the weights $\bw = \bone$ and using $\mD = \diag^{-1}\left(\frac{1}{n}\sum_{i=0}^{n-1} \mA_{ii}\right)$.  When the prior mean is a constant $\tau \in \bbR$, i.e., $M(\bx) = \tau$ for all $\bx \in \calX$, the WLOOCV loss becomes 
% $$\calL(\btheta) = (\by - \tau \bone)^\intercal \mC (\by - \tau \bone).$$
% Setting the derivative with respect to $\tau$ equal to $0$ gives the optimal 
% $$\stau = \frac{\bone^\intercal \mC \by}{\bone^\intercal \mC \bone}.$$

\Section{Kernels} \label{sec:gp_kernels}

This section will review a number of available SPD kernels. This is by no means an exhaustive list, see \cite[Chapter 4]{rasmussen.gp4ml}, \cite{fasshauer.meshfree_approx_methods_matlab}, or \cite[Chapter 2]{duvenaud2014automatic} for additional details. We will begin by collecting results from \Cref{sec:si_kernels} on shift-invariant (SI) kernels and \Cref{sec:dsi_kernels} on digitally-shift-invariant (DSI) kernels. We will only present product-form SI and DSI kernels \eqref{eq:product_kernels}. Mixture of product SI and DSI kernels \eqref{eq:mixture_product_kernel} are also theoretically compatible in the fast GPs framework we will present in \Cref{sec:fast_gps}, but their evaluation cost is exponential in the dimension and therefore impractical in most applications. After presenting product-form SI and DSI kernels, we will review some popular kernels including the squared exponential (SE) kernel, the Mat\'ern kernels, and the rational quadratic kernel. \Cref{fig:gp_draws} visualizes a few draws from a GP prior when using the popular SE kernel in \Cref{def:se_kernel}, a SI kernel from \Cref{def:si_kernels}, or a DSI kernel from \Cref{def:dsi_kernels}. 

\begin{figure}[!ht]
    \centering
    \includegraphics[width=1\linewidth]{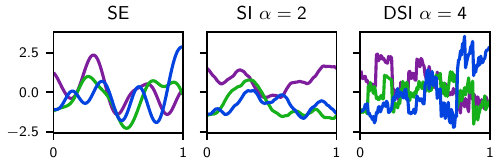}
    \caption{Prior draws from $1$-dimensional GPs having a squared exponential (SE) kernel with $\eta = 10^{-2}$, a shift-invariant (SI) kernel with $\alpha=2$ and $\eta = 1$ (see \Cref{sec:fft_si_kernels_r1lattices} or \Cref{def:si_kernels} built from univariate $\tK_\alpha$), and a digitally-SI (DSI) kernel with $\alpha=4$ and $\eta=1$ (see \Cref{sec:fwht_dsi_kernels_dnb2s} or \Cref{def:dsi_kernels} built form univariate $\ddK_{\alpha,0,0}$). The global scaling parameter $\gamma$ is chosen so all draws have similar ranges. Note the SI draws are periodic and the DSI draws are discontinuous.}
    \label{fig:gp_draws}
\end{figure}

\Subsection{Summary of Shift-Invariant (SI) Kernels}

\begin{definition} \label{def:si_kernels}
    The shift-invariant (SI) product kernels studied in \Cref{sec:si_kernels} take the form 
    $$K(\bx,\bx') = \gamma \prod_{j=1}^d \left(1+\eta_j K_{\alpha_j}(x_j,x_j')\right), \qquad \forall\; \bx,\bx' \in [0,1)^d$$
    where $\balpha \in \bbR_{>1/2}^d$ are smoothness parameters. For $\alpha \in \bbN$ and $\beta,\beta' \in \bbN_0$ satisfying $2 \alpha - \beta - \beta' > 1$, we have 
    \begin{align*}
        K_{\alpha}^{(\beta,\beta')}(x,x') &:= \tK_{\alpha}^{(\beta,\beta')}((x-x') \bmod 1,0) \\
        &= \frac{(-1)^{\alpha+\beta'+1}(2 \pi)^{2 \alpha}}{(2\alpha-\beta-\beta')!} B_{2\alpha - \beta - \beta'}((x-x') \bmod 1)
    \end{align*}
    where $B_p$ is the $p^\mathrm{th}$ Bernoulli polynomial. Explicitly, 
    \begin{equation}
        B_p(x) = \sum_{j=0}^p \binom{p}{j} \tB_{p-j} x^j = \sum_{j=0}^p c_{p,j} x^j
        \label{eq:bernoulli_poly}
    \end{equation}
    where $\tB_j$ is the $j^\mathrm{th}$ Bernoulli numbers. Some $\bc_p \in \bbR^{p+1}$ for small $p$ are given in \Cref{tab:bernoulli_coeffs}. 
    \begin{table}[!ht]
        % Table[(p choose j)*Bernoulli(p-j),{j,0,p}] for p=10
        \caption{Coefficients $\bc_p$ for the Bernoulli polynomial $B_p$ defined in \eqref{eq:bernoulli_poly}.}
        \centering
        \begin{tabular}{r | c c c c c c c c c c} 
            $p/j$ & $0$ & $1$ & $2$ & $3$ & $4$ & $5$ & $6$ & $7$ & $8$ & $9$ \\
            \hline 
            $0$ & $1$ \\ 
            $1$ & $-1/2$ & $1$ \\ 
            $2$ & $1/6$ & $-1$ & $1$ \\
            $3$ & $0$ & $1/2$ & $-3/2$ & $1$ \\
            $4$ & $-1/30$ & $0$ & $1$ & $-2$ & $1$ \\
            $5$ & $0$ & $-1/6$ & $0$ & $5/3$ & $-5/2$ & $1$ \\
            $6$ & $1/42$ & $0$ & $-1/2$ & $0$ & $5/2$ & $-3$ & $1$ \\
            $7$ & $0$ & $1/6$ & $0$ & $-7/6$ & $0$ & $7/2$ & $-7/2$ & $1$ \\
            $8$ & $-1/30$ & $0$ & $2/3$ & $0$ & $-7/3$ & $0$ & $14/3$ & $-4$ & $1$ \\
            $9$ & 0 & $-3/10$ & $0$ & $2$ & $0$ & $-21/5$ & $0$ & $6$ & $-9/2$ & 1 \\
            %$10$ & $5/66$ & $0$ & $-3/2$ & $0$ & $5$ & $0$ & $-7$ & $0$ & $15/2$ & $-5$ & $1$ %\\
            % $11$ & $0$ & $5/6$ & $0$ & $-11/2$ & $0$ & $11$ & $0$ & $-11$ & $0$ & $55/6$ & $-11/2$ & $1$ \\ 
            % $12$ & $-691/2730$ & $0$ & $5$ & $0$ & $-33/2$ & $0$ & $22$ & $0$ & $-33/2$ & $0$ & $11$ & $-6$ & $1$ \\
            \hline 
        \end{tabular}
        \label{tab:bernoulli_coeffs}
    \end{table}
    For $\bbeta,\bbeta' \in \bbN_0$ satisfying $2\balpha - \bbeta-\bbeta' > 1$ elementwise, we have 
    \begin{align*}
        K^{(\bbeta,\bbeta')}(\bx,\bx') &= \gamma \prod_{j=1}^d \left(1_{\beta_j+\beta_j'=0} + \eta_j K^{(\beta_j,\beta_j')}(x_j,x_j')\right), \\
        \int_{[0,1)^d} K^{(\bbeta,\bbeta')}(\bx,\bx') \D \bx' &= \int_{[0,1)^d} \int_{[0,1)^d} K^{(\bbeta,\bbeta')}(\bx,\bx') \D \bx' \D \bx = \gamma
    \end{align*}
    for all $\bx \in [0,1)^d$ where $1_{\beta_j+\beta_j'=0}$ is $1$ if $\beta_j+\beta_j'=0$ and $0$ otherwise. In \Cref{sec:gps_deriv_informeds} we will see that, for $\balpha>3/2$ elementwise, GPs using these SI kernels are mean squared differentiable. 
\end{definition}

\Subsection{Summary of Digitally-Shift-Invariant (DSI) Kernels}

\begin{definition} \label{def:dsi_kernels}
    The digitally-shift-invariant (DSI) product kernels studied in \Cref{sec:dsi_kernels} take the form 
    $$K(\bx,\bx') = \gamma \prod_{j=1}^d \left(1+\eta_j K_{\alpha_j}(x_j,x_j')\right), \qquad \forall\; \bx,\bx' \in [0,1)^d$$
    where $\balpha \in \bbN$ are smoothness parameters. Here we only list results for the base $b=2$ kernels, although for most results \Cref{sec:dsi_kernels} has analogs for arbitrary prime bases. We have at least two computable DSI kernel options which are defined using $\beta(x) := - \lfloor \log_2(x) \rfloor$ and $x \oplus x' := \sum_{k \in \bbN} ((\mx_k + \mx_k') \bmod 2) 2^{-k}$, the XOR of binary expansions $x = \sum_{k \in \bbN} \mx_k 2^{-k}$ and $x' = \sum_{k \in \bbN} \mx_k' 2^{-k}$, with $x,x' \in [0,1)$. 
    \begin{enumerate}
        \item For $\alpha \in \bbN_{\geq 2}$, 
        \begin{align*}
            K_\alpha(x,x') &:= \ddK_{\alpha,0,0}(x \oplus x',0) \\
            \ddK_{\alpha,0,0}(x,0) &= \begin{cases} -1+-\beta(x) x + \frac{5}{2}\left[1-t_1(x)\right], & \alpha = 2 \\ -1+\beta(x)x^2-5\left[1-t_1(x)\right]x+\frac{43}{18}\left[1-t_2(x)\right], & \alpha = 3 \\ -1 -\frac{2}{3}\beta(x)x^3+5\left[1-t_1(x)\right]x^2 - \frac{43}{9}\left[1-t_2(x)\right]x \\
                \quad +\frac{701}{294}\left[1-t_3(x)\right]+\beta(x)\left[\frac{1}{48} \sum_{a=0}^\infty \frac{\wal_{2^a}(x)}{2^{3a}} - \frac{1}{42}\right], & \alpha = 4 \end{cases}
        \end{align*}
        where for $\nu \in \bbN$ we use $t_\nu(x) = 2^{-\nu \beta(x)}$. Note that for any $\alpha \in \bbN_{\geq 2}$ one may compute this $\ddK_{\alpha,0,0}(x,x')$ at $\calO(\alpha \# (x \oplus x'))$ costs using \cite[Theorem 2]{baldeaux.polylat_efficient_comp_worse_case_error_cbc} where $x \oplus x'$ has at most $\# (x \oplus x')$ digits in its base $b$ representation. The $\alpha \in \{3,4\}$ kernels are right differentiable everywhere on $[0,1)^d$, see \Cref{thm:msdiff_K}, but GPs which use these kernels are not mean squared differentiable as we will show in \Cref{sec:gps_deriv_informeds}. 
        \item For $\alpha \in \bbR_{>0}$, 
        \begin{align*}
            K_\alpha(x,x') := \dddK_\alpha(x,x') = \begin{cases} \upsilon(\alpha), & x=x' \\ \upsilon(\alpha) - 2^{-\alpha(\beta(x \oplus x')-1)}\left(\upsilon(\alpha)+1\right), & x \neq x' \end{cases} 
        \end{align*} 
        with $\upsilon(\alpha) = 2^{\alpha+1}/(2^{\alpha+1}-2)$. These kernels have right derivative $0$ everywhere, see \Cref{thm:derivs_dddK}, and for $\alpha>2$ they are mean squared differentiable as we will show in \Cref{sec:gps_deriv_informeds}. 
    \end{enumerate}
    For either of these DSI kernels,
    \begin{align*}
        \int_{[0,1)^d} K(\bx,\bx') \D \bx' &= \int_{[0,1)^d} \int_{[0,1)^d} K(\bx,\bx') \D \bx' \D \bx = \gamma
    \end{align*}
    for all $\bx \in [0,1)^d$ as was the case for the SI kernels summarized in \Cref{def:si_kernels}. 

\end{definition}

\Subsection{Popular Radially Symmetric Kernels}

Throughout the remainder of this section we will denote the scaled distance by 
$$\lVert \bx - \bx'\rVert_\bEta := \sqrt{\sum_{j=1}^d \frac{(x_j-x_j')^2}{2 \eta_j^2}}$$ 
where $\bEta \in \bbR_{>0}^d$ are lengthscales. Radially symmetric kernels are kernels which are only a function of the scaled distance between locations, i.e., $K(\bx,\bx') = Q(\lVert \bx - \bx'\rVert_\bEta)$ for some $Q: [0,\infty) \to \bbR$. We define a few of the more popular radially symmetric kernels below, see \cite[Appendix D]{fasshauer.meshfree_approx_methods_matlab} for a more in depth discussion. A global scale $\gamma>0$ will also be considered. 

\begin{definition} \label{def:se_kernel}
    The squared exponential (SE) kernel, also called the Gaussian kernel, takes the form 
    $$K(\bx,\bx') = \gamma \exp\left(-\lVert \bx - \bx'\rVert_\bEta^2\right).$$
    This may also be written as a product of one-dimensional kernels. For any $x \in [0,1]^d$ we have 
    \begin{align*} 
        \int_{[0,1]^d} K(\bx,\bx') \D \bx' &= \gamma \left(\frac{\pi}{2}\right)^{d/2} \prod_{j=1}^d \sqrt{\eta_j} \left(\mathrm{erf}\left(\frac{x}{\sqrt{2 \eta_j}}\right) - \mathrm{erf}\left(\frac{x-1}{\sqrt{2\eta_j}}\right)\right), \\
        \int_{[0,1]^d} \int_{[0,1]^d} K(\bx,\bx') \D \bx' \D \bx &= \gamma \prod_{j=1}^d \left[2\eta_j^2\left(\exp\left(-\frac{1}{2\eta_j^2}\right)-1\right)+\sqrt{2\pi}\eta_j\mathrm{erf}\left(\frac{1}{\sqrt{2}\eta_j}\right)\right]
    \end{align*}
    where $\mathrm{erf}(x) = 2/\sqrt{\pi} \int_0^x e^{-t^2} \D t$ is the error function. 
\end{definition}

\begin{definition} \label{def:matern_kernel}
    The Mat\'ern kernel with smoothness parameter $\alpha \in \bbR_{\geq 1/2}$ is 
    $$K(\bx,\bx') = \gamma \frac{2^{1-\alpha}}{\Gamma(\alpha)} \left(\sqrt{2 \alpha} \lVert \bx - \bx'\rVert_\bEta\right)^\alpha K_\alpha\left(\sqrt{2 \alpha} \lVert \bx - \bx'\rVert_\bEta\right)$$
    where $\Gamma$ is the gamma function and, in this definition only, $K_\alpha$ denotes the modified Bessel function. As $\alpha \to \infty$ the Mat\'ern kernel converges to the squared exponential kernel in \Cref{def:se_kernel}. For $\alpha = p+1/2$ with $p \in \bbN_0$ we have 
    \begin{align}
        K(\bx,\bx') &= \gamma \frac{p!}{(2p)!} e^{-\sqrt{2p+1}\lVert \bx - \bx'\rVert_\bEta} \sum_{i=0}^p \frac{(p+i)!}{i!(p-i)!} \left(2 \sqrt{2p+1} \lVert \bx - \bx'\rVert_\bEta\right)^{p-i} \nonumber \\
        &= \gamma \exp\left(-\sqrt{2 \alpha} \lVert \bx - \bx'\rVert_\bEta\right) \sum_{j=0}^{\alpha-1/2} c_{\alpha,j} \lVert \bx - \bx'\rVert_\bEta^j \label{eq:matern_alpha_half}
    \end{align}
    where $\bc_\alpha \in \bbR_{>0}^{\alpha+1/2}$ are some coefficients. Some $\bc_\alpha$ for small $\alpha$ are given in \Cref{tab:matern_kernel_constants}. 

    \begin{table}[!ht]
        % Table[p!/((2p)!) * ((p+j)!)/((j!)*((p-j)!)) * (2*sqrt(2*p+1))^(p-j),{j,p,0,-1}] for p=3
        \caption{Coefficients $\bc_\alpha$ for the Mat\'ern kernel when $\alpha = p+1/2$ with $p \in \bbN_0$ in the form of \eqref{eq:matern_alpha_half}.}
        \centering
        \begin{tabular}{r | c c c c c c c} 
            $\alpha/j$ & $0$ & $1$ & $2$ & $3$ & $4$ & $5$ & $6$ \\
            \hline 
            $1/2$ & $1$ \\
            $3/2$ & $1$ & $\sqrt{3}$ \\
            $5/2$ & $1$ & $\sqrt{5}$ & $5/3$ \\ 
            $7/2$ & $1$ & $\sqrt{7}$ & $14/5$ & $7 \sqrt{7}/15$ \\ 
            $9/2$ & $1$ & $3$ & $27/7$ & $18/7$ & $27/35$ \\ 
            $11/2$ & $1$ & $\sqrt{11}$ & $44/9$ & $11\sqrt{11}/9$ & $121/63$ & $121\sqrt{11}/945$ \\
            $13/2$ & $1$ & $\sqrt{13}$ & $65/11$ & $52\sqrt{13}/33$ & $338/99$ & $169\sqrt{13}/495$ & $2197/10395$ \\ 
            %$15/2$ & $1$ & $\sqrt{15}$ & $90/13$ & $25 \sqrt{15}/13$ & $750/143$ & $90\sqrt{15}/143$ & $100/143$ & $25 \sqrt{15}/1001$ %\\
            % $17/2$ & $1$ & $\sqrt{17}$ & $119/15$ & $34\sqrt{17}/15$ & $289/39$ & $578\sqrt{17}/585$ & $9826/6435$ & $19652\sqrt{17}/225225$ & $83521/2027025$ \\
            \hline 
        \end{tabular}
        \label{tab:matern_kernel_constants}
    \end{table}
     
\end{definition}

\begin{definition} \label{def:rq_kernel}
    The rational quadratic kernel with scale mixture parameter $\alpha \in \bbR_{>0}$ takes the form 
    $$K(\bx,\bx') = \gamma \left(1+\frac{\lVert \bx - \bx'\rVert_\bEta}{\alpha}\right)^{-\alpha}$$
    As $\alpha \to \infty$ the rational quadratic kernel converges to the squared exponential kernel in \Cref{def:se_kernel}.  
\end{definition}

\Section{Connections Between GPs and RKHS Kernel Interpolation} \label{sec:gp_rkhs_connections} 

In this section we review a few known connections between GP regression and RKHS kernel interpolation. Details on RKHS theory are available in \cite{aronszajn.theory_of_reproducing_kernels}. Additional connections between GPs and RKHSs are discussed in \cite{kanagawa.gp_rkhs_connections,owhadi.book_operator_adapted_wavelets}.  

\Subsection{Approximation}

\begin{theorem} \label{thm:representer_thm_deriv_info}% https://en.wikipedia.org/wiki/Representer_theorem
    Let $H$ be an RKHS with kernel $K$. Suppose we are given an arbitrary error function $E: \bbR^n \to \bbR$ and a regularization function $R: \bbR \to \bbR$. Then
    $$\exists\; \bc \in \bbR^n \qquad\text{such that}\qquad \bc^\intercal \bK(\cdot) \in \argmin_{f \in H} \left[E\left(\boldsymbol{f}\right)+R\left(\lVert f \rVert_H\right)\right]$$
    where $\boldsymbol{f} = (f(\bx_i))_{i=0}^{n-1}$ are function evaluations at some points $(\bx_i)_{i=0}^{n-1} \subset \calX$. 
\end{theorem}
\begin{proof}
    For any $f \in H$ define the orthogonal decomposition 
    $$f = \bc^\intercal \bK + \nu$$
    where $\nu(\bx_i) = \llangle \nu, K(\cdot,\bx_i) \rrangle_H = 0$ for $0 \leq i < n$. To see $E\left(\boldsymbol{f}\right)$ is independent of $\nu$, notice that 
    $$f(\bx_i) = \llangle f,K(\cdot,\bx_i) \rrangle_H = \llangle \bc^\intercal \bK + \nu,K(\cdot,\bx_i)\rrangle_H = \llangle \bc^\intercal, \bK(\cdot,\bx_i)\rrangle_H$$ 
    for $0 \leq i < n$. To see $R\left(\llVert f \rrVert_H\right)$ attains its minimum when $\nu=0$, recall $R$ is non-decreasing and 
    $$R(\lVert f \rVert_H) = R\left(\left\lVert \bc^\intercal \bK + \nu\right\rVert_H\right) = R\left(\sqrt{\left\lVert \bc^\intercal \bK \right\rVert_H^2 + \lVert \nu \rVert_H^2}\right).$$
\end{proof}

\begin{theorem} \label{thm:rkhs_approx_H}
    Suppose $\by \in \bbR^n$ and $\xi \in \bbR_{>0}$ are given. Then $f \nmid \by: \calX \to \bbR$ defined as 
    \begin{equation}
        f(\cdot) \nmid \by := \bK^\intercal(\cdot) \tmK^{-1} \by
        \label{eq:rkhs_f_given_y_beta}
    \end{equation}
    attains the minimum of 
    \begin{equation}
        \left(\by - \boldsymbol{f}\right)^\intercal \left(\by - \boldsymbol{f}\right) + \xi^2 \lVert f \rVert_H^2
        \label{eq:rkhs_L_H}
    \end{equation}
    among all $f \in H$.
\end{theorem}
\begin{proof}
    \Cref{thm:representer_thm_deriv_info} implies a minimizer takes the form
    \begin{equation}
        f \nmid \by = \bK^\intercal \bc
        \label{eq:rkhs_f_given_y}
    \end{equation}
    for some $\bc \in \bbR^n$. Since $\lVert f \nmid \by \rVert_H^2 = \bc^\intercal \mK \bc$, plugging $f \nmid \by$ into \eqref{eq:rkhs_L_H} gives 
    \begin{align*}
        L(\bc) &:= \left(\by - \mK\bc\right)^\intercal \left(\by - \mK\bc\right) + \xi^2 \bc^\intercal \mK \bc \\
        &= \by^\intercal \by - 2 \bc^\intercal \mK\by + \bc^\intercal \mK \tmK\bc.
    \end{align*}
    $L(\bc)$ is minimized where the gradient is $\bzero$ which occurs when 
    \begin{align*}
        \mK \tmK\bc = \mK\by.
    \end{align*}
    Since $L(\bc)$ is convex in $\bc$ and $\tmK$ is non-singular, a minimizer of $L(\bc)$ is 
    $$\bc = \tmK^{-1}\by.$$
    Plugging this $\bc$ into \eqref{eq:rkhs_f_given_y} yields \eqref{eq:rkhs_f_given_y_beta}.
\end{proof}

\begin{corollary}
    The optimal RKHS kernel interpolant in \Cref{thm:rkhs_approx_H} is equivalent to the posterior mean \eqref{eq:post_mean} of a GP with zero prior mean, i.e., 
    $$f \nmid \by = \bbE[f \mid \by] \qquad\text{subject to}\qquad  M(\bx)=0, \qquad \forall\; \bx \in \calX.$$
\end{corollary}

\Subsection{Point-wise Error}

Let us define the noisy kernel  
\begin{equation}
    \tK(\bx,\bx') = K(\bx,\bx') + \xi 1_{\bx=\bx'}
    \label{eq:noisy_kernel}
\end{equation}
where $1_{\bx=\bx'}$ is $1$ if $\bx=\bx'$ and $0$ otherwise.

\begin{lemma} \label{lemma:sup_eq_norm_tH}
    For any RKHS $H$ and any $\bc \in \bbR^n$, we have 
    $$\sup_{f \in H: \lVert f \rVert_H \leq 1} \bc^\intercal \boldsymbol{f} = \llVert \bc^\intercal \bK \rrVert_H$$
\end{lemma}
\begin{proof}
    The reproducing property and the Cauchy--Schwartz inequality imply
    $$\sup_{f \in H: \lVert f \rVert_H \leq 1} \bc^\intercal \boldsymbol{f} = \sup_{f \in H: \lVert f \rVert_H \leq 1} \llangle f, \bc^\intercal \bK \rrangle_H \leq \llVert \bc^\intercal \bK \rrVert_H.$$
    Let $g = \bc^\intercal \bK$. Then $g/\lVert g \rVert_H \in H$ has $H$-norm $1$ and attains the above upper bound.  
\end{proof}

\begin{theorem} \label{thm:rkhs_x_error_H_noisy_no_deriv}
    Using the notations $H := H(K)$ and $\tH := H(\tK)$, we have that for any $\bx \in \calX$, 
    $$\sup_{\tf \in \tH: \lVert \tf \rVert_\tH \leq 1} \llvert \tf(\bx) - f(\bx) \nmid \tbf \rrvert^2 = \tK(\bx,\bx) - \bK^\intercal(\bx) \tmK^{-1} \bK(\bx) - 2 \xi f(\bx) \nmid \bone(\bx)$$
    where $\bone(\bx) = (1_{\bx=\bx_i})_{i=0}^{n-1}$. 
\end{theorem}
\begin{proof}
    From \Cref{thm:rkhs_approx_H}, \Cref{lemma:sup_eq_norm_tH}, and using the definition of the noisy kernel $\tK$ in \eqref{eq:noisy_kernel}, we have 
    \begin{align*}
        \sup_{f \in \tH: \lVert \tf \rVert_\tH \leq 1} &\left(\tf(\bx) - f(\bx) \nmid \tbf\right)
        = \sup_{f \in \tH: \lVert \tf \rVert_\tH \leq 1} \left(\tf(\bx) - \bK^\intercal(\bx) \tmK^{-1} \tbf\right) \\
        &= \llVert \tK(\cdot,\bx) - \bK^\intercal(\bx) \tmK^{-1} \tbK \rrVert_{\tH} \\
        &= \sqrt{\tK(\bx,\bx) - 2 \bK^\intercal(\bx) \tmK^{-1} \tbK(\bx) + \bK^\intercal(\bx) \tmK^{-1} \tmK\tmK^{-1} \bK(\bx)} \\
        &= \sqrt{\tK(\bx,\bx) - \bK^\intercal(\bx) \tmK^{-1} \bK(\bx) - 2 \xi \bK^\intercal(\bx) \tmK^{-1} \bone(\bx)} \\
        &= \sqrt{\tK(\bx,\bx) - \bK^\intercal(\bx) \tmK^{-1} \bK(\bx) - 2 \xi f(\bx) \nmid \bone(\bx)}.
    \end{align*}
    We may take care of the absolute value by noting the symmetry
    $$\sup_{\tf \in \tH: \lVert \tf \rVert_\tH \leq 1} \left(f(\bx) \nmid \tbf - \tf(\bx)\right)
    = \llVert \tK(\cdot,\bx) - \bK^\intercal(\bx) \tmK^{-1} \tbK \rrVert_{\tH}.$$
\end{proof} 

\begin{corollary} \label{corr2:rkhs_x_error_H_noisy_no_deriv}
    For $\bx \in \calX$ with  $\bx \neq \bx_i$ for $0 \leq i < n$, the posterior variance in \eqref{eq:post_cov} satisfies 
    $$\sup_{\tf \in \tH: \lVert \tf \rVert_\tH \leq 1} \llvert \tf(\bx) - f(\bx) \nmid \tbf \rrvert^2 = \tK(\bx,\bx) - \bK^\intercal(\bx) \tmK^{-1} \bK(\bx) = \Var[f(\bx) | \boldsymbol{f}] + \xi.$$
\end{corollary}

\Section{Fast GPs} \label{sec:fast_gps}

In this section we will briefly review fast GPs using SI and DSI kernels along with special point sets which enable reduced computational costs and storage requirements. These ideas were originally proposed in \cite{zeng.spline_lattice_digital_net,zeng.spline_lattice_error_analysis} and more recently used in the context of fast Bayesian cubature with fast kernel hyperparameter optimization \cite{rathinavel.bayesian_QMC_lattice,rathinavel.bayesian_QMC_sobol,rathinavel.bayesian_QMC_thesis}, see \Cref{sec:background_fgps} for more background. We will defer our discussions on fast Bayesian cubature and fast kernel hyperparameter optimization to \Cref{sec:fmtgps} where fast multitask GPs are discussed. The standard GPs discussed earlier in this chapter and the fast GPs discussed here are a special case of their multitask formulations

Suppose we have $n=2^m$ points with $m \in \bbN_0$, and we are using one of the two following point set $(\bx_i)_{i=0}^{2^m-1}$ and kernel $K$ pairings:
\begin{enumerate}
    \item A shifted rank-1 lattice in radical inverse order (\Cref{sec:lattices}) with an SI kernel (\Cref{def:si_kernels}). 
    \item A base $2$ digitally-shifted digital net in radical inverse order (\Cref{sec:dnets}) with a DSI kernel (\Cref{def:dsi_kernels}). Higher-order digital nets and/or linear matrix scrambling are permitted, but nested uniform scrambling is not allowed. 
\end{enumerate}
Recall the comparison between standard GPs and fast GPs given in \Cref{tab:com_kernel_costs}.

Under such pairings, \Cref{thm:fast_eigenvalue_computation} implies that the Gram matrix has an eigendecomposition 
$$\tmK = \mV \tmLambda \overline{\mV}, \qquad \tblambda = \tmLambda \bone =\sqrt{n} \mV \tmK_{:,0}$$
where $\tmK_{:,0}$ is the zeroth column of $\tmK$ and $\mV$ is a fast transform matrix satisfying \Cref{cond:fast_transform}. Then \Cref{corr:fast_gram_matrix_comps} implies that for any $\by \in \bbR^n$ we may write 
$$\tmK \by = \mV (\tby \odot \tblambda), \qquad \tmK^{-1} \by = \mV (\tby ./ \tblambda), \qquad \log \lvert \tmK \rvert = \bone^\intercal \log \lvert \tblambda \rvert, \qquad \mathrm{trace}(\mK) = \bone^\intercal \tblambda$$
where $\tby = \overline{\mV} \by$, the elementwise (Hadamard) product is $\odot$, elementwise division is denoted by $./$, and $\bone^\intercal \log \lvert \tblambda \rvert = \sum_{i=0}^{n-1} \log \lvert \tlambda_i \rvert.$ This implies fast GPs require only $\calO(n)$ storage and $\calO(n \log n + d n)$ computations to fit, including hyperparameter optimization via either NMLL (\Cref{sec:mll_sgp}) or GCV (\Cref{sec:gcv_sgp}).

\Section{Multitask GPs} \label{sec:mtgps} 

Often times one has access to simulations which permit evaluation at varying levels of fidelity. For example, a numerical PDE solver may be able to compute solutions with increasing mesh resolution and increasing cost. We will assume the dimension $d$ is the same on each level, as is often the case for solvers of parameterized PDEs where $d$ controls the number of terms kept in the Karhunen--Lo\`eve expansion. In such cases, we would like a GP surrogate to exploit both inter-task and intra-task covariance, where the tasks may be evaluations at different fidelity levels or more generally correlated objectives. The multitask GPs we describe here were proposed in \cite{bonilla2007multi}, see \Cref{sec:background_mtgps} for additional background. 

Suppose we have $L$ correlated functions (tasks), each acting on $\calX \subseteq \bbR^d$. We will denote this set of functions by 
$$f: \{1,\dots,L\} \times \calX \to \bbR,$$ 
where $f(\ell,\bx)$ evaluates task $\ell \in \{1,\dots,L\}$ at parameters $\bx \in \calX$. Rather than fitting independent GPs to $f(1,\cdot),\dots,f(L,\cdot)$, we would like to fit a single GP which additionally exploits covariance information between tasks. 

Like standard GPs, multitask GPs are defined with respect to an SPD kernel $K: (\{1,\dots,L\} \times \calX) \times (\{1,\dots,L\} \times \calX) \to \bbR$, see \Cref{def:spd_SPSD_kernels}, and a prior mean $M: \{1,\dots,L\} \times \calX \to \bbR$. Then our prior is
\begin{align*}
    f &\sim \mathrm{GP}(M,K), \\ 
    \bbE[f(\ell,\bx)] &= M(\ell,\bx), \\
    \Cov[f(\ell,\bx),f(\ell',\bx')] &= K((\ell,\bx),(\ell',\bx')).
\end{align*}
Suppose at level $\ell \in \{1,\dots,L\}$ we observe
$$\by_\ell = \boldsymbol{f}_\ell + \bvarepsilon_\ell, \qquad \boldsymbol{f}_\ell = (f(\ell,\bx_{\ell i}))_{i=0}^{n_\ell-1}, \qquad (\bx_{\ell i})_{i=0}^{n_\ell-1} \subset \calX, \qquad \bvarepsilon_\ell \sim \calN(0,\xi_\ell \mI).$$
Here $\by_\ell$ are noisy function evaluations at collocations points $(\bx_{\ell i})_{i=0}^{n_\ell-1}$ with IID zero-mean Gaussian noise $\bvarepsilon_\ell$ having common noise variance (nugget) $\xi_\ell > 0$. We also assume $\bvarepsilon_1,\dots,\bvarepsilon_L$ are independent.

We will denote 
\begin{align*}
    \bK_\ell(\ell',\bx') &:= (K((\ell',\bx'),(\ell,\bx_{\ell i})))_{i=0}^{n_\ell-1} \in \bbR^{n_\ell}, \\
    \bM_\ell &:= (M(\ell,\bx_{\ell i}))_{i=0}^{n_\ell-1} \in \bbR^{n_\ell}, \\
    \mK_{\ell \ell'} &:= (K((\ell,\bx_{\ell i}),(\ell',\bx_{\ell' i'})))_{i,i'=0}^{n_\ell,n_{\ell'}} \in \bbR^{n_\ell \times n_{\ell'}}, \\
    \tmK_{\ell \ell'} &:= \mK_{\ell \ell'} + \xi_\ell 1_{\ell=\ell'} \mI_{n_\ell}
\end{align*}
where $\mI_{n_\ell}$ is the $n_\ell \times n_\ell$ identity matrix and $1_{\ell=\ell'}$ is $1$ if $\ell=\ell'$ and $0$ otherwise. Let $N = n_1+\dots+n_L$ be the total number of observations. We will collect these level dependent vectors and matrices into block vectors and block matrices
$$\boldsymbol{f} := \begin{pmatrix} \boldsymbol{f}_1 \\ \vdots \\ \boldsymbol{f}_L\end{pmatrix} \in \bbR^N, \qquad \by := \begin{pmatrix} \by_1 \\ \vdots \\ \by_L\end{pmatrix} \in \bbR^N, \qquad \bM := \begin{pmatrix} \bM_1 \\ \vdots \\ \bM_L \end{pmatrix} \in \bbR^N,$$
$$\bK(\ell,\bx) := \begin{pmatrix} \bK_1(\ell,\bx) \\ \vdots \\ \bK_L(\ell,\bx) \end{pmatrix} \in \bbR^N, \qquad \mK := \begin{pmatrix} \mK_{11} & \cdots & \mK_{1L} \\ \vdots & \ddots & \vdots \\ \mK_{L1} & \cdots & \mK_{LL} \end{pmatrix} \in \bbR^{N \times N},$$
$$\mXi := \begin{pmatrix} \xi_1 \mI_{n_1} & & \\ & \ddots & \\ & & \xi_L \mI_{n_L} \end{pmatrix} \in \bbR^{N \times N},$$
$$\tmK := \mK + \mXi \in \bR^{N \times N}.$$
As with standard GPs, the posterior distribution of $f$ after observing $\by$ is also a GP. Using the above notations, we may write the posterior distribution as 
\begin{align*} 
    f | \by &\sim \mathrm{GP}(\bbE[f | \by],\mathrm{Cov}[f,f | \by]), \\
    \bbE[f(\ell,\bx) | \by] &= M(\ell,\bx) + \bK^\intercal(\ell,\bx) \tmK^{-1} \left[\by - \bM\right], \\
    \mathrm{Cov}[f(\ell,\bx),f(\ell',\bx') | \by] &= K((\ell,\bx),(\ell',\bx')) - \bK^\intercal(\ell,\bx) \tmK^{-1} \bK(\ell',\bx').
\end{align*}
Without additional structure, multitask GPs cost $\calO(N^3 + d N^2)$ and require $\calO(N^2)$ storage, including hyperparameter optimization. In \Cref{sec:fmtgps} we will show how these costs can be reduced using special SI/DSI kernels along with special point sets. 

Let us introduce some additional notations to be used in the following subsections. We collect powers of inverse Gram matrices into blocks
$$\tmK^{-1} = \mA = \begin{pmatrix} \mA_{11} & \cdots & \mA_{1L} \\ \vdots & \ddots & \vdots \\ \mA_{L1} & \ddots & \mA_{LL} \end{pmatrix}, \qquad \tmK^{-2} = \mB = \begin{pmatrix} \mB_{11} & \cdots & \mB_{1L} \\ \vdots & \ddots & \vdots \\ \mB_{L1} & \ddots & \mB_{LL} \end{pmatrix}.$$ 
We will also use the ``level summing'' matrix 
$$\mT = \begin{pmatrix} \bone_1 & & \\ & \ddots & \\ & & \bone_L \end{pmatrix} \in \bbR^{N \times L}$$
where $\bone_\ell$ is the length $n_\ell$ vector of $1$. For example, if $L=2$ with $n_1=2$ and $n_2=3$, then 
$$\mT = \begin{pmatrix} 1 & 0 \\ 1 & 0 \\ 0 & 1 \\ 0 & 1 \\ 0 & 1 \end{pmatrix} \in \bbR^{5 \times 2}.$$

First, \Cref{sec:hpopt_mtgps} will discuss multitask GP hyperparameter optimization, then \Cref{sec:mtgp_bayesian_cubature} will discuss Bayesian cubature for multitask GPs, and finally \Cref{sec:mtgp_product_structure} will introduce product structure for multitask GP kernels. 

\Subsection{Hyperparameter Optimization (HPOPT)} \label{sec:hpopt_mtgps}

Here we will discuss HPOPT for multitask GPs which closely follows \Cref{sec:hpopt_sgps}. As was the case there, we note that the kernel typically relies on hyperparameters $\btheta$, e.g., $\btheta = \{\balpha, \gamma, \bEta,\mG,\bt\}$ where $\balpha$ are smoothness parameters, $\gamma$ is a global scaling parameter, $\bEta$ are lengthscales, and $\mG,\bt$ control inter-task correlations as we will discuss in \Cref{sec:mtgp_product_structure}. As was the case for standard GPs, we will discuss HPOPT via the NMLL (\Cref{sec:mll_mtgp}) or GCV (\Cref{sec:gcv_mtgp}) loss. 

\Subsection{HPOPT: Negative Marginal Log-Likelihood} \label{sec:mll_mtgp}

\begin{theorem} \label{thm:optimal_tau_mtgp_mll}
    If the prior mean are constants $\btau=(\tau_1,\dots,\tau_L)^\intercal$ at each level, i.e., $M(\ell,\bx) = \tau_\ell$ for any $\ell \in \{1,\dots,L\}$ and any $\bx \in \calX$, then the NMLL is, up to an additive constant independent of $\btheta$, proportional to 
    \begin{equation}
        \begin{aligned}
            \calL_\mathrm{NMLL}(\btau,\btheta) &:= \begin{pmatrix} \by_1^\intercal - \tau_1 \bone_1^\intercal & \cdots & \by_L^\intercal - \tau_L \bone_L^\intercal \end{pmatrix} \begin{pmatrix} \mA_{11} & \cdots & \mA_{1L} \\ \vdots & \ddots & \vdots \\ \mA_{L1} & \ddots & \mA_{LL} \end{pmatrix} \begin{pmatrix} \by_1 - \tau_1 \bone_1 \\ \vdots \\ \by_L - \tau_L \bone_L \end{pmatrix} \\
            &+ \log \lvert \tmK \rvert.
        \end{aligned}
        \label{eq:mtgp_mll}
    \end{equation}
    Moreover, the optimal prior mean constants $\sbtau$ which minimizes the NMLL satisfy 
    \begin{equation}
        \begin{pmatrix} \bone_1^\intercal \mA_{11} \bone_1 & \cdots & \bone_1^\intercal \mA_{1L} \bone_L \\ \vdots & \ddots & \vdots \\ \bone_L^\intercal \mA_{L1} \bone_1 & \cdots & \bone_L^\intercal \mA_{LL} \bone_L \end{pmatrix} \begin{pmatrix} \stau_1 \\ \vdots \\ \stau_L \end{pmatrix} = \mT^\intercal \mA \by.
        \label{eq:optimal_tau_mtgp_mll}
    \end{equation}
\end{theorem}
\begin{proof}
    Using our updated notations the NMLL is still proportional to \eqref{eq:mll}, so 
    $$\calL_\mathrm{NMLL}(\btau,\btheta) = \sum_{\ell,\ell'=1}^L (\by_\ell^\intercal - \tau_\ell \bone_\ell^\intercal) \mA_{\ell\ell'} (\by_{\ell'} - \tau_{\ell'}\bone_{\ell'}) + \log \lvert \tmK \rvert.$$
    Setting the derivative with respect to $\tau_\ell$ equal to $0$ gives 
    \begin{align*}
        0 &= \partial \calL_\mathrm{NMLL}(\btau,\btheta)/(\partial \tau_\ell) \\
        &= \partial \left[(\by_\ell-\tau_\ell \bone_\ell^\intercal) \tmA_{\ell\ell'} (\by_\ell-\tau_\ell\bone_\ell) + 2 \sum_{\ell' \neq \ell} (\by_\ell-\tau_\ell \bone_\ell^\intercal) \mA_{\ell\ell'} (\by-\tau_{\ell'} \bone_{\ell'})\right]/(\partial \tau_\ell) \\
        &= 2 \tau_\ell \bone_\ell^\intercal \mA_{\ell\ell} \bone_\ell - 2 \bone_\ell^\intercal \mA_{\ell\ell}\by_\ell - 2 \sum_{\ell' \neq \ell} \bone_\ell^\intercal \mA_{\ell\ell'} (\by_{\ell'}-\tau_{\ell'} \bone_{\ell'}),
    \end{align*}
    so the optimal $\stau_\ell$ satisfies 
    \begin{align*}
        \stau_\ell &= \frac{\bone_\ell^\intercal \mA_{\ell\ell} \by_\ell + \sum_{\ell \neq \ell} \bone_\ell^\intercal \mA_{\ell\ell'} (\by_{\ell'} - \stau_{\ell'} \bone_{\ell'})}{\bone_\ell^\intercal \mA_{\ell\ell} \bone_\ell} \\
        &= \frac{\stau_\ell \bone_\ell^\intercal \mA_{\ell\ell} \bone_\ell + \sum_{\ell'=1}^L \bone_\ell^\intercal \mA_{\ell\ell'} (\by_{\ell'} - \stau_{\ell'} \bone_{\ell'})}{\bone_\ell^\intercal \mA_{\ell\ell} \bone_\ell} \\
        &= \stau_\ell + \frac{\sum_{\ell'=1}^L \bone_\ell^\intercal \mA_{\ell\ell'} (\by_{\ell'} - \stau_{\ell'} \bone_{\ell'})}{\bone_\ell^\intercal \mA_{\ell\ell} \bone_\ell} \\
        \therefore\; 0 &= \frac{\sum_{\ell'=1}^L \bone_\ell^\intercal \mA_{\ell\ell'} (\by_{\ell'} - \stau_{\ell'} \bone_{\ell'})}{\bone_\ell^\intercal \mA_{\ell\ell} \bone_\ell}, \\
        \therefore\; \sum_{\ell'=1}^L \bone_\ell^\intercal \mA_{\ell\ell'} \by_{\ell'} &= \sum_{\ell'=1}^L \stau_{\ell'} \bone_\ell^\intercal \mA_{\ell\ell'} \bone_{\ell'}.
    \end{align*}
    Expanding over all $\ell$ gives the desired linear system. 
\end{proof}

\Subsection{HPOPT: Generalized Cross Validation} \label{sec:gcv_mtgp}

\begin{theorem} \label{thm:optimal_tau_mtgp_gcv}
    If the prior mean are constants $\btau=(\tau_1,\dots,\tau_L)^\intercal$ at each level, i.e., $M(\ell,\bx) = \tau_\ell$ for any $\ell \in \{1,\dots,L\}$ and any $\bx \in \calX$, then the GCV loss is proportional to  
    \begin{equation}
        \calL_\mathrm{GCV}(\btau,\btheta) := \frac{\begin{pmatrix} \by_1^\intercal - \tau_1 \bone_1^\intercal & \cdots & \by_L^\intercal - \tau_L \bone_L^\intercal \end{pmatrix}}{\trace^2(\mA)} \begin{pmatrix} \mB_{11} & \cdots & \mB_{1L} \\ \vdots & \ddots & \vdots \\ \mB_{L1} & \ddots & \mB_{LL} \end{pmatrix} \begin{pmatrix} \by_1 - \tau_1 \bone_1 \\ \vdots \\ \by_L - \tau_L \bone_L \end{pmatrix}.
        \label{eq:mtgp_gcv}
    \end{equation}
    Moreover, the optimal prior mean constants $\sbtau$ which minimizes the GCV loss satisfies 
    \begin{equation} 
        \begin{pmatrix} \bone_1^\intercal \mB_{11} \bone_1 & \cdots & \bone_1^\intercal \mB_{1L} \bone_L \\ \vdots & \ddots & \vdots \\ \bone_L^\intercal \mB_{L1} \bone_1 & \cdots & \bone_L^\intercal \mB_{LL} \bone_L \end{pmatrix} \begin{pmatrix} \stau_1 \\ \vdots \\ \stau_L \end{pmatrix} = \mT^\intercal \mB \by.
        \label{eq:optimal_tau_mtgp_gcv}
    \end{equation}
\end{theorem}
\begin{proof}
    Using our updated notations, the GCV loss is still proportional to \eqref{eq:gcv},
    $$\calL_\mathrm{GCV}(\btau,\btheta) = \mathrm{trace}^{-2}(\mA) \sum_{\ell,\ell'=1}^L (\by_\ell^\intercal - \tau_\ell \bone_\ell^\intercal) \mB_{\ell\ell'} (\by_{\ell'} - \tau_{\ell'}\bone_{\ell'}).$$
    Setting the derivative with respect to $\tau_\ell$ equal to $0$ gives 
    \begin{align*}
        0 = 2 \tau_\ell \bone_\ell^\intercal \mB_{\ell\ell} \bone_\ell - 2 \bone_\ell^\intercal \mB_{\ell\ell}\by_\ell - 2 \sum_{\ell' \neq \ell} \bone_\ell^\intercal \mB_{\ell\ell'} (\by_{\ell'}-\tau_{\ell'} \bone_{\ell'}).
    \end{align*}
    The remaining steps directly follow the proof of \Cref{thm:optimal_tau_mtgp_mll} with $\mA$ replaced by $\mB$, so the result follows.
\end{proof}

\Subsection{Bayesian Cubature} \label{sec:mtgp_bayesian_cubature}

We are often interested in a linear combination of integrals on each level. With weights $\bomega \in \bbR^L$, we denote
\begin{equation*}
    \nu_\bomega := \sum_{\ell=1}^L \omega_\ell \int_\calX f(\ell,\bx) \D \bx.
\end{equation*}
The posterior distribution of $\nu_\bomega$ is the Gaussian 
$$\nu_\bomega \sim \calN(\bbE[\nu_\bomega | \by],\bbV[\nu_\bomega | \by])$$
with respective mean and variance 
\begin{align*}
    \bbE[\nu_\bomega | \by] &= \bomega^\intercal \brho \qquad\text{and}\\
    \bbV[\nu_\bomega | \by] &= \bomega^\intercal \mSigma \bomega 
\end{align*}
where 
\begin{align*}
    \brho &= \tbomega + \tmE \tmK^{-1}(\by-\bM)  \in \bbR^L, \\
    \mSigma &= \approxhat{\mE} - \tmE^\intercal \tmK^{-1}\tmE \in \bbR^{L \times L}, \\
    \tbomega &= \left(\int_\calX M(\ell,\bx) \D \bx\right) \in \bbR^L, \\
    \tmE &= \left(\int_\calX \bK(\ell,\bx) \D \bx\right)_{\ell=1}^L \in \bbR^{N \times L}, \\
    \approxhat{\mE} &= \left(\int_\calX \int_\calX K((\ell,\bx),(\ell',\bx')) \D \bx' \D \bx\right)_{\ell,\ell'=1}^L \in \bbR^{L \times L}.
\end{align*}

\begin{theorem} \label{thm:mtgp_bc_optimal_weights}
    Suppose we are given $\bchi \in \bbR^L$, e.g., $\bchi = (0,\dots,0,1)^\intercal$ or $\bchi = \bone$. Let us denote the mean squared error (MSE)
    \begin{equation*}
        \mathrm{MSE}(\bomega) = \bbE[(\nu_\bomega-\bbE[\nu_\bchi | \by])^2 | \by] = \bbV[\nu_\bomega | \by] + \mathrm{BIAS}^2(\nu_\bomega)
    \end{equation*}
    where 
    $$\mathrm{BIAS}(\nu_\bomega) = \bbE[\nu_\bomega | \by] - \bbE[\nu_\bchi | \by].$$ 
    Then 
    \begin{equation} \label{eq:mtgp_bc_optimal_weights} 
        \sbomega := \argmin \mathrm{MSE}(\bomega) = (\brho^\intercal \bchi) \left(\mSigma + \brho\brho^\intercal\right)^{-1} \brho.
    \end{equation}
    Moreover,
    $$\mathrm{MSE}(\sbomega) = (\brho^\intercal \bchi)^2 (1 -  \brho^\intercal (\mSigma + \brho \brho^\intercal)^{-1} \brho).$$
\end{theorem}
\begin{proof}
    The bias variance decomposition is standard:
    \begin{align*}
        \bbE[(\nu_\bomega-\bbE[\nu_\bchi | \by])^2 | \by] &= \bbE[\nu_\bomega^2 | \by] - 2 \bbE[\nu_\bomega | \by] \bbE[\nu_\bchi | \by] + \bbE^2[\nu_\bchi | \by] \\
        &= \bbV[\nu_\bomega | \by] + \bbE^2[\nu_\bomega | \by] - 2 \bbE[\nu_\bomega | \by] \bbE[\nu_\bchi | \by] + \bbE^2[\nu_\bchi | \by] \\
        &= \bbV[\nu_\bomega | \by] + \left(\bbE[\nu_\bomega | \by] - \bbE[\nu_\bchi | \by]\right)^2.
    \end{align*} 
    To prove the minimizer, write 
    \begin{equation*}
        \mathrm{MSE}(\bomega) = \bomega^\intercal \mSigma \bomega + (\bomega-\bchi)^\intercal \brho \brho^\intercal (\bomega-\bchi)
    \end{equation*}
    so that 
    $$\frac{\partial \mathrm{MSE}(\bomega)}{\partial \bomega} = 2 \mSigma \bomega + 2 \brho \brho^\intercal \bomega - 2 \brho\brho^\intercal \bchi = 0$$ 
    when 
    $$(\mSigma + \brho \brho^\intercal) \bomega = \brho (\brho^\intercal \bchi).$$
    Writing
    \begin{align*}
        \mathrm{MSE}(\bomega) = \bomega^\intercal (\mSigma + \brho \brho^\intercal)\bomega + \bchi^\intercal \brho \brho^\intercal \bchi - 2 \bchi^\intercal \brho \brho^\intercal \bomega
    \end{align*}
    and plugging in the optimal $\sbomega$ from \eqref{eq:mtgp_bc_optimal_weights} gives the expression for the optimal MSE. 
\end{proof}

Notice that regardless of $\bchi$, the minimizer of $\mathrm{MSE}(\bomega)$ is always proportional to $(\mSigma + \brho \brho^\intercal)^{-1} \brho$. 

\Subsection{Product Structure} \label{sec:mtgp_product_structure}

Almost all literature and implementations of multitask GPs assume $K$ has a product structure  
\begin{equation}
    K((\ell,\bx),(\ell',\bx')) = R(\ell,\ell') Q(\bx,\bx'), \qquad \forall\; \ell,\ell' \in \{1,\dots,L\}, \quad \forall\; \bx,\bx' \in [0,1]^d 
    \label{eq:mt_product_kernel}
\end{equation}
for some kernels $R: \{1,\dots,L\} \times \{1,\dots,L\} \to \bbR$ and $Q: \calX \times \calX \to \bbR$. Here $Q$ is taken to be a standard kernel such as those in \Cref{sec:gp_kernels}, e.g., a squared exponential or (digitally-)shift-invariant one. The kernel $R$ is typically parameterized directly by a low rank matrix $\mG \in \bbR^{L \times s}$ with $s \in \{0,\dots,L\}$ and a diagonal vector $\bt \in \bbR_{>0}^L$ so that the Gram matrix is 
$$\mR = (R(\ell,\ell'))_{\ell,\ell'=1}^L = \mG \mG^\intercal + \mathrm{diag}(\bt).$$

\Section{Fast Multitask GPs} \label{sec:fmtgps}

This section extends the fast GPs formulation in \Cref{sec:fast_gps} to the multitask setting. Of course, the single task setting is a special case of this multitask formulation. The resulting structures and algorithms we derive here are not available elsewhere in the literature. 

Let us assume we are using a product kernel as in \eqref{eq:mt_product_kernel} and have sample sizes $n_\ell = 2^{m_\ell}$ for $\ell \in \{1,\dots,L\}$, so the total number of points is $N = n_1 + \dots n_L$. We will enable fast multitask GPs by either 
\begin{enumerate}
    \item choosing $Q$ to be a shift-invariant (SI) kernel (\Cref{def:si_kernels}) and $(\bx_{\ell i})_{i=0}^{2^{m_\ell}-1}$ to be shifted rank-1 lattices in radical inverse order with possibly different shifts on each level, or 
    \item choosing $Q$ to be a digitally-shift-invariant (DSI) kernel (\Cref{def:dsi_kernels}) and $(\bx_{\ell,i})_{i=0}^{2^{m_\ell}-1}$ to be digitally-shifted base $2$ digital nets in radical inverse order with possibly different digital shifts on each level. As the generating matrices must be the same on each level, a single linear matrix scramble (LMS) may be applied to all generating matrices, but different LMS on each level is \emph{not} permitted. Nested uniform scrambling (NUS) is also prohibited. 
\end{enumerate}
Of course, this requires the domain to be $\calX = [0,1)^d$. 

After potentially reordering levels, let us assume $n_1 \geq \cdots \geq n_L$ without loss of generality. Then, under either of the two special pairings above, we may use \Cref{lemma:remove_krons} and \Cref{thm:fast_eigenvalue_computation} to say that for any $1 \leq \ell \leq \ell' \leq L$,
\begin{align*}
    \tmK_{\ell\ell'} &= \mV_\ell \tmLambda_{\ell\ell'} \overline{\mV_{\ell'}}, \\
    \tblambda_{\ell\ell'} &= \tmLambda_{\ell\ell'} \bone = \sqrt{2^{m_{\ell'}}}\;\overline{\mV_\ell} (\tmK_{\ell\ell'})_{:,0}.
\end{align*}
Here $\mV_\ell := \mV_{m_\ell} \in \bbC^{2^{m_\ell} \times 2^{m_\ell}}$ are fast transform matrices satisfying \Cref{cond:fast_transform}, $\tmLambda_{\ell\ell'}$ is a $2^{m_\ell-m_{\ell'}} \times 1$ block matrix with $2^{m_{\ell'}} \times 2^{m_{\ell'}}$ diagonal blocks, and $(\tmK_{\ell\ell'})_{:,0}$ is the zeroth column of $\tmK_{\ell\ell'}$. Therefore, we may write 
\begin{align*}
    \tmK &= 
    \begin{pmatrix}
        \mV_1 \tmLambda_{11} \overline{\mV_1} & \cdots & \mV_1 \tmLambda_{1L} \overline{\mV_L} \\
        \vdots & \ddots & \vdots \\
        \mV_L \overline{\tmLambda_{1L}}\; \overline{\mV_1} & \cdots & \mV_L \tmLambda_{LL} \overline{\mV_L} 
    \end{pmatrix} \\
    &= 
    \begin{pmatrix}
        \mV_1 & & \\
        & \ddots & \\
        & & \mV_L
    \end{pmatrix}
    \begin{pmatrix}
        \tmLambda_{11} & \cdots & \tmLambda_{1L} \\
        \vdots & \ddots & \vdots \\
        \overline{\tmLambda_{1L}} & \cdots & \tmLambda_{LL} 
    \end{pmatrix}
    \begin{pmatrix}
        \overline{\mV_1} & & \\
        & \ddots & \\
        & & \overline{\mV_L}
    \end{pmatrix} \\
    &=: \mV \tmLambda \overline{\mV}.
\end{align*}
We will often denote the collection of the first $\ell$ row blocks and first $\ell'$ column blocks by  
$$\tmLambda_{:\ell,:\ell'} := \begin{pmatrix} \tmLambda_{11} & \cdots & \tmLambda_{1\ell'} \\ \vdots & \ddots & \vdots \\ \overline{\tmLambda_{\ell1}} & \cdots & \tmLambda_{\ell\ell'} \end{pmatrix}$$

Both $\mV,\overline{\mV} \in \bbC^{N \times N}$ are fixed and known. Multiplication by $\mV,\overline{\mV}$ can be done quickly using fast transforms. As detailed in \Cref{sec:fast_kernel_methods}, the fast transforms are 
\begin{enumerate}
    \item the FFTBR/IFFTBR algorithms in the case of SI kernels and shifted rank-1 lattices in radical inverse order, and 
    \item the FWHT algorithm in the case of DSI kernels and base $2$ digitally-shifted digital nets in radical inverse order. 
\end{enumerate}
It is therefore sufficient to work with the sparse matrix $\tmLambda$. Multiplication by $\tmLambda$ is straightforward. In order to fit a multitask GP, we will also need to invert $\tmLambda$ and compute the determinant. 

To motivate our fast inverse and determinant algorithm, let us consider an example with $L=3$ tasks having $n_1 = 8$, $n_2 = 4$ and $n_3 = 2$. Then $\tmLambda$ has the following sparsity pattern
\begin{equation}
    \tmLambda = 
    \begin{pmatrix}
        \tmLambda_{11} & \tmLambda_{12} & \tmLambda_{13} \\
        \overline{\tmLambda_{12}} & \tmLambda_{22} & \tmLambda_{23} \\
        \overline{\tmLambda_{13}} & \overline{\tmLambda_{23}} & \tmLambda_{33}
    \end{pmatrix}
    = 
    \left[\begin{array}{cccccccc|cccc|cc}
    . &   &   &   &   &   &   &   & . &   &   &   & . &   \\
      & . &   &   &   &   &   &   &   & . &   &   &   & . \\
      &   & . &   &   &   &   &   &   &   & . &   & . &   \\
      &   &   & . &   &   &   &   &   &   &   & . &   & . \\
      &   &   &   & . &   &   &   & . &   &   &   & . &   \\
      &   &   &   &   & . &   &   &   & . &   &   &   & . \\
      &   &   &   &   &   & . &   &   &   & . &   & . &   \\
      &   &   &   &   &   &   & . &   &   &   & . &   & . \\
      \hline
    . &   &   &   & . &   &   &   & . &   &   &   & . &   \\
      & . &   &   &   & . &   &   &   & . &   &   &   & . \\
      &   & . &   &   &   & . &   &   &   & . &   & . &   \\
      &   &   & . &   &   &   & . &   &   &   & . &   & . \\
      \hline 
    . &   & . &   & . &   & . &   & . &   & . &   & . &  \\
      & . &   & . &   & . &   & . &   & . &   & . &   & .
  \end{array}\right].
  \label{eq:example_eval_mat}
\end{equation}
Then
\begin{align*}
    \tmLambda_{:2,:2}^{-1} 
    &:= \begin{pmatrix}
        \tmLambda_{11} & \tmLambda_{12} \\
        \overline{\tmLambda_{12}} & \tmLambda_{22}
    \end{pmatrix}^{-1}
    = 
    \begin{pmatrix}
        \tmLambda_{11}^{-1} + \tmLambda_{11}^{-1} \tmLambda_{12} \mS_2^{-1} \overline{\tmLambda_{12}} \tmLambda_{11}^{-1} & -\tmLambda_{11}^{-1} \tmLambda_{12} \mS_2^{-1} \\
        - \mS_2^{-1} \overline{\tmLambda_{12}} \tmLambda_{11}^{-1} & \mS_2^{-1}
    \end{pmatrix} \\
    &= 
    \left[\begin{array}{cccccccc|cccc}
    . &   &   &   & . &   &   &   & . &   &   &   \\
      & . &   &   &   & . &   &   &   & . &   &   \\
      &   & . &   &   &   & . &   &   &   & . &   \\
      &   &   & . &   &   &   & . &   &   &   & . \\
    . &   &   &   & . &   &   &   & . &   &   &   \\
      & . &   &   &   & . &   &   &   & . &   &   \\
      &   & . &   &   &   & . &   &   &   & . &   \\
      &   &   & . &   &   &   & . &   &   &   & . \\
      \hline 
    . &   &   &   & . &   &   &   & . &   &   &   \\
      & . &   &   &   & . &   &   &   & . &   &   \\
      &   & . &   &   &   & . &   &   &   & . &   \\
      &   &   & . &   &   &   & . &   &   &   & .
  \end{array}\right]
\end{align*}
where
$$\mS_2 = \tmLambda_{22} - \overline{\tmLambda_{12}} \tmLambda_{11}^{-1} \tmLambda_{12}$$
is the Schur complement. We also have that 
$$\lvert \tmLambda_{:2,:2} \rvert = \lvert \tmLambda_{11} \rvert \cdot \lvert \mS_2 \rvert.$$
Then in the next step we may say 
\begin{align*}
    \tmLambda^{-1} = \tmLambda_{:3,:3}^{-1} 
    &:= \begin{pmatrix}
        \tmLambda_{:2,:2} & \tmLambda_{:2,3} \\
        \overline{\tmLambda_{:2,3}} & \tmLambda_{33}
    \end{pmatrix}^{-1}
    = 
    \begin{pmatrix}
        \tmLambda_{:2,:2}^{-1} + \tmLambda_{:2,:2}^{-1} \tmLambda_{:2,3} \mS_3^{-1} \overline{\tmLambda_{:2,3}} \tmLambda_{:2,:1}^{-1} & -\tmLambda_{:2,:2}^{-1} \tmLambda_{:2,3} \mS_3^{-1} \\
        \mS_3^{-1} \overline{\tmLambda_{:2,3}} \tmLambda_{:2,:2}^{-1} & \mS_3^{-1}
    \end{pmatrix}
\end{align*}
where 
$$\mS_3 = \tmLambda_{33} - \overline{\tmLambda_{:2,3}} \tmLambda_{:2,:2}^{-1} \tmLambda_{:2,3}$$
is another Schur complement and 
$$\lvert \tmLambda \rvert = \lvert \tmLambda_{:3,:3} \rvert = \lvert \tmLambda_{:2,:2} \rvert \cdot \lvert \tmLambda_3 \rvert.$$
Crucially, $\mS_\ell$ are diagonal and $\tmLambda^{-1}_{:\ell,:\ell}$ only depends on $\tmLambda^{-1}_{:\ell-1,:\ell-1}$, $\tmLambda_{:\ell-1,\ell}$, and $\tmLambda_{\ell\ell}$. 

\Cref{algo:inv_det_Theta} gives a general routine for computing $\tmLambda^{-1}$ and $\lvert \tmLambda \rvert$. Then, \Cref{thm:costs_fast_mtgp} analyzes the cost and storage requirements for computing $\tmLambda$, multiplying by $\tmK$, inverting $\tmK$, and finding the determinant of $\tmK$. 

\begin{algorithm}[!ht]
    \fontsize{12}{10}\selectfont
    \caption{Compute the Inverse $\tmLambda^{-1}$ and Determinant $\lvert \tmLambda \rvert$}
    \label{algo:inv_det_Theta}
    \begin{algorithmic}
        \Require $\tmLambda$ diagonal block matrix with $m_1 \geq \dots \geq m_L$.
        \State $\mD \gets \tmLambda_{11}$ \Comment{A $2^{m_1} \times 2^{m_1}$ diagonal matrix.}
        \State $\mA \gets \mD^{-1}$ \Comment{A $2^{m_1} \times 2^{m_1}$ diagonal matrix. Costs $\calO(2^{m_1})$.} 
        \State $\varrho \gets  \lvert \mD \rvert$ \Comment{A scalar. Costs $\calO(2^{m_1})$.}
        \State $\ell \gets 2$
        \While{$\ell \leq L$}
            %\\ \Comment{$\mA$ is a $\sum_{\ell'=1}^{\ell-1} 2^{m_{\ell'}-m_{\ell-1}} \times \sum_{\ell'=1}^{\ell-1} 2^{m_{\ell'}-m_{\ell-1}}$ block matrix with $2^{m_{\ell-1}} \times 2^{m_{\ell-1}}$ diagonal blocks.} 
            \State $\mD \gets \tmLambda_{\ell\ell}$ \Comment{A $2^{m_\ell} \times 2^{m_\ell}$ diagonal matrix.}
            \State $\mB \gets \tmLambda_{:(\ell-1),\ell}$ \Comment{A $\sum_{\ell'=1}^{\ell-1} 2^{m_{\ell'} - m_\ell} \times 1$ block matrix with $2^{m_\ell} \times 2^{m_\ell}$ diagonal blocks.}
            \State $\mE \gets \mA \mB$ \Comment{A $\sum_{\ell'=1}^{\ell-1} 2^{m_{\ell'}-m_{\ell}} \times 1$ block matrix with $2^{m_\ell} \times 2^{m_\ell}$ diagonal blocks. Costs $\calO([\sum_{\ell'=1}^{\ell-1} 2^{m_{\ell'} - m_{\ell-1}}][\sum_{\ell'=1}^{\ell-1} 2^{m_{\ell'}}])$.}
            \State $\mF \gets \overline{\mB}^\intercal \mE$ \Comment{A $2^{m_\ell} \times 2^{m_\ell}$ diagonal matrix. Costs $\calO(\sum_{\ell'=1}^{\ell-1} 2^{m_{\ell'}})$.}
            \State $\mS \gets \mD-\mF$ \Comment{A $2^{m_\ell} \times 2^{m_\ell}$ diagonal matrix (the Schur complement). Costs $2^{m_\ell}$.}
            \State $\varrho \gets \varrho \lvert \mS \rvert$ \Comment{A scalar. Costs $\calO(2^{m_\ell})$. Equivalent to $\lvert \tmLambda_{:\ell,:\ell} \rvert$.}
            \State $\mG \gets \mS^{-1}$ \Comment{A $2^{m_\ell} \times 2^{m_\ell}$ diagonal matrix. Costs $\calO(2^{m_\ell})$.}
            \State $\mH \gets \mE \mG$ \Comment{A $\sum_{\ell'=1}^{\ell-1} 2^{m_{\ell'}-m_{\ell}} \times 1$ block matrix with $2^{m_\ell} \times 2^{m_\ell}$ diagonal blocks. Costs $\calO(\sum_{\ell'=1}^{\ell-1} 2^{m_{\ell'}})$.}
            \State $\mJ \gets \mH \overline{\mE}^\intercal$ \Comment{A $\sum_{\ell'=1}^{\ell-1} 2^{m_{\ell'}-m_{\ell}} \times \sum_{\ell'=1}^{\ell-1} 2^{m_{\ell'}-m_{\ell}}$ block matrix with $2^{m_\ell} \times 2^{m_\ell}$ diagonal blocks. Costs $\calO([\sum_{\ell'=1}^{\ell-1} 2^{m_{\ell'} - m_\ell}][\sum_{\ell'=1}^{\ell-1} 2^{m_{\ell'}}])$.}
            \State $\mM \gets \mA + \mJ$ \Comment{A $\sum_{\ell'=1}^{\ell-1} 2^{m_{\ell'}-m_{\ell}} \times \sum_{\ell'=1}^{\ell-1} 2^{m_{\ell'}-m_{\ell}}$ block matrix with $2^{m_\ell} \times 2^{m_\ell}$ diagonal blocks. Costs $\calO([\sum_{\ell'=1}^{\ell-1} 2^{m_{\ell'} - m_\ell}][\sum_{\ell'=1}^{\ell-1} 2^{m_{\ell'}}])$.}
            \State $\mA \gets \begin{pmatrix} \mM & -\mH \\ -\overline{\mH}^\intercal & \mG \end{pmatrix}$ \Comment{A $\sum_{\ell'=1}^\ell 2^{m_{\ell'}-m_\ell} \times \sum_{\ell'=1}^\ell 2^{m_{\ell'}-m_\ell}$ block matrix with $2^{m_\ell} \times 2^{m_\ell}$ diagonal blocks. Equivalent to $\tmLambda_{:\ell,:\ell}^{-1}$.}
            \State $\ell \gets \ell+1$
        \EndWhile
        \State $\tmLambda^{-1} \gets \mA$ \Comment{A $\sum_{\ell=1}^L 2^{m_{\ell}-m_L} \times \sum_{\ell=1}^L 2^{m_{\ell}-m_L}$ block matrix with $2^{m_L} \times 2^{m_L}$ diagonal blocks.}
        \State $\lvert \tmLambda \rvert \gets \varrho$ \Comment{A scalar.}
        \\\Return $\tmLambda^{-1},\lvert \tmLambda \rvert$
    \end{algorithmic}
\end{algorithm}

\begin{theorem} \label{thm:costs_fast_mtgp}
    Assume $n_\ell = 2^{m_\ell}$ with $m_\ell \in \bbN_0$ and, without loss of generality, assume $n_1 \geq \cdots \geq n_L$ and $N=n_1+\cdots+n_L$. Evaluating $\tmLambda \in \bbC^{N \times N}$ requires 
    \begin{align*}
        \calO\left(\sum_{\ell=1}^L (L-\ell+1) (n_\ell \log(n_\ell) + d n_\ell)\right) \text{ computations and } \\ 
        \calO\left(\sum_{\ell=1}^L (L-\ell+1) n_\ell\right) \text{ storage}.
    \end{align*}
    Computing $\tmK \by = \mV \tmLambda \overline{\mV} \by$ for any $\by \in \bbR^N$ requires 
    $$\calO\left(\sum_{\ell=1}^L n_\ell \log(n_\ell) + \sum_{\ell=1}^L (L-\ell+1) n_\ell\right) \text{ computations}.$$
    Then, evaluating $\lvert \tmK \rvert = \lvert \tmLambda \rvert$ and $\tmLambda^{-1}$ using \Cref{algo:inv_det_Theta} requires 
    $$\calO\left(\sum_{\ell=2}^L \left(\sum_{\ell'=1}^{\ell-1} n_{\ell'}\right)^2/n_\ell\right) \text{ computations and } \calO(N^2/n_L) \text{ storage}.$$
    Having stored $\tmLambda^{-1}$, evaluating $\tmK^{-1} \by= \mV \tmLambda^{-1} \overline{\mV} \by$ for any $\by \in \bbR^N$ requires
    $$\calO\left(\sum_{\ell=1}^L n_\ell \log(n_\ell) + N^2/n_L\right) \text{ computations}.$$
\end{theorem}

\begin{corollary}
    Assume $n:=n_\ell = 2^{m}$ with $m \in \bbN_0$ for all $\ell \in \{1,\dots,L\}$ so $N=Ln$. Evaluating $\tmLambda \in \bbC^{N \times N}$ requires 
    $$\calO\left(L^2(n \log n + dn)\right) \text{ computations and } \calO\left(L^2 n\right) \text{ storage}.$$
    Computing $\tmK \by = \mV \tmLambda \overline{\mV} \by$ for any $\by \in \bbR^N$ requires 
    $$\calO\left(L n \log n + L^2 n\right) \text{ computations}.$$
    Then, evaluating $\lvert \tmK \rvert = \lvert \tmLambda \rvert$ and $\tmLambda^{-1}$ using \Cref{algo:inv_det_Theta} requires 
    $$\calO\left(L^3 n\right) \text{ computations and } \calO(L^2n) \text{ storage}.$$
    Having stored $\tmLambda^{-1}$, evaluating $\tmK^{-1} \by= \mV \tmLambda^{-1} \overline{\mV} \by$ for any $\by \in \bbR^N$ requires
    $$\calO\left(L n \log n + L^2 n\right) \text{ computations}.$$
\end{corollary}

Let us introduce some additional notations useful in the next subsections. Similar to what was done for standard multitask GPs, we will collect powers of $\tmLambda$ into blocks
$$\tmLambda^{-1} = \mGamma = \begin{pmatrix} \mGamma_{00} & \cdots & \mGamma_{1L} \\ \vdots & \ddots & \vdots \\ \mGamma_{L1} & \ddots & \mGamma_{LL} \end{pmatrix}, \qquad \tmLambda^{-2} = \mUpsilon = \begin{pmatrix} \mUpsilon_{00} & \cdots & \mUpsilon_{1L} \\ \vdots & \ddots & \vdots \\ \mUpsilon_{L1} & \ddots & \mUpsilon_{LL} \end{pmatrix},$$
so that
$$\mA = \tmK^{-1} = \mV \mGamma \overline{\mV}, \qquad \mB = \tmK^{-2} = \mV \mUpsilon \overline{\mV}.$$
Moreover, let us collect 
$$\mPi = \left(\sqrt{n_\ell n_{\ell'}} (\mGamma_{\ell\ell'})_{00}\right)_{\ell,\ell'=1}^L \in \bbR^{L \times L}, \qquad \mP = \left(\sqrt{n_\ell n_{\ell'}} (\mUpsilon_{\ell\ell'})_{00}\right)_{\ell,\ell'=1}^L \in \bbR^{L \times L}$$
where $(\mGamma_{\ell\ell'})_{00}$ and $(\mUpsilon_{\ell\ell'})_{00}$ are the upper left elements of $\mGamma_{\ell\ell'}$ and $\mUpsilon_{\ell\ell'}$ respectively. 

\Cref{sec:hpopt_fmtgps} will discuss hyperparameter optimization (HPOPT) for fast multitask GPs and then \Cref{sec:fmtgp_bayesian_cubature} will discuss fast Bayesian cubature. Both HPOPT and Bayesian cubature for fast multitask GPs do not generally increase the cost or storage requirements beyond that of solving  linear system in $\tmK$ and computing the determinant $\lvert \tmK \rvert$, both of which were analyzed in \Cref{thm:costs_fast_mtgp}.

\Subsection{Hyperparameter Optimization (HPOPT)} \label{sec:hpopt_fmtgps}

This section closely follows \Cref{sec:hpopt_mtgps} which discussed HPOPT for standard multitask GPs via either the NMLL (\Cref{sec:mll_fmtgp}) or GCV (\Cref{sec:gcv_fmtgp}) losses.% The WLOOCV loss for fast multitask GPs follows the form in \Cref{sec:wloocv_mtgp}, specifically in \Cref{thm:optimal_tau_mtgp_wloocv}, and we do not present a simplified form here. Nevertheless, computations can still be done quickly with the fast matrix multiplication and linear system solution routines presented in \Cref{thm:costs_fast_mtgp}.

\Subsection{HPOPT: Negative Marginal Log-Likelihood (NMLL)} \label{sec:mll_fmtgp}

\begin{theorem} \label{thm:optimal_tau_fmtgp_mll}
    When the prior mean are constants $\btau=(\tau_1,\dots,\tau_L)^\intercal$ at each level, i.e., $M(\ell,\bx) = \tau_\ell$ for any $\ell \in \{1,\dots,L\}$ and any $\bx \in [0,1]^d$, then the optimal $\sbtau$ which minimizes the NMLL \eqref{eq:mtgp_mll} satisfies 
    $$\mPi \sbtau = \mT \mA \by.$$ 
\end{theorem}
\begin{proof}
    Recall the optimal $\sbtau$ in \eqref{eq:optimal_tau_mtgp_mll} and that
    $$\bone_\ell^\intercal \mA_{\ell\ell'} = \bone_\ell^\intercal \mV_\ell \mGamma_{\ell\ell'} \overline{\mV_{\ell'}} = \sqrt{n_\ell} \be_\ell \mGamma_{\ell\ell'} \overline{\mV_{\ell'}} = \sqrt{n_\ell} (\mGamma_{\ell \ell'})_{0,:} \overline{\mV_{\ell'}},$$
    $$\bone_\ell^\intercal \mA_{\ell\ell'}\bone_{\ell'} = \bone_\ell^\intercal \mV_\ell \mGamma_{\ell\ell'} \overline{\mV_{\ell'}} \bone_{\ell'} = \sqrt{n_\ell n_{\ell'}}\be_\ell \mGamma_{\ell\ell'} \be_\ell = \sqrt{n_\ell n_{\ell'}} = \sqrt{n_\ell n_{\ell'}} (\mGamma_{\ell \ell'})_{00}$$
    where $\be_\ell=(1,0,0,\dots)^\intercal \in \bbR^{n_\ell}$ and $(\mGamma_{\ell \ell'})_{0,:}$ is the zeroth row of $\mGamma_{\ell \ell'}$. 
\end{proof}

\Subsection{HPOPT: Generalized Cross Validation (GCV)} \label{sec:gcv_fmtgp}

\begin{theorem} \label{thm:optimal_tau_fmtgp_gcv}
    When the prior mean are constants $\btau=(\tau_1,\dots,\tau_L)^\intercal$ at each level, i.e., $M(\ell,\bx) = \tau_\ell$ for any $\ell \in \{1,\dots,L\}$ and any $\bx \in [0,1]^d$, then the optimal $\sbtau$ which minimizes the GCV loss \eqref{eq:mtgp_gcv} satisfies 
    $$\mP \sbtau = \mT \mA^2 \by.$$ 
\end{theorem}
\begin{proof}
    Recalling the optimal $\sbtau$ in \eqref{eq:optimal_tau_mtgp_gcv}, the proof is analogous to that of \Cref{thm:optimal_tau_fmtgp_mll} but with $\mA,\mGamma$ replaced by $\mB,\mUpsilon$ respectively.  
\end{proof}

\Subsection{Bayesian Cubature} \label{sec:fmtgp_bayesian_cubature}

Recall that for the product-form SI kernels in \Cref{def:si_kernels} or product-form DSI kernels in \Cref{def:dsi_kernels} that 
$$\int_{[0,1]^d} Q(\bx,\bx') \D \bx' = \int_{[0,1]^d} \int_{[0,1]^d} Q(\bx,\bx') \D \bx' \D \bx = \gamma$$
for any $\bx \in [0,1)^d$ where $\gamma \in \bbR_{>0}$ is the global scaling parameter. 

\begin{theorem}
    Following the notation in \Cref{sec:mtgp_bayesian_cubature}, we have 
    \begin{align*}
        \mSigma &= \gamma \mR - \gamma^2 \mR \mPi \mR \in \bbR^{L \times L}.
    \end{align*}
\end{theorem}
\begin{proof}
    \begin{align*}
        \tmE =& \begin{pmatrix} \mR_{11} \bone_1 & \cdots & \mR_{1L} \bone_L \\ \vdots & \ddots & \vdots \\ \mR_{L1} \bone_1 & \cdots & \mR_{LL} \bone_L \end{pmatrix} q \in \bbR^{N \times L} \\
        \approxhat{\mE} &= \gamma^2 \mR \in \bbR^{L \times L} \\
        \tmE^\intercal \tmK^{-1} \tmE =& q^2 \begin{pmatrix} \mR_{11} \bone_1^\intercal & \cdots & \mR_{1L} \bone_L^\intercal \\ \vdots & \ddots & \vdots \\ \mR_{L1} \bone_1^\intercal & \cdots & \mR_{LL} \bone_L^\intercal \end{pmatrix} \begin{pmatrix} \mV_1 & & \\ & \ddots & \\ & & \mV_L \end{pmatrix} \begin{pmatrix} \mGamma_{11} & \cdots & \mGamma_{1L} \\ \vdots & \ddots & \vdots \\ \mGamma_{L1} & \ddots & \mGamma_{LL} \end{pmatrix} \\
        &\begin{pmatrix} \overline{\mV_1} & & \\ & \ddots & \\ & & \overline{\mV_L} \end{pmatrix} \begin{pmatrix} \mR_{11} \bone_1 & \cdots & \mR_{1L} \bone_1 \\ \vdots & \ddots & \vdots \\ \mR_{L1} \bone_L & \cdots & \mR_{LL} \bone_L \end{pmatrix} \\
        =& q^2  \begin{pmatrix} \mR_{11} \be_1^\intercal & \cdots & \mR_{1L} \be_L^\intercal \\ \vdots & \ddots & \vdots \\ \mR_{L1} \be_1^\intercal & \cdots & \mR_{LL} \be_L^\intercal \end{pmatrix} \begin{pmatrix} n_1 \mGamma_{11} & \cdots & \sqrt{n_1 n_L} \mGamma_{1L} \\ \vdots & \ddots & \vdots \\ \sqrt{n_L n_1} \mGamma_{L1} & \ddots & n_L \mGamma_{LL} \end{pmatrix} \\
        &\begin{pmatrix} \mR_{11} \be_1 & \cdots & \mR_{1L} \be_1 \\ \vdots & \ddots & \vdots \\ \mR_{L1} \be_l & \cdots & \mR_{LL} \be_L \end{pmatrix} \\
        =& q^2 \mR \mPi \mR
    \end{align*}
    where again $\be_\ell=(1,0,0,\dots)^\intercal \in \bbR^{n_\ell}$.
\end{proof}

\Section{Derivative-Informed GPs} \label{sec:gps_deriv_informeds}

% Sum[1/2^(p+1)*(1+Sum[1/2^(q+1)*(1+Sum[(1+r/2)/2^(r+1),{r,0,q-1}]),{q,0,p-1}]),{p,b-1,inf}]
% Sum[1/2^(p+1)*(1+Sum[1/2^(q+1)*(1+Sum[1/2^(r+1)*(1+Sum[(1+s/2)/2^(s+1),{s,0,r-1}]),{r,0,q-1}]),{q,0,p-1}]),{p,b-1,inf}]

Here we discuss GPs with derivative information \cite[Chapter 9.4]{rasmussen.gp4ml}, which may be seen as a special case of optimal recovery in Hilbert spaces as discussed in \cite[Chapter 16]{wendland2004scattered}. \Cref{sec:background_gps_derivatives} provides additional background on GPs which incorporate derivatives. Our development of derivative informed fast GPs using SI/DSI kernels is novel in the literature. 

\Subsection{Background on Derivatives of GPs}

\begin{definition}
    A process $f$ is mean squared differentiable if the following limit exists:
    $$\lim_{h_1,h_2 \to 0} \bbE\left[\left(\frac{f(x+h_1)-f(x)}{h_1} - \frac{f(x+h_2)-f(x)}{h_2}\right)^2\right].$$
    When the limits are replaced by right or left limits the process is said to be right or left mean squared differentiable respectively. 
\end{definition}

\begin{lemma} \label{lemma:msdiff_kernel_conditions}
    If a SPSD kernel $K: \calX \times \calX \to \bbR$ satisfies $K(x,x) = K(0,0)$ for all $x \in \calX$ with $\calX \subseteq \bbR$, then a Gaussian process $f \sim \calN(0,K)$ is mean squared differentiable if and only if both the following limit exists:
    $$\lim_{h \to 0} \frac{K(x+h,x)-K(x,x)}{h^2}\qquad\text{and}$$
    $$\lim_{h_1,h_2 \to 0} \frac{K(x+h_1,x+h_2)-K(x+h_1,x) - K(x+h_2,x) + K(x,x)}{h_1h_2}.$$
    Right and left mean squared differentiability follow from replacing limits by right and left limits respectively. 
\end{lemma}
\begin{proof}
    \begin{align*}
        \bbE&\left[\left(\frac{f(x+h_1)-f(x)}{h_1} - \frac{f(x+h_2)-f(x)}{h_2}\right)^2\right] \\
        &= \bbE\left[\left(\frac{f(x+h_1)-f(x)}{h_1}\right)^2\right] + \bbE\left[\left(\frac{f(x+h_2)-f(x)}{h_2}\right)^2\right] \\
        &- 2 \bbE\left[\frac{f(x+h_1)-f(x)}{h_1} \cdot \frac{f(x+h_2)-f(x)}{h_2}\right].
    \end{align*}
    Now, 
    \begin{align*}
        \lim_{h \to 0} \bbE\left[\left(\frac{f(x+h)-f(x)}{h}\right)^2\right] &= \lim_{h \to 0} \frac{\bbE\left[f^2(x+h) - 2 f(x+h)f(x) + f^2(x)\right]}{h^2} \\
        &= \lim_{h \to 0} \frac{K(x+h,x+h) - 2 K(x+h,x) + K(x,x)}{h^2} \\
        &= 2\lim_{h \to 0} \frac{K(x,x)-K(x+h,x)}{h^2},
    \end{align*}
    and 
    \begin{align*}
        \lim_{h_1,h_2 \to 0} & \bbE\left[\frac{f(x+h_1)-f(x)}{h_1} \cdot \frac{f(x+h_2)-f(x)}{h_2}\right] \\
        &= \lim_{h_1,h_2 \to 0} \bbE\left[\frac{f(x+h_1)f(x+h_2)-f(x+h_1)f(x)-f(x+h_2)f(x)+f^2(x)}{h_1h_2}\right] \\
        &= \lim_{h_1,h_2 \to 0} \frac{K(x+h_1,x+h_2)-K(x+h_1,x) - K(x+h_2,x) + K(x,x)}{h_1h_2}
    \end{align*}
\end{proof}

If a process $f$ is mean squared differentiable, we will denote its derivative process as $f^{(1)}$. If $f^{(1)}$ is mean squared differentiable, then we will denote its derivative by $f^{(2)}$, and so on. If it exists, we will say $f^{(\beta)}$ is the $\beta^{th}$ derivative of the process $f$. For multidimensional domains $\calX \subseteq \bbR^d$ and $f: \calX \to \bbR$, as we assume going forward, we will denote $f^{(\bbeta)}$, with $\bbeta \in \bbN_0^d$, to be the process taking $\beta_j$ derivatives in dimension $j \in \{1,\dots,d\}$. 

When a function is mean squared differentiable, then, as differentiation is a linear operator, we may apply the derivatives directly to the kernel. The squared exponential kernel \Cref{def:se_kernel} is infinitely mean squared differentiable. Discussions on the differentiability of other kernels are readily available, see for example \cite{fasshauer.meshfree_approx_methods_matlab,banerjee2003directional}. When GP derivatives exist, the required kernel derivatives may often be readily obtained using automatic differentiation.

\Subsection{Connection to Multitask GPs}

Suppose we have $L$ compatible derivative multi-indices $\bbeta_1,\dots,\bbeta_L \in \bbN_0^d$, and we observe at level $\ell \in \{1,\dots,L\}$ derivative observations
$$\by_\ell = f^{(\bbeta_\ell)} + \bvarepsilon_\ell.$$
As with the multitask GPs discussed in \Cref{sec:mtgps}, we will assume IID zero-mean Gaussian noise $\bvarepsilon_\ell$ with common noise variance (nugget) $\xi_\ell >0$. For computational purposes, we will also assume $\bvarepsilon_1,\dots,\bvarepsilon_L$ are independent as was the case for multitask GPs. 

To construct a GP over $f$ which incorporates derivative observations, we must define a  prior mean $M:\calX \to \bbR$ and a prior covariance kernel $K:\calX \times \calX \to \bbR$ which are $\bbeta_\ell$ times mean squared differentiable for every $\ell \in \{1,\dots,L\}$. For non-zero prior means, differentiability still holds as long as the prior mean is sufficiently differentiable. GPs with derivatives can be seen as a special case of the multitask GPs in \Cref{sec:mtgps} with
\begin{align*}
    M(\ell,\bx) &:= M^{(\bbeta_\ell)}(\bx), \\
    K((\ell,\bx),(\ell',\bx)) &:= K^{(\bbeta_\ell,\bbeta_{\ell'})}(\bx,\bx')
\end{align*} 
for all $\bx,\bx' \in \calX$ and $\ell,\ell' \in \{1,\dots,L\}$. 

\Subsection{Fast Derivative-Informed GPs}

The following corollary concerns derivatives of GPs with our product-form shift-invariant (SI) and digitally-shift-invariant (DSI) kernels which have $\calX = [0,1]^d$. As these kernels are of product-form, their per-dimension differentiability holds for their multidimensional forms. 

\begin{corollary} \label{corollary:msdiff_si_dsi_kernels}
    Suppose the GP prior mean $M:[0,1]^d \to \bbR$ is sufficiently mean squared differentiable, e.g., the constant means we considered earlier.
    \begin{enumerate}
        \item For the product-form SI kernel $K$ in \Cref{def:si_kernels} built from one-dimensional kernels $\tK_{\alpha_j}$  with smoothness parameters $\balpha \in \bbR_{>1/2}^d$, the GP $f \sim \mathrm{GP}(M,K)$ is mean squared differentiable $\bbeta \in \bbN_0^d$ times, i.e., $f^{(\bbeta)}$ exists, so long as $2\balpha-\bbeta>1$ elementwise. The derivative forms are described in \Cref{def:dsi_kernels}.  
        \item For the product-form DSI kernel $K$ in \Cref{def:dsi_kernels} built from one-dimensional kernels $\ddK_{\alpha_j,0,0}$ with smoothness parameters $\balpha \in \{2,3,4\}^d$, the GP $f \sim \mathrm{GP}(M,K)$ is not mean squared differentiable. However, if $\balpha \in \{3,4\}^d$, then the posterior mean is right differentiable on all dyadic rationals $x \in \bbB$ with derivative given in \Cref{lemma:b2_alpha234_digital_dright_derivs}.  
        \item For the product-form DSI kernel $K$ in \Cref{def:dsi_kernels} built from one-dimensional kernels $\dddK_{\alpha_j}$  with smoothness parameters $\balpha \in \bbR_{>0}^d$, the GP $f \sim \mathrm{GP}(M,K)$ is mean squared differentiable $\bbeta \in \bbN_0^d$ times, i.e., $f^{(\bbeta)}$ exists, so long as $\balpha-\bbeta>2$ elementwise. All derivatives, when they exist, are equal to $0$, see \Cref{thm:derivs_dddK}. 
    \end{enumerate}
\end{corollary}
\begin{proof}
    All claims use \Cref{lemma:msdiff_kernel_conditions}. The claim regarding SI kernels follows from applying \Cref{thm:msdiff_Kcheckhat}. The claim regarding the first DSI kernel built from $\ddK_{\alpha_j,0,0}$ follows from \Cref{thm:msdiff_K}. The claim regarding the second DSI kernel built from $\dddK_\alpha$ follows from \Cref{thm:msdiff_Kdotdot}. 
\end{proof}

All the technology developed in \Cref{sec:fmtgps} for fast multitask GPs immediately applies to fast derivative-informed GPs. The only requirement is that the derivative kernels $K^{(\bbeta_\ell,\bbeta_{\ell'})}$ are still shift-invariant (SI) if we are using lattices, or digitally-shift-invariant (DSI) if we are using digital nets. \Cref{def:si_kernels} shows this is in fact the case for SI kernels. We have yet to find practically useful DSI kernels which are also mean squared differentiable.

% \begin{figure}[!ht]
%     \centering
%     \includegraphics[width=.75\linewidth]{figs/gram_matrix_structures.pdf}
%     \caption{Gram matrix structures when pairing a squared exponential (SE) kernel with regular grid points, a shift-invariant (SI) kernel with lattice points, and a digitally-SI (DSI) kernel with a digital net. When observing just the function $f$, the SE kernel with grid points has Kronecker structure, the SI kernel with lattice points has circulant structure, and the DSI kernel with digital net structure has RSBT${}_2$ structure. If we include gradient observations, then these same structures can be found in each block of the Gram matrix. Here $n=16$ points are used in $d=2$ dimensions.}
%     \label{fig:gram_mat_structures}
% \end{figure}

\Section{Numerical Experiments} \label{sec:gps_numerical_experiments}

The numerical experiments use our \texttt{FastGPs} package (\url{https://alegresor.github.io/fastgps/}) version 1.0 or later, which relies on both the \texttt{QMCPy} and \texttt{PyTorch} packages. \texttt{FastGPs} is readily installed using the command \texttt{pip install -U fastgps}, after which we may import our required packages:
\begin{lstlisting}[style=Python]
import fastgps as fgps
import qmcpy as qp
import torch
\end{lstlisting}
\noindent While we do not provide code snippets here, extensive documentation and examples are available at the project URL above. \Cref{sec:FastBayesianMLQMC} contains additional experiments using the fast Bayesian cubature support in \texttt{FastGPs} to enable multilevel Quasi-Monte Carlo without replications. 

Our experiments use the $\alpha=2$ SI product kernel in \Cref{def:si_kernels} and $\alpha=4$ DSI product kernel in \Cref{def:dsi_kernels} with higher-order univariate kernels $\ddK_{\alpha,0,0}$. LD point sets are generated using \texttt{QMCPy}'s default generating vector for lattices and default generating matrices for digital nets. We found our results to be robust to these choices. 
All optimizations of the MLL were run to convergence using the resilient back propagation (Rprop) algorithm \cite{riedmiller1993direct} with learning rate $0.1$. We found optimization of the SI and DSI kernels converged in orders of magnitude fewer iterations than for squared exponential (SE) kernels. 
Accuracy will be measured by the $L_2$ relative error 
\begin{equation}
    \sqrt{\int_{[0,1]^d} \llvert \sf(\bx)-\bbE[f(\bx) | \by] \rrvert^2 \D \bx / \int_{[0,1]^d} \llvert \sf(\bx) \rrvert^2 \D \bx}
\end{equation}
between the true benchmark functions $\sf$ and the posterior means $\bbE[f(\bx) | \by]$. This is approximated numerically using the $\ell_2$ norm and LD Halton points. All computations were performed on the CPUs of a 2023 MacBook Pro with an M3 Max processor.

We are interested in comparing zero-mean GPs with SE, SI, and DSI kernels both with and without gradient information. We note that our benchmarks are smooth, two-sided differentiable functions which do not necessarily align with our fast GP modeling assumptions. Rather than benchmarking only on functions for which fast GPs are theoretically justified, we are instead interested in seeing how our methods apply to more realistic functions where our kernel assumptions may not hold.  

\Subsection{Low-dimensional functions without derivatives} 

We begin in the low-dimensional setting without derivative information and compare across SE GPs with lattice points, fast SI GPs with lattice points, and fast DSI GPs with digital nets. \Cref{fig:1d_gps} considers a set of visually diverse $d=1$ curves. \Cref{fig:2d_gp_error} considers the $d=2$-dimensional Ackley function \cite{ackley2012connectionist}. Notice the periodicity of SI predictions and the discontinuity of DSI predictions as reflected in the kernels.

\begin{figure*}[!ht]
    \centering
    \includegraphics[width=.95\linewidth]{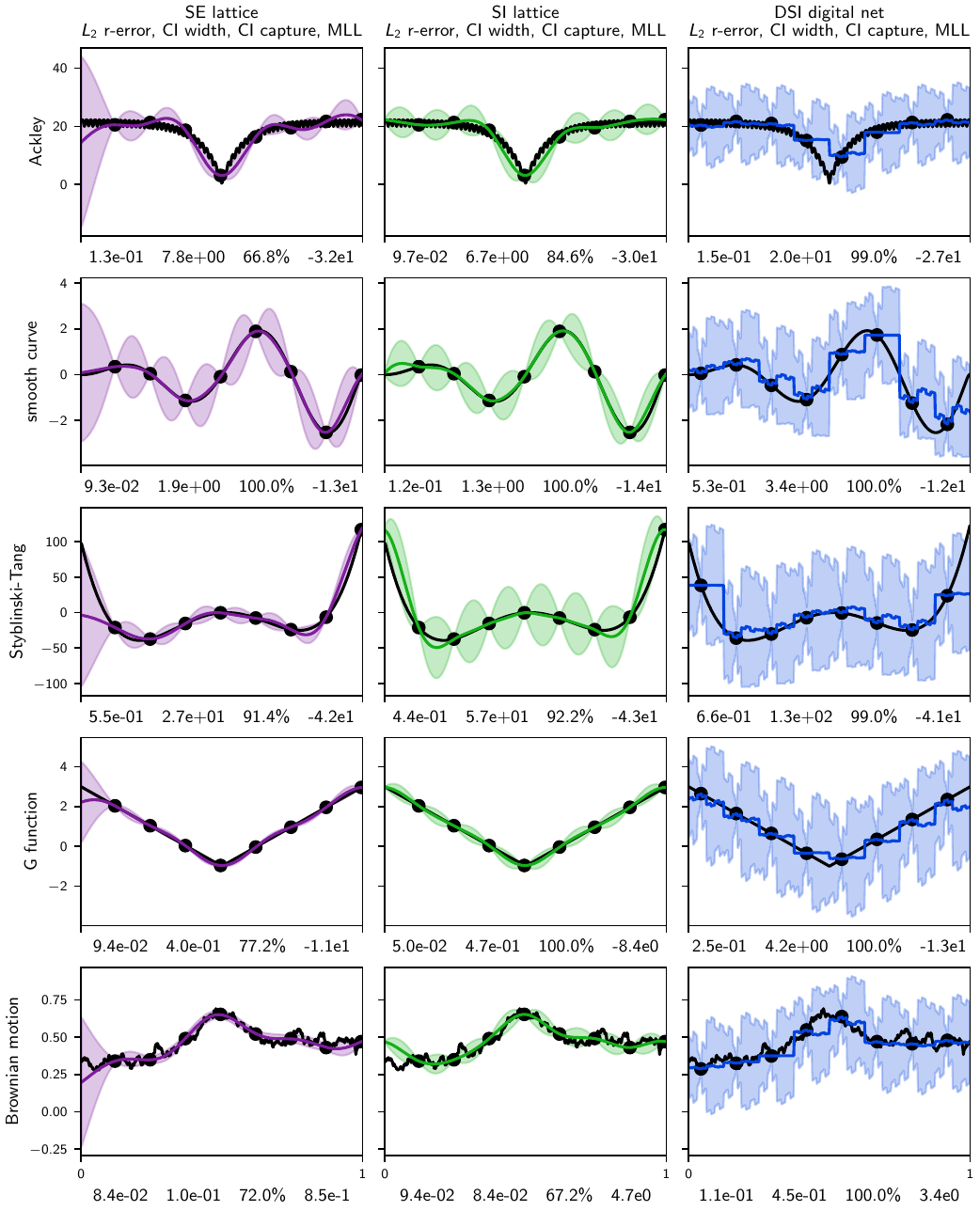}
    \caption{Comparison of Gaussian process (GP) methods on visually diverse example functions in dimension $d=1$. Black lines are the true functions, colored lines are the posterior means, and shaded regions show 99\% confidence intervals (CIs). The figure lists the $L_2$ relative errors, the average CI width, the proportion of inputs for which the  CI captures the true function, and the optimized MLL.}
    \label{fig:1d_gps}
\end{figure*}

For the $d=1$ examples, similar accuracies and MLLs were achieved by all three methods. SE and SI GPs yield similar $99\%$ confidence intervals (CIs)
\begin{equation}
    \bbE[f(\bx) | \by] \pm 2.58 \sqrt{\bbV[f(\bx) | \by]} 
\end{equation}
as evidence by their similar (average) CI widths and CI capture proportions. DSI GPs are often less confident in their predictions, but the CI capture proportions always match the desired 99\% confidence. SE GPs tend to be overconfident as evidenced by the low capture fractions. SI GPs tend to yield moderately sized CIs with moderate coverage while DSI GPs give large CIs with high coverage. 

For the $d=2$ Ackley function fit to $n=4096$ points, our fast SI and DSI GP methods achieved slightly better accuracy than the classic SE GP method and were around $2000$ and $1000$ times faster respectively. However, the optimized MLL was over $6000$ for SE GP and only around $-0.004$ for SI and DSI GPs. 

\Subsection{Benchmarks with and without gradients} 

Next, \Cref{tab:benchmarks} shows accuracy, fitting times, and marginal log-likelihoods (MLLs) for GPs using $n=1024$ points both with and without gradient information on a number of synthetic benchmarks from the Virtual Library of Simulation Experiments \cite{VLSE}. For the SI kernel, the periodizing baker transform $\phi(x) = 1-2\lvert x-1/2 \rvert$ was used to improve estimates for the gradient-informed Branin and Hartmann functions. 
%We note that the common set of benchmarks we have chosen are smooth, two-sided differentiable functions which do not align with our fast GP modeling assumptions of either periodicity or only right-differentiability at dyadic rationals. Rather than benchmarking only on functions for which fast GP is theoretically guaranteed, we are interested in seeing how our methods apply to more realistic functions when statistical assumptions are not necessarily satisfied.  

\begin{table}[!ht]
    \caption{Comparison of Gaussian process (GP) error, fitting time, and MLL estimates with $n=1024$ points on benchmark functions: Ackley $d=1$, Branin $d=2$, Camel $d=2$, Styblinski--Tang (StyTang) $d=2$, and Hartman $d=6$. SE is the standard GP using lattice points. SI is the fast GP with lattices. DSI is the fast GP with digital nets. $f$ indicates the GP was fit to function values only, while $(f,\nabla f)$ indicates the GP was fit with function and gradient information.}
    \centering
    \begin{tabular}{r|cc|cc|c}
        & \multicolumn{5}{c}{$L_2$ relative error} \\
        & \multicolumn{2}{c|}{SE} & \multicolumn{2}{c|}{SI} & \multicolumn{1}{c}{DSI} \\
        & $f$ & $(f, \nabla f)$ & $f$ & $(f, \nabla f)$ & $f$ \\
        \hline
            Ackley& 2.8e-4 & 3.6e-1 & 1.0e-4 & 8.8e-5 & 1.9e-1 \\
            Branin& 6.8e-6 & 3.5e-7 & 1.7e-1 & 2.4e-3 & 1.8e-2 \\
            Camel & 2.3e-5 & 2.5e-6 & 2.1e-2 & 1.4e-1 & 4.3e-2 \\
            StyTang & 6.6e-5 & 3.3e-6 & 3.2e-2 & 3.5e-1 & 5.0e-2 \\
            Hartmann & 3.3e-2 & 7.8e-3 & 5.7e-2 & 3.7e-2 & 8.0e-2 \\
        \hline 
        \hline 
        & \multicolumn{5}{c}{time per optimization step} \\
        & \multicolumn{2}{c|}{SE} & \multicolumn{2}{c|}{SI} & \multicolumn{1}{c}{DSI} \\
        & $f$ & $(f, \nabla f)$ & $f$ & $(f, \nabla f)$ & $f$ \\
        \hline
            Ackley& 5.3e-2 & 2.7e-1 & 5.7e-4 & 1.4e-3 & 1.1e-3 \\
            Branin& 5.2e-2 & 1.0e0 & 5.2e-4 & 2.8e-3 & 1.1e-3 \\
            Camel& 4.7e-2 & 1.0e0 & 5.5e-4 & 2.8e-3 & 1.1e-3 \\
            StyTang& 4.9e-2 & 1.1e0 & 5.2e-4 & 2.8e-3 & 1.1e-3 \\
            Hartmann& 8.5e-2 & 2.2e1 & 7.5e-4 & 1.3e-2 & 1.1e-3 \\
        \hline 
        \hline 
        & \multicolumn{5}{c}{MLL (Marginal Log-Likelihood)} \\
        & \multicolumn{2}{c|}{SE} & \multicolumn{2}{c|}{SI} & \multicolumn{1}{c}{DSI} \\
         & $f$ & $(f, \nabla f)$ & $f$ & $(f, \nabla f)$ & $f$ \\
        \hline
            Ackley& 8.7e2 & -1.9e7 & 5.0e2 & -4.7e3 & -1.8e3 \\
            Branin & 3.4e3 & 1.3e4 & -4.8e3 & -1.4e4 & -2.8e3 \\
            Camel & 3.3e3 & 1.0e4 & -1.8e3 & -1.4e4 & -1.0e3 \\
            StyTang & 3.4e3 & 1.1e4 & -2.4e3 & -1.8e4 & -8.5e2 \\
            Hartmann & 1.1e3 & 6.7e3 & 6.5e2 & -5.5e3 & 6.2e2 \\
        \hline 
    \end{tabular}
    \label{tab:benchmarks}
\end{table}

Here we found results to be highly dependent on the benchmark. On the simpler $d=2$ Branin, Camel, and Styblinski--Tang (StyTang) benchmarks, SE was orders of magnitude more accurate than SI and DSI GPs both with and without gradient information. However, on the highly oscillatory $d=1$ Ackley function, SI is able to achieve comparable accuracy, and SI was able to efficiently utilize gradient information while SE cannot. 
Moreover, on the $d=6$ Hartmann example, all three methods attain similar accuracy without gradient information. SI was not able to utilize gradient information on either the Camel or StyTang benchmarks. This lack of efficiency may be attributed to the benchmarks not satisfying the assumptions of fast GPs as evidenced by the lower MLLs for SI/DSI GPs compared to SE GPs.

For GPs on $n=1024$ points without gradients, the SI GP was consistently around $100$ times faster than the SE GP. Including gradients at each of the $n=1024$ points makes SI over $350$ times faster than SE on $d=2$ examples and almost $1700$ times as fast on the $d=6$ Hartmann function. DSI typically takes around twice as long as SI, partially due to the requirement to compute the DSI sum term $S$ in \eqref{eq:K4_sum_term}. The $\alpha=2$ and $\alpha=3$ DSI kernels are faster to evaluate and did not significantly degrade accuracy in our testing.

\Subsection{Accuracy vs timing vs dimension for StyTang} 

Finally, in \Cref{fig:speed_accuracy_comp_full} we track effects of dimension $d \in \{1,2,4,8,16\}$ and sample size $n$ on the accuracy and time for the Styblinski--Tang benchmark \cite{styblinski1990experiments}.  We found that as $d$ increases the SI and DSI GPs become significantly more accurate than the SE GPs. For SI kernels and small $n$, the benefit of gradient information appears to increase with $d$; this benefit is diminished at a horizon for $n$ that also appears to increase with $d$. In other words, we found gradient information for fast GPs to be most valuable for large $d$ and small $n$.

\begin{figure*}[!ht]
    \centering
    \includegraphics[width=1\linewidth]{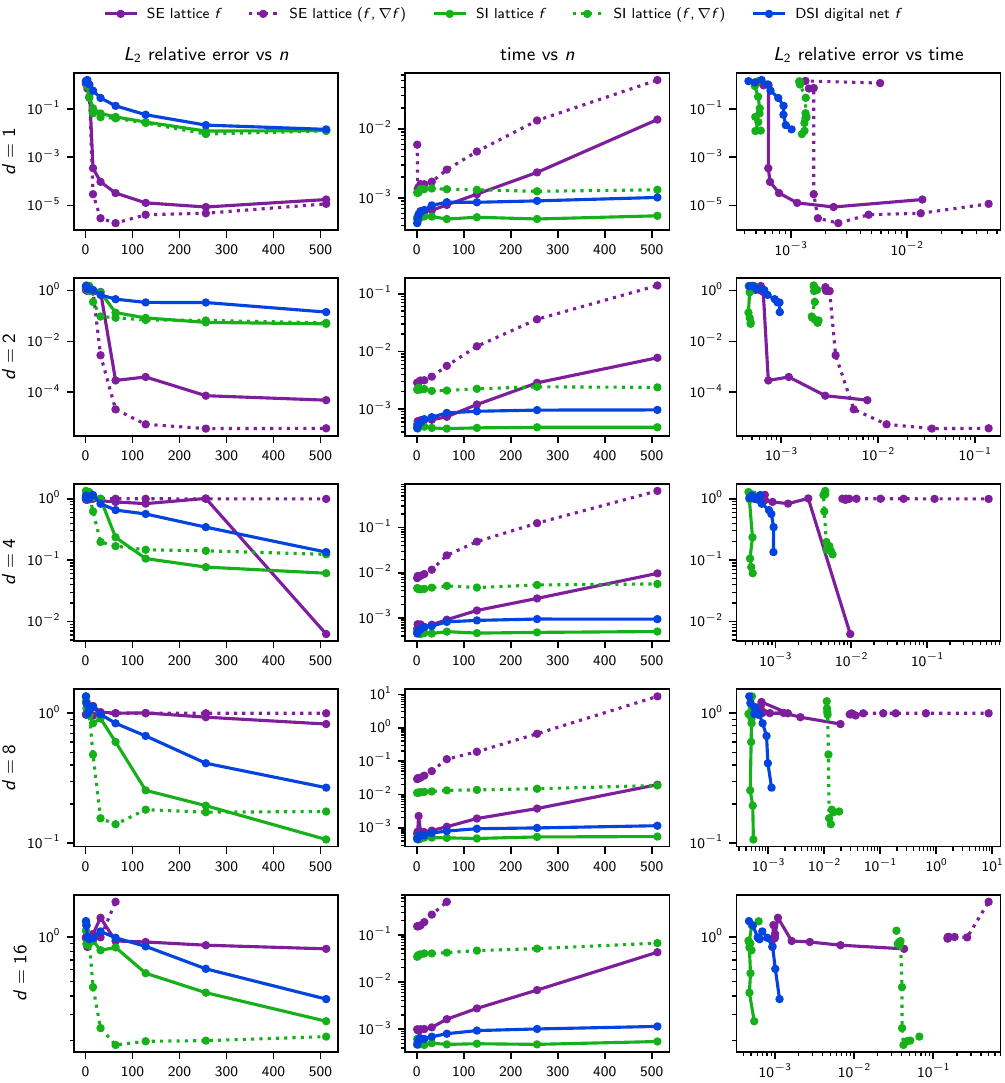}
    \caption{Gaussian processes (GPs) of varying dimension $d$ on the Styblinski--Tang benchmark function. Here we compare the error, number of samples $n$, time per optimization step, and inclusion of gradient information $\nabla f$. GP fitting was run for $250$ optimization steps.}
    \label{fig:speed_accuracy_comp_full}
\end{figure*}

\Chapter{Applications in UQ and SciML} 

This chapter will highlight a number of application projects in: 
\begin{enumerate}
    \item Building fast Gaussian processes for multilevel Monte Carlo without replications (\Cref{sec:FastBayesianMLQMC}). 
    \item Using Gaussian process to predict probability of system failure (\Cref{sec:pfgpci}).
    \item Solving the Darcy flow equation using multilevel Gaussian processes (\Cref{sec:gp4darcy}). 
    \item Modeling radiative heat transfer using neural operators (\Cref{sec:RTE_DeepONet}).
    \item Solving nonlinear PDEs with random coefficients to machine precision accuracy using scientific machine learning (\Cref{sec:CHONKNORIS}). 
\end{enumerate}
Each section corresponds to a paper written with collaborators across academia, national labs, and industry. We will focus on the technical contributions of each project. The full papers contain additional problem context and connections to the broader literature. 

\Section{Fast Bayesian Multilevel Quasi-Monte Carlo} \label{sec:FastBayesianMLQMC}

\begin{quotation}
    This section follows \cite{sorokin.FastBayesianMLQMC}, a paper I worked on during my 2025 DOE SCGSR (US Department of Energy Office of Science Graduate Student Research) Program fellowship in collaboration with Pieterjan Robbe, Fred J. Hickernell, Gianluca Geraci, and Michael S. Eldred.
\end{quotation}

Existing multilevel Quasi-Monte Carlo (MLQMC) methods use multiple independent randomizations of a low-discrepancy (LD) sequence in order to compute error estimates on each level. Using multiple randomizations is inherently less efficient than sampling more points from a single LD sequence. We propose to cast the MLQMC problem into a Bayesian cubature framework which requires only a single LD sequence and provides error estimates in terms of the posterior variance. Moreover, using \Cref{sec:fast_gps} we may perform Gaussian process (GP) regression at only $\calO(n \log n)$ cost in the number of points $n$, including hyperparameter optimization, by pairing certain LD sequences to matching kernel forms. As existing MLQMC procedures iteratively double the sample size on the level of maximum utility, we develop a novel utility function for fast Bayesian MLQMC which trades off the cost of doubling and the forecasted posterior variance after doubling, assuming fixed GP hyperparameters. We also showcase a new digitally-shift-invariant (DSI) kernel of adaptive smoothness for use with fast digital net GPs which takes a hyperparameter-weighted sum of existing higher-order smoothness DSI kernels. A series of numerical experiment validate the performance of our fast Bayesian approximations and error estimates for both single-level problems and multilevel problems with a fixed number of levels. We find the proposed fast Bayesian MLQMC methods indeed provide superior approximations compared to standard MLQMC methods with multiple randomizations, and the Bayesian error estimates using digital nets are accurate for the true errors.  

Quasi-Monte Carlo was detailed in \Cref{sec:qmc} with an initial overview of multilevel Monte Carlo (MLMC) \cite{giles2015multilevel,giles.MLMC_path_simulation} and multilevel Quasi-Monte Carlo (MLQMC) \cite{giles.mlqmc_path_simulation} methods available in \Cref{sec:stop_crit_multilevel}. We will compare against the classic MLQMC method with replications which on each level employs the replicated QMC estimator described in \Cref{sec:stop_crit_qmc_rep_student_t}. \Cref{sec:fast_gps} provided background on fast GPs with low-discrepancy points and matching (digitally-)shift-invariant kernels, their hyperparameter optimization, and their application to fast Bayesian cubature. A unified introduction is give in the full paper \cite{sorokin.FastBayesianMLQMC}. 

The remainder of this section is organized as follows. \Cref{sec:BMLQMC_methods} details our methods. We describe the existing IID-MLMC method in \Cref{sec:mlmc}, then we describe the existing MLQMC method with multiple randomizations (replications)  in \Cref{sec:mlqmc}, and then we describe fast Bayesian cubature along with our proposed fast Bayesian MLQMC method in \Cref{sec:bmlqmc}. Algorithms are provided for each of these three methods. Afterwards, \Cref{sec:numerical_experiments} gives numerical experiments comparing the three algorithms for both single-level problems and multilevel problems with a fixed number of levels. Here will only present results for the two QMC methods using digital nets. Results for lattices are available in the full text \cite{sorokin.FastBayesianMLQMC}.

\Subsection{Methods} \label{sec:BMLQMC_methods}

Assume we have access to quantities of interest $f_\ell: [0,1]^d \to \bbR$ where $\ell \in \{1,\dots,L\}$ indexes the fidelity of the simulation, with larger $\ell$ corresponding to higher-fidelity simulations. Sometimes the dimension $d$ varies with the level so $f_\ell: [0,1]^{d_\ell} \to \bbR$; for example, $d_\ell$ may be the  number of monitoring times of an asset price for option pricing with $d_\ell$ increasing in $\ell$. In such cases, we take $d = \max(d_1,\dots,d_L)$ and subset dimensions accordingly. We will also assume the number of levels is fixed, although extensions to adaptive multilevel schemes where the highest-fidelity level is determined by the remaining bias should be straightforward \cite{giles2015multilevel}.

We are interested in approximating the expectation of the maximum-fidelity simulation $\nu := \bbE[f_L(\bX)]$ subject to $\bX \sim \calU[0,1]^d$. For problems with non-uniform stochasticity, one may employ a change of variables to produce a compatible form, see \Cref{sec:variable_transforms}. As is standard for MLMC methods, we will expand the above quantity of interest as a telescoping sum 
\begin{equation} \label{eq:tele_mlmc}
    \nu = \bbE[f_L(\bX)] = \sum_{\ell=1}^L \bbE[f_\ell(\bX) - f_{\ell-1}(\bX)] = \sum_{\ell=1}^L \bbE[Y_\ell(\bX)] = \sum_{\ell=1}^L \mu_\ell
\end{equation}
where $Y_\ell := f_\ell-f_{\ell-1}$, $\mu_\ell := \bbE[Y_\ell]$, and we set $f_0 = 0$ so $Y_1 = f_1$.
We will denote by $C_\ell$ the cost of evaluating $Y_\ell$ and by $V_\ell := \bbV[Y_\ell(\bX)]$ the variance of the $\ell^\mathrm{th}$ difference, again subject to $\bX \sim \calU[0,1]^d$.

\Subsubsection{Multilevel Monte Carlo with Independent Points} \label{sec:mlmc}

Classic Monte Carlo methods approximate the true mean of an expectation by the $n$-sample mean at IID (independent and identically distributed) sampling locations. Standard MLMC methods \cite{giles.MLMC_path_simulation} use IID samples on each level, 
$$\bx_{1,0},\dots,\bx_{1,n_1-1},\dots,\bx_{L,0},\dots,\bx_{L,n_L-1} \simiid \calU[0,1]^d,$$
to create the estimate 
\begin{equation} \label{eq:mlmc}
    \hnu = \sum_{\ell=1}^L \hmu_\ell \qqtqq{where} \hmu_\ell = \frac{1}{n_\ell} \sum_{i=0}^{n_\ell-1} Y_\ell(\bx_{\ell,i}).
\end{equation}
This gives $\bbV[\hnu] = \sum_{\ell=1}^L V_\ell / n_\ell$, which, for a given error tolerance on $\bbV[\hnu]$, will optimally allocate $n_\ell \propto \sqrt{V_\ell/ C_\ell}$. As $V_\ell$ is not known, it is typically approximated by an unbiased estimator. For consistency with the QMC schemes we will present in \Cref{sec:mlqmc} and \Cref{sec:bmlqmc}, we will use an iterative doubling version of this IID-MLMC scheme as detailed in \Cref{alg:mlmc}. 

\begin{algorithm}[!ht]
    \fontsize{12}{10}\selectfont
    \caption{\texttt{MC}: MLMC With Independent Points}
    \label{alg:mlmc}
    \begin{algorithmic}
        \Require $N>0$ \Comment{the budget}
        \Require $C_1,\dots,C_L > 0 $ \Comment{the cost of evaluating $Y_1,\dots,Y_L$ respectively}
        \Require $n_1^\mathrm{next},\dots,n_L^\mathrm{next} \in \bbN$ satisfying $\sum_{\ell=1}^L n_\ell^\mathrm{next} C_\ell \leq N$ \Comment{the initial sample sizes}
        \State $n_\ell \gets 0$ for $\ell \in \{1,\dots,L\}$ \Comment{the number of model evaluations on each level}
        \State $\calL \gets \{1,\dots,L\}$ \Comment{the set of levels to update}
        \While{true}
            \State Generate $\bx_{\ell,i} \sim \calU[0,1]^d$ for $\ell \in \calL$ and $n_\ell \leq i < n_\ell^\mathrm{next}$ all independently
            \State Evaluate $Y_\ell(\bx_{\ell,i})$ for $\ell \in \calL$ and $n_\ell \leq i < n_\ell^\mathrm{next}$
            \State $\hmu_\ell \gets \frac{1}{n_\ell} \sum_{i=0}^{n_\ell-1} Y_\ell(\bx_{\ell,i})$ for $\ell \in \calL$ \Comment{estimate of $\mu_\ell$}
            \State $\sigma_\ell^2 \gets \frac{1}{n_\ell-1} \sum_{i=0}^{n_\ell-1} (Y_\ell(\bx_{\ell,i}) - \hmu_\ell)^2$ for $\ell \in \calL$\Comment{estimate of $V_\ell$}
            \State $\calL_\mathrm{feasible} \gets \{\ell \in \{1,\dots,L\}: \sum_{\ell'=1}^L C_{\ell'} n_{\ell'} + C_\ell n_\ell \leq N\}$ \Comment{feasible set of levels}
            \If{$\calL_\mathrm{feasible} = \emptyset$} break \EndIf \Comment{exit loop before going over budget}
            \State $\starhat{\ell} \gets \argmax_{\ell \in \calL_\mathrm{feasible}} \frac{\sigma_\ell^2}{n_\ell C_\ell}$ \Comment{choose the level of maximum utility} 
            \State $\calL \gets \{\starhat{\ell}\}$ and $n_{\starhat{\ell}} \gets n_{\starhat{\ell}}^\mathrm{next}$ and $n_{\starhat{\ell}}^\mathrm{next} \gets 2n_{\starhat{\ell}}$ \Comment{double samples on the chosen level}
        \EndWhile
        \State $\hnu \gets \sum_{\ell=1}^L \hmu_\ell$ \Comment{estimate of $\nu$}
        \State $\sigma^2 \gets \sum_{\ell=1}^L \hsigma_\ell^2/n_\ell$ \Comment{estimate of $\bbV\left[\hnu\right]$}
        \\ \Return $\hnu,\sigma,\{n_\ell\}_{\ell=1}^L$ \Comment{the estimate, its standard error, and samples}
    \end{algorithmic}
\end{algorithm}

\Subsubsection{Multilevel Quasi-Monte Carlo with Replications} \label{sec:mlqmc}

For Quasi-Monte Carlo (QMC) methods, we will consider shifted rank-1 lattices and digitally-shifted base $2$ digital nets, both in the extensible radical inverse order. These low-discrepancy sequences were plotted alongside IID points in \Cref{fig:points}. From the unrandomized lattice or digital sequence $(\bz_i)_{i \in \bbN_0}$, we will write the (digitally) shifted point set $(\bx_i)_{i \in \bbN_0}$ with $\bx_i = \bz_i \oplus \bDelta$ and shift $\bDelta \in [0,1)^d$. For fixed $(\bz_i)_{i \in \bbN_0}$, if $\bDelta \sim \calU[0,1]^d$ then $\bx_i \sim \calU[0,1]^d$. For lattices $\oplus$ corresponds to addition modulo $1$, while for digital nets $\oplus$ performs the elementwise exclusive or (XOR) between binary expansions. See \Cref{sec:lattices} and \Cref{sec:dnets} for complete definitions and additional details on lattices and digital nets respectively. 

The multilevel QMC (MLQMC) method from \cite{giles.mlqmc_path_simulation} uses $R$ independent shifts on each level 
$$\bDelta_{1,1},\dots,\bDelta_{1,R},\dots,\bDelta_{L,1},\dots,\bDelta_{L,r} \simiid \calU[0,1)^d$$
to construct $LR$ replications of LD point sets 
\begin{equation} \label{eq:mlqmc} 
    (\bx_{\ell,r,i})_{i=0}^{n_\ell-1}, \qquad \bx_{\ell,r,i} := \bz_i \oplus \bDelta_{\ell,r}, \qquad 1 \leq \ell \leq L, \quad 1 \leq r \leq R.
\end{equation}
One then replaces $\hnu$ in \eqref{eq:mlmc} by an approximation derived from the $LR$ independent sample means $\tmu_{\ell,r}$:
$$\hnu = \sum_{\ell=1}^L \hmu_\ell \qqtqq{where} \hmu_\ell = \frac{1}{R} \sum_{r=1}^R \tmu_{\ell,r} \qqtqq{and} \tmu_{\ell,r} = \frac{1}{n_\ell} \sum_{i=0}^{n_\ell-1} Y_\ell(\bx_{\ell,r,i}).$$
Let $\tV_{\ell,n_\ell} = \bbV[1/n_\ell\sum_{i=0}^{n_\ell-1} Y_\ell(\bz_i \oplus \bDelta)]$ be the variance of the LD sample mean with respect to the random shift $\bDelta \sim \calU[0,1]^d$, which should go to $0$ as $n_\ell \to \infty$. Then we have $\bbV\left[\hnu\right] = \sum_{\ell=1}^L \tV_{\ell,n_\ell}/R$. Rather than increasing the number of randomizations $R$, it will be more efficient to increase the number of LD point $n_\ell$ and exploit the space filling properties of extensible LD sequences. We present the standard multilevel QMC method in \Cref{alg:mlqmc}, which exploits the fact that the extensible lattices and base $2$ digital sequences we consider achieve nice space filling properties at samples sizes which are powers of $2$.

\begin{algorithm}[!ht]
    \fontsize{12}{10}\selectfont
    \caption{\texttt{RQMC}: MLQMC With Replications}
    \label{alg:mlqmc}
    \begin{algorithmic}
        \Require $N>0$ \Comment{the budget}
        \Require $C_1,\dots,C_L > 0 $ \Comment{the cost of evaluating $Y_1,\dots,Y_L$ respectively}
        \Require $R>0$ \Comment{the fixed number of randomizations for each level, we use $R=8$.} 
        \Require $n_1^\mathrm{next},\dots,n_L^\mathrm{next} \in \bbN$ powers of two satisfying $R \sum_{\ell=1}^L n_\ell^\mathrm{next} C_\ell \leq N$ \\ \Comment{the initial sample sizes}
        \Require A generating vector for lattices or generating matrices for digital nets. 
        \State Generate $\bDelta_{1,1},\dots,\bDelta_{1,R},\dots,\bDelta_{L,1},\dots,\bDelta_{L,r} \simiid \calU[0,1)^d$
        \State $n_\ell \gets 0$ for $\ell \in \{1,\dots,L\}$ \Comment{the number of evaluations per randomization}
        \State $\calL \gets \{1,\dots,L\}$ \Comment{the set of levels to update}
        \While{true}
            \State Generate $x_{\ell,r,i} = \bz_i \oplus \bDelta_{\ell,r}$ for $\ell \in \calL$ and $1 \leq r \leq R$ and $n_\ell \leq i < n_\ell^\mathrm{next}$
            \State Evaluate $Y_\ell(\bx_{\ell,r,i})$ for $\ell \in \calL$ and $1 \leq r \leq R$ and $n_\ell \leq i < n_\ell^\mathrm{next}$
            \State $\tmu_{\ell,r} \gets \frac{1}{n_\ell} \sum_{i=0}^{n_\ell-1} Y_\ell(\bx_{\ell,r,i})$ for $\ell \in \calL$ and $1 \leq r \leq R$ \\ \Comment{per-randomization estimate of $\mu_\ell$}
            \State $\hmu_\ell \gets \frac{1}{R} \sum_{r=1}^{R} \tmu_{\ell,r}$ for $\ell \in \calL$ \Comment{aggregate estimate of $\mu_\ell$}
            \State $\tsigma_\ell^2 \gets \frac{1}{R-1} \sum_{r=1}^{R} (\tmu_{\ell,r} - \hmu_\ell)^2$ for $\ell \in \calL$ \Comment{estimate of $\tV_{\ell,n_\ell}$}
            \State $\calL_\mathrm{feasible} \gets \{\ell \in \{1,\dots,L\}: \sum_{\ell'=1}^L R C_{\ell'} n_{\ell'} + R C_\ell n_\ell \leq N\}$ \Comment{feasible set of levels}
            \If{$\calL_\mathrm{feasible} = \emptyset$} break \EndIf \Comment{exit loop before going over budget}
            \State $\starhat{\ell} \gets \argmax_{\ell \in \calL_\mathrm{feasible}} \frac{\tsigma_\ell^2}{R n_\ell C_\ell}$ \Comment{choose the level of maximum utility}
            \State $\calL \gets \{\starhat{\ell}\}$ and $n_{\starhat{\ell}} \gets n_{\starhat{\ell}}^\mathrm{next}$ and $n_{\starhat{\ell}}^\mathrm{next} \gets 2n_{\starhat{\ell}}$ \Comment{double samples on the chosen level}
        \EndWhile
        \State $\hnu \gets \sum_{\ell=1}^L \hmu_\ell$ \Comment{estimate for $\nu$}
        \State $\tsigma^2 \gets \sum_{\ell=1}^L \tsigma_\ell^2/R$ \Comment{estimate of $\bbV[\hnu]$}    
        \\ \Return $\hnu,\tsigma,\{R n_\ell\}_{\ell=1}^L$ \Comment{the estimate, its standard error, and samples}
    \end{algorithmic}
\end{algorithm}

\Subsubsection{Fast Bayesian Multilevel Quasi-Monte Carlo Without Replications} \label{sec:bmlqmc}

We will assume $Y_\ell$ is a Gaussian processes (GP) \cite{rasmussen.gp4ml} and, while not necessarily being justified theoretically, we make the modeling assumption that $Y_1,\dots,Y_L$ are independent GPs. For now, let us consider a fixed $\ell \in \{1,\dots,L\}$ and drop $\ell$ subscripts. At each level we prescribe $Y \sim \mathrm{GP}(\tau,K)$ where $\tau$ is a constant prior mean and $K$ is the kernel which depends on hyperparameter $\btheta = \{\gamma,\bEta,\tau\}$ (or $\btheta = \{\gamma,\bEta,\tau,\ba\}$ if using the adaptive smoothness kernel in \Cref{def:BMLQMC_dsi_kernel}). For our fast GP constructions, we will use product kernels as in \eqref{eq:product_kernels} which take the form 
\begin{equation}
    K(\bx,\bx') = \gamma \prod_{j=1}^d (1+\eta_j R(x_j,x_j'))
    \label{eq:prod_kernel_BMLQMC}
\end{equation}
for some univariate kernel $R$. For lattices, our univariate kernels will take the form of \Cref{def:si_kernels} with $R := K_\alpha$; we will use $\alpha=1$ to assume minimal periodicity. For digital nets, we use a new digitally-shift-invariant (DSI) kernel of adaptive smoothness defined below. This adaptive DSI kernel takes a hyperparameter-weighted sum of individual DSI kernels $\tR_\alpha$ of various smoothness $\alpha$ where $\tR_\alpha$ correspond to forms derived in \Cref{sec:dsi_kernels_smooth_functions}. 

\begin{definition}[Adaptive smoothness digitally-shift-invariant (DSI) product kernel] \label{def:BMLQMC_dsi_kernel}
    Our DSI kernels take the form of \eqref{eq:prod_kernel_BMLQMC} using a weighted sum of univariate kernels with hyperparameter weights $\ba >0$ so that 
    \begin{equation}
        \begin{aligned} 
        R(x,x') &:= \tR(x \oplus x') = a_1 \tR_1(x,x') + a_2 \tR_2(x,x') + a_3 \tR_3(x,x') + a_4 \tR_4(x,x') \quad\text{with } \\
        \tR_\alpha(x) &= \begin{cases} 6\left(1-\frac{1}{2} t_1(x)\right), & \alpha=1 \\ -1+-\beta(x) x + \frac{5}{2}\left[1-t_1(x)\right], & \alpha = 2 \\ -1+\beta(x)x^2-5\left[1-t_1(x)\right]x+\frac{43}{18}\left[1-t_2(x)\right], & \alpha = 3 \\ -1 -\frac{2}{3}\beta(x)x^3+5\left[1-t_1(x)\right]x^2 - \frac{43}{9}\left[1-t_2(x)\right]x \\
        \quad +\frac{701}{294}\left[1-t_3(x)\right]+\beta(x)\left[\frac{1}{48} \sum_{a=1}^\infty \frac{\mx_a}{2^{3(a-1)}} - \frac{1}{42}\right], & \alpha = 4 \end{cases}.
        \end{aligned}
        \label{eq:dsi_kernels}
    \end{equation}
    Here $\beta(x) = - \lfloor \log_2(x) \rfloor$, $t_\nu(x) = 2^{-\nu \beta(x)}$, and $x \oplus x' = \sum_{a=1}^\infty ((\mx_a + \mx_a') \bmod 2) 2^{-(a+1)}$ is the exclusive or (XOR) between binary expansions of $x = \sum_{a =0}^\infty \mx_a 2^{-a}$ and $x' = \sum_{a=0}^\infty \mx_a' 2^{-a}$ as used for digital nets in \Cref{sec:dnets}. 
\end{definition}

For the fast Bayesian cubature discussed in \Cref{sec:fast_gps}, we have posterior distribution of the integral is the Gaussian random variable
$$\mu := \bbE_\bX[Y(\bX)] \sim \calN(\hmu,\hV_n), \qquad \bX \sim \calU[0,1]^d$$
where $\hmu := \bbE[\mu | \mX, \bY]$ and $\hV_n = \bbV[\mu | \mX]$ for $\mX := (\bx_i)_{i=0}^{n-1}$ and $\bY := (Y(\bx_i))_{i=0}^{n-1}$. Under our choice of product kernels built from univariate (digitally-)shift-invariant kernels, and assuming we optimize the prior mean $\tau$ using the marginal log-likelihood (MLL) loss in \Cref{sec:mll_sgp}, then we have 
\begin{equation} \label{eq:fgp_post_cubature_mean_var}
    \hmu = \frac{1}{n} \sum_{i=0}^{n-1} Y(\bx_i) \qqtqq{and} \hV_n = \gamma \left[1-\left(\frac{1}{n} \sum_{i=0}^{n-1} \calK(\bx_i,\bx_0)\right)^{-1}\right]
\end{equation}
where $\calK(\bx,\bx') = \prod_{j=1}^d \left(1+\eta_j R(x, x_j')\right)$ is the unscaled version of the product kernel in \eqref{eq:prod_kernel_BMLQMC}. 

A key observation is that $\hV_n$ only depends on the points $\mX$, not the function evaluations $\bY$. This implies that, for fixed hyperparameters $\btheta$, we may compute the projected variance $\hV_\hn$ for any $\hn \geq n$ points from the given LD sequence. We prefer sample sizes $\hn$ which are powers of two as there the LD sequences we consider attain desirable uniformity properties and enable fast GP computations. While one may compute $\hV_\hn$ exactly for any $\hn$, for later utility we will define $\hV_\hn$ when $\hn$ is not a power of $2$ to be the log-log linear interpolation between surrounding powers of two so that
% $$\log_2(\hV_{n}) := \log_2(\hV_{2^p}) + \log_2\left(\frac{\hV_{2^{p+1}}}{\hV_{2^p}}\right) \left(\log_2(\hn) - p\right) \qquad 2^p < \hn < 2^{p+1}$$
% $$\log_2(\hV_{n}) := \log_2(\hV_{2^p}) -p \log_2\left(\frac{\hV_{2^{p+1}}}{\hV_{2^p}}\right) + \log_2\left(\frac{\hV_{2^{p+1}}}{\hV_{2^p}}\right) \log_2(\hn) \qquad 2^p < \hn < 2^{p+1}$$
% $$\log_2(\hV_{n}) := \log_2\left(\frac{\hV_{2^p}^{p+1}}{\hV_{2^{p+1}}^p}\right) + \log_2\left(\frac{\hV_{2^{p+1}}}{\hV_{2^p}}\right) \log_2(\hn) \qquad 2^p < \hn < 2^{p+1}$$
\begin{equation} \label{eq:logloglininterp}
    \hV_\hn := \left(\frac{\hV_{2^p}^{p+1}}{\hV_{2^{p+1}}^p}\right) \hn^{\log_2\left(\hV_{2^{p+1}}/\hV_{2^p}\right)} \qqtqq{when} 2^p < \hn < 2^{p+1} \qqtqq{for} p \in \bbN_0
\end{equation}
with $p = \lfloor \log_2(\hn) \rfloor$, see \Cref{fig:logloglininterp}.  

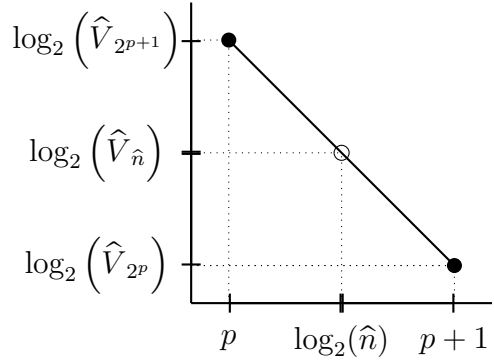
\begin{figure}[!ht]
    \centering 
    \begin{tikzpicture}
        \draw[-,thick] (0,0) -- (4,0);
        \draw[-,thick] (0,0) -- (0,4);
        \draw[-,thick] (0.5,3.5) -- (3.5,0.5);
        \fill (0.5,3.5) circle [radius=.1];
        \fill (3.5,0.5) circle [radius=.1];
        \draw (2,2) circle [radius=.1];
        \draw[|-|,thick] (.5,0) -- (2,0);
        \draw[|-|,thick] (2,0) -- (3.5,0);
        \draw[|-|,thick] (0,.5) -- (0,2);
        \draw[|-|,thick] (0,2) -- (0,3.5);
        \draw[-,thin,dotted] (2,0) -- (2,2);
        \draw[-,thin,dotted] (0,2) -- (2,2);
        \draw[-,thin,dotted] (0.5,0) -- (0.5,3.5);
        \draw[-,thin,dotted] (0,0.5) -- (3.5,0.5);
        \draw[-,thin,dotted] (0,0.5) -- (3.5,0.5);
        \draw[-,thin,dotted] (0,3.5) -- (0.5,3.5);
        \draw[-,thin,dotted] (3.5,0) -- (3.5,0.5);
        \node at (0.5,-0.5) {$p$};
        \node at (2,-0.5) {$\log_2(\hn)$};
        \node at (3.5,-0.5) {$p+1$};
        \node at (-1.25,0.5) {$\log_2\left(\hV_{2^p}\right)$};
        \node at (-1.25,2) {$\log_2\left(\hV_\hn\right)$};
        \node at (-1.25,3.5) {$\log_2\left(\hV_{2^{p+1}}\right)$};
    \end{tikzpicture}
    \caption{Log-log linear interpolation for $\hV_\hn$ when $\hn$ is not a power of $2$ using $p = \lfloor \log_2(\hn) \rfloor$.}
    \label{fig:logloglininterp}
\end{figure}

Let us now consider the multilevel Bayesian cubature setup where we fit independent GPs to each difference $Y_\ell$. We will select shifts $\bDelta_1,\dots,\bDelta_L \in [0,1)^d$ and use LD points $\bx_{\ell,i} = \bz_i \oplus \bDelta_{\ell}$ on each level. Here the shifts are arbitrary and need not be independent as was the case for MLQMC with replication in \Cref{sec:mlqmc}. One may even choose the same shift on each level, $\bDelta_1 = \dots = \bDelta_L$, to use the same collocation points on each level as is common in recycling MLMC methods \cite{robbe2019recycling, kumar2017multigrid}. Different hyperparameters $\btheta_\ell$ will be fit on each level. Using \eqref{eq:fgp_post_cubature_mean_var} with level-dependent notations, we estimate $\nu$ in \eqref{eq:tele_mlmc} by the posterior cubature mean 
\begin{equation*} \label{eq:nu_bmlqmc_pcmean}
    \hnu := \bbE[\nu | (\mX_\ell,\bY_\ell)_{\ell=1}^L] = \sum_{\ell=1}^\ell \hmu_\ell \qqtqq{where} \hmu_\ell := \bbE[\mu_\ell | \mX_\ell,\bY_\ell] = \frac{1}{n_\ell} \sum_{i=0}^{n_\ell-1} Y_\ell(\bx_{\ell,i}). 
\end{equation*}
Additionally, under the assumption of independent GPs $Y_1,\dots,Y_L$, the posterior cubature variance becomes 
\begin{equation} \label{eq:nu_bmlqmc_pcvar}
    \begin{aligned}
    \hV_{n_1,\dots,n_L} &:= \bbV[\nu | (\mX_\ell)_{\ell=1}^L] = \sum_{\ell=1}^L \hV_{\ell,n_\ell} \qtq{where} \\
    \hV_{\ell,n_\ell} &:= \bbV[\mu_\ell | \mX_\ell] = \gamma_\ell \left[1-\left(\frac{1}{n_\ell} \sum_{i=0}^{n_\ell-1} \calK_\ell(\bx_{\ell,i},\bx_{\ell,0})\right)^{-1}\right]
    \end{aligned}
\end{equation}
with $\calK_\ell$ depending on hyperparameters $\btheta_\ell \setminus \{\gamma\}$, i.e., all hyperparameters except the global scaling factor $\gamma$. If one is not willing to assume independent GPs $Y_1,\dots,Y_L$, then the Cauchy--Schwarz inequality may be used to bound 
$$\bbV[\nu | (\mX_\ell)_{\ell=1}^L] \leq \left(\sum_{\ell=1}^L \sqrt{\hV_{\ell,n_\ell}}\right)^2.$$
In our numerical experiments we use the independent GP assumption and take the standard error to be $\hV_{n_1,\dots,n_L}$ in \eqref{eq:nu_bmlqmc_pcvar} as we found it provided robust error estimation across our suite of test problems.  

As was the case for MLQMC with replications in \Cref{alg:mlqmc}, we will employ a scheme which iteratively doubles the sample size at selected levels. Our complete fast Bayesian MLQMC method is summarized in \Cref{alg:bmlqmc} which selects the level at which to double using the procedure in \Cref{alg:level_select_bmlqmc}. At each iteration, our selected level will depend on the projected variance $\hV_{\ell,\hn_\ell}$ for certain $\hn_\ell$. Specifically, we compare subsequent levels in decreasing cost-of-doubling order and iteratively select the level with the greater decrease in projected variance for the same cost. Concretely, when comparing levels $\ell$ and $\ell'$, if $n_\ell C_\ell = n_{\ell'} C_{\ell'}$, i.e., the cost of doubling on each level is equivalent, then we choose to move forward with the level which gives the larger decrease in variance between $V_{\ell,n_\ell} - V_{\ell,2n_\ell}$ and $V_{\ell',n_{\ell'}}-V_{\ell',2n_{\ell'}}$. If $n_\ell C_\ell > n_{\ell'} C_{\ell'}$, i.e., the cost of doubling on level $\ell$ is greater than the cost of doubling on level $\ell'$, then we choose to move forward with level which gives the larger decrease in variance between $V_{\ell,n_\ell} - V_{\ell,2n_\ell}$ and $V_{\ell',n_{\ell'} - V_{\ell',\hn_{\ell'}}}$ where $\hn_{\ell'}$ is chosen to satisfy $n_\ell C_\ell = C_{\ell'}(\hn_{\ell'} - n_{\ell'})$ so doubling the sample size on level $\ell$ would require the same cost as increasing the sample size to $\hn_{\ell'}$ on level $\ell'$. As $\hn_{\ell'}$ will usually not be a power of two, we use the log-log interpolation between surrounding powers of two as written in \eqref{eq:logloglininterp}. The fact that the log-log linear interpolation recovers the case when $n_\ell C_\ell = n_{\ell'} C_{\ell'}$ is reflected in the simplified presentation of \Cref{alg:level_select_bmlqmc}. 

\begin{algorithm}[!ht]
    \fontsize{12}{10}\selectfont
    \caption{\texttt{level\_select\_BQMC}: Level selection for fast Bayesian MLQMC}
    \label{alg:level_select_bmlqmc}
    \begin{algorithmic}
        \Require $\calL_\mathrm{feasible} \subseteq \{1,\dots,L\}$ with $\tL := \lvert \calL_\mathrm{feasible} \rvert>0$ elements \Comment{levels to consider}
        \Require $C_1,\dots,C_L > 0 $ \Comment{the cost of evaluating $Y_1,\dots,Y_L$ respectively}
        \Require $\btheta_1,\dots,\btheta_L$ \Comment{GP hyperparameters for each level} 
        \Require $n_1,\dots,n_L > 0$ \Comment{the current number of samples on each level}
        \State Set unique $l_1,\dots,l_\tL \in \calL_\mathrm{feasible}$ so that $n_{l_1} C_{l_1} \geq \cdots \geq n_{l_\tL} C_{l_\tL}$ 
        \newline \Comment{order feasible levels by non-increasing cost for doubling samples on each level}
        \State $\ell \gets l_1$ \Comment{initialize the selected level}
        \For{$k=2,\dots,\tL$}
            \State $\ell' \gets \ell_k$ \Comment{will compare levels $\ell$ and $\ell'$ with $n_\ell C_\ell \geq n_{\ell'} C_{\ell'}$}
            \State $\hn_{\ell'} \gets n_\ell C_\ell/C_{\ell'} + n_{\ell'}$ \Comment{equal costs for increasing sample sizes}
            \State $p \gets \lfloor \log_2(\hn_{\ell'}) \rfloor$ \Comment{implies $2^p \leq \hn_{\ell'}  < 2^{p+1}$}
            \State $\hV_{\ell',\hn_{\ell'}} \gets (\hV_{\ell',2^p}^{p+1}/\hV_{\ell',2^{p+1}}^p) \hn_{\ell'}^{\log_2\left(\hV_{\ell',2^{p+1}}/\hV_{\ell',2^p}\right)}$ \Comment{log-log interpolation as in \eqref{eq:logloglininterp}} 
            \State  \Comment{if $n_\ell C_\ell = n_{\ell'} C_{\ell'}$ directly compute $\hV_{\ell',\hn_{\ell'}} = \hV_{\ell',2n_{\ell'}}$ and avoid evaluating $\hV_{\ell',4n_{\ell'}}$}
            \If{$\hV_{\ell',n_{\ell'}} - \hV_{\ell',\hn_{\ell'}} \geq \hV_{\ell,n_\ell} - \hV_{\ell,2n_\ell}$} \\ \Comment{check for a greater projected decrease in variance for the same cost}
                \State $\ell \gets \ell'$ \Comment{update the selected level}
            \EndIf
        \EndFor
        \\ \Return $\ell$ \Comment{the selected level}
    \end{algorithmic}
\end{algorithm}

\begin{algorithm}[!ht]
    \fontsize{12}{10}\selectfont
    \caption{\texttt{BQMC}: Fast Bayesian MLQMC Without Replications}
    \label{alg:bmlqmc}
    \begin{algorithmic}
        \Require $N>0$ \Comment{the budget}
        \Require $C_1,\dots,C_L > 0 $ \Comment{the cost of evaluating $Y_1,\dots,Y_L$ respectively}
        \Require $n_1^\mathrm{next},\dots,n_L^\mathrm{next} \in \bbN$ powers of two satisfying $\sum_{\ell=1}^L n_\ell^\mathrm{next} C_\ell \leq N$ \\ \Comment{the initial sample sizes} 
        \Require A generating vector for lattices or generating matrices for digital nets. 
        \Require $\btheta_1,\dots,\btheta_L$ initial hyperparameters containing a global scale $\gamma$, lengthscales $\bEta$, and, if using weighted sums of DSI kernels as in \Cref{def:BMLQMC_dsi_kernel}, the kernel weights $\ba$ 
        \Require $\bDelta_1,\dots,\bDelta_L \in [0,1)^d$ \Comment{(digital) shifts, we use $\bDelta_1,\dots,\bDelta_L \simiid \calU[0,1)^d$.}
        \State $n_\ell \gets 0$ for $\ell \in \{1,\dots,L\}$ \Comment{the number of evaluations}
        \State $\calL \gets \{1,\dots,L\}$ \Comment{the set of levels to update}
        \While{true}
            \State Generate $x_{\ell,i} = \bz_i \oplus \bDelta_\ell$ for $\ell \in \calL$ and $n_\ell \leq i < n_\ell^\mathrm{next}$
            \State Evaluate $Y_\ell(\bx_{\ell,i})$ for $\ell \in \calL$ and $n_\ell \leq i < n_\ell^\mathrm{next}$
            \State Update $\btheta$ to optimize the NMLL for $\ell \in \calL$
            \State $\hmu_\ell \gets \frac{1}{n_\ell} \sum_{i=0}^{n_\ell-1} Y_\ell(\bx_{\ell,i})$ for $\ell \in \calL$ \Comment{posterior cubature mean}
            \State $\hV_{\ell,n_\ell} \gets \gamma_\ell \left[1-\left(\frac{1}{n_\ell} \sum_{i=0}^{n_\ell-1} \calK_\ell(\bx_{\ell,i},\bx_{\ell,0})\right)^{-1}\right]$ for $\ell \in \calL$ \\ \Comment{posterior cubature variance}
            \State $\calL_\mathrm{feasible} \gets \{\ell \in \{1,\dots,L\}: \sum_{\ell'=1}^L C_{\ell'} n_{\ell'} + C_\ell n_\ell \leq N\}$ \Comment{feasible set of levels}
            \If{$\calL_\mathrm{feasible} = \emptyset$} break \EndIf \Comment{exit loop before going over budget}
            \State $\starhat{\ell} \gets \texttt{level\_select\_BQMC}\left(\calL_\mathrm{feasible},(C_\ell,\btheta_\ell,n_\ell)_{\ell=1}^L\right)$ \Comment{\Cref{alg:level_select_bmlqmc}}
            \State $\calL \gets \{\starhat{\ell}\}$ and $n_{\starhat{\ell}} \gets n_{\starhat{\ell}}^\mathrm{next}$ and $n_{\starhat{\ell}}^\mathrm{next} \gets 2n_{\starhat{\ell}}$ \Comment{double samples on the chosen level}
        \EndWhile
        \State $\hnu \gets \sum_{\ell=1}^L \hmu_\ell$ \Comment{estimate for $\nu$}
        \State $\hV_{n_1,\dots,n_L} \gets \sum_{\ell=1}^L \hV_{\ell,n_\ell}$ \Comment{posterior cubature variance assuming independent GPs}    
        \\ \Return $\hnu,\sqrt{\hV_{n_1,\dots,n_L}},\{n_\ell\}_{\ell=1}^L$ \Comment{the estimate, its standard error, and samples}
    \end{algorithmic}
\end{algorithm}

\Subsection{Numerical Experiments} \label{sec:numerical_experiments}

In this section we present a number of numerical experiments for both single-level (Q)MC (\Cref{sec:examples_single_level}) and multilevel (Q)MC (\Cref{sec:examples_multilevel}) using the three presented (Q)MC algorithms:
\begin{itemize}
    \item \textbf{\texttt{MC}} (Multilevel) Monte Carlo with IID points as described in \Cref{sec:mlmc,} and \Cref{alg:mlmc}.
    \item \textbf{\texttt{RQMC}} (Multilevel) QMC with replications as described in \Cref{sec:mlqmc} and \Cref{alg:mlqmc}.
    \item \textbf{\texttt{BQMC}} (Multilevel) Bayesian QMC without replications as described in \patchoverfull \Cref{sec:bmlqmc} and \Cref{alg:bmlqmc}.
\end{itemize}
Our Python implementation of these algorithms uses the fast Gaussian process regression package \texttt{FastGPs} from \Cref{sec:fastgps_features} and the QMC software package \texttt{qmcpy} \Cref{sec:qmcpy_features}. We will run $250$ independent trials of each (multilevel) (Q)MC approximation for a given problem and budget. For QMC methods, the full text \cite{sorokin.FastBayesianMLQMC} considers both lattices and the digital nets. Here we will only present results for digital nets. We will always use the \texttt{joe-kuo-6.21201} digital net generating matrices from \url{https://web.maths.unsw.edu.au/~fkuo/sobol/index.html} \cite{joe2003remark,joe2008constructing}. This is the default choices in \texttt{qmcpy}. Comparisons against other generating matrices is a valuable avenue for future work. 
% \url{https://github.com/QMCSoftware/LDData/blob/main/dnet/joe_kuo.6.21201.txt}
%  \url{https://github.com/QMCSoftware/LDData/blob/main/lattice/kuo.lattice-33002-1024-1048576.9125.txt}

Most of the single-level test functions in \Cref{sec:examples_single_level} were considered in \cite{lecuyer.RQMC_CLT_bootstrap_comparison}, albeit with the ridge functions having equal weights. There, comprehensive experiments showed that RQMC methods outperform alternative bootstrap methods in terms of confidence interval coverage across a range of benchmarks. Their experiments tested up to dimension $d=32$ and only observed RQMC coverage failures when using $R=5$ randomizations. We will use $R=8$ replications in our RQMC testing. As $R$ decreases, RQMC will typically yield better true errors with less accurate standard errors. We note that our choice of $R=8$ is rather aggressive in comparison to existing literature which often uses at least $R=25$.

\Subsubsection{Single-Level Problems} \label{sec:examples_single_level}

Here we consider single-level test functions $Y$ and run (Q)MC algorithms to estimate $\mu = \bbE[Y(\bX)]$ for $\bX \sim \calU[0,1]^d$ and $d=32$. The ridge functions will be defined as $Y(\bx) = g(u(\bx))$ for $u(\bx) = \sum_{j=1}^d c_j \Phi^{-1}(x_j)$ with weights $(c_j)_{j=1}^d$ and $\Phi$ the CDF of the standard normal $\calN(0,1)$. We will only present results for \emph{sparse weights} $c_j = 2^{-j} / \sqrt{\sum_{j'=1}^d 2^{-2j'}}$ which make the problem QMC-friendly, but we observed similar findings for ridge functions with \emph{equal weights} $c_j = d^{-1/2}$. Both sparse and equal weights make $u(\bX) \sim \calN(0,1)$ when $\bX \sim \calU[0,1]^d$. Our four test functions are listed below.

\begin{enumerate}
    \item \textbf{sumxex:} $Y(\bx) = -d+\sum_{j=1}^d x_j\exp(x_j)$. This is a smooth and additive integrand which is easy for QMC methods.
    \item \textbf{ridge PL sparse:} $Y(\bx) = g(u(\bx))$ with $g(u)=\max\left(u-1,0\right) - \phi(1)+\Phi(-1)$ and $\phi$ the density of $\calN(0,1)$. This is a continuous piecewise linear (PL) function with a kink, a feature often observed in problems from computational finance. 
    \item \textbf{ridge JSU sparse:} $Y(\bx) = g(u(\bx))$ with $g(u) = -\eta+F^{-1}\left(\Phi\left(v\right)\right)$ where $F$ is the CDF of a Johnson's SU distribution \cite{johnson1949systems} with parameters $\gamma=\delta=\lambda=1$ and $\xi=0$, and $\eta$ is the mean of that distribution. For $\bX \sim \calU[0,1]^d$, $Y(\bX)$ has skewness $-5.66$ and kurtosis $96.8$ making it heavy tailed. 
    \item \textbf{Genz corner-peak 2:} $Y(\bx) = \left(1+\sum_{j=1}^d c_j x_j\right)^{-(d+1)}$ with coefficients $c_j = j^{-2}/(4\sum_{j'=1}^d (j')^{-2})$. This is a corner-peak version of the Genz function \cite{genz1993comparison} (as opposed to the oscillatory version) with coefficients of the second kind (out of three common options). The different coefficient options represent increasing levels of anisotropy and decreasing effective dimension.
\end{enumerate}

In the top row of \Cref{fig:dnet.convergence} we plot both the median true errors and median standard errors for single-level problems with an increasing budget $N=n$. For these single-level problems, the true error and standard error for IID Monte Carlo always converge at the expected $\calO(n^{-1/2})$ rate with the true errors closely matching the standard errors. The RQMC methods also yield standard errors which closely match true errors, with both converging faster than IID-MC methods. RQMC converges like $\calO(n^{-2})$ for the easy sumxex problem, slightly below $\calO(n^{-1})$ for both the ridge PL sparse and ridge JSU sparse problems, and between $\calO(n^{-1})$ and $\calO(n^{-3/2})$ for the Genz corner-peak 2 function. As expected, the BQMC methods almost always provide better true errors compared to the RQMC methods. In the majority of cases, the true error rates for BQMC match those for RQMC, but with smaller constant multiples. One exception is the true errors of BQMC-net for the sumxex problem which appears to converge at a rate slightly better than $\calO(n^{-3})$ compared to the $\calO(n^{-2})$ rate for the RQMC-net methods.

We found the BQMC standard errors are often able to match the true error convergence rates and provide slight improvements over the corresponding RQMC standard errors. However, the rate constants for the BQMC standard error are rather conservative for the true errors when higher-order convergence beyond $\calO(n^{-2})$ is observed. For example, when applying BQMC-net to the easy sumxex problem, the true error converges like $\calO(n^{-3})$ while the standard error converges with a rate around $\calO(n^{-2})$. One potential method to improve BQMC-net standard error convergence would be to expand the DSI sum kernels in \Cref{def:BMLQMC_dsi_kernel} to consider even higher-order smoothness kernels beyond $\alpha=4$. However, evaluating higher-order smoothness kernels is more expensive in terms of both computations and memory, and a larger set of hyperparameters weights $\ba$ would need to be optimized. We also found the BQMC-net standard error convergence rate appears slightly worse than that for RQMC-net for the Genz corner-peak $2$ function.

In \Cref{fig:dnet.sl.error_scatters} we plot the standard error and true error for each trial against each other. Again we see the conservative nature the BQMC standard error estimations in comparison to RQMC methods. This is especially evident for the sumxex and Genz corner-peak 2 functions at higher budgets $n$. For the two ridge functions, we found the BQMC standard errors were accurate for the true errors with both medians almost exactly matching.

\begin{figure}[!ht]
    \centering
    \includegraphics[width=1\textwidth]{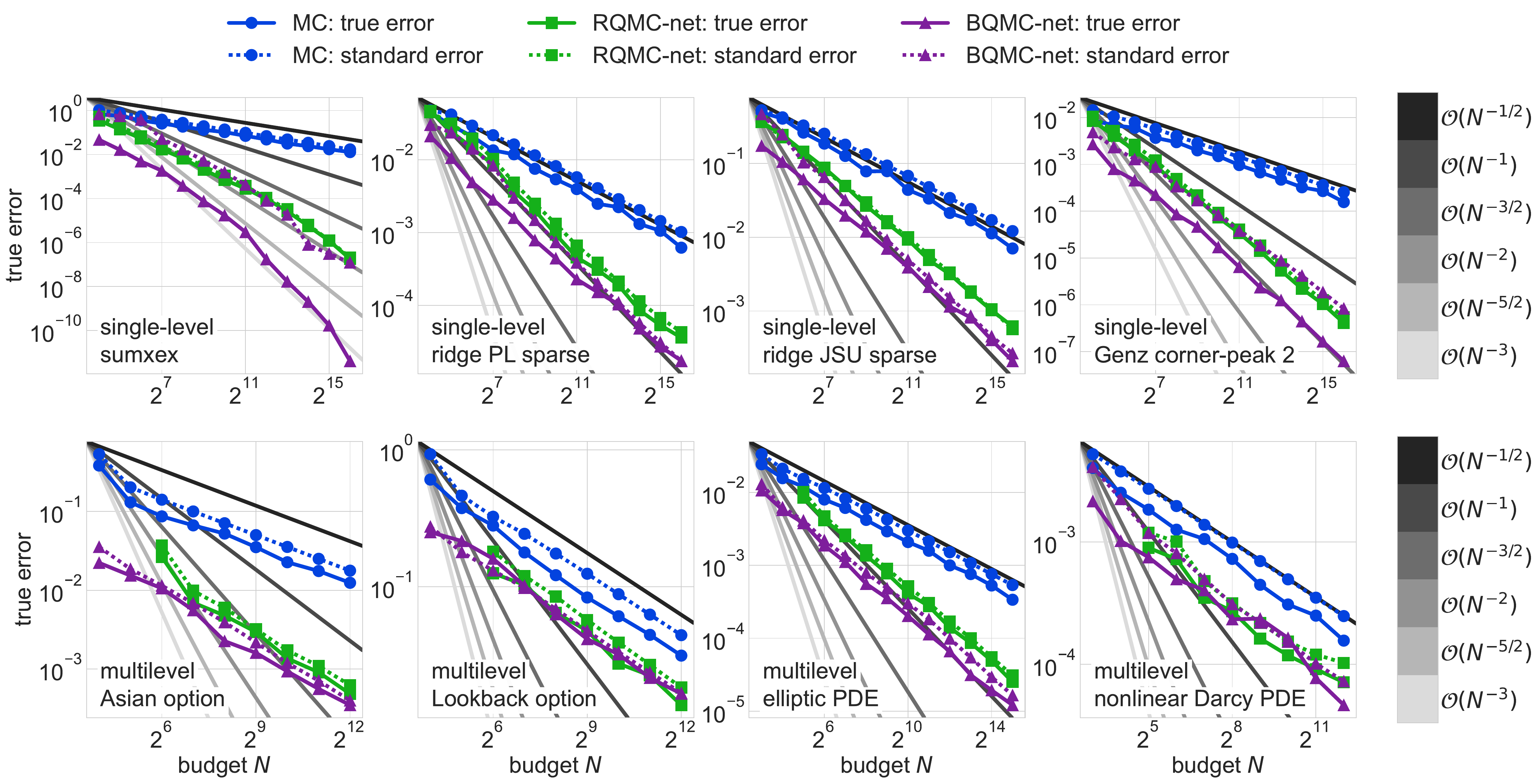}
    \caption{Median true error and standard error versus budget for Monte Carlo (MC) with IID points, Quasi-Monte Carlo with replications (RQMC), and Quasi-Monte Carlo with fast Bayesian cubature (BQMC).} 
    \label{fig:dnet.convergence}
\end{figure}

\begin{figure}[!ht]
    \centering
    \includegraphics[width=1\textwidth]{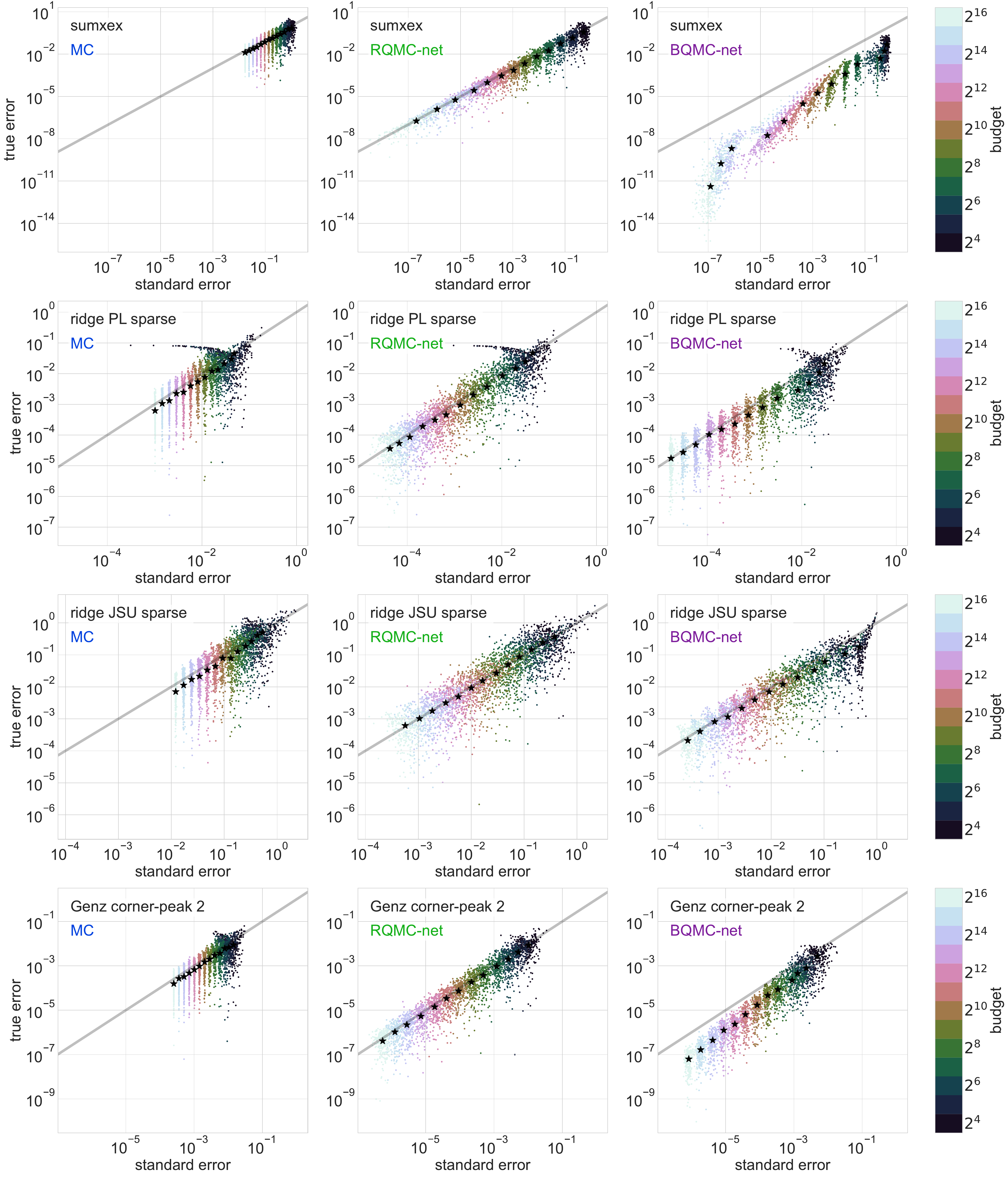}
    \caption{Standard error versus true error across trials for (single-level) Monte Carlo (MC) with IID points, Quasi-Monte Carlo with replications (RQMC), and Quasi-Monte Carlo with fast Bayesian cubature (BQMC). The stars represent the median true and standard errors for each budget.} 
    \label{fig:dnet.sl.error_scatters}
\end{figure}

\Subsubsection{Multilevel Problems} \label{sec:examples_multilevel}

Here we consider multilevel test functions with a fixed number of levels. We will describe our test functions for option pricing (\Cref{sec:bqmc_ml_opt_pricing}), solving a one-dimensional elliptic PDE (\Cref{sec:bqmc_elliptic}), and solving a two-dimensional nonlinear Darcy flow PDE (\Cref{sec:bqmc_nonlinear_darcy}). Our multilevel numerics follow in \Cref{sec:numerics_ml}.  \Cref{tab:ml_decay} gives the approximate means $\mu_\ell$ and standard deviations $\sqrt{V_\ell}$ for each problem. Our cost always increases with level, and we always normalize so that the cost on the maximum level is $1$.  

\begin{table}[!ht]
    \caption{Decay of the means $\mu_\ell$ and standard deviations $\sqrt{V_\ell}$ of differences with increasing level $\ell$.}
    \centering
    \resizebox{\textwidth}{!}{
    \begin{tabular}{rrrrrrrrr} 
        \multicolumn{9}{c}{mean $\mu_\ell = \bbE[Y_\ell(\bX)] = \bbE[f_\ell(\bX)-f_{\ell-1}(\bX)]$ with $\bX \sim \calU[0,1]^d$} \\
        problem/$\ell$ & 1 & 2 & 3 & 4 & 5 & 6 & 7 & 8 \\
        \hline 
        Asian option & 6.3e0  & -3.0e-1 & -1.5e-1 & -7.2e-2 & -3.7e-2 & -1.8e-2 & -9.3e-3 & -4.8e-3 \\
        Lookback option & 1.3e1 & 1.4e0 & 8.9e-1 & 5.8e-1& 4.1e-1 & 2.8e-1 & 1.8e-1 & 1.3e-1 \\
        elliptic PDE & 1.6e-1 & -1.1e-2 & 2.5e-3 & 1.5e-3  \\
        nonlinear Darcy PDE & 4.1e-2 & 8.5e-4 & 3.3e-4 \\
        \hline 
        \hline 
        \multicolumn{9}{c}{standard deviation $\sqrt{V_\ell} = \sqrt{\bbV[Y_\ell(\bX)]} = \sqrt{\bbV[f_\ell(\bX)-f_{\ell-1}(\bX)]}$ with $\bX \sim \calU[0,1]^d$} \\
        problem/$\ell$ & 1 & 2 & 3 & 4 & 5 & 6 & 7 & 8 \\
        \hline 
        Asian option & 8.7e0 & 6.8e-1 & 3.5e-1 & 1.7e-1 & 8.4e-2 & 4.3e-2 & 2.1e-2 & 1.1e-2 \\
        Lookback option & 1.3e1 & 2.1e0 & 1.4e0 & 9.0e-1 & 6.3e-1 & 4.2e-1 & 2.9e-1 & 2.0e-1 \\
        elliptic PDE & 1.4e-1 & 6.2e-2 & 1.0e-2 & 3.5e-3 \\
        nonlinear Darcy PDE & 7.8e-2 & 6.7e-3 & 1.9e-3 \\
        \hline 
    \end{tabular}
    }
    \label{tab:ml_decay}
\end{table}

\Subsubsection{Multilevel Option Pricing} \label{sec:bqmc_ml_opt_pricing}

Suppose we are given a financial option with starting price $S_0$, strike price $K$, interest rate $r$, and volatility $\sigma$ which is exercised at time $1$. We use $S_0=K=100$, $r=0.05$, and $\sigma=0.2$. We will assume the asset path follows a geometric Brownian motion $S(t) = S_0 e^{(r-\sigma^2/2)t+\sigma B(t)}$ where $B(t)$ is a standard Brownian motion. At level $\ell$, we monitor the asset at $d_\ell = 2^{2+\ell}$ times $(j/d_\ell)_{j=1}^{d_\ell}$ so that $(B(1/d_\ell),B(2/d_\ell),\dots,B(1))^\intercal \sim \calN\left(\bzero,\mSigma_\ell\right)$ where $\mSigma_\ell = \left(\min(j/d_\ell,j'/d_\ell)\right)_{j,j'=1}^{d_\ell}$. We will use the eigendecomposition $\mSigma_\ell = \mA_\ell \mA_\ell^\intercal$ to write $(B(1/d_\ell),B(2/d_\ell),\dots,B(1))^\intercal \sim \mA_\ell \Phi^{-1}(\bX_\ell)$ where $\Phi^{-1}$ is the inverse CDF of the standard normal applied elementwise to $\bX_\ell \sim \calU[0,1]^{d_\ell}$. The cost $C_\ell$ on level $\ell$ is proportional to $2^\ell$. We will consider the following two options with $L=8$ levels of discretization.

\begin{enumerate}
    \item \textbf{Asian Call Option with Geometric Averaging} Here the discounted fair prices of the infinite-dimensional options with continuous monitoring is \\ $\bbE\left[\max\left(\exp\left(\int_0^1 \log(S(t)) \D t\right) - K,0\right)\right]e^{-r}$, and for discrete monitoring at level $\ell$ we use the finite-dimensional approximation 
    $$f_\ell(\bX) = \max\left(\prod_{j=1}^{d_\ell} S(j/d_\ell) - K,0\right)e^{-r}.$$ 
    \item \textbf{Lookback Option} Here the discounted fair price of the infinite-dimensional option with continuous monitoring is $\bbE\left[S(1)-\min_{0 \leq t \leq 1} S(t)\right] e^{-r}$, and for discrete monitoring at level $\ell$ we use the finite-dimensional approximation 
    $$f_\ell(\bX) = \left[S(1) - \min(S(1/d_\ell),S(2/d_\ell),\dots,S(1))\right] e^{-r}.$$
\end{enumerate}

\Subsubsection{Multilevel Elliptic PDE} \label{sec:bqmc_elliptic}

Let us consider the one-dimensional elliptic PDE $-\nabla(e^{a(u,\bx)} \nabla F(u,\bx)) = g(u)$ with $u \in [0,1]$ and boundary conditions $F(0,\bx) = F(1,\bx) = 0$ for all $\bx \in [0,1]^d$. The forcing term $g$ is set to the constant $g(u) = 1$ for all $u \in [0,1]$. We will generate $a(u,\bX) = \sum_{j=1}^d \Phi^{-1}(X_j) \sin(\pi k u)/j$ with $\bX \sim \calU[0,1]^d$ in $d=8$ dimensions and $\Phi$ the CDF of the standard normal so $\Phi^{-1}(X_j) \sim \calN(0,1)$. Let us denote by $F_\ell$ the level $\ell$ numerical PDE solution using a finite difference method with $2^{1+\ell}+1$ evenly spaced mesh points $(k/2^{1+\ell})_{k=0}^{2^{1+\ell}}$. We will take the discretized solution to be $f_\ell(\bX) = F_\ell(1/2,\bX)$ and use $L=4$ levels. For each query of $f_\ell$ we need to solve a tridiagonal linear system, which can be done with linear complexity. Therefore, we take the cost $C_\ell$ on level $\ell$ to be proportional to $2^\ell$. 

\Subsubsection{Multilevel Nonlinear Darcy Flow PDE} \label{sec:bqmc_nonlinear_darcy}

\begin{figure}[!ht]
    \centering
    \includegraphics[width=1\textwidth]{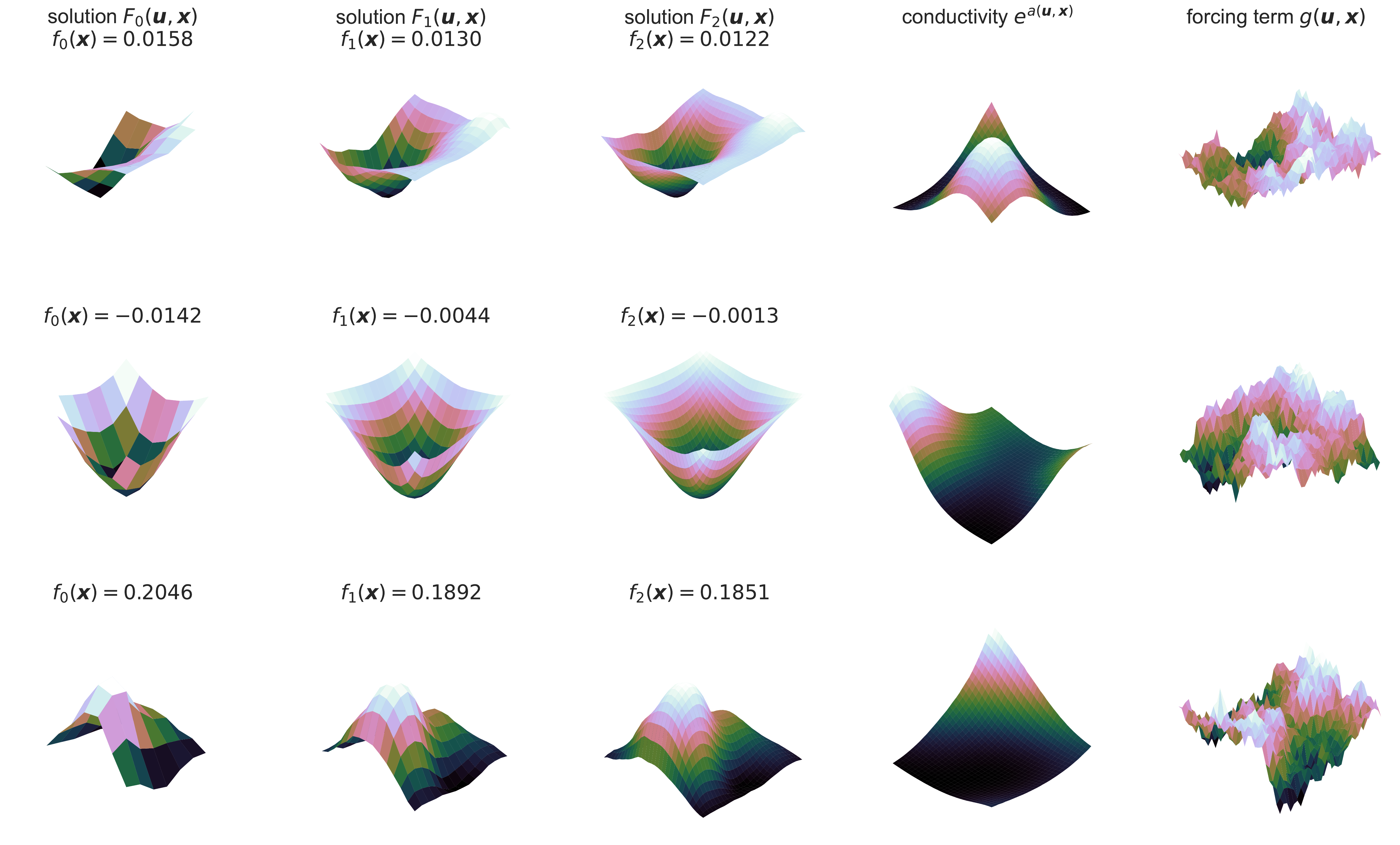}
    \caption{Visualization of the Darcy flow PDE where each row represents a different realization $\bx$ of both the random conductivity and forcing term. The resolution increases with each of the $L=3$ levels with the conductivity and forcing term shown at their maximum resolution. Here the quantity of interest on level $\ell\in \{1,2,3\}$ is $f_\ell(\bX) = \max_{\bu \in [0,1]^d} F_\ell(\bu,\bX)$ for $\bX \sim \calU[0,1]^d$.} 
    \label{fig:darcy}
\end{figure}

We now consider a two-dimensional nonlinear Darcy flow problem 
\begin{equation*} \label{eq:BMLQMC_Darcy}
\begin{cases}
    - \nabla \cdot (e^{a(\bu,\bx)} . \nabla F(\bu,\bx)) +  (F(\bu,\bx))^3 = g(\bu,\bx), & \bu \in [0,1]^2 \\
    F(\bu,\bx) = 0, & \bu \in \partial [0,1]^2
\end{cases}
\end{equation*}
where $g$ is a forcing term and $e^a$ represents the conductivity. Here we assume $g$ and $a$ are both random and independently drawn from Gaussian processes with the underlying stochasticity of both processes controlled by $\bX \sim \calU[0,1]^d$. We will take $g$ to be a draw from a GP with a $1/2$-Mat\'ern kernel and $a$ to be a draw from a GP with a Gaussian kernel. The exact parameterizations of these GPs can be found in the implementation code. As with the elliptic PDE example, we will let $F_\ell$ denote the level $\ell$ numerical PDE solution solved on a computational grid with $2^{2+\ell}-1$ mesh points $(k/2^{2+\ell})_{k=1}^{2^{2+\ell}-1}$ in each dimension. We numerically solve the PDE via iterative linearizations and a Levenberg--Marquardt scheme \cite{levenberg1944method,marquardt1963algorithm}. Our quantity of interest is $\bbE[\max_\bu(F(\bu,\bx))]$, so we take the discretized solution to be $f_\ell(\bX) = \max_\bu F_\ell(\bu,\bX)$ with $\bX \sim \calU[0,1]^d$. \Cref{fig:darcy} visualizes realizations of the conductivity, forcing term, and PDE solutions at varying levels of resolution. The numerical solver is run to machine precision backwards error. Theoretically, the cost of evaluating $f_\ell$ is $\calO(S_\ell 2^{6\ell})$, with $S_\ell$ the number of Levenberg--Marquardt steps, as at each step we need to solve a dense linear system in a $(2^{2+2\ell}-1) \times (2^{2+2\ell}-1)$ matrix. Practically, we take $L=3$ levels and set the costs proportional to the respective average runtimes $(4.2 \times 10^{-5},9.9 \times 10^{-5},2.8 \times 10^{-3})$. Due to the large scale nature of this problem with $d=3844$ dimensions, we chose to only consider the $\alpha=4$ DSI kernel in \Cref{def:BMLQMC_dsi_kernel}, i.e., we fixed $\beta_1=\beta_2=\beta_3=0$ and set $\beta_4=1$. This enabled faster hyperparameter optimization and reduced storage. Even so, our Darcy flow experiment required over $12$ hours running in parallel on $5$ NVIDIA A100 80 GB GPUs. 

\Subsubsection{Results for Multilevel Numerical Experiments} \label{sec:numerics_ml}

In the bottom row of \Cref{fig:dnet.convergence} we plot both the median true errors and median standard errors for multilevel problems with an increasing budget $N$. As was the case with the single-level problems, for these multilevel problems we again find IID Monte Carlo methods achieve matching true and standard errors with convergence like $\calO(N^{-1/2})$ for each problem. The RQMC methods also yielded accurate standard error estimates which closely matched the true errors for all problems. RQMC converged like $\calO(N^{-1/2})$ for both option pricing problems and the Darcy flow PDE with smaller rate constants than IID Monte Carlo in each case. For the elliptic PDE, RQMC converges like $\calO(N^{-1})$. For BQMC we are able to accommodate smaller budgets $N$ compared to the RQMC methods which require multiple replications, as evidence by the curves extending further to the left for BQMC than RQMC.

As with single-level problems, BQMC-net outperforms RQMC-net in terms of both true and standard errors in the vast majority of cases. We observed a better rate constant for BQMC-net than RQMC-net on both the Asian option and elliptic PDE examples. On the harder Lookback option and nonlinear Darcy examples we observed very similar errors between BQMC-net and RQMC-net. We did observe slightly better RQMC-net median true errors for the Darcy problem with the $N=2^9$ and $N=2^{10}$ budgets compared to BQMC-net, but BQMC-net then yields better true errors at our largest tested budgets of $N=2^{11}$ and $N=2^{12}$. Due to the lack of higher-order convergence, we also found the BQMC-net standard error estimate to be accurate for the true errors, even in small sample regimes. 

In \Cref{fig:dnet.ml.error_scatters} we plot the standard error and true error for each trial against each other. For BQMC methods we observed that as the budget increases the standard errors become tighter across the trials. In order words, for larger budgets $N$ the BQMC methods become more consistent and confident in their standard error predictions. This feature is shared with IID Monte Carlo methods, but not with RQMC methods which always aggregate a fixed number of independent estimates $R$. 

Finally, in \Cref{fig:dnet.ml.sample_allocation_group_level} we plot the sample allocation versus budget. IID Monte Carlo often seems to allocate a greater proportion of the budget to lower levels in comparison to RQMC and BQMC methods. The RQMC and BQMC methods yielded similar allocations across most problems for both point sets. BQMC did show a slight preference for allocating more budget to moderate levels while RQMC tended to allocate a slightly greater proportion of the budget to the high and low levels. 

\begin{figure}[!ht]
    \centering
    \includegraphics[width=1\textwidth]{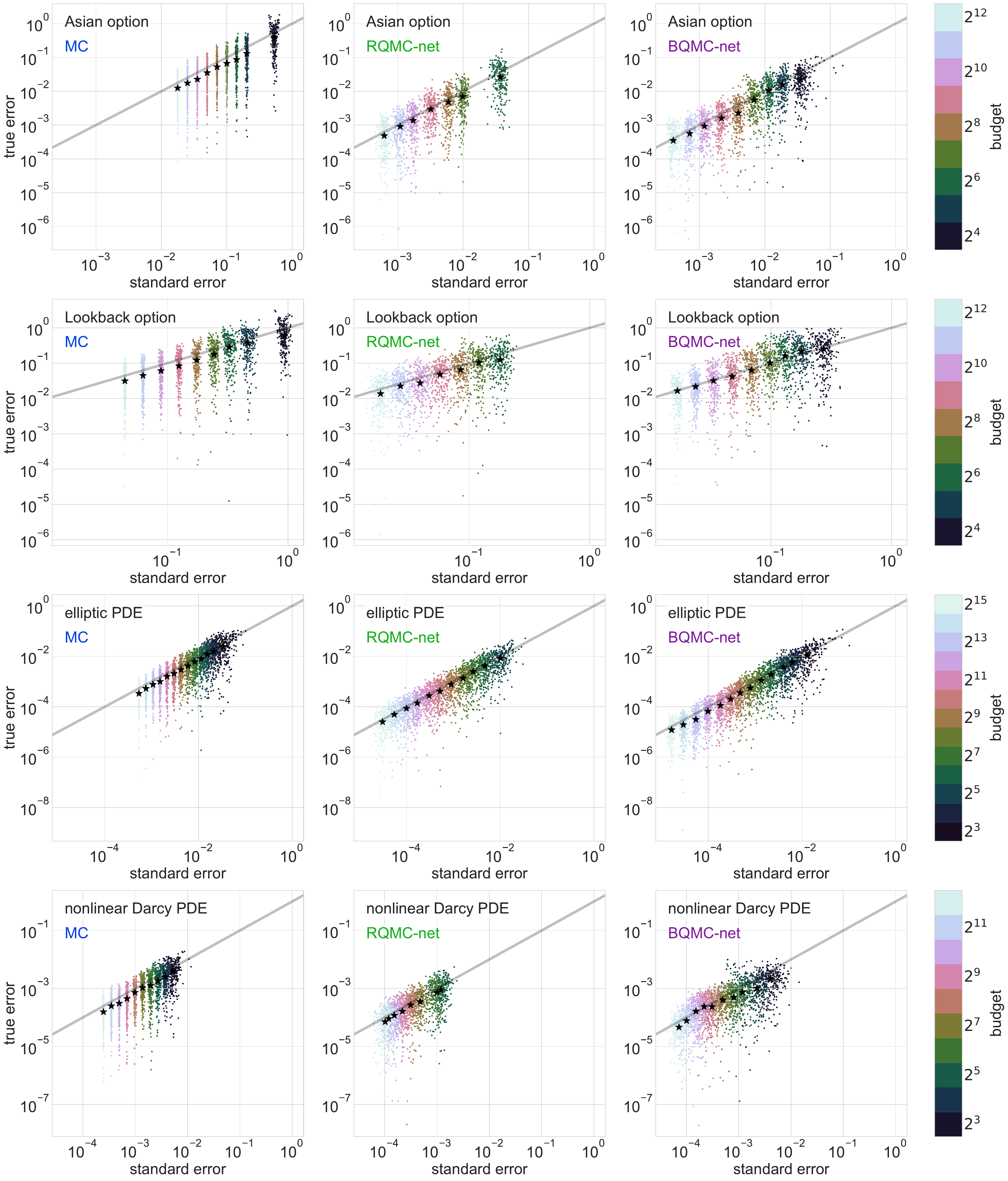}
    \caption{Standard error versus true error across trials for multilevel Monte Carlo (MC) with IID points, Quasi-Monte Carlo with replications (RQMC), and Quasi-Monte Carlo with fast Bayesian cubature (BQMC). The stars represent the median true and standard errors for each budget.} 
    \label{fig:dnet.ml.error_scatters}
\end{figure}

\begin{figure}[!ht]
    \centering
    \includegraphics[width=1\textwidth]{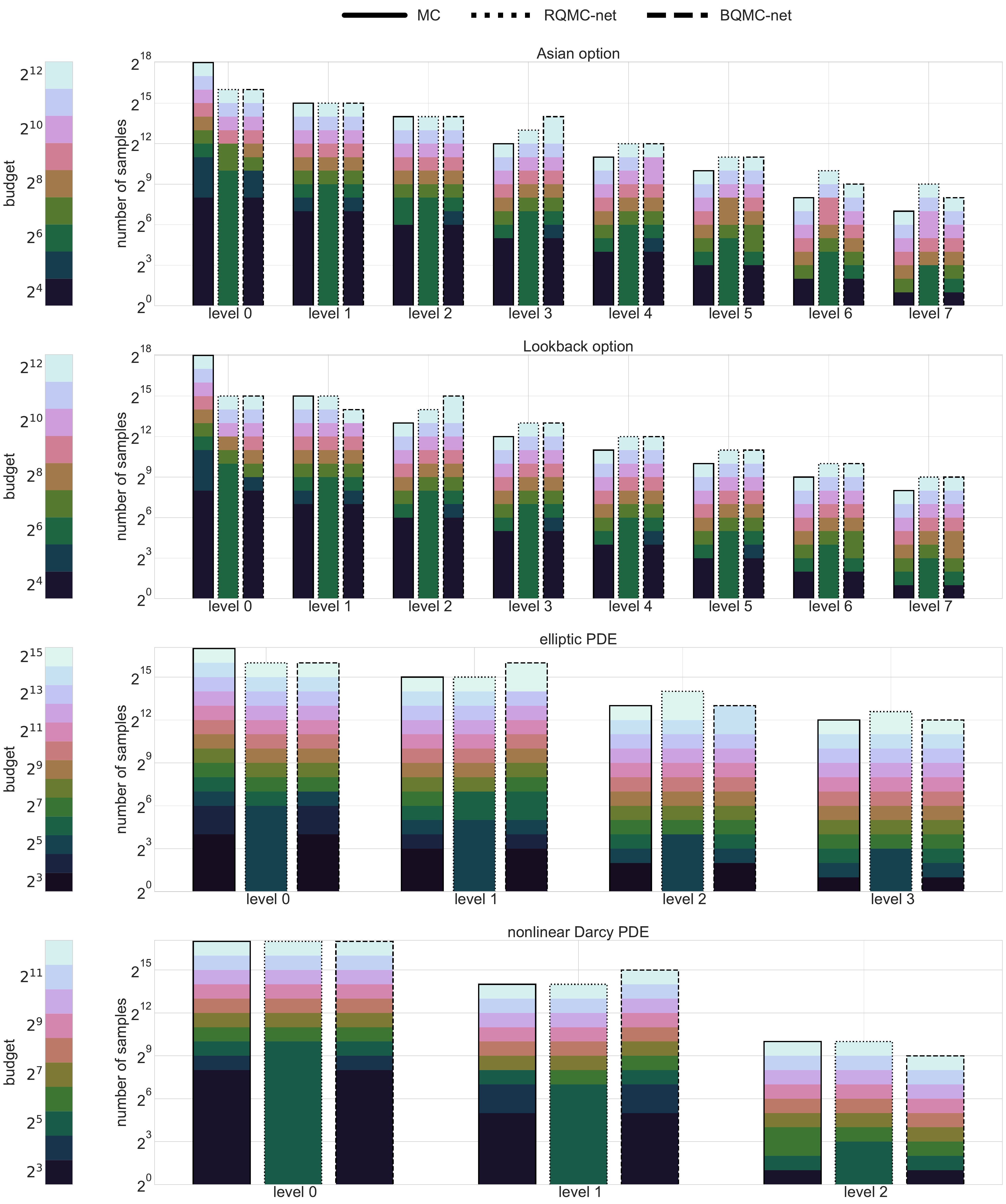}
    \caption{Sample allocation by budget versus level for multilevel Monte Carlo (MC) with IID points, Quasi-Monte Carlo with replications (RQMC), and Quasi-Monte Carlo with fast Bayesian cubature (BQMC).} 
    \label{fig:dnet.ml.sample_allocation_group_level}
\end{figure}

\Section{Gaussian Processes for Probability of Failure Estimation} \label{sec:pfgpci}

\begin{quotation}
    This section follows \cite{sorokin.adaptive_prob_failure_GP}, a paper I worked on during my 2022 Givens Associate Intern appointment at Argonne National Laboratories in collaboration with Vishwas Rao.
\end{quotation}

Efficiently approximating the probability of system failure has gained increasing importance as expensive simulations begin to play a larger role in reliability quantification tasks in areas such as structural design, power grid design, and safety certification among others. This work derives credible intervals on the probability of failure for a simulation which we assume is a realization of a Gaussian process. We connect these intervals to binary classification error and comment on their applicability to a broad class of iterative schemes proposed throughout the literature. A novel iterative sampling scheme is proposed which can suggest multiple samples per batch for simulations with parallel implementations. We empirically test our scalable, open-source implementation on a variety of simulations including a Tsunami model where failure is quantified in terms of maximum wave height.

Specifically, the novel contributions of this section are as follows.
\begin{itemize}
    \item Provide credible intervals that hold with guaranteed confidence for schemes which derive probability of failure estimates from a probabilistic surrogate model.
    \item Propose a sampling scheme for iteratively updating a Gaussian process surrogate to approximate probability of failure. This scheme is suitable to high performance computing (HPC) settings where the expensive simulation may be evaluated at multiple parameter configurations in parallel.  
    \item Provide a scalable, open-source implementation into the \texttt{QMCPy} library. 
\end{itemize}

\Subsection{Binary Classification with Gaussian Processes} \label{sec:binary_classificaiton_GP}

Suppose we have access to a simulation $g: [0,1]^d \to \bbR$ where we assume failure occurs whenever $g(\bu) \geq 0$. Our goal is to quantify the probability of system failure subject to random parameters $\bU \sim \calU[0,1]^d$; see \Cref{sec:variable_transforms} for a discussion on transformations to uniform randomness. We will model $g$ as a Gaussian process (GP) with mean $M: [0,1]^d \to \bbR$ and SPD covariance kernel $K: [0,1]^d \times [0,1]^d \to \bbR$. The mean $M$ and kernel $K$ may correspond to either the prior or posterior distributions, see \Cref{sec:gps} for discussion on GP regression. 

Here we have two sources of uncertainty: there is uncertainty in the canonical $\bU$ on probability space $([0,1]^d,\calB([0,1]^d),\bbU)$ where $\bbU$ is the Lebesgue measure on $[0,1]^d$, and there is uncertainty in the canonical stochastic GP $g$ on probability space $(\Omega,\calF,\bbG)$ indexed by elements in $[0,1]^d$. Here $\bbU$ may be thought of as a ``horizontal'' probability over the parameter space while $\bbG$ may be thought of as the ``vertical'' probability over the GP function space.

Let us define the random \emph{horizontal failure and success regions} $F,S: \Omega \to \calB([0,1]^d)$ by
\begin{align*}
    F(g) &:= \{\bu \in [0,1]^d: g(\bu) \geq 0\} \qqtqq{and} S(g) &:= \{\bu \in [0,1]^d: g(\bu) < 0\}
\end{align*}
respectively. Then our quantity of interest is the \emph{horizontal probability of failure} $P: \Omega \to [0,1]$ defined as 
\begin{equation}
    P(g) := \bbU(F(g)) = \bbE_\bbU[1_{F(g)}(\bU)].
    \label{eq:P(g)}
\end{equation}
We will also define the \emph{vertical probability of failure} $p: [0,1]^d \to [0,1]$ by 
\begin{equation}
    p(\bU) := \bbE_\bbG[1_{F(g)}(\bU)] = \Phi\left(\frac{M(\bU)}{\sigma(\bU)}\right)
    \label{eq:p(U)}
\end{equation}
where $\Phi$ is the CDF of a standard normal distribution and $\sigma^2(\bU) := K(\bU,\bU)$ is the variance of the GP at $\bU$. The framework presented here applies to more general probabilistic models when the last equality above is replaced appropriately for non-Gaussian processes. 

Let the deterministic \emph{predicted horizontal failure and success regions} $\hat{F},\hat{S} \subseteq [0,1]^d$ be
\begin{align*}
    \hat{F} &:= \{\bu \in [0,1]^d: p(\bu) \geq 1/2 \} = \{\bu \in [0,1]^d: M(\bu) \geq 0\}, \\
    \hat{S} &:= \{\bu \in [0,1]^d: p(\bu) < 1/2\} = \{\bu \in [0,1]^d: M(\bu) < 0\}.
\end{align*}
In other words, $\hat{F}$ and $\hat{S}$ are the regions where the GP mean fails and succeeds respectively. Then we may define the \emph{true positive, true negative, false positive, and false negative} functions $\mathrm{TP},\mathrm{FP},\mathrm{TN},\mathrm{FN}: \Omega \to \calB([0,1]^d)$ by  
\begin{alignat*}{2}
    \mathrm{TP}(g) &= \hat{F} \cap F(g), \qquad 
    \mathrm{FP}(g) &= \hat{F} \cap S(g), \\
    \mathrm{TN}(g) &= \hat{S} \cap S(g), \qquad 
    \mathrm{FN}(g) &= \hat{S} \cap F(g).
\end{alignat*}
Here positive indicates failure and negative indicates success. The bottom two panels of \Cref{fig:TP_FP_TN_FN} visualize these four disjoint regions covering $[0,1]^d$ for some $g_1,g_2 \in \Omega$. 

Notice that $1_{\mathrm{TP}(g)}(\bU) + 1_{\mathrm{TN}(g)}(\bU) + 1_{\mathrm{FP}(g)}(\bU) + 1_{\mathrm{FN}(g)}(\bU) = 1$. Then the \emph{expected accuracy} $\mathrm{ACC}: [0,1]^d \to [0,1]$ is naturally defined as
\begin{align*}
  \mathrm{ACC}(\bU) &:= \bbE_\bbG \left[1_{\mathrm{TP}(g)}(\bU)+1_{\mathrm{TN}(g)}(\bU)\right] \\
  &= 1_{\hat{F}}(\bU)p(\bU) + 1_{\hat{S}}(\bU)[1-p(\bU)] \\
  &= \max\{p(\bU),1-p(\bU)\},
\end{align*}
and the \emph{expected error rate} $\mathrm{ERR}: [0,1]^d \to [0,1]$ is defined as 
\begin{equation}
    \begin{aligned}
    \mathrm{ERR}(\bU) &:= 1-\mathrm{ACC}(\bU) \\
    &= 1_{\hat{F}}(\bU)[1-p(\bU)] + 1_{\hat{S}}(\bU)p(\bU) \\
    &= \min\{p(\bU),1-p(\bU)\}.
    \end{aligned}
    \label{eq:ERR}
\end{equation}
The $\max$ and $\min$ expressions result from the fact that $\bU \in \hat{F}$ if and only if $p(\bU) \geq 1/2$. 

\begin{figure}[!ht]
    \centering
    \includegraphics[width=\textwidth]{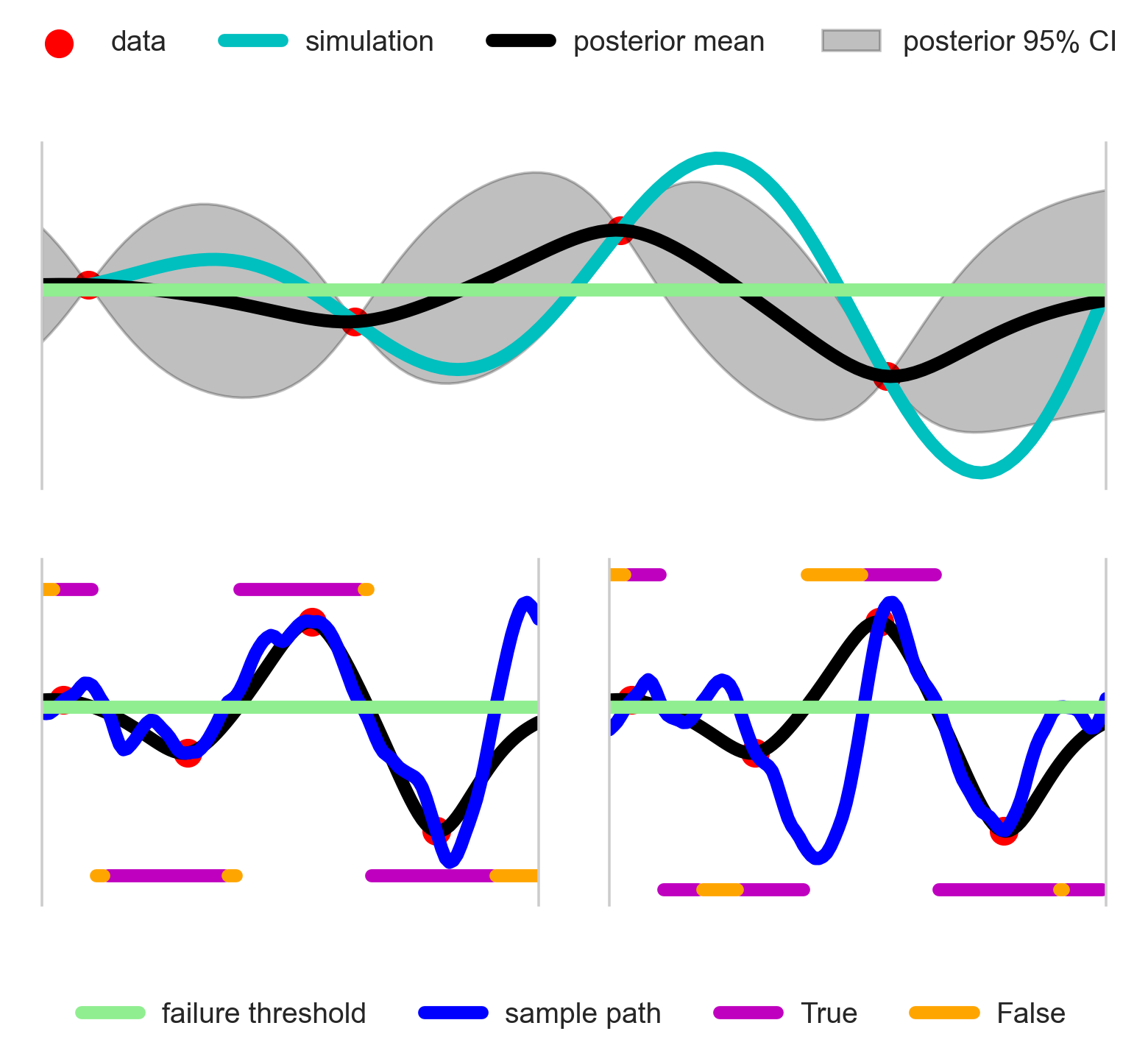}
    \caption{The top figure shows the true simulation and a Gaussian process fit to a few data points. The bottom two plots show some Gaussian process sample paths $g_1,g_2 \in \Omega$ with their corresponding $\mathrm{TP}$, $\mathrm{FP}$, $\mathrm{TN}$, and $\mathrm{FN}$ regions (true positive, false positive, true negative, and false negative regions). The predicted failure $\mathrm{TP}$ and $\mathrm{FP}$ regions are shown along the top of each plot while the predicted success $\mathrm{TN}$ and $\mathrm{FN}$ regions are shown along the bottom of each plot.}
    \label{fig:TP_FP_TN_FN}
\end{figure}

\Subsection{Estimators and Credible Intervals} \label{sec:estimators_error_bounds}

Our main contributions are \Cref{thm:ci_P_check} and \Cref{thm:ci_P_hat} in this section which derive error bounds for the horizontal probability of failure quantity of interest $P(g)$ in \eqref{eq:P(g)}. Specifically, we derive credible intervals of the form $[\underline{P},\overline{P}]$ so that 
\begin{equation}
    \bbG\left(P \in [\underline{P},\overline{P}]\right) \geq 1-\alpha.
    \label{eq:ci}
\end{equation}
where $\alpha \in (0,1)$ is some uncertainty threshold.

\begin{theorem}[Credible Interval from Posterior Mean Estimate] \label{thm:ci_P_check}
    Denote the posterior mean estimate by
    \begin{equation}
        \cP := \bbE_\bbG[P(g)] = \bbE_\bbG[\bbE_\bbU[1_{F(g)}(\bU)]] = \bbE_\bbU[p(\bU)].
        \label{eq:P_check}
    \end{equation} 
    Then \eqref{eq:ci} holds when 
    \begin{equation}
        \underline{P} = \max\left\{\cP-\check{\gamma},0\right\}, \qquad \overline{P} = \min\left\{\cP+\check{\gamma},1\right\}, \qquad \check{\gamma} := \frac{2\bbE_\bbU\left[p(\bU)(1-p(\bU))\right]}{\alpha}. 
        \label{eq:p_bounds_check}
    \end{equation}
\end{theorem}
\begin{proof}
    Markov's inequality, Jensen's inequality, and the triangle inequality imply that for $\gamma > 0$ we have 
    \begin{align*}
        \gamma \bbG\left(\left\lvert P -\cP \right\rvert \geq \gamma \right) &\leq \bbE_\bbG \left[\left\lvert P -\cP \right\rvert\right] \\
        &= \bbE_\bbG\left[\left\lvert \bbE_\bbU [1_{F(g)}(\bU)] - \bbE_\bbU [p(\bU)] \right\rvert \right] \\
        &= \bbE_\bbG\left[\left\lvert \bbE_\bbU \left[1_{F(g)}(\bU) - p(\bU)\right]\right\rvert\right] \\
        &= \bbE_\bbG\left[\left\lvert \bbE_\bbU\left[1_{F(g)}(\bU)\left(1-p(U)\right) - 1_{S(g)}(\bU)p(\bU)\right] \right\rvert \right] \\
        &\leq \bbE_\bbG\left[\bbE_\bbU\left[1_{F(g)}(\bU)\left(1-p(\bU)\right) + 1_{S(g)}(\bU)p(\bU)\right] \right] \\
        &= \bbE_\bbU\left[\bbE_\bbG\left[1_{F(g)}(\bU)\right]\left(1-p(\bU)\right) + \bbE_\bbG\left[1_{S(g)}(\bU)\right]p(\bU)\right] \\
        &= 2\bbE_\bbU\left[p(\bU)\left(1-p(\bU)\right)\right].
    \end{align*}
\end{proof}

\begin{theorem}[Credible Interval from Predicted Failure Region Estimate] \label{thm:ci_P_hat}
    Denote the probability of the predicted failure region estimate by 
    \begin{equation}
        \hat{P} := \bbU(\hat{F}) = \bbE_\bbU[1_{\hat{F}}(\bU)]. \label{eq:P_hat}
    \end{equation}
    Then \eqref{eq:ci} holds when 
    \begin{equation}
        \underline{P} = \max\left\{\hat{P}-\hat{\gamma},0\right\}, \qquad \overline{P} = \min\left\{\hat{P}+\hat{\gamma},1\right\}, \qquad \hat{\gamma} := \frac{\bbE_\bbU\left[\mathrm{ERR}(\bU)\right]}{\alpha}
        \label{eq:p_bounds_hat}
    \end{equation}
    for $\mathrm{ERR}(\bU) = \min\{p(\bU),1-p(\bU)\}$ as defined in \eqref{eq:ERR}.
\end{theorem}
\begin{proof}
    Markov's inequality, Jensen's inequality, and the triangle inequality imply that for $\gamma > 0$ we have 
    \begin{align*}
        \gamma \bbG\left(\left\lvert P - \hat{P} \right\rvert \geq \gamma \right) &\leq \bbE_\bbG \left[\left\lvert P - \hat{P} \right\rvert\right] \\
        &= \bbE_\bbG\left[\left\lvert \bbE_\bbU\left[1_{F(g)}(\bU) - 1_{\hat{F}}(\bU)\right]\right\rvert\right] \\
        &= \bbE_\bbG\left[\left\lvert \bbE_\bbU\left[1_{F(g)}(\bU)1_{\hat{S}}(\bU) - 1_{S(g)}(\bU)1_{\hat{F}}(\bU)\right]\right\rvert\right] \\
        &\leq \bbE_\bbG\left[\bbE_\bbU\left[1_{F(g)}(\bU)1_{\hat{S}}(\bU) + 1_{S(g)}(\bU)1_{\hat{F}}(U)\right]\right] \\
        &= \bbE_\bbU\left[\bbE_\bbG\left[1_{F(g)}(\bU)1_{\hat{S}}(\bU) + 1_{\bar{F}}(\bU)\right]1_{\hat{F}}(\bU)\right] \\
        &= \bbE_\bbU[\bbE_\bbG[1_\mathrm{FN}(\bU)+1_\mathrm{FP}(\bU)]] \\
        &= \bbE_\bbU[\mathrm{ERR}(\bU)].
    \end{align*}
\end{proof}

Since for any $\bU \in [0,1]^d$ either $p(\bU) \geq 1/2$ or $1-p(\bU) \geq 1/2$, we see that $\min\{p(\bU),1-p(\bU)\} \leq 2p(\bU)\left(1-p(\bU)\right)$. Therefore, the credible interval in \Cref{thm:ci_P_hat} is tighter than the credible interval in \Cref{thm:ci_P_check}, and going forward we take $[\underline{P},\overline{P}]$ to be defined as in \eqref{eq:p_bounds_hat}. 

\Subsection{Adaptive Algorithm} \label{sec:adaptive_sampling_algorithm_cost}

Our proposed algorithm sequentially updates a surrogate GP to refine the probability of failure estimate and shrink the resulting credible interval. The general iterative procedure is as follows. 
\begin{enumerate}
    \item \textbf{\textbf{Input}} a GP prior distribution $\bbG_0$ specified by a prior mean function $M_0$ and prior covariance function $K_0$.
    \item \textbf{\textbf{Input}} a simulation $g$ which we assume is a realization of the GP. The prior number of samples of $g$ is $n \gets 0$. 
    \item \textbf{\textbf{Step 1}} Evaluate $g$ at some batch of nodes $\bx_1,\dots,\bx_b$ and set $n \gets n+b$.
    \item \textbf{\textbf{Step 2}} Based on all previous evaluations of $g$, update the GP posterior distribution to $\bbG_n$ with posterior mean $M_n$ and posterior covariance function $K_n$. 
    \item \textbf{\textbf{Step 3}} Compute $N$-sample Quasi-Monte Carlo (see \Cref{sec:qmc}) approximations $\hat{P}^\mathrm{QMC}_n$ and $\hat{\gamma}^\mathrm{QMC}_{n}$ of $\hat{P}_n$ and $\hat{\gamma}_{n}$ respectively as defined in \eqref{eq:P_hat}, \eqref{eq:p_bounds_hat} based on the posterior distribution $\bbG_n$. 
    \item \textbf{\textbf{Step 4}} If the approximate $1-\alpha$ credible interval 
    $$[\underline{P}^\mathrm{QMC}_{n},\overline{P}^\mathrm{QMC}_{n}] := [\max\{\hat{P}^\mathrm{QMC}_n-\hat{\gamma}^\mathrm{QMC}_{n},0\},\min\{\hat{P}^\mathrm{QMC}_n+\hat{\gamma}^\mathrm{QMC}_{n},1\}]$$
    is desirably narrow or the budget for sampling $g$ has expired, we are done. Otherwise, return to Step 1 to continue refining the approximate estimate and credible interval. 
\end{enumerate}

Despite not being written explicitly, the estimates $\hat{P}_n^\mathrm{QMC}$ and $\hat{\gamma}_n^\mathrm{QMC}$ in Step 3 depend on the desired uncertainty threshold $\alpha$, the posterior distribution $\bbG_n$, and the $N$ QMC sampling nodes $\mU_N := \{\bu_i\}_{i=0}^{N-1}$. We choose to leave the low-discrepancy nodes $\mU_N$ unchanged across iterations. $N$ is chosen to be large, e.g., $N=2^{20}$, as the QMC estimates are computed using only the posterior mean $M_n$ and posterior covariance $K_n$ which are relatively cheap to evaluate. 

The choice of sampling scheme in Step 1 is arbitrary. A number of deterministic one step ($b=1$) look ahead schemes are proposed in \cite{bae.pf_gp_uncertainty_reduction,bect.sequential_design_experiments_pf,wagner2022rare,vazquez2009sequential,bichon2008efficient,lv2015new,renganathan.PF_CAMERA_multifidelity,dalbey2014gaussian,echard2013combined}. While some are theoretically applicable to $b>1$, the optimization required in practice quickly becomes intractable. We wish to allow $b>1$ so parallel implementations of simulations on HPC systems may be fully utilized.

To this end, we propose to sample $b$ IID points from the distribution with unnormalized density $2\mathrm{ERR}_n \leq 1$ as defined in \eqref{eq:ERR} under $\bbG_n$. We use rejection sampling \cite{flury1990acceptance,ripley2009stochastic,rubinstein2016simulation} to draw IID samples from the unnormalized density $\varrho = 2\mathrm{ERR}_n$, see \Cref{alg:RS}. The random number of tries $T$ required to draw a single sample from density $\varrho$ with rejection sampling follows a geometric distribution $T \sim \mathrm{Geom}(c)$ where $c = \bbE\left[\varrho(\bU)\right]$ is equivalent to the rejection sampling efficiency with $\bU \sim \calU[0,1]^d$. Therefore, the expected number of tries required to draw $b$ IID points from $2\mathrm{ERR}_n$ is $b/(2\bbE_\bbU[\mathrm{ERR}_n(\bU)])=b/(2\alpha \hat{\gamma}_n)$. This randomized, non-greedy scheme shows strong empirical performance in the next section. 
 
\begin{algorithm}[!ht]
    \fontsize{12}{10}\selectfont
    \caption{$\boldsymbol{\mathrm{AlgRS}}(\varrho,b)$: Rejection Sampling} \label{alg:RS}
    \begin{algorithmic}
        \Require $\varrho:[0,1]^d \to [0,1]$ \Comment{unnormalized density to sample from}
        \Require $b \in \bbN$ \Comment{the number of samples to draw}
        \State{$\mX_0 \gets \emptyset$ \Comment{initialize an empty point set of IID draws from $\varrho$}}
        \State{$i \gets 0$ \Comment{initialize accepted counter}}
        \State{$t \gets 0$ \Comment{initialize tried counter}}
        \While{$i<b$}
            \State{draw $\bX_t \sim \calU[0,1]^d$ independent of $\{\bX_j: j<t\}, \{\bU_j: j<t\}$ \Comment{draw candidate}}
            \State{draw $\bU_t \sim \calU(0,1)$ independent of $\{\bX_j: j \leq t\}, \{\bU_j: j < t\}$ \Comment{draw threshold}}
            \If{$\bU_t \leq \varrho(\bX_t)$}
              \State{$i \gets i+1$ \Comment{increment accepted counter}}
              \State{$\mX_i \gets \mX_{i-1} \cup \{\bX_t\}$ \Comment{add candidate to the accepted set}}
            \EndIf
            \State{$t \gets t+1$ \Comment{increment tried counter}}
        \EndWhile
        \\\Return{$\mX_{b}$ \Comment{return the $b$ IID draws from $\varrho$}}
    \end{algorithmic}
\end{algorithm}

\Subsection{Numerical Experiments}\label{sec:pfgpci_numerical_experiments}

The following numerical experiments are based on our open-source Python implementation in \texttt{QMCPy}. 
%Our examples are reproducible in notebook form\footnote{\url{https://qmcsoftware.github.io/QMCSoftware/demos/talk_paper_demos/ProbFailureSorokinRao/prob_failure_gp_ci/}}.  
Since $\mathrm{ERR}_0 \propto 1_{[0,1]^d}$, a uniform density, the initial samples are selected from a randomized low-discrepancy sequence rather than drawing IID points using rejection sampling as indicated in the algorithm. As discussed in \Cref{sec:hpopt_sgps}, the prior mean $M_0$ and kernel $K_0$ often involve hyperparameters which may be optimized to better fit the data. We will use the Mat\'ern kernels in \Cref{def:matern_kernel} which have parameters for smoothness, lengthscales, and a global scale. Our implementation allows one to either re-optimize hyperparameters at each iteration or keep them fixed after a certain number of iterations. 

We consider a number of benchmark examples in small and medium dimensions. These are listed in \Cref{table:toy_examples} alongside their true solution and corresponding figure reference. Such examples are often defined by an original function $\tilde{g}: \calT \to \bbR$ with respect to non-uniform random variable $\bV$ on $\calT$ and where failure occurs when the simulation exceeds some $\xi \in \bbR$. We may perform a change of variables so $\bV \sim \varphi(\bU)$ where $\bU \sim \calU[0,1]^d$ as desired and the probability of failure of $g(\bU) := \tilde{g}(\varphi(\bU))-\xi$ is equivalent to the probability of failure of $\tilde{g}(\bV)$. See \Cref{sec:variable_transforms} or \cite{choi.QMC_software} for a more rigorous variable transformation framework and further examples.

\begin{table}[!ht]
    \caption{Benchmark problems across a variety of dimensions.}
    \centering
    \begin{tabular}{c c c c c}
        problem & dimension & true $P(g)$ & figure\\
        \hline
        Sine & $1$ & $0.50$ & \Cref{fig:sine} \\
        \hline \\
        Multimodal \cite{bichon2008efficient} & $2$ & $0.30$ & \Cref{fig:multimodal} \\
        \hline \\
        Four Branch \cite{schobi2017rare} & $2$ & $0.21$ & \Cref{fig:fourbranch} \\
        \hline \\
        Ishigami \cite{ishigami1990importance} & $3$ & $0.16$ & \Cref{fig:ishigami} \\
        \hline \\
        Hartmann \cite{balandat2020botorch} & $6$ & $0.0074$  & \Cref{fig:hartmann} \\ 
        \hline
    \end{tabular}
    \label{table:toy_examples}
\end{table}

As a more realistic example, we look at the probability the maximum height of a Tsunami exceeds $3$ meters SSHA (sea surface height anomaly) subject to uniform uncertainty on the origin of the Tsunami. The Tsunami simulation was run using the \texttt{UM-Bridge} software \cite{umbridge.software}, an interface for deploying containerized models. The Tsunami simulation is detailed  in \cite{seelinger2021high}. 

In all the examples we see the credible interval captures the true mean and the width of the credible interval decreases steadily across iterations. The true error is usually at least one order of magnitude below the error indicated by the credible interval. Notice the desirable behavior where the samples begin to cluster near the failure boundary in later iterations. We found results to be highly dependent on the choice of the prior kernel, prior mean, and their hyperparameters. 

\begin{figure}[!ht]
    \centering
    \includegraphics[width=\textwidth]{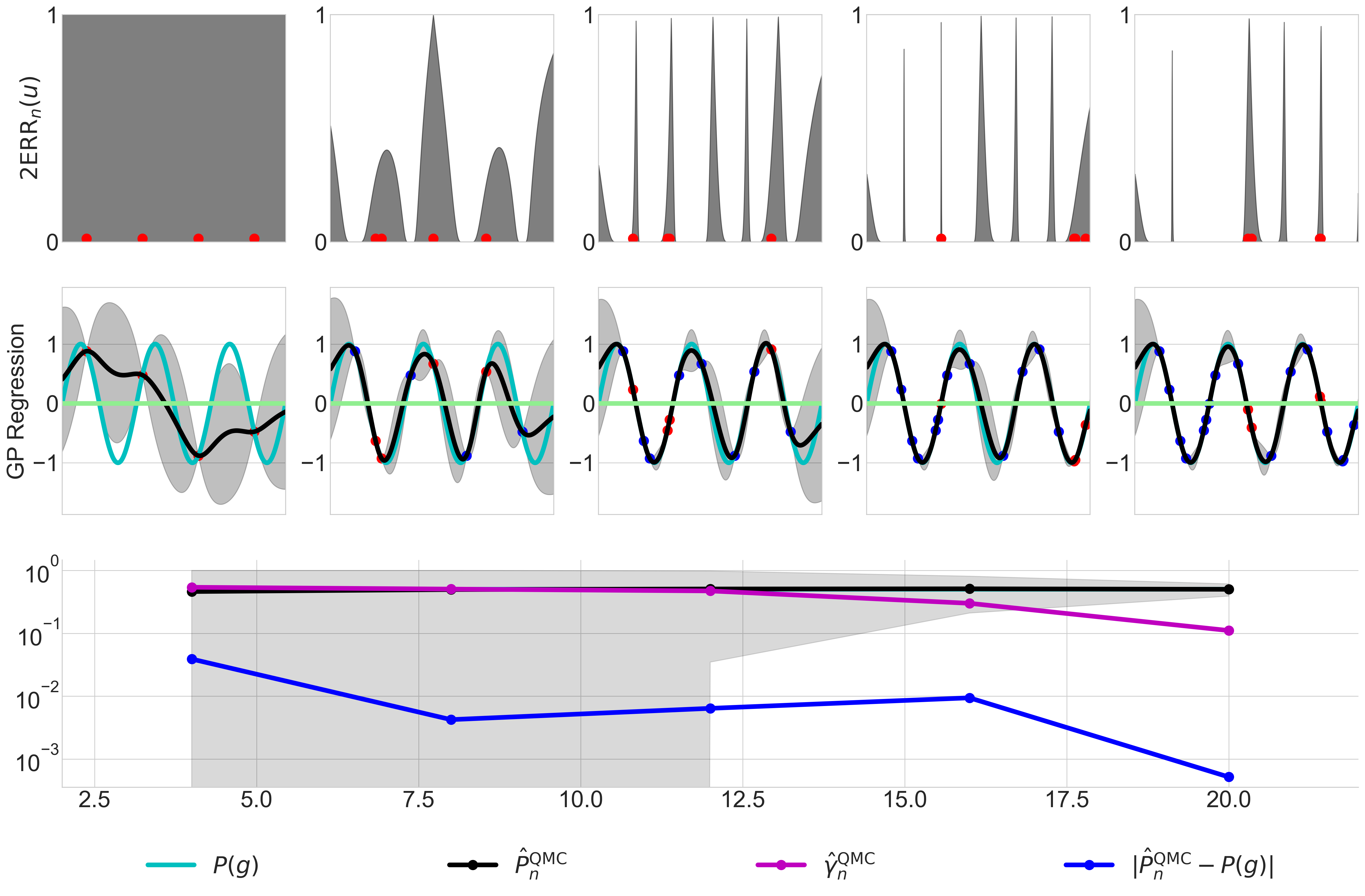}
    \caption{Sine function, 1-dimensional. Moving left to right across columns traverses algorithm iterations. In each column, the top row shows the unnormalized density $2\mathrm{ERR}_n$ from which samples are drawn. The second row shows the true function in cyan with the corresponding Gaussian process visualized by its posterior mean in black and a 95\% point-wise confidence interval in gray. Red point are new points introduced in this iteration while blue points are those form previous iterations. The green line is the failure threshold. The bottom plot shows convergence of the algorithm against the number of points $n$ with the resulting approximate credible interval in gray.}
    \label{fig:sine}
\end{figure}

\begin{figure}[!ht]
    \centering
    \includegraphics[width=\textwidth,trim={0 0cm 0 0},clip]{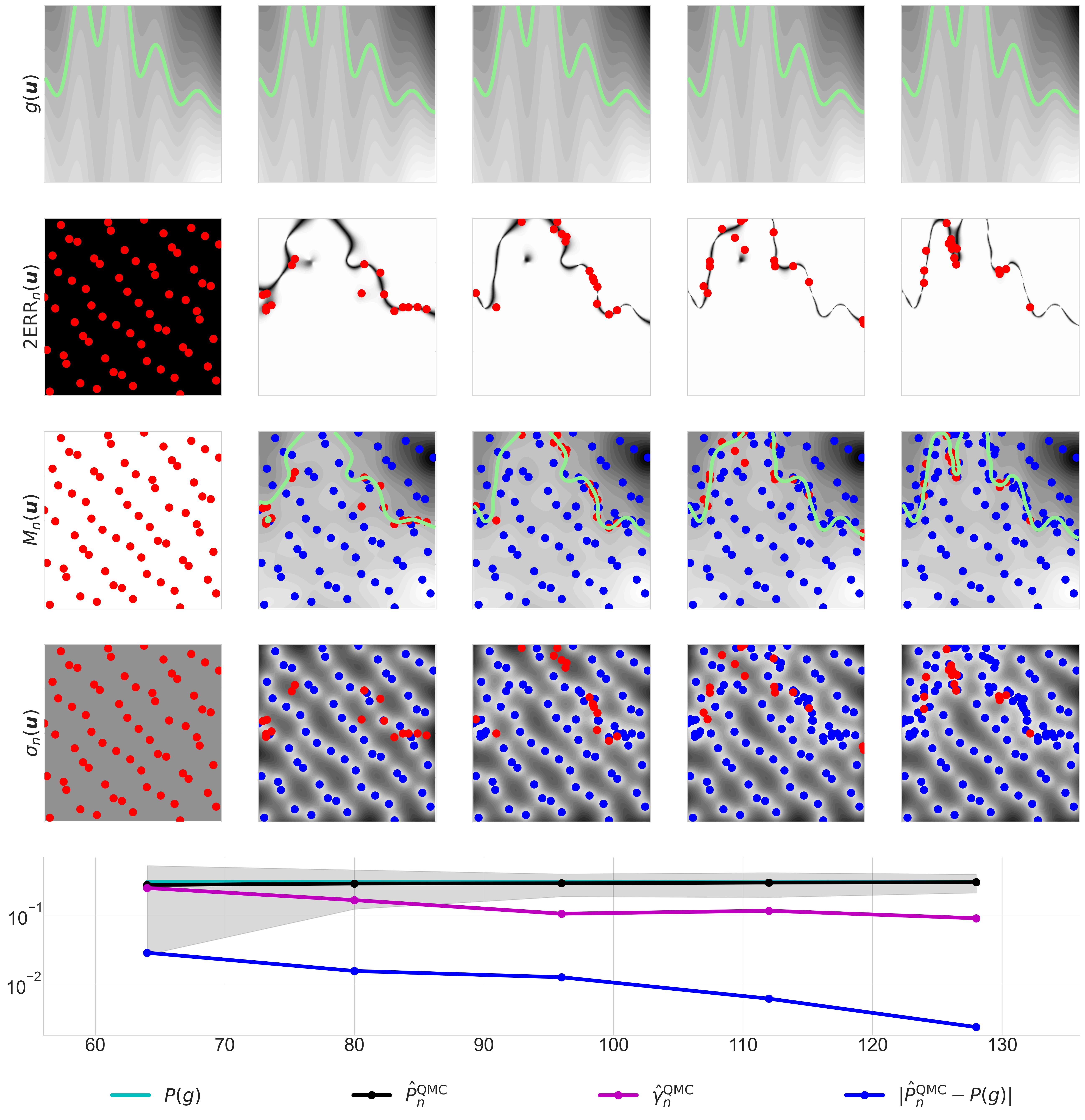}
    \caption{Multimodal function, 2-dimensional. Moving left to right across columns traverses algorithm iterations. In each column, the top row visualizes the true function with true failure boundary $\{\bu \in [0,1]^d: g(\bu) = 0\}$ in green. The second row plots the unnormalized sampling density $2\mathrm{ERR}_n$ while the third and fourth row plot the Gaussian process posterior mean and standard deviation respectively. In the third row we plot the predicted failure boundary based on the posterior mean in green. The final row visualizes convergence as in \Cref{fig:sine}.} 
    \label{fig:multimodal}
\end{figure}

\begin{figure}[htbp]
    \centering
    \includegraphics[width=\textwidth,trim={0 0 0 0},clip]{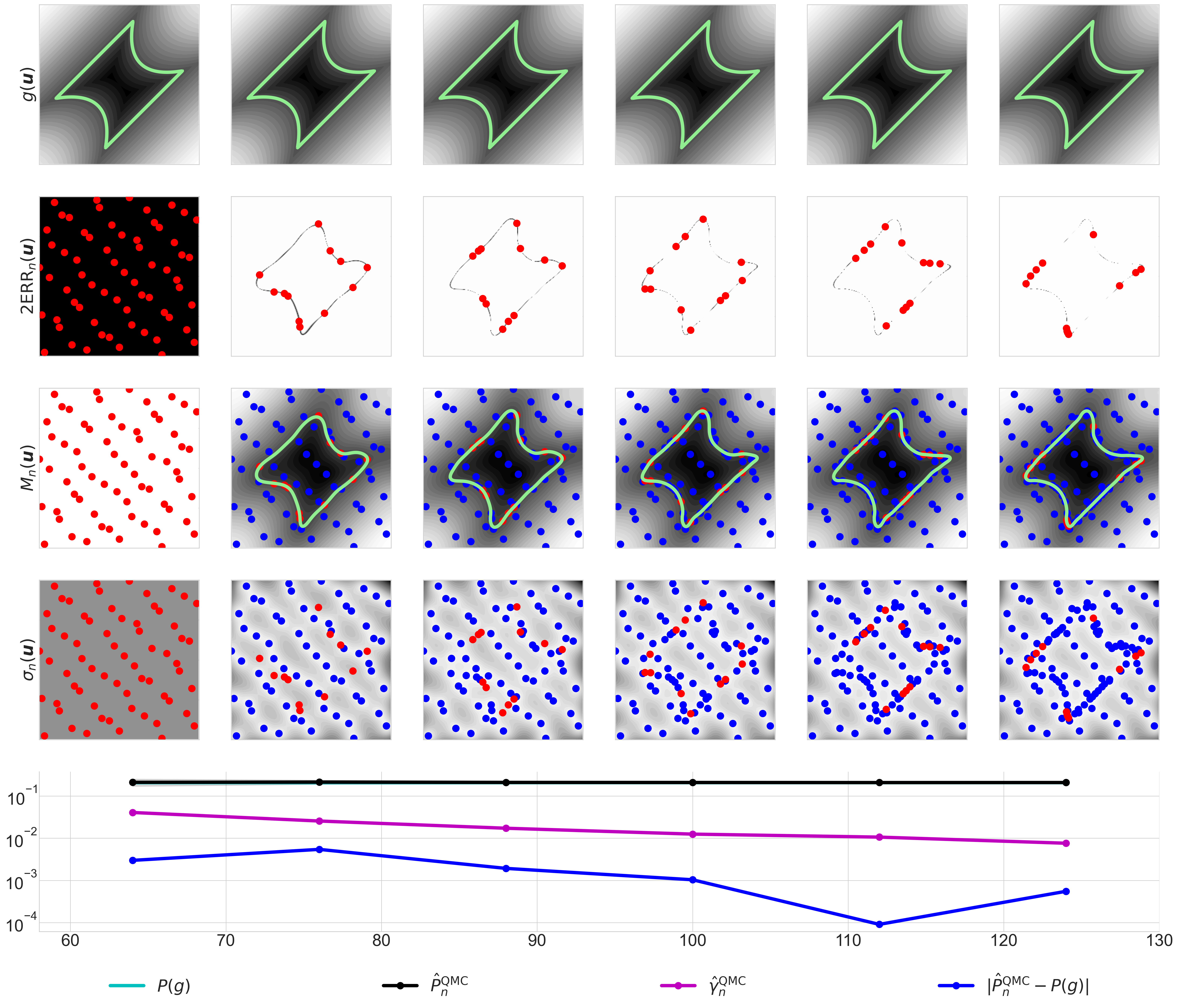}
    \caption{Four Branch function, 2-dimensional. See \Cref{fig:multimodal}.}
    \label{fig:fourbranch}
\end{figure}

\begin{figure}[!ht]
    \centering
    \includegraphics[width=.8\textwidth,trim={0 0 0 0},clip]{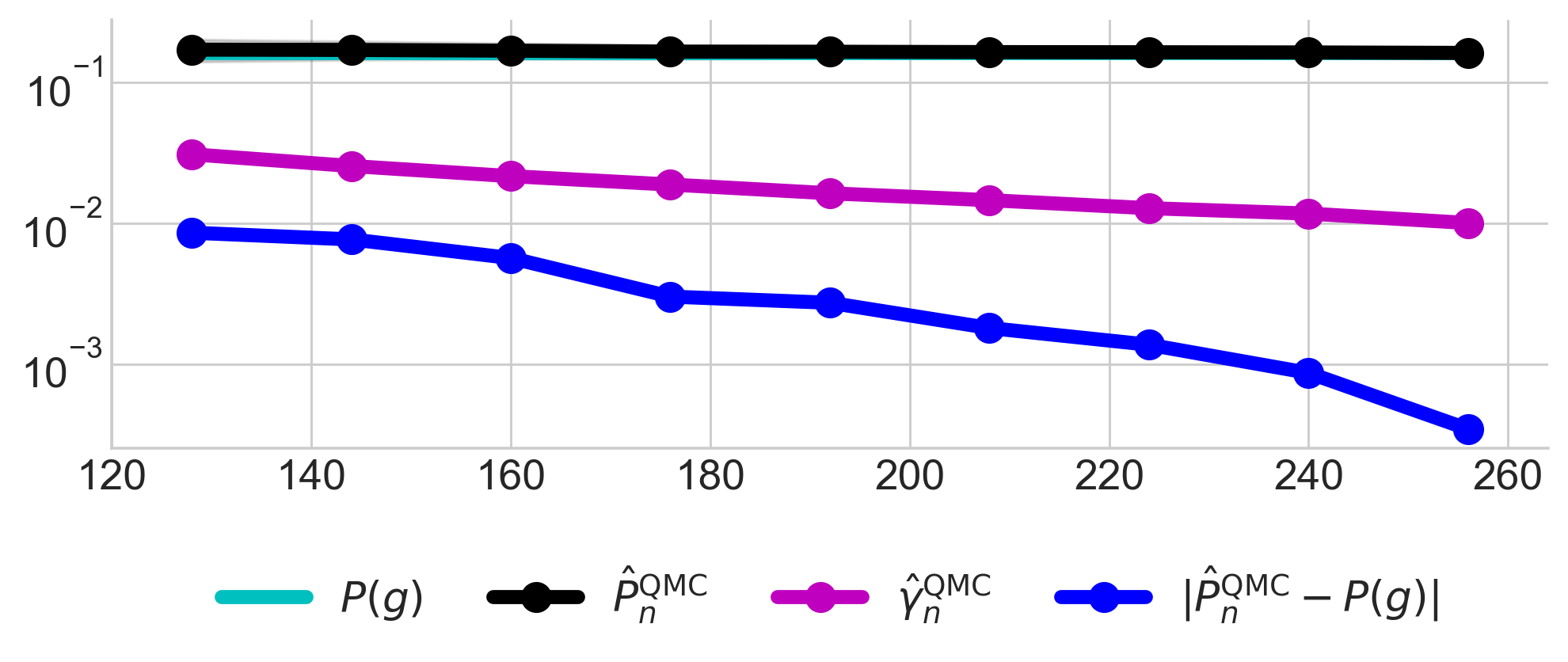}
    \caption{Ishigami function, 3-dimensional. Convergence plots as in \Cref{fig:sine}.}
    \label{fig:ishigami}
\end{figure}

\begin{figure}[!ht]
    \centering
    \includegraphics[width=.8\textwidth,trim={0 0 0 0},clip]{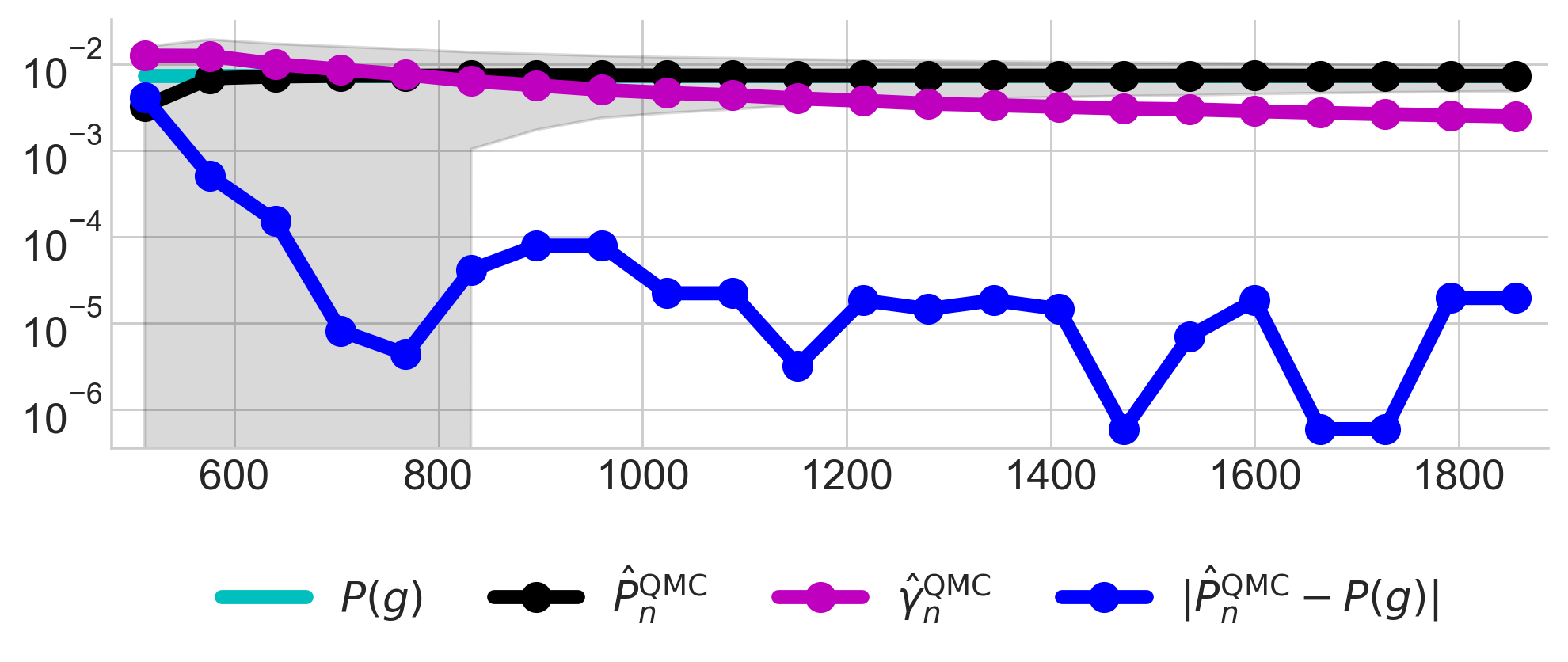}
    \caption{Hartmann function, 6-dimensional. Convergence plots as in \Cref{fig:sine}.}
    \label{fig:hartmann}
\end{figure}

\begin{figure}[!ht]
    \centering
    \includegraphics[width=\textwidth,trim={0 0 0 0},clip]{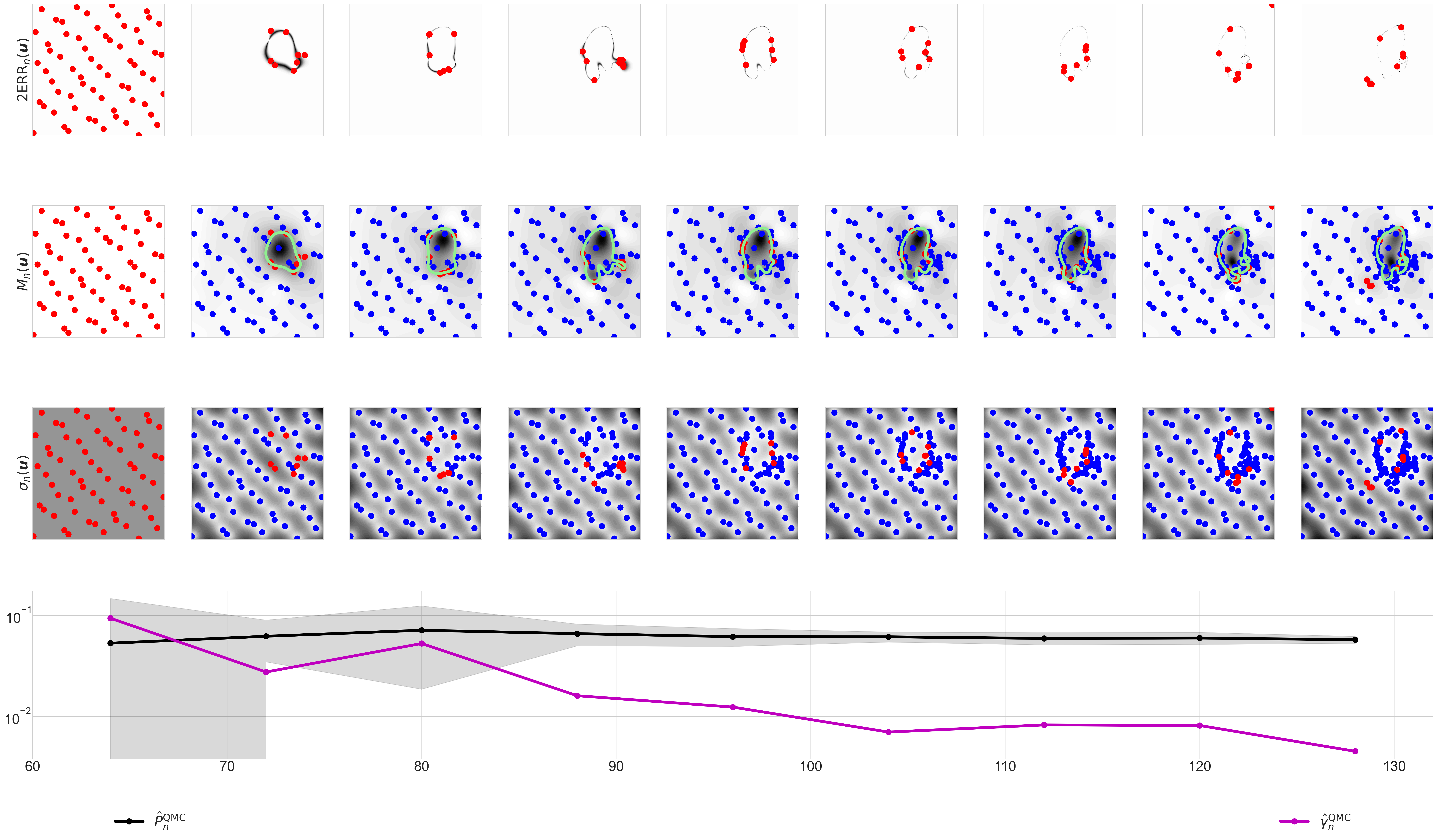}
    \caption{Tsunami simulation, 2-dimensional. Similar to \Cref{fig:multimodal}, except the first row is removed as we can no longer plot the true function. Also, since we no longer know the true mean the convergence plot in the bottom row cannot visualize our true error.}
\end{figure}

\Subsection{Remaining Challenges} \label{sec:conclusions_future_work}

Let us end this section by highlighting some remaining challenges and directions for future work.
\begin{itemize}
    \item The credible intervals quantify uncertainty in $P$. However, the credible interval bounds $\underline{P}$ and $\overline{P}$ cannot be computed exactly, so we resort to high precision QMC approximation. These QMC approximates come with their own relatively small uncertainty which should be accounted for in more rigorous credible intervals. 
    \item There are many sampling schemes to refine the posterior distribution and thus shrink the width of the credible interval. A valuable contribution would be analytically or empirically compare the convergence rates of credible interval width across sampling schemes. 
    \item There is a delicate interplay between the time it takes to run a simulation, fit a GP regression model, and propose samples. For simulations that are expensive to evaluate, it may be worthwhile to spend more time performing GP hyperparameter optimization and looking for good next sampling locations. For cheap functions, one may wish to match the sampling scheme and kernel in order to rapidly update a fast GP surrogate, see \Cref{sec:fast_gps}. Significant savings may be had by optimizing the trade-off between time spent simulating and time spent between simulations.
    \item This section has discussed the simple case where the simulation is assumed to be deterministic. However, many functions are not deterministic but contain uncertainty of their own. Moreover, simulations may have cheaper, lower-fidelity counterparts. Extending the framework to noisy, multi-fidelity, and/or multilevel problems would greatly widen the scope of applicability and potentially yield significant computational savings. 
\end{itemize}

\Section{Subsurface Flow Through Porous Media with Multilevel Fast Gaussian Processes} \label{sec:gp4darcy}

\begin{quotation}
    This section follows \cite{sorokin.gp4darcy}, a publication I worked on during my 2023 Graduate Internship appointment at Los Alamos National Laboratory in collaboration with Aleksandra Pachalieva, Daniel O'Malley, James M. Hyman, Nicolas W. Hengartner, and Fred J. Hickernell.
\end{quotation}

% 1. The question we are addressing is 
Limiting the injection rate to restrict the pressure below a threshold at a critical location can be an important goal of simulations that model the subsurface pressure between injection and extraction wells.   
The pressure is approximated by the solution of Darcy's partial differential equation for a given permeability field. The subsurface permeability is modeled as a random field since it is known only up to statistical properties. This induces uncertainty in the computed pressure.
Solving the partial differential equation for an ensemble of random permeability simulations enables estimating a probability distribution for the pressure at the critical location.   
% 2. The bottleneck is 
These simulations are computationally expensive, and practitioners often need rapid online guidance for real-time pressure management.  
% 3. The Methods we use to address this problem are (what we did)
An ensemble of numerical partial differential equation solutions is used to construct a Gaussian process regression model that can quickly predict the pressure at the critical location as a function of the extraction rate and permeability realization. The Gaussian process surrogate analyzes the ensemble of numerical pressure solutions at the critical location as noisy observations of the true pressure solution, enabling robust inference using the conditional Gaussian process distribution. 
% 4. Our principal results are 
Our first novel contribution is to identify a sampling methodology for the random environment and matching kernel technology for which fitting the Gaussian process regression model scales as $\mathcal{O}(n \log n)$ instead of the typical $\mathcal{O}(n^3)$ rate in the number of samples $n$ used to fit the surrogate, see \Cref{sec:fast_gps}.  
The surrogate model allows almost instantaneous predictions for the pressure at the critical location as a function of the extraction rate and permeability realization. Our second contribution is a novel algorithm to calibrate the uncertainty in the surrogate model to the discrepancy between the true pressure solution of Darcy's equation and the numerical solution. 
% 5. What these results mean
Although our method is derived for building a surrogate for the solution of Darcy's equation with a random permeability field, the framework broadly applies to solutions of other partial differential equations with random coefficients.

\Subsection{Introduction}

\Cref{fig:confidences_lineplot} shows the GP model prediction for the expected confidence that the pressure at the critical location will be below a given threshold. For example, suppose the threshold pressure is 2\,mmH$_2$O (millimeters of water). In this case, the extraction rate must be at least 0.01\,m$^3$/s (cubic meters per second) to have 90\% confidence that the pressure at the critical location will be below the threshold. Notice the confidence computed using the GP surrogate increases in both the extraction rate and threshold, which matches the physics of the simulation.

\begin{figure}[!ht]
    \centering
    \includegraphics[width=1\textwidth]{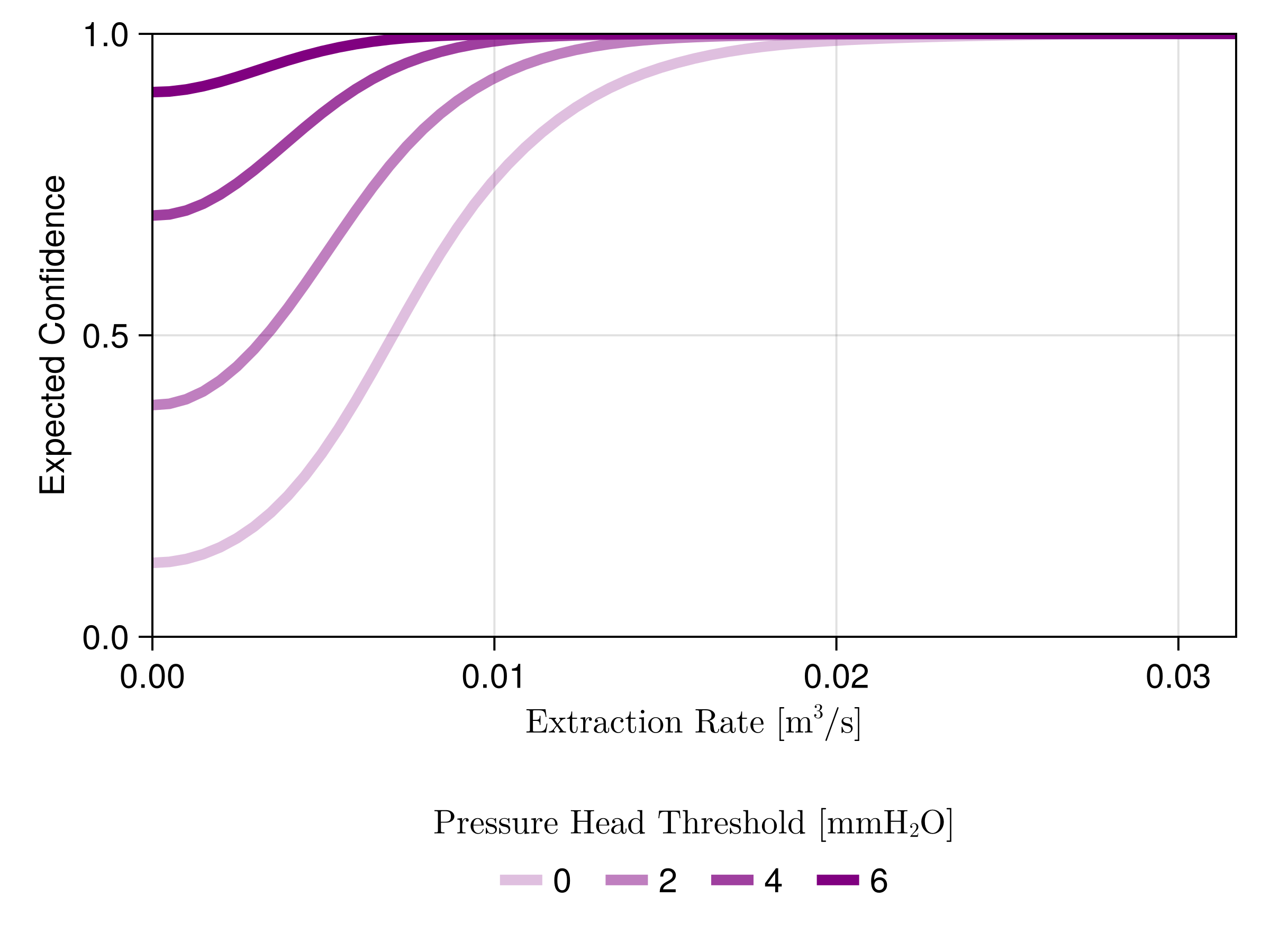}
    \caption{The expected confidence in maintaining pressure below a threshold as a function of extraction rate [m$^3$/s]. \Cref{fig:confidences_heatmap} extends this plot to a continuum of thresholds.}
    \label{fig:confidences_lineplot}
\end{figure}

\Cref{fig:subsurface_flow_workflow} shows the workflow of our approach:
\begin{enumerate}
\item We sample the feasibility space by generating a uniform low-discrepancy sample for the extraction rates and permeability fields. Then, for each extraction-permeability pairing, we solve Darcy's equation numerically at a sequence of increasing fidelities using the \texttt{DPFEHM} software package \cite{DPFEHM.jl}. For this work, we used the default solver in \texttt{DPFEHM}, which is a conjugate gradient iterative solver preconditioned with algebraic multigrid. The simulation errors depend on the fidelity of the permeability field and the fidelity of the finite volume numerical solver. The decay of differences between simulations at increasing fidelities informs an approximate upper bound on the noise variance. The differences between simulations at the highest fidelity and the target fidelity inform an approximate lower bound on the noise variance. These bounds are used in the next step for optimizing the noise variance. Once these bounds have been found, we may choose to sample more at the target fidelity as these numerical solutions will be used to build the Gaussian process (GP) regression model. 
\item With numerical solutions in hand, the observations at the target fidelity are used to fit a fast GP model. We emphasize that the target fidelity is not necessarily the maximum among the increasing sequence of fidelities at which we numerically solve the PDE. The target fidelity is chosen based on budgetary restrictions and the GP surrogate's noise will adapt to this choice. The GP model may be fit quickly since the chosen low-discrepancy sampling locations and matching kernel have induced a circulant kernel matrix structure, see \Cref{sec:fast_gps}. This fast fitting cost includes optimizing hyperparameters such as the noise variance. The noise variance optimization is initialized to the approximate upper bound and restricted to be above the approximate lower bound derived in the previous step.
\item The GP surrogate fit to the target fidelity observations can then be used to select an optimal extraction rate in real-time for any pressure threshold. The key is that the GP surrogate is much faster to evaluate than the high-fidelity numerical PDE solver and the GP models the true analytic PDE solution, not the numerical one. This enables rapid real-time analysis for a variety of objectives from an error-aware model.
\end{enumerate}

\begin{figure}[!ht]
    \centering
    \includegraphics[width=1\textwidth]{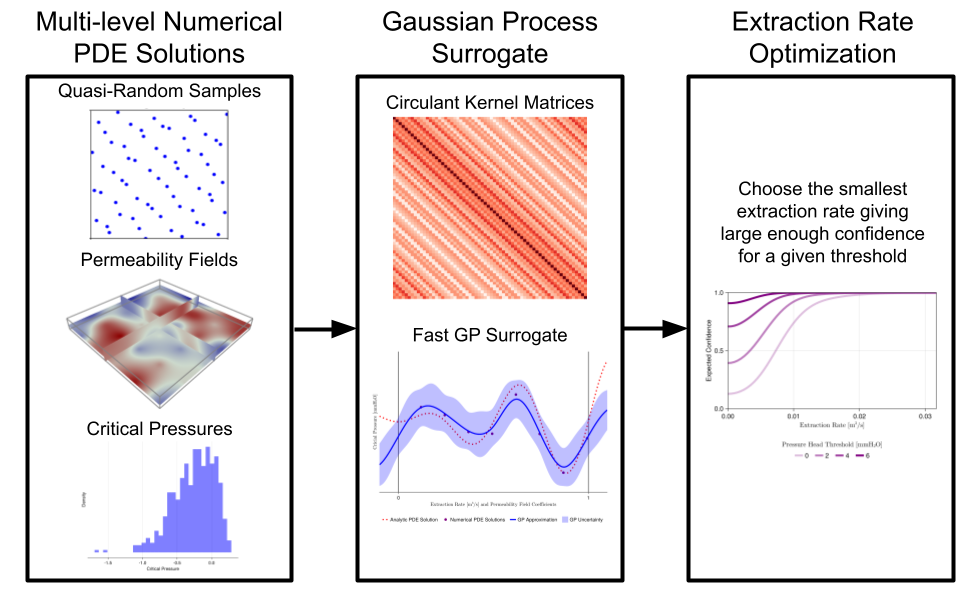}
    \caption{Workflow diagram visualizing the three stages of our method. First, the possible extraction rate and permeability field realizations are sampled with a low-discrepancy sequence. 
    Each pair is input to the numerical partial differential equation solver, which returns an approximation for the pressure at a critical location. Next, a GP model is optimized to the data relating extraction rate and permeability to pressure at the critical location. The optimization for $n$ samples is done at $\mathcal{O}(n \log n)$ cost by exploiting structure in the Gram kernel matrix induced by using low-discrepancy samples and matching kernels, see \Cref{sec:fast_gps}. The trained GP model can be quickly evaluated to identify the lowest extraction rate, which rarely over-pressurizes a critical location.}
    \label{fig:subsurface_flow_workflow}
\end{figure}

\Cref{sec:methods} introduces the modeling equations, notation for the problem formulation, and the existing methods we build upon later in the chapter. This includes details on numerical solutions of Darcy's equation and an overview of GP modeling. Our novel theoretical contributions are detailed in \Cref{sec:novel_contributions} where we first discuss a method for calibrating the GP noise to the numerical PDE solution error and then tie in our fast $\calO(n \log n)$ GP construction from \Cref{sec:fast_gps}. \Cref{sec:results} discusses details of our numerical simulations, exemplifies the use of the trained GP model for real-time pressure management, and explores the efficacy of both the  Gaussian noise assumption and calibration routine. This section concludes by applying our method to the Darcy problem with a three-dimensional subsurface, emphasizing the generality of our algorithm. 

\Subsection{Methods} \label{sec:methods}

This subsection describes the problem and the model equations of interest. We start by formulating Darcy's equation and describing our quantity of interest: the confidence (probability) that the pressure at a critical location stays below a given threshold. Using a Gaussian process (GP) regression surrogate gives a distribution over possible pressures at the critical location leading to a random confidence.  Next, we discuss the numerical solution of Darcy's equation including the permeability field discretization using the Karhunen--Loève expansion and the two-point flux finite volume method. Finally, we describe how the GP surrogate views the numerical PDE solutions as noisy evaluations of the true pressure solutions in order to fit a distribution over admissible pressure solutions.

\Subsubsection{Problem Formulation}

Consider a pressure management problem of a single-phase fluid in a heterogeneous permeability field. Darcy's partial differential equation can model the pressure throughout the subsurface 
\begin{equation}
    \nabla  \cdot (G(x)  \cdot \nabla H(x)) = f(x),
    \label{eq:darcy}
\end{equation}
when the subsurface permeability field is known. Darcy's equation describes the pressure head $H(x)$ in the subsurface over a domain $D \subset \mathbb{R}^2$ with permeability field $G(x)$ and external forcing function $f(x)$. The steady-state Darcy equation \eqref{eq:darcy} allows us to evaluate the long-term impact of the injection and extraction on the pressure head. For pressure management, the forcing function $f$ is composed of an
injection rate $w \geq 0$ at $x_\text{injection}$ and an extraction rate
$-r \leq 0$ at $x_\text{extraction}$. Following 
\cite{pachalieva.physics_informed_ML_differential_programming}, we write 
\begin{equation}
    f(x;r) := \begin{cases} 
        w, & x = x_\text{injection} \\ 
        -r, & x = x_\text{extraction} \\
        0, & x \in D \backslash \{x_\text{injection},x_\text{extraction}\}
    \end{cases}.
    \label{eq:flow_rate}
\end{equation}
Throughout this section, we treat the \emph{injection rate} $w$ as fixed and focus on optimizing the \emph{extraction rate} $r$, assuming $r$ does not exceed $w$.

The details of the permeability field $G$ are rarely known in practice. Instead, one is often given statistical properties of that field, such as the mean permeability and spatial correlations of variations of the permeability around its mean. Given these descriptors, it is convenient to model the permeability $G$ as a \emph{random log-normal} field. Solving Darcy's equation \eqref{eq:darcy} using a stochastic permeability field $G$ induces randomness in the pressure $H$.

Our goal is to quickly estimate the probability that the pressure at the \emph{critical location} $x_\text{critical}$ remains below a desired \emph{threshold} $\bar{h}$ as a function of the extraction rate $r$. This enables practitioners to implement control policies to manage pressure at critical locations in real-time. 
Let $H^c(r,G) := H(x_\text{critical};r,G)$ denote the pressure at the critical location,  which is a function of the extraction rate, $r$, and permeability field, $G$. 
For a fixed upper bound $\bar h$, we seek to evaluate the \emph{confidence},
\begin{equation}
    c(r) := P_G(H^c(r,G) \leq \bar{h}),
    \label{eq:confidence}
\end{equation}
where the probability $P_G$ is taken over the distribution of the permeability field $G$. \Cref{fig:numerical_pde.birdseye} illustrates the described setup. 

\begin{figure}[!ht]
    \centering
    \includegraphics[width=1\textwidth]{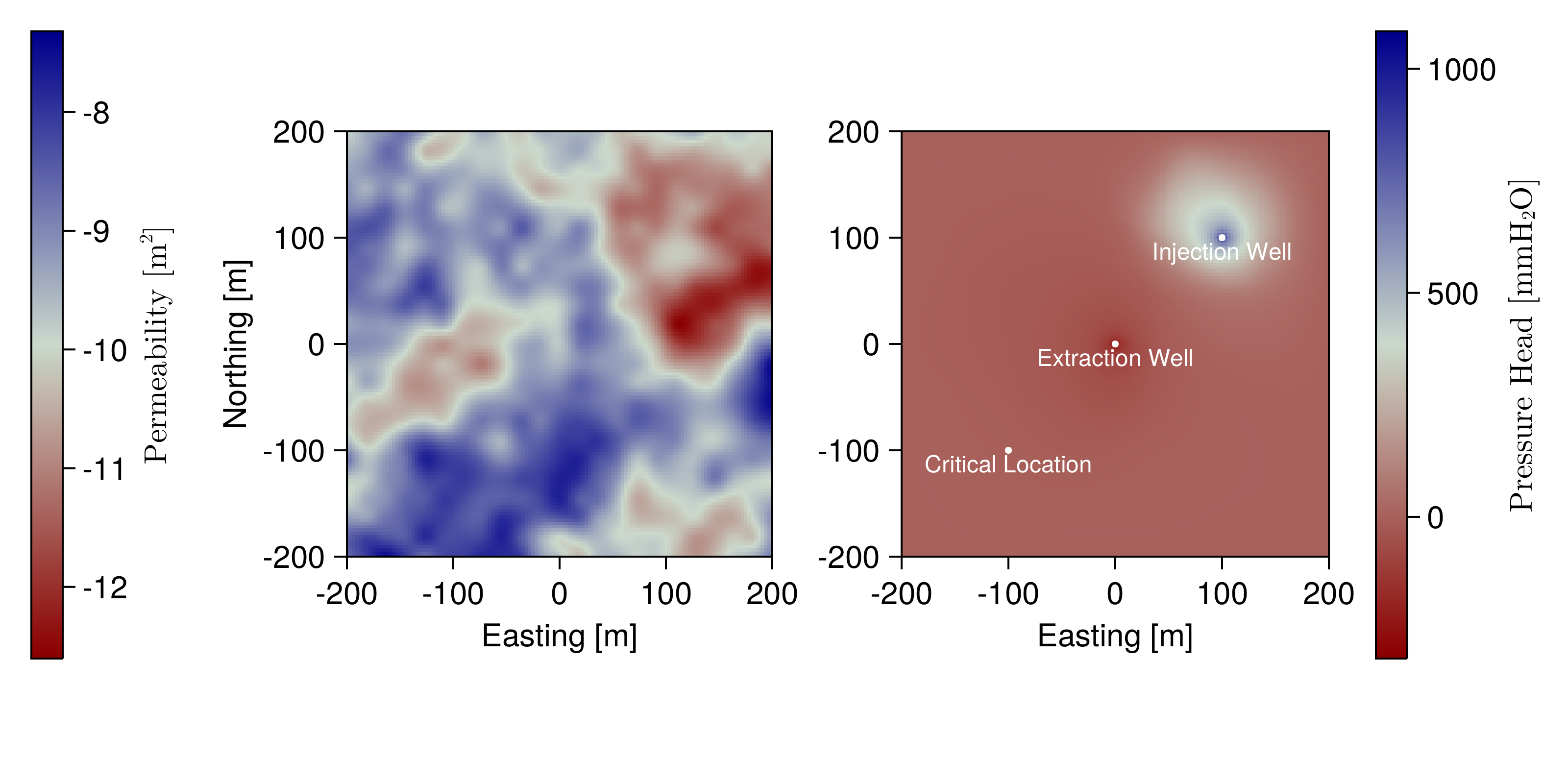}
    \caption{The left plot shows a realization of the log-permeability field $\log\, G(x)$[m$^2$] and the right plot shows the resulting pressure $H(x; r,G)$ for some extraction rate $r$ [m$^3$/s]. The fixed injection, extraction, and critical locations in our two-dimensional setup are also shown.}
    \label{fig:numerical_pde.birdseye}
\end{figure}

Although the confidence, $c(r)$, cannot be computed explicitly, it can be approximated for each fixed extraction rate $r$ by numerically solving for critical pressure $H^c(r,G)$ in Darcy's equation \eqref{eq:darcy} for many realizations of $G$. Unfortunately, this method is biased as the numerical critical pressure is only an approximation for the true  critical pressure $H^c(r,G)$.  Also, the associated cost of solving the PDE multiple times is impractical for practitioners desiring fast online inference. 

Our approach provides rapid solutions and error estimates for the confidence $c(r)$ by building a surrogate model for the critical pressure $H^c(r,G)$. This statistical approach treats the numerically computed critical pressures as noisy observations of the analytic critical pressures. GP modeling is a natural and efficient approach for this framework and can provide immediate online estimates for $c(r)$ as a function of the extraction rate. 

Given $n$ numerical critical pressure observations, the GP surrogate $H^c_n(r,G)$ estimates the critical pressure $H^c(r,G)$. We plug this estimate into \eqref{eq:confidence} and get the \emph{conditional confidence} 
\begin{equation}
    \hat{C}_n(r) := P_G(H^c_n(r,G) \leq \bar{h} | H_n^c).
    \label{eq:conditional_confidence}
\end{equation}
The \emph{expected conditional confidence} is a natural estimate for $c(r)$ denoted by 
\begin{equation}
    c_n(r) := \mathbb{E}_{H_n^c} \left[\hat{C}_n(r)\right] = P_{(G,H_n^c)}(H^c_n(r,G) \leq \bar{h}).
    \label{eq:expected_conditional_confidence}
\end{equation}
We approximate the unknown analytic solution $c(r)$ by the computationally tractable $c_n(r)$, which only uses the surrogate model. After interchanging expectations, the above equation can be efficiently computed with (Quasi-)Monte Carlo, see \Cref{sec:qmc}.

\Subsubsection{Numerical Solution of Darcy's Equation}

To solve Darcy's equation \eqref{eq:darcy}, we apply a standard two-point flux finite volume method in the square domain $D$ on a discrete mesh. A truncated Karhunen--Loève expansion represents the log of the log-normal permeability field $G$ over $D$. This enables us to draw samples of $G$, which can be evaluated at a mesh grid of any fidelity. 

We say the physical domain has discretization dimension $d$ when the finite volume mesh has $d+1$ mesh points in each dimension of $D$. The choice of $d=2^m$ creates nested mesh grids. While there is no restriction on the mesh grids with an equal number of points in each dimension, this reduces the number of parameters we must consider when approximating the numerical error later in this section. 

The Karhunen--Loève expansion \cite{karhunen.kl_expansion} %huang.convergence_kl,ghanem.stochastic_finite_elements_kl,shinozuka.simulation_stochastic_processes,grigoriu.spectral_represetation_simulation
of the permeability field may be used to find a good finite-dimensional approximation of $G$. Specifically, we may write
\begin{equation}
    \log \,  G(x) = \sum_{j=1}^\infty \sqrt{\lambda_j} \varphi_j(x) Z_j
    \label{eq:KL_expansion}
\end{equation}
where $\varphi_j$ are deterministic and orthonormal and $Z_1, Z_2, \dots$ are independent standard Gaussian random variables. Ordering $\lambda_1 \geq \lambda_2 \geq \lambda_3 \geq \dots$, we approximate $G$ by 
\begin{equation}
    \log \, 
 G^s(x) := \sum_{j=1}^s \sqrt{\lambda_j} \varphi_j(x) Z_j 
    \label{eq:KL_expansion_approx}
\end{equation}
which optimally compacts the variance into earlier terms. \Cref{fig:numerical_pde.full} shows different pairs of $s$ and $d$ for a common realization of $Z_1,Z_2,\dots$. Notice the greater detail in $G^s$ as $s$ increases and the finer mesh over $D$ as $d$ increases. We will often call the pair $(s,d)$ the fidelity of the numerical solution.

\begin{figure}[!ht]
    \centering
    \includegraphics[width=1\textwidth]{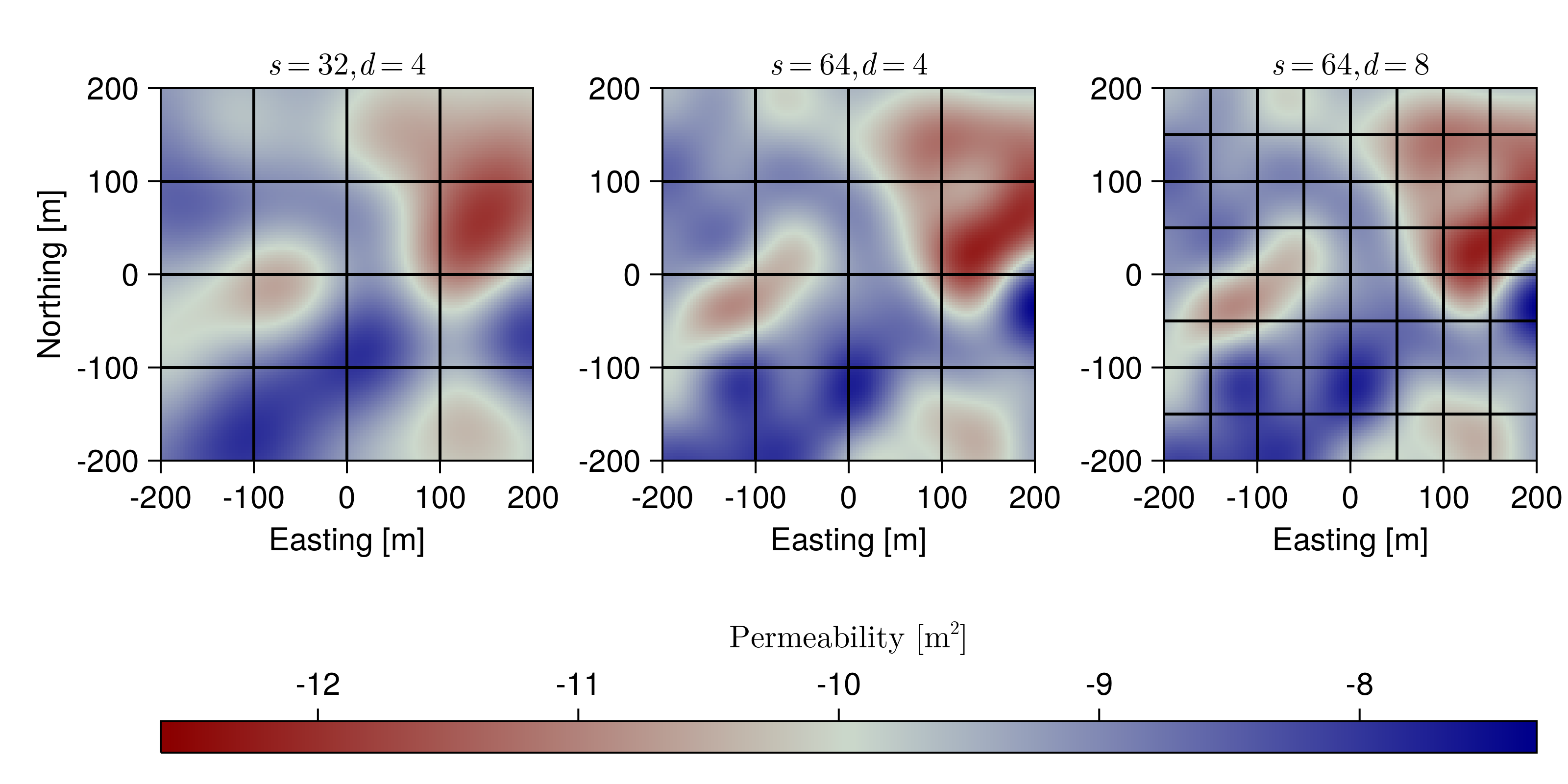}
    \caption{The same realization of the log-permeability field as in \Cref{fig:numerical_pde.birdseye} but for various choices of permeability discretization dimension $s$ and domain discretization dimension $d$. In the left plot, a small $s$ and $d$ are chosen, corresponding to a lack of small-scale changes in the permeability realization and a coarse mesh grid over the domain, respectively. Moving from the left to center plot maintains the same mesh grid while increasing $s$ to yield more small-scale changes in the realization. Moving from the center to the right plot keeps the same realization while increasing $d$ to yield a finer mesh over $D$.}
    \label{fig:numerical_pde.full}
\end{figure}

We let $H^c_{s,d}(r,\boldsymbol{Z}^{s})$ denote the \emph{numerical critical pressure} computed by solving the PDE \eqref{eq:darcy} with domain discretization dimension $d$ and permeability discretization dimension $s$. Here $\boldsymbol{Z}^{s} = (Z_1,\dots,Z_s)$ is a vector of independent standard Gaussians, which uniquely determine the approximate permeability field $G^s$. In our implementation, we use the \texttt{GaussianRandomFields.jl} package \cite{GaussianRandomFields.jl} to simulate permeability fields $G^s$ and solve the PDE numerically with the \texttt{DPFEHM.jl} package \cite{DPFEHM.jl}. 

\Subsubsection{A Probabilistic GP Surrogate}

We model the relationship between the inputs of extraction rate $r$ and permeability field $G$ and the output critical pressure $H^c(r,G)$. The model is built on observed numerical critical pressures $Y^n := \{H^c_{s,d}(r_i,\boldsymbol{Z}_i^s)\}_{i=1}^n$ at strategically chosen sampling locations $(r_i,G_i^s)_{i=1}^n$. Following \Cref{sec:gps}, our GP surrogate views $Y^n$ as noisy observations of $H^c$ with 
\begin{equation}
    H_{s,d}^c(r,\boldsymbol{Z}^s) = H^c(r,G) + \varepsilon_{s,d}.
    \label{eq:gp_model}
\end{equation}
We model the noise $\varepsilon_{s,d}$, which encodes the discretization error, as a random variable.  That random variable is assumed to be independent of the sampling location $(r,\boldsymbol{Z}^s)$ but dependent on the fidelity $(s,d)$.  We further assume that the errors are zero-mean Gaussians with variance $\zeta_{s,d}$, i.e., 
\begin{equation}
    \varepsilon_{s,d} \sim \mathcal{N}(0,\zeta_{s,d}).
    \label{eq:noise_var_normal_assumption}
\end{equation}
Assuming homogeneous variances enables us to exploit numerical tricks that lead to fast computations of the posterior expectations.
Later we will discuss a method for approximating bounds on the noise variance $\zeta_{s,d}$, which are then used during hyperparameter optimization.

GP modeling assumes $H^c$ is a GP, and therefore, the conditional distribution of $H^c$ given $Y^n$ is also a GP \cite{rasmussen.gp4ml}. We use this conditional, or posterior, distribution on $H^c$ as an error-aware surrogate. Closed form expressions for the posterior mean and variance can be found in \Cref{sec:gps}. Note that here, the GP takes as input the $1+s$-dimensional vector $t := (r,\boldsymbol{Z}^s)$. \Cref{fig:noisy_lattice_qgp.4} illustrates an example posterior GP. While the figure assumes $t$ is one-dimensional, which is impossible for our problem, GPs extend naturally to arbitrarily large dimensions. We emphasize that the posterior GP is a surrogate for the critical pressure $H^c$, not the numerical critical pressure $H^c_{s,d}$ whose evaluations are used for fitting. Our GP surrogate provides a distribution on $H^c$ whose expectation can be taken as a point estimate for the analytic critical pressure solution.

\Subsection{Theory} \label{sec:novel_contributions}

This subsection describes in detail the novel theoretical contributions of this section. We will approximate upper and lower bounds on the variance of the error between the numerical PDE solutions used to fit the GP and the true PDE solutions. When optimizing the GP noise variance $\zeta_{s,d}$, the approximate upper bound is used as a starting value and the optimization is restricted to search above the approximate lower bound. We have paired these with the fast GP surrogate techniques from \Cref{sec:fast_gps} which enable GP fitting costs to be reduced from $\mathcal{O}(n^3)$ to $\mathcal{O}(n \log n)$ using an intelligent design of experiments and matching GP kernel. This speedup technology is more generally applicable to surrogate modeling when one has control over the design of experiments.

Recall the GP model with zero-mean Gaussian noise assumes the numerical solution is unbiased for the analytic solution. Therefore, the noise variance $\zeta_{s,d}$ is the mean squared error (MSE) between the analytic PDE solution and the numerical PDE solutions. This section derives approximate upper and lower bounds on the root mean squared error (RMSE) $\sqrt{\zeta_{s,d}}$. The upper bound is used as a starting point to calibrate $\zeta_{s,d}$ when performing hyperparameter optimization of the GP model. This bound is derived by tracking the decay in average solution differences as the problem fidelity is increased. The lower bound is used to restrict the search domain for this hyperparameter optimization. This heuristic lower bound is derived by looking at the MSE of differences between solutions at the maximum and target fidelities.

We start by approximating an upper bound on $\zeta_{s,d}$. First, notice the assumption of zero-mean Gaussian noise in \eqref{eq:noise_var_normal_assumption} implies that the standard deviation of the GP noise may be written as the RMSE
\begin{equation}
    \begin{aligned}
        \sqrt{\zeta_{s,d}} &= \left\lVert H^c(R,G)-H^c_{s,d}(R,\boldsymbol{Z}^{s}) \right\rVert \\
        &= \sqrt{\mathbb{E}_{\left(R,\boldsymbol{Z}^{s}\right)} \left[H^c(R,G)-H^c_{s,d}(R,\boldsymbol{Z}^{s})\right]^2} \\
        &=: \mathrm{RMSE}_{s,d}.
    \end{aligned}
    \label{eq:post_var_N_as_expected_val}
\end{equation}
Here the extraction rate $R$ is assumed to be uniformly distributed between $0$ and $w$, i.e., $R \sim \mathcal{U}[0,w]$. Moreover, the extraction rate is assumed to be independent of $\boldsymbol{Z}^{s}$.

Following ideas from multilevel Monte Carlo \cite{giles.MLMC_path_simulation,giles2015multilevel,robbe.multi_index_qmc}, we choose strictly increasing sequences $(s_j)_{j \geq 0}$ and $(d_j)_{j \geq 0}$ with $s = s_T$ and $d = d_T$ and rewrite \eqref{eq:post_var_N_as_expected_val} as the telescoping sum 
\begin{equation}
    \mathrm{RMSE}_{s_T,d_T} = \left\lVert \sum_{j=T+1}^\infty \left[\Delta_{s_j}(R,\boldsymbol{Z}^{s_j}) + \Delta_{d_j}(R,\boldsymbol{Z}^{s_j}) \right]\right\rVert
    \label{eq:telescoping_post_var_N}
\end{equation}   
where 
\begin{align*}
    \Delta_{s_{j+1}}(R,\boldsymbol{Z}^{s_{j+1}}) &= H^c_{s_{j+1},d_j}(R,\boldsymbol{Z}^{s_{j+1}}) - H^c_{s_j,d_j}(R,\boldsymbol{Z}^{s_j}), \\
    \Delta_{d_{j+1}}(R,\boldsymbol{Z}^{s_{j+1}}) &= H^c_{s_{j+1},d_{j+1}}(R,\boldsymbol{Z}^{s_{j+1}}) - H^c_{s_{j+1},d_j}(R,\boldsymbol{Z}^{s_{j+1}}).
\end{align*}
Here $T \geq 0$ is the index of the target fidelity whose observations will be used to fit the GP model.
Let us assume that 
\begin{equation}
    \lVert \Delta_{s_j}(R,\boldsymbol{Z}^{s_j}) \rVert = 2^{b_s} s_j^{a_s} \quad\text{and}\quad  \lVert \Delta_{d_j}(R,\boldsymbol{Z}^{s_j}) \rVert = 2^{b_d}d_j^{a_d}.
    \label{eq:logloglinear_delta_dims}
\end{equation}
The parameters $(a_s,b_s)$ and $(a_d,b_d)$ will be fit using linear regression in the log-log domain. Let $s_j = v_s 2^j$ and $d_j = v_d 2^j$ where $v_s$ and $v_d$ are the respective initial values chosen by the user. Applying the triangle inequality to \eqref{eq:telescoping_post_var_N} gives 
\begin{equation}
    \begin{aligned}
        \mathrm{RMSE}_{s_T,d_T} &\leq \sum_{j=T+1}^\infty \left[2^{b_s}v_s^{a_s} \left(2^{a_s}\right)^j + 2^{b_d} v_d^{a_d} \left(v_d^{a_d}\right)^j\right] \\
        &= 2^{b_s} v_s^{a_s} \frac{2^{(T+1) a_s}}{1-2^{a_s}} + 2^{b_d}v_d^{a_d} \frac{2^{(T+1) a_d}}{1-2^{a_d}} \\
        &=: \overline{\mathrm{RMSE}}_{s_T,d_T}
    \end{aligned}
    \label{eq:bar_rmse_upper_bound}
\end{equation}
using the expression for the sum of a geometric series.

\Cref{fig:convergence_GP4Darcy} illustrates the above idea for both a two-dimensional subsurface (above plot) and three-dimensional subsurface (below plot). First, we pick an $M \geq T$ so that $(s_M,d_M)$ is the maximum fidelity at which we will numerically solve the PDE and $(s_T,d_T)$ is the target fidelity whose numerical solutions will be used to fit the GP surrogate. We emphasize that $M \geq T$ since we do not require the GP surrogate to be built on maximum fidelity solves. Now, at every fidelity $(s_j,d_j)$ with $1 \leq j \leq M$ we solve the PDE at the same $(R_i,\boldsymbol{Z}_i^{s_M})_{i=1}^m$ points to get $\{H^c_{s_j,d_j}(R_i,\boldsymbol{Z}_i^{s_j})\}_{i=1}^m$. Here $\boldsymbol{Z}_i^{s_j}$ is the first $s_j$ element of $\boldsymbol{Z}_i^{s_M}$. For $ 1 \leq j \leq M$ we make the approximations 
\begin{equation}
    \begin{aligned}
        \lVert \Delta_{s_j}(R,\boldsymbol{Z}^{s_j}) \rVert \approx \sqrt{\frac{1}{m} \sum_{i=1}^m \Delta^2_{s_j}(R_i,\boldsymbol{Z}^{s_j})}, \qquad \lVert \Delta_{d_j}(R,\boldsymbol{Z}^{s_j}) \rVert \approx \sqrt{\frac{1}{m} \sum_{i=1}^m \Delta^2_{d_j}(R_i,\boldsymbol{Z}^{s_j})} 
    \end{aligned}
    \label{eq:RMSEs_approxes}
\end{equation}
corresponding to the plotted red dots and blue squares, respectively. The slope intercept pairings $(a_s,b_s)$ and $(a_d,b_d)$ from \eqref{eq:logloglinear_delta_dims} are fit to the values in \eqref{eq:RMSEs_approxes} with lines in the respective colors. The upper bounds $\overline{\mathrm{RMSE}}_{s_j,d_j}$ from \eqref{eq:bar_rmse_upper_bound} are visualized by the purple stars. Notice that the model will find $\overline{\mathrm{RMSE}}_{s_T,d_T}$ for any target fidelity $(s_T,d_T)$ we choose. As $d_T$ increases, the mesh size shrinks, and the PDE becomes more expensive to solve numerically. As $s_T$ increases, the input dimension to the GP model grows, and more PDE solves are required to build an accurate model. Notice that the RMSE is dominated by the error in the permeability field discretization rather than the error in the domain discretization.  

\begin{figure}[!ht]
    \centering
    \begin{center}\textbf{2-Dimensional Subsurface} \end{center}
    \includegraphics[width=.75\textwidth]{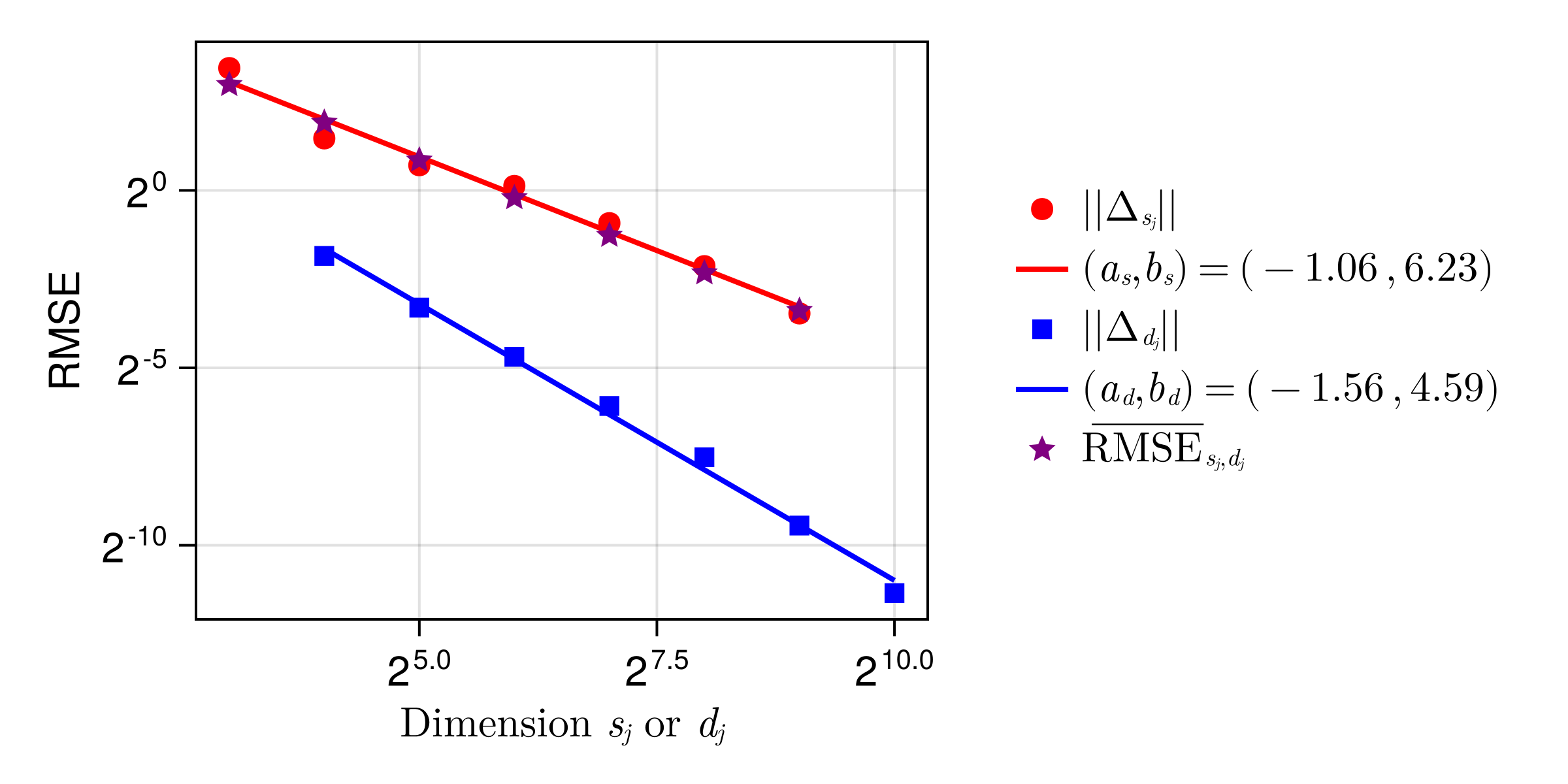}
    \begin{center}\textbf{3-Dimensional Subsurface} \end{center}
    \includegraphics[width=.75\textwidth]{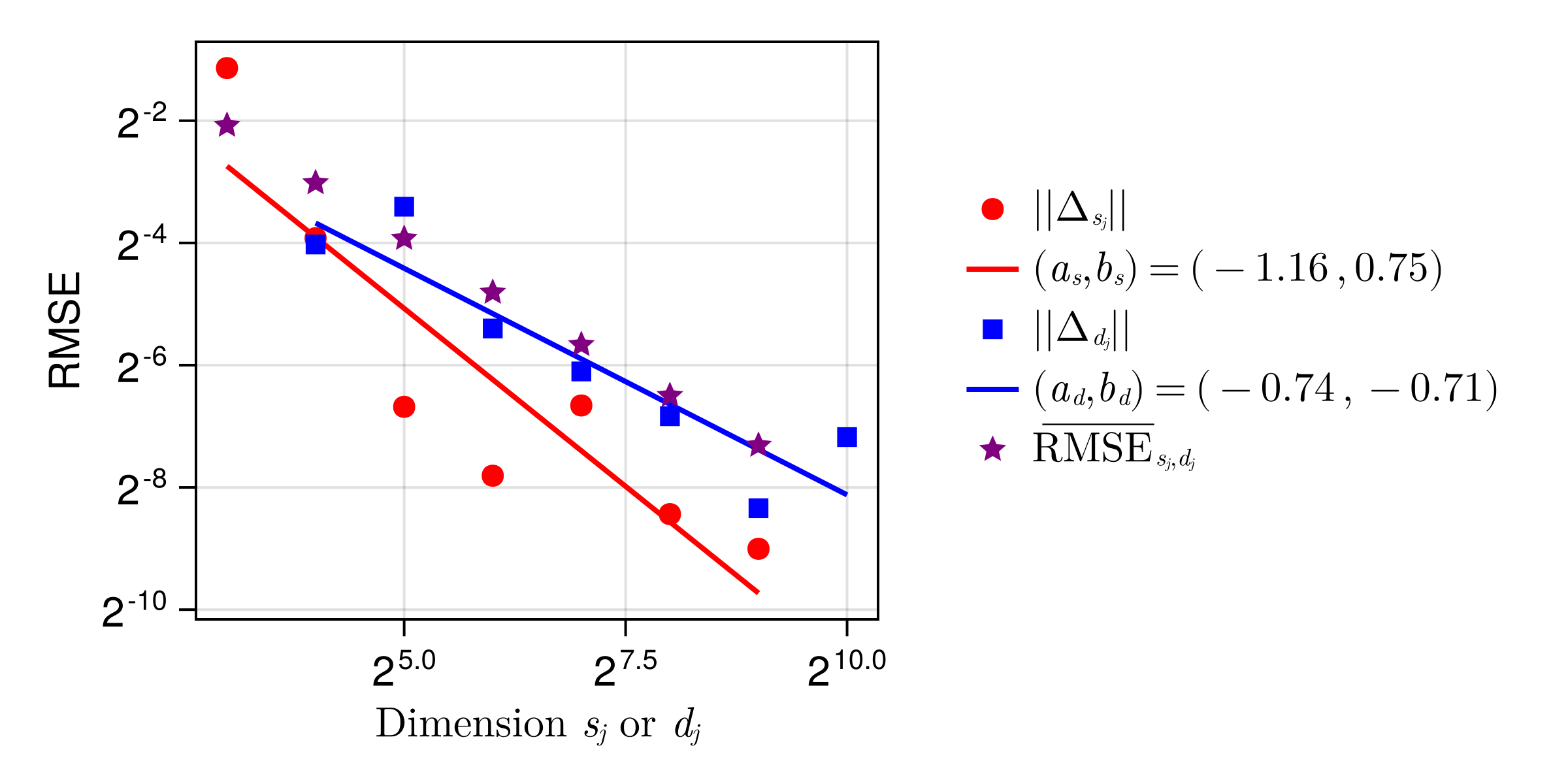}
    \caption{The Karhunen--Loève expansion of the permeability field is approximated by the sum of the first $s_j$ terms. The physical domain $D$ is discretized into a grid with mesh width $1/d_j$ in each dimension. The RMSE of the difference between numerical critical pressures with discretizations $(s_j,d_j/2)$ and $(s_j/2,d_j/2)$ is scatter plotted as $\lVert \Delta_{s_j} \rVert$. The RMSE of the difference between numerical critical pressures with discretizations $(s_j,d_j)$ and $(s_j,d_j/2)$ is scatter plotted as $\lVert \Delta_{d_j} \rVert$.  Simple regression models are fit to these two scatter trends where $a$ and $b$ are the slope and intercept, respectively, of the plotted lines. These models are extrapolated through an infinite telescoping sum to derive the approximate upper bound on the RMSE of the difference between the numerical critical pressure with discretization $(s_j,d_j)$ and the target critical pressure.}
    \label{fig:convergence_GP4Darcy}
\end{figure}

We now approximate a lower bound on $\zeta_{s,d} = \zeta_{s_T,d_T}$. Recall that $\zeta_{s_T,d_T}$ is MSE between analytic solutions and the numerical solutions at target fidelity $(s_T,d_T)$. Practically speaking, we expect this to be at  least as large as the MSE between the solutions at the maximum fidelity $(s_M,d_M)$ and target fidelity $(s_T,d_T)$, i.e., 
\begin{align*}
    \sqrt{\zeta_{s_T,d_T}} &= \sqrt{\mathbb{E}_{\left(R,\boldsymbol{Z}^{s_M}\right)} \left[H^c(R,G)-H^c_{s_M,d_M}(R,\boldsymbol{Z}^{s_M})\right]^2} \\
    &\gtrapprox \sqrt{\mathbb{E}_{\left(R,\boldsymbol{Z}^{s_M}\right)} \left[H^c_{s_M,d_M}(R,\boldsymbol{Z}^{s_M})-H^c_{s_T,d_T}(R,\boldsymbol{Z}^{s_T})\right]^2} \\
    &=: \underline{\mathrm{RMSE}}_{s_T,d_T}
\end{align*}
Similar to \eqref{eq:RMSEs_approxes}, we approximate this heuristic lower bound by the sample RMSE: 
\begin{equation}
    \underline{\mathrm{RMSE}}_{s_T,d_T} \approx \sqrt{\frac{1}{m} \sum_{i=1}^m \tilde{\Delta}_{M,T,i}^2}.
    \label{eq:noise_var_lower_bound}
\end{equation}
where 
\begin{equation}
    \tilde{\Delta}_{M,T,i} = H^c_{(s_M,d_M)}(R_i,\boldsymbol{Z}_i^{s_M})-H^c_{s_T,d_T}(R_i,\boldsymbol{Z}_i^{s_T}).
    \label{eq:tildeDelta}
\end{equation}

\Subsection{Numerical Experiments} \label{sec:results}

This subsection discusses the numerical experiments used to test our method. We stress that the experiments are meant to be a proof of concepts which is easily generalized to more realistic problems rather than an exhaustive real-life application. We begin by describing implementation specifics such as transformations required to make the fast GP framework compatible with Darcy's problem and estimation techniques for the expected conditional confidence. We then describe two numerical experiments used to evaluate our method. The first is a straightforward application of our trained GP model for real-time analysis of the relationship between extraction rate and confidence in maintaining a low enough pressure at the critical location. The second experiment visualizes the calibration procedure for the GP noise variance and explores the requisite assumption of Gaussian noise. Finally, we emphasize the generality of our algorithm by applying it to the same Darcy problem but with a three-dimensional subsurface.

\Subsubsection{Implementation Considerations}

To enable the reproducibility of our results, we describe the specifics for approximating the confidence $c(r)$ in \eqref{eq:confidence} by the expected conditional confidence $c_n(r)$ in \eqref{eq:expected_conditional_confidence}. Recall that lattice sequences are defined in the unit cube $[0,1]^{1+s}$. Moreover, the posterior mean of a GP surrogate fit with lattice sampling locations, and a matching covariance kernel has a periodic posterior mean. We first transform the extraction rate $r$ and independent Gaussians $\boldsymbol{Z}^s$, which defines the permeability field $G^s$, to the unit cube domain. We then periodize the critical pressure using the baker transform.

To transform our problem to the unit cube $[0,1]^{1+s}$, recall the assumption that the extraction rate no larger than the injection rate, i.e., $r \in [0,w]$. Let $\Phi^{-1}$ denote the inverse distribution function of a standard Gaussian random variable. For any $r_u \in [0,1]$ and $u_1,u_2,\dots \in (0,1)$ we may use \eqref{eq:KL_expansion} to write
\begin{equation}
    M^c(r_u,u_1,u_2,\dots) := H^c\left(wr_u,\exp\left(\sum_{j \geq 1} \sqrt{\lambda_j} \varphi_j(x) \Phi^{-1}(u_j)\right)\right).
\end{equation}
For independent standard uniform random variables $U_1,U_2,\dots \overset{\mathrm{IID}}{\sim} \mathcal{U}[0,1]$ we have 

$$P_G(H^c(r,G) \leq \bar{h}) = P_{(U_1,U_2,\dots)}(M^c(r/w,U_1,U_2,\dots) \leq \bar{h}).$$ 

To periodize $M^c(r_u,u_1,u_2,\dots)$, define the baker transform \cite[Chapter 16]{owen.mc_book_practical} 
\begin{equation}
    b(u) = 1 - 2 \left\lvert u - \frac{1}{2}\right\rvert = \begin{cases} 2u, & 0 \leq u \leq 1/2 \\ 2(1-u) & 1/2 \leq u \leq 1 \end{cases}
\end{equation}
so that 
\begin{equation}
    \mathring{M}^c(r_u,u_1,u_2,\dots) := M^c(b(r_u),b(u_1),b(u_2),\dots).
\end{equation}
Since $b(U) \sim \mathcal{U}[0,1]$ when $U \sim \mathcal{U}[0,1]$, we have 
\begin{equation}
    \begin{aligned}
        c(r) &= P_{(U_1,U_2,\dots)}\left(M^c(r/w,U_1,U_2,\dots) \leq \bar{h}\right) \\
        &= P_{(U_1,U_2,\dots)}\left(\mathring{M}^c(r/(2w),U_1,U_2,\dots) \leq \bar{h}\right).
    \end{aligned}
    \label{eq:periodic_confidence}
\end{equation}

Let $\mathring{M}^c_n(r_u,u_1,\dots,u_s)$ denote the posterior Gaussian process for \patchoverfull $\mathring{M}^c(r_u,u_1,u_2,\dots)$. Substituting  $\mathring{M}^c_n(r_u,u_1,\dots,u_s)$ into \eqref{eq:periodic_confidence} and taking the expectation gives 
\begin{equation}
    \begin{aligned}
        \mathring{c}_n(r) &:= P_{(\mathring{M}^c_n,U_1,\dots,U_s)}(\mathring{M}^c_n(r/(2w),U_1,\dots,U_s) \leq \bar{h}) \\
        &= \mathbb{E}_{(U_1,\dots,U_s)}\left[\Phi\left(\frac{\bar{h} - \mathring{M}_n(r,U_1,\dots,U_s)}{\mathring{\sigma}_n(r,U_1,\dots,U_s)}\right)\right].
    \end{aligned}
    \label{eq:periodic_expected_conditional_confidence}
\end{equation}
Equation \eqref{eq:periodic_confidence} motivates us approximating $c(r)$ by $\mathring{c}_n(r)$. Here the final inequality follows from Fubini's theorem \cite{fubini1907sugli} and $\mathring{m}_n$ and $\mathring{\sigma}_n$ are the GP posterior mean and standard deviation of $\mathring{M}_n^c$. We use the Quasi-Monte Carlo estimate 
\begin{equation}
    \mathring{c}_n(r) \approx \frac{1}{N} \sum_{i=1}^N \Phi\left(\frac{\bar{h} - \mathring{m}_n(r,U_{i1},\dots,U_{is})}{\mathring{\sigma}_n(r,U_{i1},\dots,U_{is})}\right)
    \label{eq:qmc_estimate_expected_conf}
\end{equation}
where $(U_{i1},\dots,U_{is})_{i=1}^N$ are low-discrepancy points, e.g., the first $N$ points of a lattice or digital sequence, see \Cref{sec:qmc} for details. 

\Subsubsection{Two-Dimensional Experiments and Analysis}

The domain of our subsurface is square with side lengths of 200\,m with the injection well, extraction well, and critical location shown in \Cref{fig:numerical_pde.birdseye}. We set the injection rate to 0.031688\,m$^3$/s (equivalent to 1 million metric tons per year [MMT/y]) and test extraction rates between -0.031688\,m$^3/s$ and 0.0\,m$^3$/s. We use a zero-mean log-normal permeability field with a Matérn covariance kernel having a correlation length 50\,m.

Our GP surrogate is fit to numerical experiments with fidelity $(s,d)=(64,128)$, i.e., $64$ were terms kept in the KL expansion and the mesh width for the finite volume solver was $1/128$ in both dimensions. The sequence of fidelities used to find upper and lower bounds for noise variance tuning were $(s_j)_{j=0}^M = (4,8,16,32,64,128,256,512)$ and $(d_j)_{j=0}^M = (8,16,32,64,128,256,512,1024)$ as shown in \Cref{fig:convergence_GP4Darcy} for the two-dimensional subsurface. At each fidelity, the PDE was solved numerically at $m=128$ extraction-permeability pairings, and $n=1024$ solves at fidelity $(s_T,d_T) = (64,128)$ were used to fit the GP. The Quasi-Monte Carlo approximation in \eqref{eq:qmc_estimate_expected_conf} was performed using $N=1024$ randomly shifted lattice points. The CPU time required for each step of experimentation are given in \Cref{tab:CPUtimes}. The condition number of the noisy circulant kernel matrix is $387$. As the condition number is the ratio between the largest and smallest eigenvalues, one may decrease the condition number by raising the lower bound on the GP noise variance.

\begin{table}[ht!]
    \caption{CPU time for different stages of the proposed method on the Darcy problem in two and three dimensions. First, the KL expansion is performed on a fine grid. Then the finite volume method from \texttt{DPFEHM} \cite{DPFEHM.jl} is used solve Darcy's equation on common permeability realizations at different fidelities. The Gaussian process (GP) regression model is then fit at the target fidelity, and GP inference may be performed at a fraction of the cost compared to a traditional solver. Note that GP fitting includes the finite volume (FV) solves. The fidelity parameters are the number of samples $m$, the number of KL terms / number of input dimensions to the GP $s$, and the domain discretization fidelity $d$. For the two-dimensional Darcy problem, the fidelity $d$ indicates a $d \times d$ computational mesh while for the three-dimensional Darcy problem the fidelity $d$ indicates a $d \times d \times 9$ computational mesh. Note that the cost of KL is independent of the number of samples, so these entries are left blank.}
    \centering
    \resizebox{\textwidth}{!}{
    \begin{tabular}{l | l | r l l l | l | r l l l}
        \textbf{Step} & \multicolumn{5}{|c|}{\textbf{Darcy 2D}} & \multicolumn{5}{c}{\textbf{Darcy 3D}} \\ 
        & \multicolumn{2}{c}{CPU time [sec]} & $m$ & $s$ & $d$ & \multicolumn{2}{c}{CPU time [sec]} & $m$ & $s$ & $d$ \\ 
        \hline 
        \hline 
        KL & \multicolumn{2}{l}{$229$} & & $512$ & $1024$ & \multicolumn{2}{l}{$2578$} &  & $512$ & $1024$ \\
        \hline 
        \multirow{11}{*}{FV solves} & \multirow{11}{*}{$7857$} & $47$ & $128$ & $4$ & $8$ & \multirow{11}{*}{$46657$} & $320$ & $128$ & $4$ & $8$ \\
        & & $76$  & $128$ & $8$ & $8$     & & $341$  & $128$ & $8$ & $8$ \\
        & & $82$  & $128$ & $8$ & $16$    & & $340$  & $128$ & $8$ & $16$ \\
        & & $137$ & $128$ & $16$ & $16$   & & $384$  & $128$ & $16$ & $16$ \\ 
        & & $220$ & $128$ & $16$ & $32$   & & $383$  & $128$ & $16$ & $32$ \\
        & & $263$ & $128$ & $32$ & $32$   & & $468$  & $128$ & $32$ & $32$ \\ 
        & & $317$ & $128$ & $32$ & $64$   & & $532$  & $128$ & $32$ & $64$ \\
        & & $336$ & $128$ & $64$ & $64$   & & $663$  & $128$ & $64$ & $64$ \\ 
        & & $340$ & $128$ & $64$ & $128$  & & $745$  & $128$ & $64$ & $128$ \\ 
        & & $375$ & $128$ & $128$ & $128$ & & $1077$ & $128$ & $128$ & $128$ \\
        & & $389$ & $128$ & $128$ & $256$ & & $1645$ & $128$ & $128$ & $256$ \\
        & & $464$ & $128$ & $256$ & $256$ & & $2309$ & $128$ & $256$ & $256$ \\ 
        & & $527$ & $128$ & $256$ & $512$ & & $5502$ & $128$ & $256$ & $512$ \\ 
        & & $671$ & $128$ & $512$ & $512$ & & $6849$ &  $128$ & $512$ & $512$ \\
        & & $908$ & $128$ & $512$ & $1024$& & $19182$&  $128$ & $512$ & $1024$\\
        & & $2704$& $1024$& $64$  & $128$  & & $5916$& $1024$ & $64$ & $128$ \\ 
        \hline
        GP fit  & \multicolumn{2}{l}{$3.5$} & $1024$ & $64$ & $128$ & \multicolumn{2}{l}{$2.0$} & $1024$ & $64$ & $128$ \\
        \hline 
        GP predict & \multicolumn{2}{l}{$3.0$} & $1024$ & $64$ & $128$ & \multicolumn{2}{l}{$2.8$} & $1024$ & $64$ & $128$ \\
        \hline 
    \end{tabular}
    }
    \label{tab:CPUtimes}
\end{table}

In \Cref{fig:confidences_heatmap} for the two-dimensional subsurface, we plot the approximate expected conditional confidence in \eqref{eq:qmc_estimate_expected_conf} for a range of extraction rates $r$ and pressure thresholds $\bar{h}$. While the surrogate is not constrained to be monotonically increasing in both extraction rate $r$ and threshold $\bar{h}$, the expected confidence appears to have this qualitative behavior. This reassures us that our surrogate captures the physics in the model. \Cref{fig:confidences_lineplot} may be viewed as slices of the left plot of \Cref{fig:confidences_heatmap} at fixed pressure threshold $\bar{h}$. For a fixed pressure threshold $\bar{h}$, numerous methods exist to use the GP model to find an extraction rate which yields a desired confidence. We emphasize this can all be done in real-time using only evaluations of the GP surrogate. 

\begin{figure}[!ht]
    \centering
    \begin{minipage}{.44\textwidth}
        \begin{center}\textbf{2-Dimensional Subsurface} \end{center}
        \includegraphics[trim={10cm 0 0 0},clip,width=1.\textwidth]{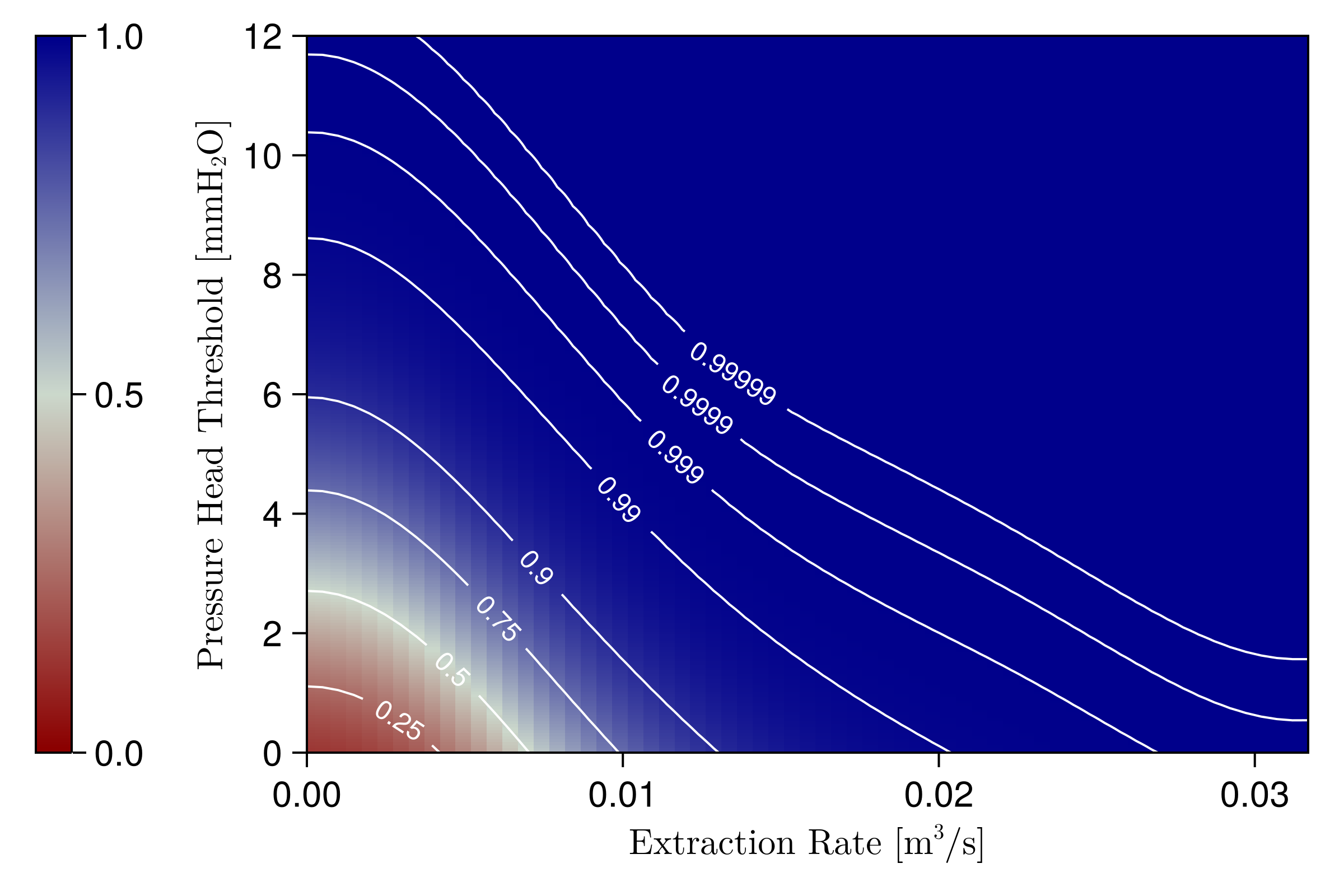}
    \end{minipage}
    \begin{minipage}{.50\textwidth}
        \begin{center}\textbf{3-Dimensional Subsurface} \end{center}
        \includegraphics[width=1.\textwidth]{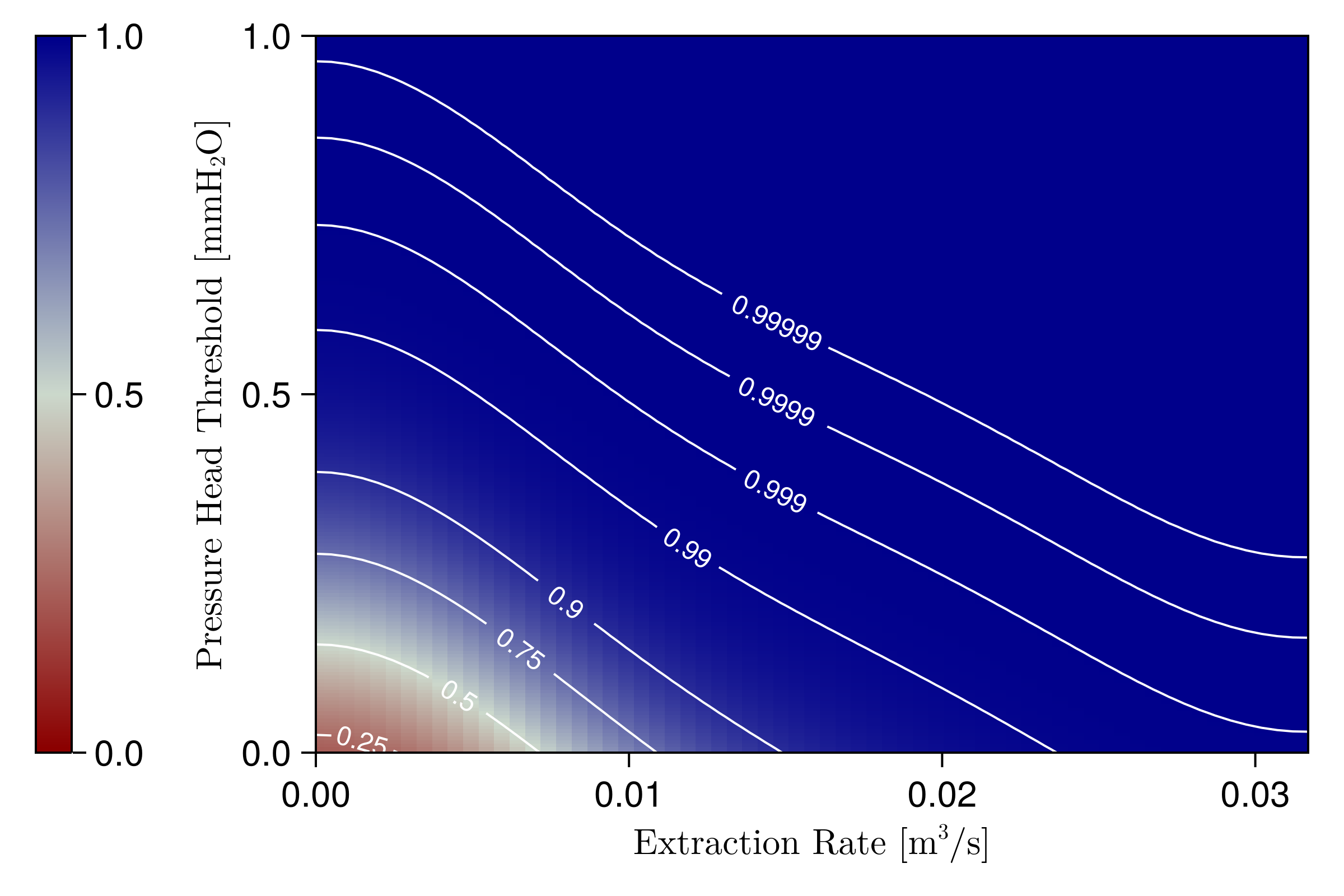}
    \end{minipage}
    \caption{Approximate expected posterior confidence from \eqref{eq:qmc_estimate_expected_conf} by extraction rate $r$ and pressure threshold $\bar{h}$.}
    \label{fig:confidences_heatmap}
\end{figure}

We now analyze our assumption of Gaussian noise for the GP model and our method of optimization. \Cref{fig:GP_noise} illustrates the critical distributions considered for noise variance fitting. A frequency plot of the errors between the maximum and target fidelities is plotted, i.e., a frequency plot of $\tilde{\Delta}_{M,T,i}$ from \eqref{eq:tildeDelta}. The sample MSE of these errors is used an approximate lower bound on the noise variance for optimization. The initial noise variance for optimization was set to the upper bound approximated the decay of differences in numerical solves at a sequence of increasing fidelities, see \eqref{eq:bar_rmse_upper_bound}. Details on both these approximate bounds will be provided later in the text. 

For the two-dimensional subsurface, \Cref{fig:GP_noise} shows the fitted GP noise variance essentially matches the lower bound. In fact, if we do not lower bound the noise variance, our optimization to maximize the marginal likelihood will choose an optimal noise variance orders of magnitude smaller than the lower bound. In practice, we observed the GP based confidence estimate is robust to the choice of observation noise. 

It may also be observed in the left plot of \Cref{fig:GP_noise} that the distribution of errors between the target and maximum fidelity numerical solutions does not appear Gaussian but instead appears to have heavier tails. The distribution of these errors should be close to the distribution of errors between the target fidelity numerical solutions and analytic solutions. This later error is what is modeled by our GP noise. The use of a GP necessitates the assumption of Gaussian noise, but our results suggest this assumption may not hold in practice.

\begin{figure}[!ht]
    \centering
    \begin{minipage}{.49\textwidth}
        \begin{center}\textbf{2-Dimensional Subsurface} \end{center}
        \includegraphics[width=1.\textwidth]{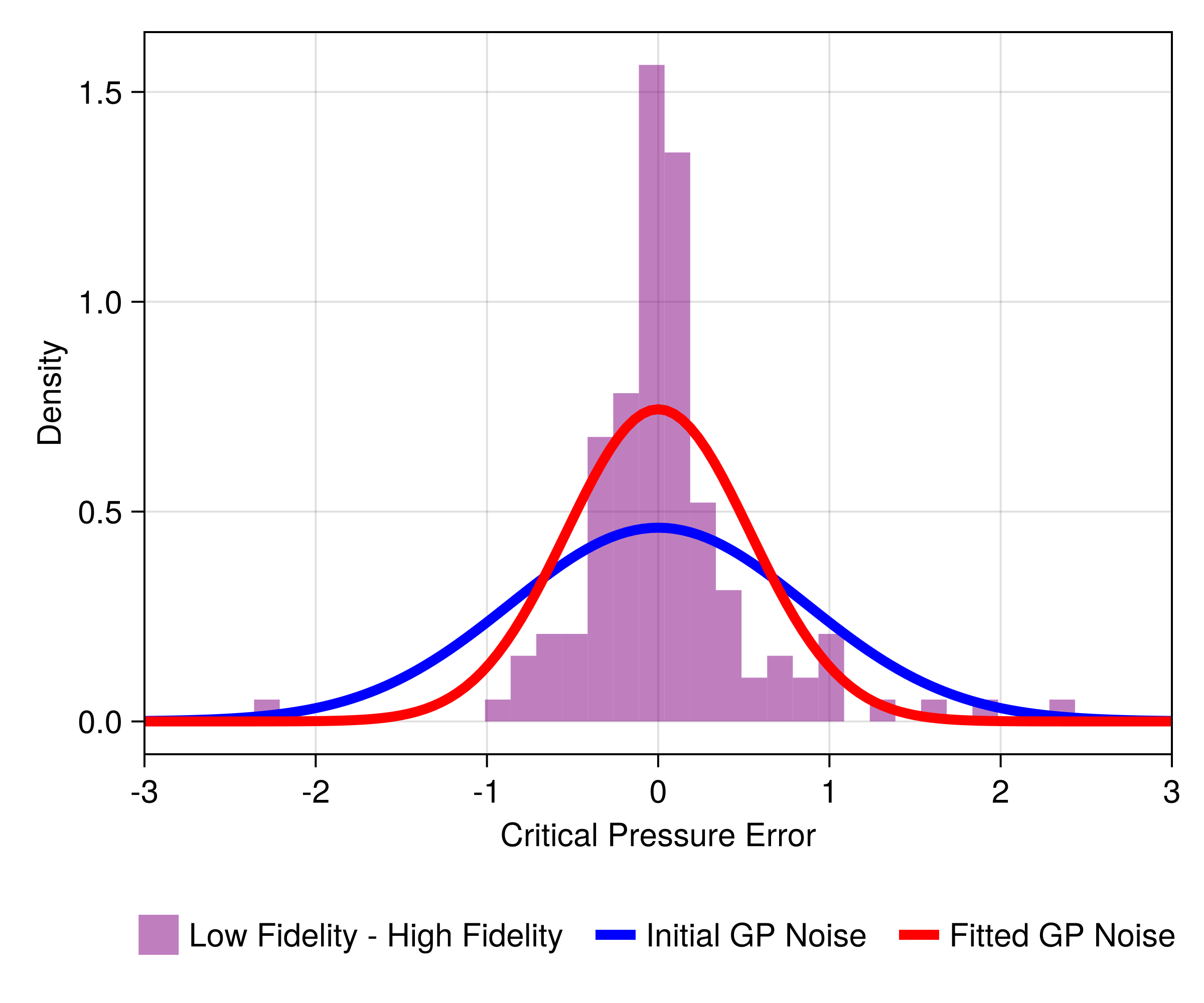}
    \end{minipage}
    \begin{minipage}{.49\textwidth}
        \begin{center}\textbf{3-Dimensional Subsurface} \end{center}
        \includegraphics[width=1.\textwidth]{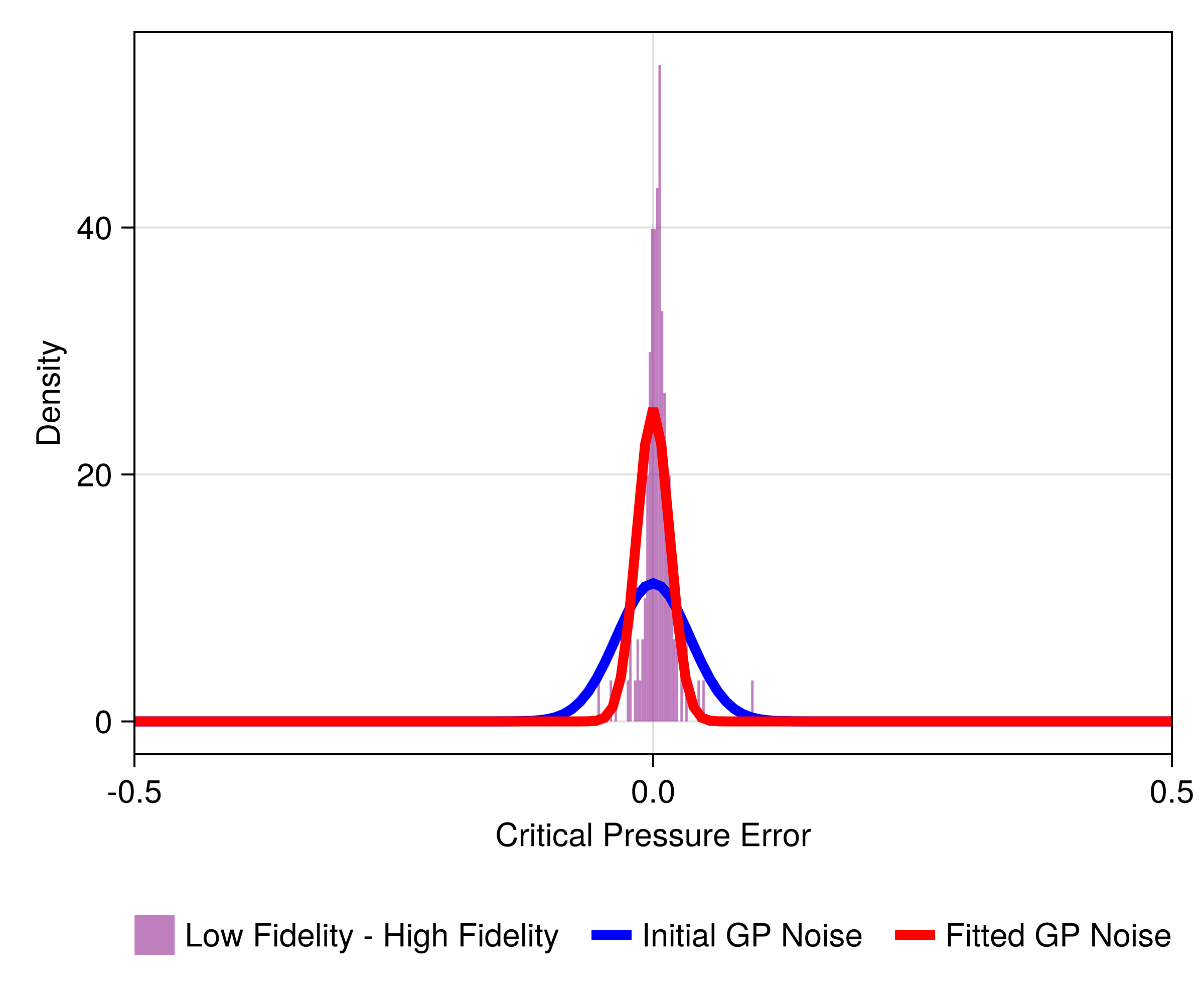}
    \end{minipage}
    \caption{Analysis of optimization for Gaussian process (GP) noise. The histogram shows frequency of errors between target and maximum fidelity observations, i.e., frequency of $\tilde{\Delta}_{M,T,i}$ from \eqref{eq:tildeDelta}. The mean square error of these $\tilde{\Delta}_{M,T,i}$ is used as a lower bound when optimizing the Gaussian process's noise variance. The initial noise variance for optimization, corresponding to the blue curve, is set to an upper bound approximated from differences in numerically solves of the PDE at a sequence of increasing fidelities, see \eqref{eq:bar_rmse_upper_bound}. The noise distribution after optimization is the red curve, which is indistinguishable from the lower bound distribution (not plotted). Our use of GP modeling necessitates the assumption of Gaussian noise which prohibits a better fit.}
    \label{fig:GP_noise}
\end{figure}

\Subsubsection{Three-Dimensional Experiments and Analysis}

To emphasize the generality of our method, we applied our algorithm to the Darcy flow problem to a three-dimensional domain. The experimental setup is the same as in the two-dimensional case, except now we set the subsurface to have a depth of 20\,m while the injection, extraction, and critical locations are all set at a depth of 10\,m. Also, the mesh grid for the finite volume solver had $(d_j+1) \times (d_j+1) \times 9$ mesh points in each dimension so the mesh width in each dimension at fidelity $j$ is $1/d_j,1/d_j,1/8$. The CPU time required for each step of experimentation are given in \Cref{tab:CPUtimes}. The condition number of the noisy circulant kernel matrix is $613$.

\Cref{fig:convergence_GP4Darcy} shows the convergence of telescoping sums used to derive an upper bound on the noise variance for the three-dimensional subsurface. The coefficients of determination are $0.81$ and $0.83$ for the $s$ and $d$ trends respectively. These are lower than the $0.98$ and $0.99$ respective values for Darcy's problem with a two-dimensional subsurface, but still large enough to justify the linear fits.

\Cref{fig:confidences_heatmap} shows, for the three-dimensional subsurface, the confidence in maintaining a low enough pressure at the critical location as a function of both pressure threshold and extraction rate. The pressure at the critical location is generally much lower in this three-dimensional setup than in the two-dimensional one. For instance, at an extraction rate of 0\,m${}^3$/s the two-dimensional setup gives a confidence of around 25\% that the pressure at the critical location is below 1\,mmH${}_2$O, while the three-dimensional setup has almost a 100\% confidence for the same extraction rate and threshold. Again, the monotonicity in both extraction rate and pressure threshold computed from the surrogate match our physical intuition for this three-dimensional subsurface problem.

Finally, \Cref{fig:GP_noise} shows the noise calibration process for the \patchoverfull three-dimensional subsurface, specifically the initial upper bound and final optimized bound. We again observe the heavier tails in the distribution of differences between target and maximum fidelities when compared to the assumed Gaussian distribution. There also appears to be a slight skew to the right in the distribution of these differences, indicating a potential bias in low-fidelity approximation. Again, we defer remedies to future work.

\Section{Solving Radiative Transfer Equations with Deep Operator Networks} \label{sec:RTE_DeepONet}

\begin{quotation}
    This section follows \cite{sorokin.RTE_DeepONet}, a publication I worked on during my 2024 Scientific Machine Learning Researcher appointment at FM (Factory Mutual Insurance Company) in collaboration with Xiaoyi Lu and Yi Wang. 
\end{quotation}

Radiative transfer is often the dominant mode of heat transfer in fires, and solving the governing radiative transfer equation (RTE) in CFD fire simulations is computationally intensive. This project develops a new versatile toolkit for training neural surrogates to solve various RTEs across different geometries and boundary conditions. We generalize previous work in the area to include unknown boundary conditions and to perform principal component analysis (PCA) for dimension reduction in this context. This enables efficient training of high-dimensional neural surrogate solvers for a large class of RTEs. The mesh free nature of these surrogates enables them to overcome the ray effect suffered by traditional solvers. Our results demonstrate that neural surrogates can provide fast and accurate radiation predictions for practical problems important to fire safety research. 

Specifically, the present study extends the work of \cite{lu2024surrogate} by developing DeepONet (deep operator network) surrogate RTE solvers for multidimensional radiation problems. We support RTEs with complex boundary conditions and use PCA to encode input functions into low-dimensional representations, resulting in more compact architectures. The new methodology is implemented in a flexible and extensible sciML toolkit for training neural surrogate RTE solvers using physics-informed, data-driven, or hybrid loss functions. We utilize the developed library to solve radiation problems in fire safety research.

\Subsection{Methods}

In CFD simulations, radiation is incorporated into the energy equation as a volumetric radiative heat loss term. This requires solving the radiative intensity field $I(\bx,\bs)$ at various locations $\bx$ and directions $\bs$ from the radiative transport equation (RTE), given by
\begin{equation}
    \calR_\mathrm{RTE}(I,(\bx,\bs)) = \bs\cdot\nabla I(\bx,\bs) +\kappa(\bx) I(\bx,\bs) - \kappa(\bx) I_{b}(\bx) = 0,
\label{eq:RTE_GE}
\end{equation}
subject to the boundary condition imposed for radiation rays emitting from bounding surfaces to the radiatively participating media within the domain
\begin{equation}
    \calR_\mathrm{BC}(I,(\bx,\bs)) = \epsilon(\bx) I_b(\bx) + \frac{\rho^{d}(\bx)}{\pi} \int I(\bx,\bs') \lvert \bn(\bx).\bs' \rvert d \Omega + \rho^{s}(\bx) I(\bx,\bs_s) - I(\bx,\bs) = 0.
    \label{eq:RTE_BC}
\end{equation}
Here $\bn$ denotes the boundary surface normal, $\bs_s =  \bs - 2 \bn (\bn . \bs)$ is the direction of a specular reflection, and $d\Omega$ is the solid angle on the unit sphere $\Omega$.
 
The input functions of the RTE are the absorption coefficient $u_1 (= -\kappa) \in \mathcal{U}_1$, the black body emissive power $u_2 (= I_b) \in \mathcal{U}_2$, the surface emissivity $u_3 (= \epsilon) \in \mathcal{U}_3$, the diffusive reflection coefficient $u_4 (= \rho^d) \in \mathcal{U}_4$, and the specular reflection coefficient $u_5 (= \rho^s) \in \mathcal{U}_5$. Here $I_b(\bx) = \sigma/\pi T^4(\bx)$ where $T$ is the temperature and $\sigma$ is the Stefan–Boltzmann constant. Boundary operators are generalized and parameterized, whereas prior work \cite{lu2024surrogate} assumed black walls at fixed temperatures. We use a DeepONet to approximate the RTE solution operator $G: \mathcal{U}_1\times \mathcal{U}_2\times \mathcal{U}_3\times \mathcal{U}_4\times \mathcal{U}_5 \to \mathcal{I}$ which maps input functions to the solution $I(\bx,\bs)\in \mathcal{I}$ as
\begin{equation}
G(u_1, \dots, u_5) = I. \label{eq:solution_operator}
\end{equation}

\begin{figure}[!ht]
    \centering
    \includegraphics[width=0.8\linewidth, clip=True]{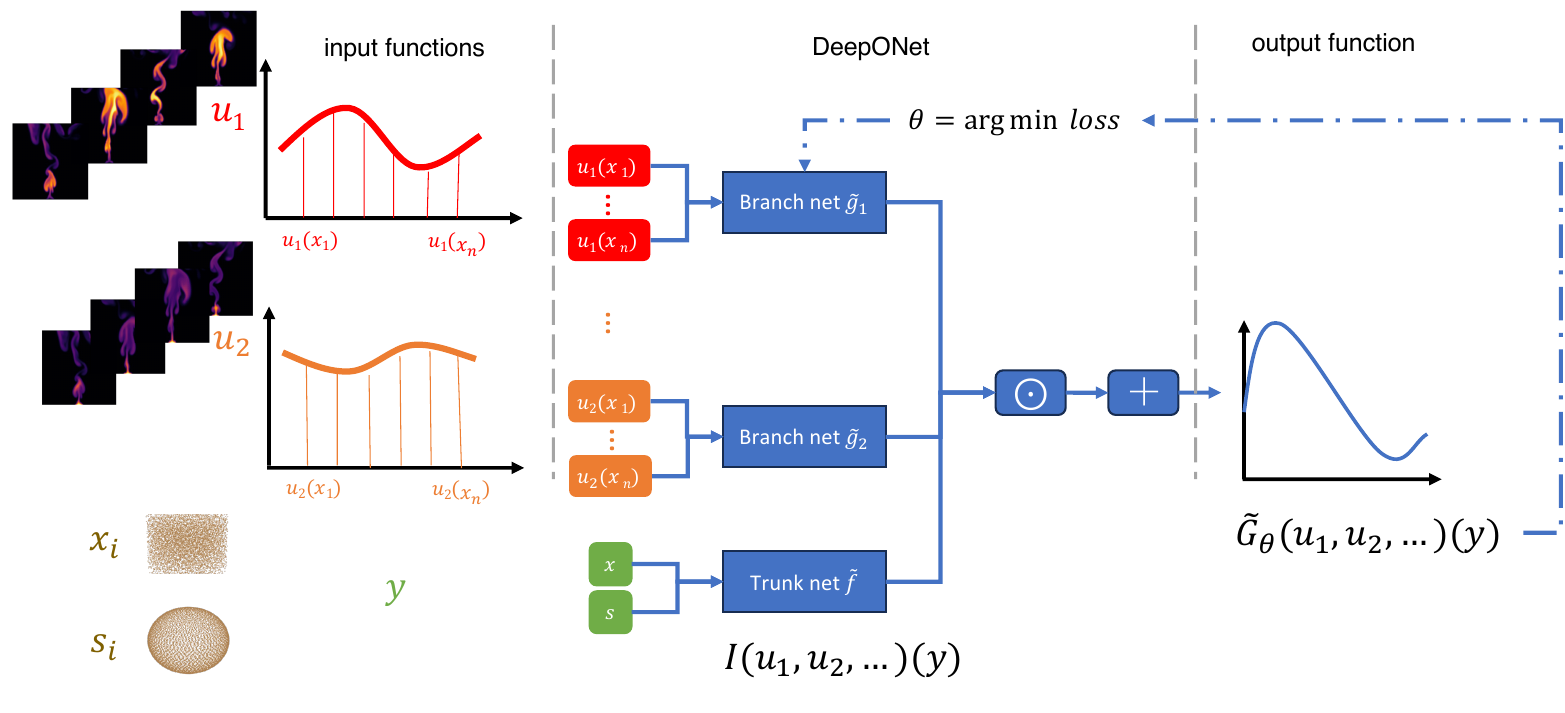}
    \caption{\captionsize The architecture of the deep operator network (DeepONet) for learning the mapping between function spaces.  We only show two of the five random coefficients for simplicity. }
    \label{fig:DeepONet_sketch}
\end{figure}

DeepONet comprises two main components: the branch net and the trunk net (see \cite{Lu21Nature}), as illustrated in \Cref{fig:DeepONet_sketch}. As shown in the leftmost panel, the branch nets $g_1,~\dots~g_5$ encode input functions $\bu_1(\bx), \dots \bu_5(\bx)$ at a fixed number of discrete sensor locations $\bx_i,~i=1,\dots,n$. The trunk net $f$ encodes the location and direction $\by=(\bx,\bs)$ where the output function $G(\bu_1,\dots \bu_5)(\by)$ is evaluated within a given domain. The multiple-input DeepONet prediction is 
\begin{equation}
    \tG_\Theta(\bu^1,\dots,\bu^5)(\by) = \calF_{\Theta_0}(\by) \odot \calG_{\Theta_1}(\bu^1) \odot \cdots \odot \calG_{\Theta_5}(\bu^5) + b
\end{equation}
where $\odot$ is the Hadamard product of the $H$ outputs from each MLP, $b$ is a trainable bias parameter, and $\Theta = \cup_{i=0}^R \Theta_i$ aggregates all networks parameters. 

The data for training a DeepONet is the Cartesian product of random coefficients $(\bu_k^1,\dots,\bu_k^5)_{k=1}^{N_\mathrm{RC}}$ and collocation points $(\bx_i,\bs_i)_{i=1}^{N_\mathrm{C}}$. The collocation points are input to the trunk net while the random coefficients are sampled at \emph{sensor locations} to get \emph{sensor values} which are input to the branch nets. Our implementation supports both aligned datasets, where the sensor locations match the collocation points, and unaligned datasets. Moreover, we lazily construct batches of the Cartesian product dataset on a GPU to significantly reduce the time spent loading data.

We support PCA to reduce the number of sensor values. Auto-encoders for sciML have also shown success for dimensionality reduction \cite{kontolati2023learning}, but they require additional training for the encoder and are not amenable to distributed evaluation. PCA instead can be done via a simple matrix-vector multiplication in a distributed manner when integrated in parallel CFD solvers. 

The weighted hybrid loss function combines terms for the RTE \eqref{eq:RTE_GE}, boundary condition \eqref{eq:RTE_BC}, and data from a traditional solver into 
\begin{equation}
    \calL(\Theta) = \omega_\mathrm{RTE} \lVert \calR_\mathrm{RTE}(\tG_\Theta, \calD_\mathrm{RTE}) \rVert + \omega_\mathrm{BC}\lVert \calR_\mathrm{BC}(\tG_\Theta,\calD_\mathrm{BC}) \rVert + \omega_\mathrm{data}\lVert \calR_\mathrm{data}(\tG_\Theta,\calD_\mathrm{data})\rVert.
    \label{eq:loss}
\end{equation}
Here $\omega$ are the weights, $\calR$ are the residuals, and $\calD$ are subsets of the sensor-collocation data with $\calD_\mathrm{data}$ also containing solution data from the reference solver. The gradient in \eqref{eq:RTE_GE} is evaluated exactly using automatic differentiation, while the integral in \eqref{eq:RTE_BC} is approximated using either Gauss--Legendre quadrature or Quasi-Monte Carlo, see \Cref{sec:qmc}. These cubature routines are also used to infer the incident radiation, radiative heat flux, or radiative heat loss. 
    
The core functionalities of the developed sciML library are implemented into two abstract classes. The first constructs the sensor-collocation datasets $\calD$ while the second defines the loss function $\calL(\Theta)$ by implementing both the residuals $\calR$ and the DeepONet $\tG_\Theta$. The complete code will be made available in a future publication.

\Subsection{Numerical Experiments}

Numerical experiments are conducted using the developed sciML library \\ \texttt{RTENet}, showcasing trained neural surrogates solving RTEs with complex boundary conditions and in practical settings of fire radiation transfer. The selected test problems include a cylinder enclosure problem from \cite[Fig. 5]{chui1992prediction}, the four special cases  considered in \cite[Section 4.3]{ge2016development}, and the small pool fire case from \texttt{FireFOAM/tutorials}. 

DeepONet training hyperparameters are summarized in \Cref{table:experimental_setup} where $L_T$ and $L_B$ are the number of hidden layers in trunk and branch nets respectively, $\gamma$ is the learning rate, $B$ is the equal batch size across datasets, and $E$ is the number of epochs. Note that the number of network parameters $\lvert \Theta \rvert$ is also a function of the number of sensor locations which is described in the following paragraphs. All training uses the AMSGrad variant of the Adam optimizer. 

\begin{table}[!ht]
    \caption{DeepONet training hyperparameters.}
    \centering
    \begin{tabular}{r | c c c | c c c c | c c c}
        & \multicolumn{3}{c|}{dataset} & \multicolumn{4}{c|}{DeepONet} & \multicolumn{3}{c}{training} \\
        problem & coefficients & $N_\mathrm{RC}$ & $N_\mathrm{C}$ & $\lvert \Theta \rvert$ & $L_T$ & $L_B$ & $H$ & $\gamma$ & $B$ & $E$ \\
        \hline 
        \cite{chui1992prediction} & $(\kappa)$ & $200$ & $12\mathrm{K}$ & $166\mathrm{K}$ & $5$ & $4$ & $128$ & $10^{-3}$ & $2\mathrm{K}$ & $500$ \\
        \cite{ge2016development} & $(\epsilon,\rho^d,\rho^s)$ & $259$ & $2\mathrm{K}$ & $150\mathrm{K}$ & $4$ & $3$ & $128$ & $10^{-3}$ & $1\mathrm{K}$ & $700$ \\
        pool fire & $(\kappa,I_b)$ & $4000$ & $365\mathrm{K}$ & $3.1\mathrm{M}$ & $5$ & $4$ & $512$ & $10^{-4}$ & $33\mathrm{K}$ & $12$ \\
        \hline 
    \end{tabular}
    \label{table:experimental_setup}
\end{table}

\begin{figure}[!ht]
    \centering 
    \includegraphics[width=1\linewidth, clip=True]{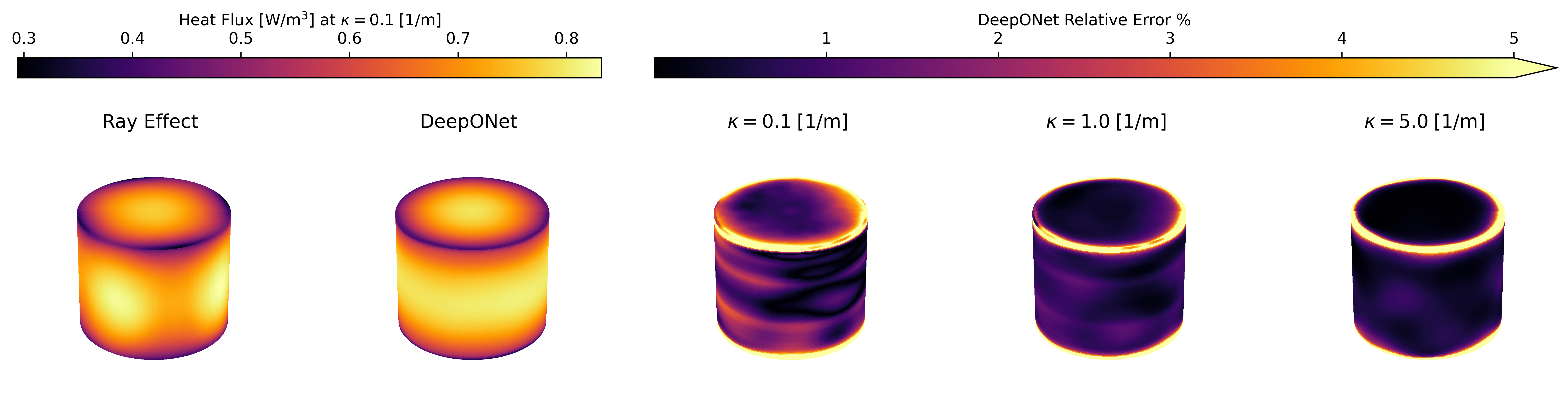}
    \caption{Deep operator networks (DeepONets) overcome the ray effect and yield small point-wise relative errors.}
    \label{fig:cylinder_deeponet_gray}
\end{figure}

A first test solves gas radiation transfer in a cylindrical enclosure. Boundaries are cold black walls and gas temperature is constant. The DeepONet was trained with $N_\mathrm{RC} = 200$ random realizations of constant $\kappa$ between $\kappa=0.01 \, [\mathrm{m}^{-1}]$ and $\kappa=6 \, [\mathrm{m}^{-1}]$ and no traditional solver data was provided, i.e., this was a purely physics-informed training. $L_2$ relative errors of $1.53\%$, $1.17\%$, and $1.88\%$ were attained for $\kappa=0.1$, $\kappa=1$, and $\kappa=5$ respectively as compared to the analytic gray gas solution. The point-wise relative errors are shown in \Cref{fig:cylinder_deeponet_gray}. The finite volume discrete ordinate method (FVDOM) solver in FireFOAM suffers from the ray effect when the angular discretization is coarse; this is quickly mitigated by the mesh-free nature of DeepONet training and inference. 

\begin{figure}[!ht]
    \centering
	\includegraphics[width=1\linewidth]{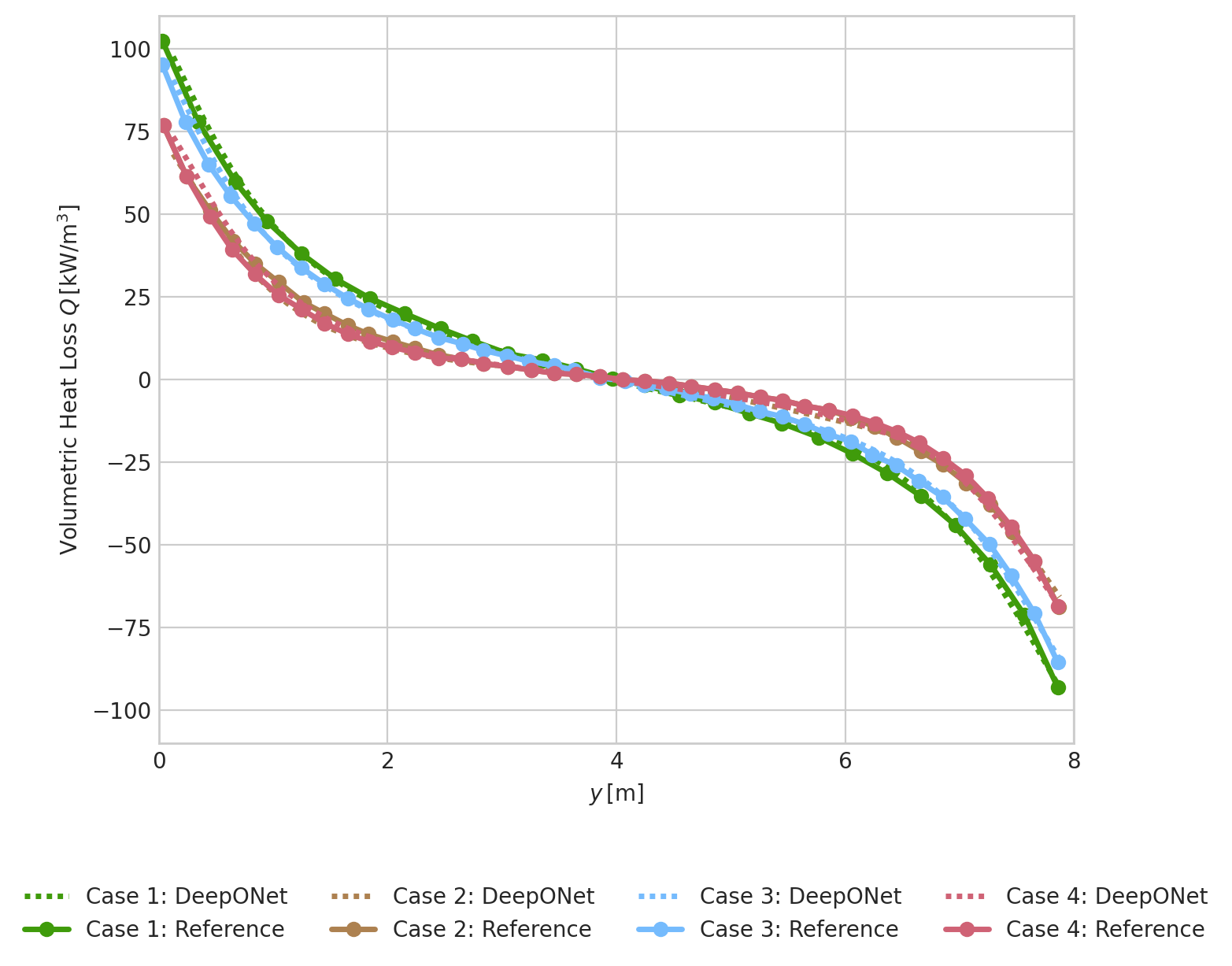}
	\caption{Deep operator networks (DeepONets) predicted heat loss compared to the higher-order spherical harmonics $P_7$ models in \cite{ge2016development}.}
	\label{fig:Ge_bc_rect_Q}
\end{figure}

A second test problem in \cite[Section 4.3]{ge2016development} considers gas radiation transfer in a rectangular domain with mixed boundary conditions. The gas temperature is prescribed, and the absorption coefficient is constant $\kappa = 0.5 \, [\mathrm{m}^{-1}]$, while $\epsilon,\rho^d,\rho^s$ are unknown piecewise constant functions satisfying $\epsilon + \rho^d + \rho^s = 1$. The top and bottom boundary have fixed $\epsilon=1$ while the left and right boundary are random constants $\epsilon_{\mathrm{LR}}$, $\rho_{\mathrm{LR}}^d$, and $\rho_{\mathrm{LR}}^s$. The DeepONet is trained on $N_\mathrm{RC} = 259$ random realizations of $\epsilon_{\mathrm{LR}}$, $\rho_{\mathrm{LR}}^d$, and $\rho_{\mathrm{LR}}^s$, and tested with the parameters specified in \cite[Section 4.3]{ge2016development}. \Cref{fig:Ge_bc_rect_Q} compares the predicted volumetric heat loss along a vertical slice of the domain between DeepONet and the $P_7$ model from \cite{ge2016development}.

The third problem trains DeepONet on $N_{\mathrm{RC}} = 4\mathrm{K}$ FireFOAM pool fire simulation snapshots, which use $151 \times 151$ cells and 16 discrete ordinates. The high-dimensional datasets are encoded into lower-dimensional latent spaces using PCA, reducing the branch net input size from $23\mathrm{K}$ to 500. \Cref{fig:poolfire} compares predicted incident radiation results on withheld testing realizations. Notably, the DeepONet approximation does not suffer from the ray effect observed in the discrete ordinate solver.

\Section{Operator Learning at Machine Precision} \label{sec:CHONKNORIS}

\begin{quotation}
    This section summarizes \cite{bacho.CHONKNORIS}, a paper I worked on during my 2025 DOE SCGSR (US Department of Energy Office of Science Graduate Student Research) Program fellowship under the guidance of Fred J. Hickernell and Pieterjan Robbe along with collaborators from the California Institute of Technology, including Houman Owhadi, Aras Bacho, Xianjin Yang, Th\'eo Bourdais, Edoardo Calvello, and Matthieu Darcy, as well as collaborators from the University of Washington, including Alex Hsu and Bamdad Hosseini.
\end{quotation}

Here we are interested in using scientific machine learning (sciML) for machine precision recovery of nonlinear PDEs with unknown/random coefficients, including both forward and inverse problems. For forward problems we want to recover the PDE solution from unknown coefficients such as a random forcing term in a nonlinear elliptic PDE, a random initial condition in Burgers' equation, or a random permeability field in a nonlinear Darcy flow equation. For inverse problems, we want to find a mapping from the (partially) observed PDE solution back to the unknown coefficients. For example, we will consider the seismic imaging problem where we model the mapping from acoustic data gather on the surface to the velocity map characterizing subsurface terrain. 

Scientific machine learning (sciML) has gained increasing interest in recent years as an efficient technology for modeling PDEs, see the brief literature review in \Cref{sec:background_sciml} or the more comprehensive review in the paper \cite{bacho.CHONKNORIS}. Operator learning sciML methods, such as Deep Operator Networks (DeepONets), Fourier Neural Operators (FNOs), and Gaussian Processes (GPs), have proven to be flexible tools for approximating mappings across both forward and inverse problems. However, such methods typically cannot achieve approximation errors below around 0.1\% ($L_2$ relative error) even after scaling up to large models. 

The goal of this section is to develop novel sciML models which are capable of achieving machine precision error in recovering discretized numerical solutions to nonlinear PDEs. The proposed method, which we call the Cholesky--Newton--Kantorovich Neural Operator Residual Iterative System, or CHONKNORIS for short, is applicable to PDEs with unknown coefficients spanning both forward and inverse problems. CHONKNORIS enables the computational efficiency and scalability of sciML operator learning models to be applied to accuracy-critical applications across scientific domains.

Solving nonlinear PDEs is often done by iteratively linearizing the system around an approximate solution and then perturbing the approximation in the direction of the linear system solution \cite{quarteroni1994numerical,zeidler2013nonlinear}. After writing the nonlinear PDE solution as the fixed point of a nonlinear least squares problem, one may use techniques like the Gauss--Newton algorithm or the (dampened) Levenberg--Marquardt algorithm to find such a solution. We use a function space generalization of these methods called the (dampened) Newton--Kantorovich (NK) iteration. The NK method iteratively perturbs an approximate PDE solution in a direction proportional to some operator acting on the PDE residual at the current approximation. Specifically, this operator is the (dampened) pseudo-inverse of the PDE residual's derivative at the current approximation. 

We use sciML to learn the mapping from an approximate PDE solution and random coefficients to the operator acting on the PDE residual. We emphasize that, while existing operator learning methods directly approximate the mapping from random coefficients to PDE solutions and typically achieve no better than 0.1\% error, our method approximates an operator-valued mapping which, when applied to PDE residuals, enables machine precision recovery of PDE solutions. Our method necessitates significantly larger sciML models. For example, in the discretized setting with $N$ collocation points, existing operator learning methods approximate a mapping from $\mathcal{O}(N)$ coefficient inputs to $\mathcal{O}(N)$ solution outputs,  while our method approximates a mapping from $\mathcal{O}(N)$ coefficient-solution inputs to $\mathcal{O}(N^2)$ residual operator outputs.

The Newton--Kantorovich (NK) method we employ may be written as 
\begin{equation}
    v_{n+1} = v_n - \left(\frac{\delta \mathcal{F}^\intercal}{\delta v}\frac{\delta \mathcal{F}}{\delta v} + \lambda I\right)^{-1} \frac{\delta \mathcal{F}^\intercal}{\delta v} \mathcal{F}(u,v_n)
\end{equation}
where $u$ are random coefficients,  $v_n := v_n(u)$ is the approximate PDE solution at iteration $n$, $\mathcal{F} := \mathcal{F}(u,v)$ is the (potentially data-driven) PDE residual equation, and $\lambda>0$ is a dampening (relaxation) parameter. CHONKNORIS learns the operator-valued mapping 
$$\mathcal{N}: (u,v) \mapsto \left(\frac{\delta \mathcal{F}^\intercal}{\delta v}\frac{\delta \mathcal{F}}{\delta v} + \lambda I\right)^{-1} \frac{\delta \mathcal{F}^\intercal}{\delta v}$$
so that
\begin{equation}
    v_{n+1} = v_n - \calN(u,v_n) \mathcal{F}(u,v_n).
\end{equation}
For forward problems, we use a sciML model $\mathcal{L}$ for the approximate Cholesky factor satisfying
\begin{equation}
    \mathcal{L}^\intercal\mathcal{L} \approx \left(\frac{\delta \mathcal{F}^\intercal}{\delta v}\frac{\delta \mathcal{F}}{\delta v} + \lambda I\right)^{-1}
\end{equation} 
so that 
\begin{equation}
    \mathcal{N} \approx \mathcal{L}^\intercal\mathcal{L} \frac{\delta \mathcal{F}^\intercal}{\delta v}.
\end{equation}

The following subsections will highlight a few of the PDE applications I was most active in developing. More extensive numerical experiments, including a foundation model for forward problems and other inverse problem applications, may be found in the full text \cite{bacho.CHONKNORIS}. The PDEs we will consider include:
\begin{itemize} 
    \item A one-dimensional nonlinear elliptic PDE with an unknown forcing term, see \Cref{sec:CHONKNORIS_elliptic}. 
    \item Burgers' equation in one dimension with unknown initial conditions leading to shocks, see \Cref{sec:CHONKNORIS_burgers}. 
    \item A nonlinear Darcy flow equation in two dimensions with fixed forcing and unknown permeability, see \Cref{sec:CHONKNORIS_darcy}. 
    \item A seismic imaging full waveform inversion problem for  recovering subsurface velocity maps from surface acoustic data, see \Cref{sec:CHONKNORIS_FWI}. 
\end{itemize} 
For each of these problems we were able to use CHONKNORIS to recover discretized solutions to within machine precision error.

\Subsection{Nonlinear Elliptic Equation}\label{sec:CHONKNORIS_elliptic}

The first example is a 1 dimensions nonlinear elliptic PDE, supplemented with periodic boundary conditions, which may be written as 
\begin{align}\label{eq:example_elliptic_pde}
\begin{cases}
    - \Delta v(x) + \kappa v(x)^3=u(x),& x\in\Omega\\
    v(0)=v(1)
\end{cases}
\end{align} 
with $\kappa = 50$. The corresponding differential operator and its derivative are
\begin{equation*}
    \mathcal{F}(u,v)=- \Delta v + \kappa v^3-u \qqtqq{and} \left[\frac{\delta \mathcal{F}}{\delta v}(u,v)\right](h) = [-\Delta + 3\kappa v^2](h)
\end{equation*}
respectively.

Cholesky factor data is generated for $N=63$ regular grid points across $1024$ realizations of $u$ and $5$ NK steps, each without relaxation, i.e.,  $\lambda = 0$. The random coefficient $u$ is sampled from a zero-mean Gaussian process with periodic kernel 
$$K(x,x') = \exp(-2/\ell \sin^2(\pi/p (x-x')))$$
with period length $p=1/2$ and lengthscale $\ell=10$. Data from $1/8$ of the realizations was withheld for testing. 

Here we fit two operator learning models. The first is a classic operator learning model $u \mapsto v$ whose prediction is used as an initial guess for the NK/CHONKNORIS method. The second is our CHONKNORIS predictor model for the Cholesky factor $v \mapsto L(v)$. For both models we use vector-valued Gaussian processes regression with tuned Mat\'ern $5/2$ and Gaussian kernels respectively. 

\Cref{fig:elliptic_pde_complete} visualizes the distributions of random coefficients $u$, solutions $v$, and convergence of error metrics. Here we measure the residual RMSE and $L_2$ relative error for discretized solutions using the NK and CHONKNORIS methods started from initial guesses $v_0=0$ or initial guesses from the classic operator learning method. The initial guess provides modest savings for both NK and CHONKNORIS. NK typically converges within 5 iterations.  CHONKNORIS, despite using approximate Cholesky factors, manages to converge to machine precision accuracy within $12$ iterations across all realizations.

\begin{figure}[H]
    \centering
    \includegraphics[width=1\linewidth]{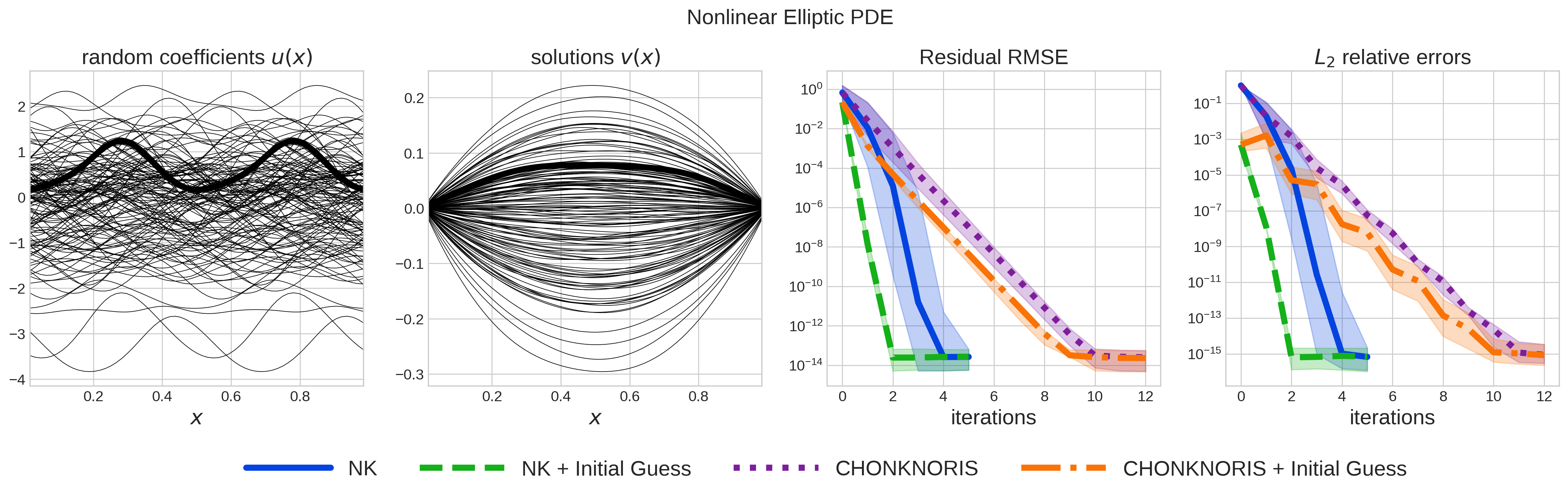}
    \caption{Results for the nonlinear elliptic PDE problem. A single coefficient-solution pair is shown in bold.  For each method the median is plotted along with $10\%-90\%$ quantiles across test realizations. Our CHONKNORIS method is able to achieve machine precision accuracy in less than $12$ iterations.}
    \label{fig:elliptic_pde_complete}
\end{figure}

\Subsection{Burgers' Equation}
\label{sec:CHONKNORIS_burgers}

The second example is the 1-dimensional Burgers' equation, supplemented with periodic boundary conditions and a random initial condition, which may be written as 
\begin{equation} \label{eq:Burger's}
    \begin{cases}
        \partial_t f = \nu \partial_{xx} f - f \partial_x f, & (x,t) \in \mathbb{T} \times [0,T], \\
        f(0,t) = f(1,t), & t \in [0,T], \\
        f(x,0) = f_0(x), & x \in \mathbb{T}
    \end{cases}.
\end{equation}
Here $\nu = 1/50$, $f_0$ is the initial condition and $\partial_t$, $\partial_x$, $\partial_{xx}$ are partial derivatives. We discretize $[0,T]$ using the $M$-point uniform grid $\{t_i\}_{i=0}^{M-1} := \{i\Delta t\}_{i=0}^{M-1}$, $\Delta t = T/(M-1)$, and apply an implicit Euler temporal  discretization leading to
$$\frac{f^{i+1}(x) - f^i(x)}{\Delta t} = \nu \partial_{xx} f^{i+1} (x) - f^{i+1}(x) \partial_{x} f^{i+1}(x)$$
where $f^{i}(x) = f(x,t_i)$ and $\Delta t = T/(M-1)$. Thus, the time-marching discretization scheme defines the next time step $f^{i+1}$ given $f^{i}$ so that $\mathcal{F}(f^{i},f^{i+1})=0$ with \begin{align*}
    \mathcal{F}(u,v)&=v - \Delta t \left(\nu \partial_{xx} v - v \partial_{x} v\right) - u\\
    \frac{\delta \mathcal{F}}{\delta v}(u,v)[h] &= h - \Delta t\left(\nu \partial_{xx} h - h \partial_{x} v - v \partial_{x} h\right)
\end{align*} Notice that $\delta \mathcal{F} / \delta v$ does \emph{not} depend on $u$.

Cholesky factor data is generated for $N=127$ regular grid points across $512$ realizations of the random initial condition $f_0$ and $4$ NK steps. Each NK step is done without relaxation, i.e., $\lambda=0$, but we store and predict Cholesky factors corresponding to $\lambda = 0.01$ for better conditioning. We take our random initial condition to be 
$$f_0(x) = \sum_{k=1}^3 a_k \sin(\pi k x)$$
where $(a_1,\dots,a_3) \sim \mathcal{N}(0,1)$. Data from $1/8$ of the realizations was withheld for testing. 

Here the CHONKNORIS Cholesky model $v \mapsto L(v)$ is a fully connected neural network (multilayer-perceptron, MLP). The MLP architecture we use has around $8.7$ million parameters across $N=127$ inputs, two hidden layers with $500$ and $1000$ nodes respectively, and an output layer of size $8128$, all connected with hyperbolic tangent (tanh) activations. 

\Cref{fig:burgers_complete} visualizes a few solutions for different initial conditions and plots the convergence of Burger's equation. The CHONKNORIS model is able to consistently recover the discretized solutions to machine precision despite many solutions containing shocks which can be difficult to capture for traditional solvers.

\begin{figure}[H]
    \centering
    \includegraphics[width=1\linewidth]{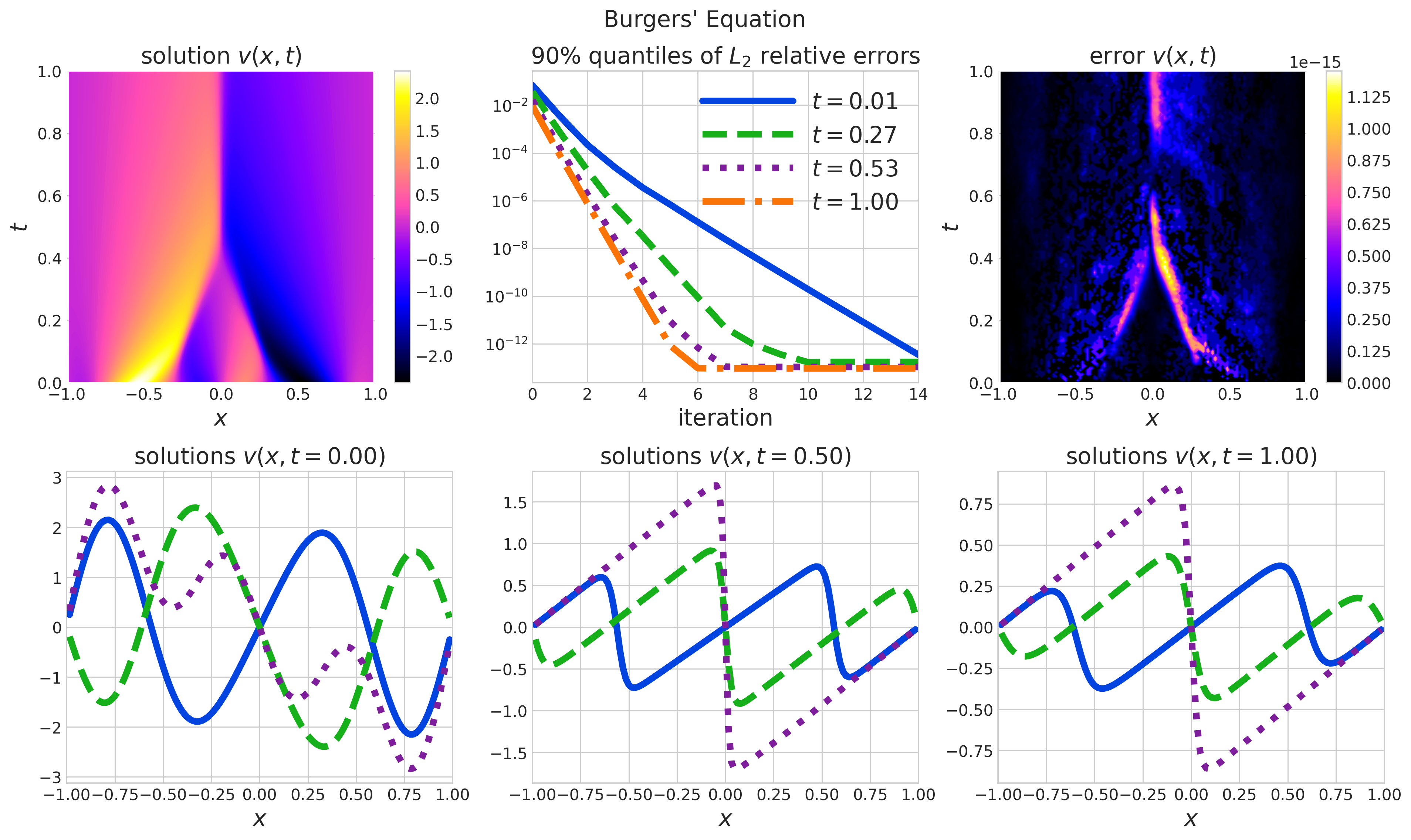}
    \caption{Results for Burgers' equation. The upper left plots show a held-out discretized solution containing shocks. The upper right shows the error in our CHONKNORIS prediction which is on the order of machine precision. The upper middle plot shows convergence of our CHONKNORIS method applied at various time steps with later time steps generally requiring fewer iterations. The bottom row shows a few solutions at different time steps with the $t=0$ case corresponding to our random initial conditions.}
    \label{fig:burgers_complete}
\end{figure}

\Subsection{Nonlinear Darcy Flow}\label{sec:CHONKNORIS_darcy}

The third example is the 2D Darcy flow equation, supplemented with homogeneous Dirichlet boundary conditions, which may be written as 
\begin{equation} \label{eq:Darcy}
\begin{cases}
    - \nabla \cdot (e^u \nabla v) +  v^3 = f, & x \in [0,1]^2 \\
    v = 0, & x \in \partial [0,1]^2
\end{cases}.
\end{equation}
Here $f$ is a forcing term and $e^u$ represents the conductivity. Expanding $- \nabla \cdot (e^u \nabla v) = - e^u [\nabla u \cdot \nabla v + \Delta v]$, we have
\begin{align}
\mathcal{F}\left(u,v\right) = -e^u [\nabla u \cdot \nabla v + \Delta v]+\kappa v^3 - f \qquad\text{and} \\
\left[\frac{\delta \mathcal{F}}{\delta v}(u,v)\right](h) = -e^u[\nabla u \cdot \nabla h + \Delta h]+3 \kappa v^2 h.
\end{align}

Cholesky factor data is generated for $N=400$ points on a $20 \times 20$ regular grid across $1024$ realizations of the random permeability $u$ and $6$ NK steps. Each NK step is done without relaxation, i.e., $\lambda=0$, but we store and predict Cholesky factors corresponding to $\lambda = 0.001$ for better conditioning. We take the fixed forcing term to be a single draw from a zero-mean GP with $5/2$ Mat\'ern kernel and constant lengthscale of $3/10$ across both dimension. The random coefficient $u$ was sampled from a Gaussian process with covariance $5(- \Delta + 1/100)^{-2}$ where $-\Delta$ denotes the Laplacian. Data from $1/8$ of the realizations was withheld for testing. 

As with the nonlinear elliptic example, we fit two operator learning models. The first is a classic operator learning model $u \mapsto v$ whose prediction is used as an initial guess for the NK/CHONKNORIS method. The second is our CHONKNORIS predictor model for the Cholesky factor $(u,v) \mapsto L(v)$. For both models we use vector-valued Gaussian processes regression with tuned Mat\'ern $5/2$ and Gaussian kernels respectively. 

\Cref{fig:darcy_complete} visualizes the distribution of errors across CHONKNORIS iterations, the fixed forcing term, and predictions for a single representative testing permeability. While the classic operator learning model achieves $L_2$ relative error on the order of $10^{-3}$, running the CHONKNORIS method for $5000$ iterations is able to recover solutions to machine precision for over $90\%$ of the withheld test realizations.  

\begin{figure}[H]
    \centering
    \includegraphics[width=1\linewidth]{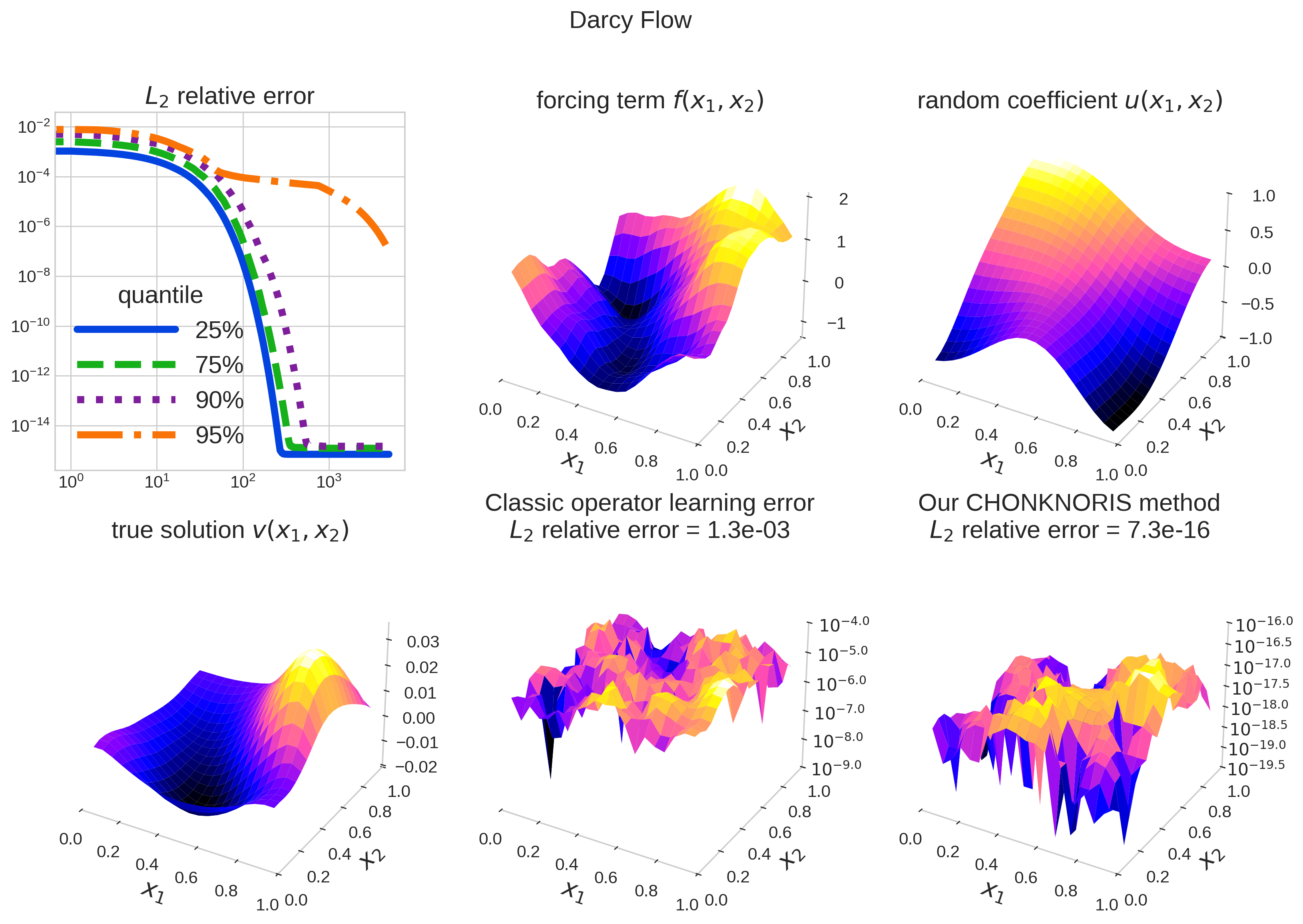}
    \caption{Results for the Darcy flow PDE. While classic operator learning methods can typically achieve $L_2$ relative errors of around $10^{-3}$, our CHONKNORIS method is capable of machine precision recovery for the vast majority of random permeability fields. Notice that more challenging realizations require more CHONKNORIS iterations.}
    \label{fig:darcy_complete}
\end{figure}

\Subsection{Seismic Imaging}\label{sec:CHONKNORIS_FWI}

We consider the acoustic wave equation relating the subsurface velocity $v$ and seismic wavefield $p$ which may be written as 
\begin{equation}\label{eq:seismic.imaging.PDE}\begin{cases}
    \Delta p(t,x)- \frac{1}{v^2(x)} p_{tt}(t,x) = s(t,x), & (x,t)\in \Omega\times [0,T],\\
    p(0,x)=p_t(0,x)=0, & x\in\Omega.
\end{cases}
\end{equation}
Here $s$ is a seismic source term, which we take to be a Ricker wavelet, and we only observe $u := p\vert_{z=0}$, the acoustic signal at the boundary $z=0$. The goal of full waveform inversion \cite{virieux2009overview} is to learn the map $u \mapsto v$.

Let $\mathcal{F}(u,v) = \calP(v)-u$ where $\calP$ is a numerical PDE solver for the forward problem which computes the acoustic signal at the boundary from a given velocity map $v$. We use automatic differentiation to compute the Jacobian of the forward solver $\mathcal{P}$. Data is sourced from the OpenFWI dataset \cite{deng2022openfwi}, and we use the same numerical forward solver they do which may be found at \url{https://csim.kaust.edu.sa/files/SeismicInversion/Chapter.FD/lab.FD2.8/lab.html} (a $2-4$ finite difference scheme with $2$nd-order accuracy in time and $4$th-order accuracy in space). Convergence of the exact NK method for a single $14 \times 14$ resolution velocity map is shown in \Cref{fig:nk_fwi_viz}.

While our previous experiments predicted the inverse Cholesky factor of the approximate Hessian, for this problem we found it easier to directly predict the Cholesky factor of the approximate Hessian and then solve the system using back-substitution. We note this does not change the computational complexity of inference. We found it necessary to simultaneously tune both the relaxation and learning rate parameters using line search in order to converge to machine precision in a reasonable number of iterations. Specifically, the learning rate $\alpha$ is used to write $v_{n+1} = v_n - \alpha\, \calN(u,v_n) \calF(u,v_n)$.

A comparison of metrics across different resolutions is given in \Cref{fig:fwi_convergence_plot}. For the rough velocity maps studied here, increasing the resolution increases the condition numbers of the relaxed approximate Hessian. CHONKNORIS is unable to exactly predict the Cholesky factor of ill-conditioned matrices, and thus resorts to inferring gradient descent steps which can make CHONKNORIS inference slow to converge. As expected, we also observe that the relaxation is decreased as NK nears convergence, indicating a smooth transition from gradient descent to Gauss--Newton updates. This behavior is more difficult to replicate with CHONKNORIS as the approximate Hessian near the solution is ill-conditioned and therefore difficult for CHONKNORIS to predict. Here we have again used a GP surrogate.

\begin{figure}[H]
    \centering 
    \includegraphics[width=1\textwidth]{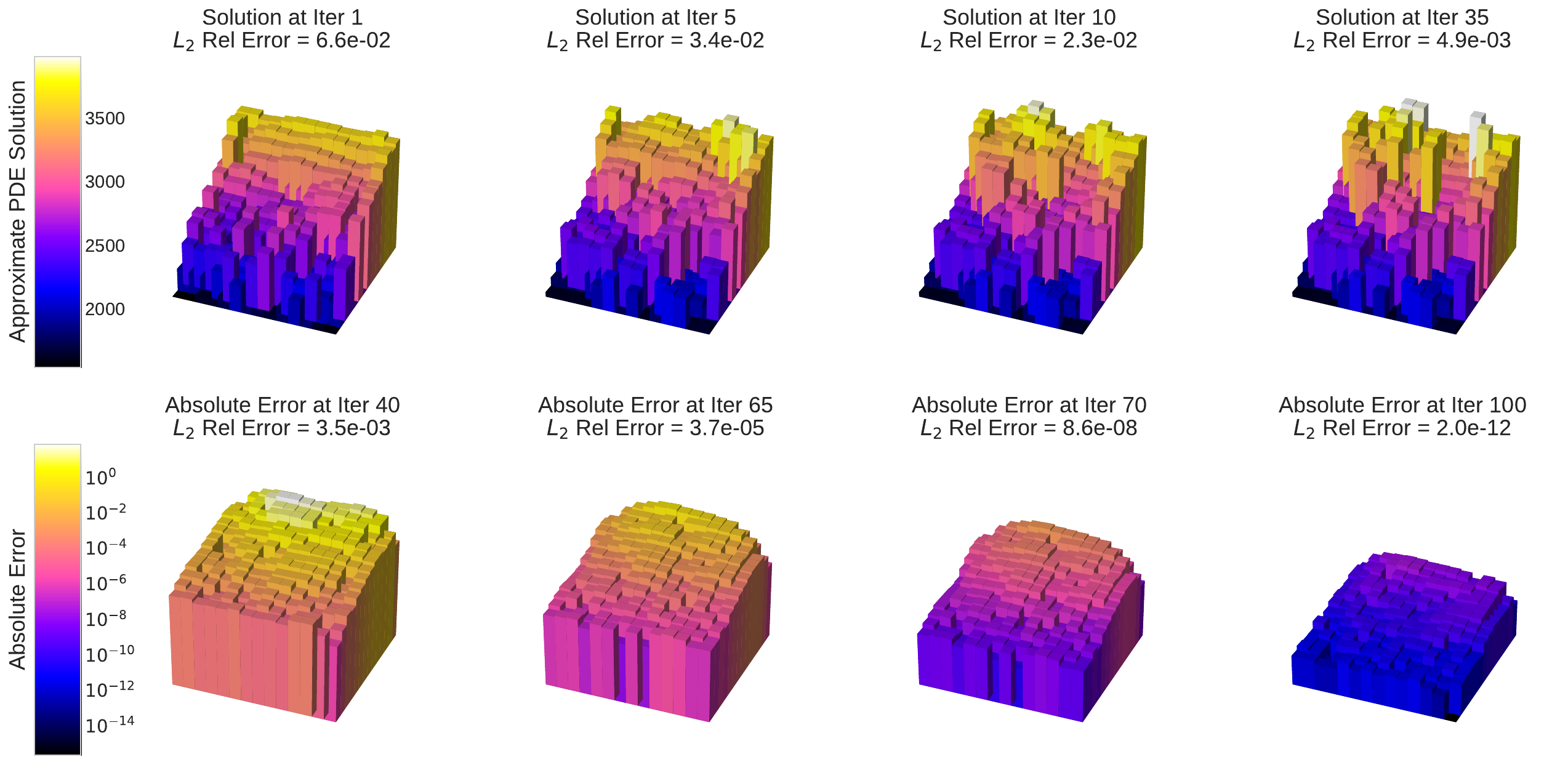}
    \caption{Visualization of the NK method applied to the FWI problem of recovering subsurface velocity maps from surface acoustic signals.} 
    \label{fig:nk_fwi_viz}
\end{figure} 

\begin{figure}[H]
    \centering 
    \includegraphics[width=1\textwidth]{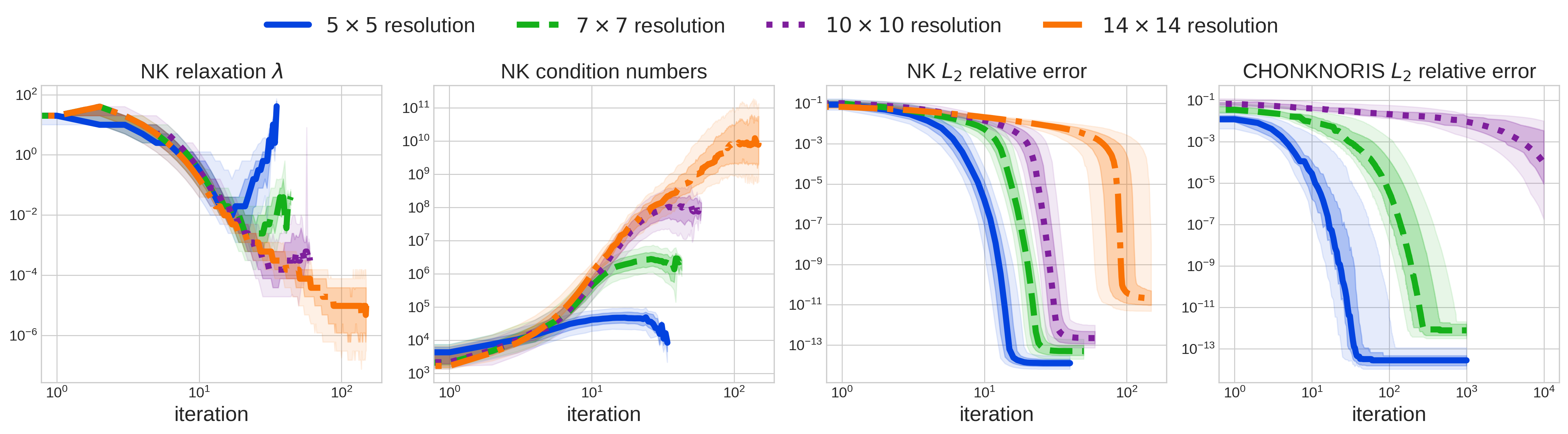}
    \caption{Convergence study for the full waveform inversion (FWI) problem. The NK method is able to converge to machine precision for the four resolutions considered. CHONKNORIS is capable of machine precision recovery for low resolutions, but is currently ineffective for predicting the ill-conditioned matrices which arise for rough coefficients, high resolutions, and/or near convergence iterations.} 
    \label{fig:fwi_convergence_plot}
\end{figure} 

\Chapter{Conclusion}

This thesis has presented a range of research across (Quasi-)Monte Carlo methods, Gaussian process regression, and machine learning. We developed efficient algorithms to
\begin{enumerate}
    \item provide efficient Quasi-Monte Carlo methods for both researchers and practitioners,
    \item enable accelerated Gaussian process regression modeling in multi-fidelity and Bayesian cubature settings, and 
    \item tackle challenging applications alongside collaborators across academia, national labs, and industry. 
\end{enumerate}
These algorithms were implemented into robust, open-source software libraries. 

\bibliography{main}

\end{document}